\PassOptionsToPackage{noend}{algorithmic}  

\documentclass{article}

\usepackage{microtype}
\usepackage{graphicx}
\usepackage{booktabs} 

\usepackage{hyperref}

\usepackage[accepted]{icml2023}

\usepackage{amsmath}
\usepackage{amssymb}
\usepackage{mathtools}
\usepackage{amsthm}

\usepackage[capitalize,noabbrev]{cleveref}

\theoremstyle{plain}
\newtheorem{theorem}{Theorem}[section]

\theoremstyle{definition}
\newtheorem{definition}[theorem]{Definition}

\theoremstyle{remark}

\usepackage[textsize=tiny]{todonotes}

\usepackage{amsfonts}
\usepackage{amsmath}
\usepackage{amsthm}
\usepackage{bm}
\usepackage{booktabs}
\usepackage{dsfont}
\usepackage{mathtools}
\usepackage{pgfplots}
\pgfplotsset{compat=1.16}
\usepackage{rotating}
\usepackage[group-separator={,},group-minimum-digits={3}]{siunitx}
\usepackage{subcaption}
\usepackage{thmtools}
\usepackage{thm-restate}
\usepackage{tikz}
\usepackage{tikz-qtree}
\usepackage{xspace}
\usetikzlibrary{automata,calc,positioning,fit,svg.path,arrows.meta,shapes.geometric,shadows}

\tikzset{
	in place/.style={
		auto=false,
		fill=white,
		inner sep=2pt,
	},
	cascaded/.style = {%
		general shadow = {%
			shadow scale = 1,
			shadow xshift = -1ex,
			shadow yshift = 1ex,
			draw,
			fill = white},
		general shadow = {%
			shadow scale = 1,
			shadow xshift = -.5ex,
			shadow yshift = .5ex,
			draw,
			fill = white},
		fill = white, 
		draw,
		minimum width = 1.5cm,
		minimum height = 2cm
	},
}

\newcommand{\alglinelabel}{%
	\addtocounter{ALC@line}{-1}
	\refstepcounter{ALC@line}
	\label
}

\theoremstyle{remark}
\newtheorem{example}{Example}

\renewcommand{\algorithmicrequire}{\textbf{Input:}}

\DeclareMathOperator{\codeif}{\mathtt{:-} }

\newcommand{\arbitraryset}{\mathcal{X}}

\newcommand{\mdp}{\mathbb{M}} 
\newcommand{\mdpstates}{\mathcal{S}} 
\newcommand{\mdpactions}{\mathcal{A}} 
\newcommand{\mdpprob}{p}
\newcommand{\mdprewfunc}{r}
\newcommand{\mdpdiscount}{\gamma}
\newcommand{\mdptermfunc}{\tau}
\newcommand{\mdpstate}{s}
\newcommand{\mdpaction}{a}
\newcommand{\history}{h}

\newcommand{\obstuple}{\bm\mdpstate}
\newcommand{\obs}{\mdpstate}
\newcommand{\obsterm}{\mdpstate^T}
\newcommand{\obsgoal}{\mdpstate^G}

\newcommand{\propset}{\mathcal{P}}
\newcommand{\lfunc}{l}
\newcommand{\trace}{\lambda}
\newcommand{\traceset}{\Lambda}
\newcommand{\proplabel}{\mathcal{L}} 
\newcommand{\proplabelset}{\mathbb{L}}

\newcommand{\rmname}{M}
\newcommand{\rmtuple}{
	\langle\rmstates,\propset,\rmtransition, \rmreward, \rminitstate,\rmstatesacc,\rmstatesrej \rangle}
\newcommand{\rmstates}{\mathcal{U}}
\newcommand{\rmstate}{u}
\newcommand{\rminitstate}{u^0}
\newcommand{\rmtransition}{\varphi}
\newcommand{\rmreward}{r}
\newcommand{\rmstatesacc}{\rmstates^A}
\newcommand{\rmstatesrej}{\rmstates^R}
\newcommand{\rmheight}{h}
\newcommand{\dnf}{\textnormal{DNF}}
\newcommand{\dnfprop}{\dnf_\propset}

\newcommand{\machineset}{\mathcal{M}}
\newcommand{\leaf}{\rmname_\top}
\newcommand{\hrm}{H}
\newcommand{\hrmtrans}{\delta_\hrm}
\newcommand{\hrmroot}{\rmname_r}
\newcommand{\hrmtuple}{\hrm=\langle\machineset,\hrmroot,\propset \rangle}
\newcommand{\hrmstate}{\bm{u}}

\newcommand{\hrmtuplep}{\hrm'=\langle\machineset',\hrmroot',\propset \rangle}

\newcommand{\stack}{\Gamma}
\newcommand{\excond}{\xi}
\newcommand{\context}{\phi}
\newcommand{\Context}{\Phi}
\newcommand{\ContextAlt}{\Psi}  

\newcommand{\opt}{\omega}
\newcommand{\optinit}{\mathcal{I}_\opt}

\newcommand{\opth}{\Omega_{\hrm}}
\newcommand{\optterm}{\Omega_{\beta}}
\newcommand{\cw}{\textsc{CraftWorld}\xspace}
\newcommand{\ww}{\textsc{WaterWorld}\xspace}
\newcommand{\learningmethod}{LHRM\xspace}
\newcommand{\rootheight}{\rmheight_r}
\newcommand{\numrmslevel}{N}
\newcommand{\numstateslevel}{U}
\newcommand{\numedgeslevel}{E}
\newcommand{\numedgestotal}{\bar{E}} 

\newcommand{\qfunc}{q}
\newcommand{\nnparams}{\bm\theta}

\newcommand{\ftreenodes}{\mathcal{F}}
\newcommand{\ftreeroot}{F_r}
\newcommand{\formtolits}{\nu}

\newcommand{\numtasks}{T}
\newcommand{\numinstances}{I}
\newcommand{\avgreturn}{R}
\newcommand{\taskinstancescore}{c}

\newcommand{\maxdisjuncts}{\kappa}
\newcommand{\numgoaltracesexp}{\rho}
\newcommand{\numshortestgoaltracesexp}{\numgoaltracesexp_s}
\newcommand{\machinesetcall}{\machineset_\mathcal{C}}

\newcommand{\asprepr}{\mathbb{A}}
\newcommand{\generalrules}{\mathcal{R}}
\newcommand{\aspprogram}{P}
\newcommand{\answerset}{AS}
\newcommand{\hrmlearningtask}{T_\hrm}
\newcommand{\ilasplearningtask}{T}
\newcommand{\ilaspbk}{\mathcal{B}}
\newcommand{\ilasphypspace}{\mathcal{S}_M}
\newcommand{\ilaspexamples}{\mathcal{E}}
\newcommand{\ilasphypothesis}{\mathcal{H}}

\definecolor{imperial}{HTML}{003E74} 

\newcommand{\Chicken}[1][]{%
	\tikz \fill [scale=1ex/300,yscale=-1,#1] svg "
	M469.333,182.857c-20.197,0-36.571-16.374-36.571-36.571c0-20.197,16.374-36.571,36.571-36.571
	c20.197,0,36.571-16.374,36.571-36.571c0-26.93-21.832-48.762-48.762-48.762c-20.197,0-36.571,16.374-36.571,36.571
	c0,20.197-16.374,36.571-36.571,36.571c-20.197,0-36.571-16.374-36.571-36.571C347.429,27.29,320.139,0,286.476,0
	s-60.952,27.29-60.952,60.952c0,13.723,4.537,26.384,12.189,36.571h36.573c94.106,0,170.667,76.561,170.667,170.667v20.269
	c5.543,2.633,11.741,4.112,18.286,4.112c30.476,0,42.667-24.381,42.667-73.143C505.905,199.231,489.531,182.857,469.333,182.857z%
	M274.286,134.095h-85.333v231.619c0,80.791,65.495,146.286,146.286,146.286h73.143V268.19
	C408.381,194.132,348.344,134.095,274.286,134.095z M298.667,268.19h-0.002c-13.442,0-24.379-10.937-24.379-24.381
	c0.005-13.446,10.942-24.381,24.381-24.381c13.444,0,24.381,10.937,24.381,24.381C323.048,257.253,312.11,268.19,298.667,268.19z%
	M91.429,304.762c-47.128,0-85.333,38.205-85.333,85.333c0,47.128,38.205,85.333,85.333,85.333
	c23.884,0,45.463-9.824,60.952-25.637v-145.03H91.429z%
	M6.095,134.095 152.381,280.381 152.381,134.095z%
	";%
}

\newcommand{\Cow}[1][]{%
	\tikz \fill [scale=1ex/300,yscale=-1,#1] svg "
	M172.904,16.987c-5.879,12.056-13.311,31.868-20.034,53.523c1.074-0.036,40.652-0.081,40.652-0.081
	C186.563,48.791,178.913,29.003,172.904,16.987z%
	M330.927,332.671H181.073c-30.986,0-56.195,25.209-56.195,56.195v49.951c0,30.986,25.209,56.195,56.195,56.195h149.854
	c30.986,0,56.195-25.209,56.195-56.195v-49.951C387.122,357.881,361.913,332.671,330.927,332.671z M199.805,438.818h-37.463
	v-49.951h37.463V438.818z M349.659,438.818h-37.463v-49.951h37.463V438.818z%
	M355.902,107.891H156.098c-24.42,0-45.241,15.661-52.973,37.463h40.485c44.758,0,81.171,36.413,81.171,81.171v68.683
	h62.439v-68.683c0-44.758,36.413-81.171,81.171-81.171h40.485C401.143,123.552,380.322,107.891,355.902,107.891z%
	M143.61,182.818H99.902v106.146c0,13.07,5.784,25.142,15.395,33.307c16.926-16.72,40.162-27.064,65.776-27.064h6.244
	v-68.683C187.317,202.425,167.71,182.818,143.61,182.818z M137.366,257.745h-0.003c-13.77,0-24.973-11.204-24.973-24.976
	c0.005-13.774,11.209-24.976,24.976-24.976c13.772,0,24.976,11.204,24.976,24.976S151.137,257.745,137.366,257.745z%
	M18.732,82.915H0v18.732c0,24.1,19.607,43.707,43.707,43.707h18.732v-18.732C62.439,102.522,42.832,82.915,18.732,82.915z%
	M338.361,16.987c-5.872,12.041-13.291,31.818-20.008,53.44c0,0,39.619,0.045,40.652,0.079
	C352.04,48.84,344.377,29.018,338.361,16.987z%
	M493.268,82.915c-24.1,0-43.707,19.607-43.707,43.707v18.732h18.732c24.1,0,43.707-19.607,43.707-43.707V82.915H493.268z%
	M368.39,182.818c-24.1,0-43.707,19.607-43.707,43.707v68.683h6.244c25.614,0,48.852,10.345,65.777,27.065
	c9.609-8.164,15.394-20.238,15.394-33.307V182.818H368.39z M374.634,257.745h-0.003c-13.77,0-24.973-11.204-24.973-24.976
	c0.005-13.774,11.209-24.976,24.976-24.976c13.772,0,24.976,11.204,24.976,24.976S388.406,257.745,374.634,257.745z
	";%
}

\newcommand{\Rabbit}[1][]{%
	\tikz \fill [scale=1ex/300,yscale=-1,#1] svg "
	M405.074,287.777c0-82.332-66.743-149.074-149.074-149.074s-149.074,66.743-149.074,149.074
	c0,15.289,2.303,30.04,6.581,43.926c-4.216,9.342-6.581,19.697-6.581,30.611c0,41.166,33.371,74.537,74.537,74.537
	c38.002,0,69.34-28.447,73.934-65.203c-13.322-0.817-27.958-19.223-33.471-35.955c-2.585-8.677,4.097-15.858,14.907-15.972
	c12.778,0,25.556,0,38.333,0c10.811,0.114,17.493,7.294,14.907,15.972c-5.514,16.731-20.15,35.138-33.471,35.955
	c4.596,36.758,35.932,65.203,73.935,65.203c41.166,0,74.537-33.371,74.537-74.537c0-10.914-2.365-21.271-6.581-30.611
	C402.771,317.817,405.074,303.066,405.074,287.777z M197.435,325.046h-0.002c-11.742,0-21.294-9.554-21.294-21.296
	c0.004-11.745,9.558-21.296,21.296-21.296c11.743,0,21.296,9.554,21.296,21.296C218.731,315.492,209.178,325.046,197.435,325.046z
	M314.565,325.046h-0.002c-11.742,0-21.294-9.554-21.294-21.296c0.004-11.745,9.558-21.296,21.296-21.296
	c11.743,0,21.296,9.554,21.296,21.296C335.861,315.492,326.308,325.046,314.565,325.046z%
	M127.724,99.656C91.517,63.449,46.791,40.961,0,32.168c8.793,46.791,31.282,91.517,67.489,127.724
	c13.534,13.534,25.4,23.71,35.849,31.101c14.231-22.281,33.203-41.253,55.484-55.487
	C151.431,125.057,141.26,113.192,127.724,99.656z%
	M384.276,99.657c-13.534,13.534-23.71,25.4-31.101,35.848c22.281,14.232,41.253,33.205,55.485,55.484
	c10.449-7.391,22.314-17.562,35.85-31.098C480.719,123.685,503.207,78.958,512,32.166C465.208,40.96,420.483,63.45,384.276,99.657
	z%
	M256,438.678c-10.919,10.714-24.133,19.096-38.823,24.343c9.713,10.347,23.511,16.813,38.823,16.813
	s29.111-6.467,38.822-16.815C280.132,457.774,266.919,449.392,256,438.678z
	";%
}
\newcommand{\RabbitE}{\Rabbit[scale=0.65]}

\newcommand{\Squid}[1][]{%
	\tikz \fill [scale=1ex/300,yscale=-1,#1] svg "
	M365.466,40.285c11.563,17.673,18.969,38.298,20.808,60.481h0.45V283.83c9.616-5.565,20.771-8.766,32.681-8.766
	c30.082,0,54.468-24.386,54.468-54.468C473.871,140.174,430.292,74.804,365.466,40.285z%
	M38.128,220.596c0,30.082,24.386,54.468,54.468,54.468c11.91,0,23.065,3.202,32.681,8.766V100.766h0.45
	c1.839-22.183,9.245-42.809,20.808-60.481C81.708,74.804,38.128,140.174,38.128,220.596z%
	M256,13.617c-54.148,0-98.043,43.896-98.043,98.043v228.766c0,18.049,14.632,32.681,32.681,32.681h130.723
	c18.049,0,32.681-14.632,32.681-32.681V111.66C354.043,57.513,310.148,13.617,256,13.617z M217.872,334.979h-0.002
	c-12.012,0-21.785-9.774-21.785-21.787c0.004-12.016,9.778-21.787,21.787-21.787c12.013,0,21.787,9.774,21.787,21.787
	C239.66,325.205,229.886,334.979,217.872,334.979z M294.128,334.979h-0.002c-12.012,0-21.785-9.774-21.785-21.787
	c0.004-12.016,9.778-21.787,21.787-21.787c12.013,0,21.787,9.774,21.787,21.787C315.915,325.205,306.141,334.979,294.128,334.979z%
	M190.638,405.787c-7.593,0-14.881-1.315-21.666-3.706c-11.718,30.786-41.526,52.727-76.376,52.727
	c-33.037,0-59.915-26.878-59.915-59.915c0-21.024,17.104-38.128,38.128-38.128v-32.681C31.765,324.085,0,355.85,0,394.894
	c0,51.057,41.538,92.596,92.596,92.596c51.719,0,95.522-34.506,109.618-81.702H190.638z%
	M441.191,324.085v32.681c21.024,0,38.128,17.104,38.128,38.128c0,33.037-26.878,59.915-59.915,59.915
	c-34.85,0-64.658-21.941-76.376-52.727c-6.786,2.391-14.073,3.706-21.666,3.706h-11.576c14.096,47.197,57.9,81.702,109.618,81.702
	c51.057,0,92.596-41.538,92.596-92.596C512,355.85,480.235,324.085,441.191,324.085z%
	M284.746,405.787h-57.491c-4.993,35.415-21.548,59.915-36.616,59.915v32.681c20.749,0,39.5-13.183,52.799-37.121
	c5.329-9.593,9.545-20.573,12.563-32.428c3.018,11.857,7.233,22.836,12.563,32.428c13.298,23.938,32.049,37.121,52.798,37.121
	v-32.681C306.294,465.702,289.739,441.202,284.746,405.787z
	";%
}

\newcommand{\Wheat}[1][]{%
	\tikz \fill [scale=1ex/300,yscale=-1,#1] svg "
	M90.313,406.025c42.4,42.4,110.8,42.4,152.8,0C201.113,363.725,132.712,363.725,90.313,406.025z%
	M84.112,246.725c-42.4,42.4-42.4,110.8,0,152.8C126.012,357.125,126.012,288.625,84.112,246.725z%
	M157.912,338.425c42,42.4,110.4,42.4,152.8,0C268.313,296.025,199.912,296.025,157.912,338.425z%
	M151.313,179.025c-42.4,42.4-42.4,110.8,0,152.8C193.712,289.825,193.712,221.425,151.313,179.025z%
	M282.713,207.025c59.9,0,108.1-48.6,108.1-108.1C330.913,98.925,282.713,147.525,282.713,207.025z%
	M225.212,271.125c42.4,42.4,110.8,42.4,152.8,0C336.013,228.825,267.513,228.825,225.212,271.125z%
	M219.012,111.825c-42.4,42.4-42.4,110.8,0,152.8C260.913,222.225,260.913,153.725,219.012,111.825z%
	M486.813,2.925c-3.9-3.9-9.7-3.9-13.6,0l-59.1,59.5c-3.9,3.9-3.9,9.7,0,13.6c1.9,1.9,7.9,5,13.2,0l59.5-59.5
	C490.713,12.625,490.713,6.825,486.813,2.925z%
	M458.013,157.625l-71.5,71.5c-3.9,3.9-3.9,9.7,0,13.6c1.9,1.9,7.7,5.2,13.2,0l71.5-71.5c3.9-3.9,3.9-9.7,0-13.6
	C467.713,153.725,461.513,153.725,458.013,157.625z%
	M260.513,103.225l71.5-71.5c3.9-3.9,3.9-9.7,0-13.6s-9.7-3.9-13.6,0l-71.1,71.9c-3.9,3.9-3.9,9.7,0,13.6
	C249.312,105.225,254.912,108.125,260.513,103.225z%
	M78.212,436.425c-2.3-0.4-7.8-3.1-14.4-9.7s-8.9-12.1-9.7-14.4c1.9-3.5,1.6-8.6-1.6-11.7c-3.9-3.9-9.7-3.9-13.6,0
	c-2.7,2.7-7,9.3-1.9,21c1.6,3.5,3.9,7.4,6.6,10.9l-40.4,41.2c-3.9,3.9-3.9,9.7,0,13.6c1.9,1.9,4.3,2.7,6.6,2.7s4.7-0.8,6.6-2.7
	l41.2-41.2c3.5,2.7,7.4,5.1,10.9,6.6c11.6,5.4,19.4,0,21-1.9c3.9-3.9,3.9-9.7,0-13.6
	C86.813,434.425,81.712,434.025,78.212,436.425z%
	";%
}

\newcommand{\Table}[1][]{%
	\tikz \fill [scale=1ex/300,yscale=-1,#1] svg "
	M512 160L256 32L0 160V208L16 216V376C16 389.255 26.7452 400 40 400H56C69.2548 400 80 389.255 80 376V248L224 320V472C224 485.255 234.745 496 248 496H264C277.255 496 288 485.255 288 472V320L432 248V376C432 389.255 442.745 400 456 400H472C485.255 400 496 389.255 496 376V216L512 208V160Z
	";%
}
\newcommand{\TableE}{\Table[scale=0.65]}

\newcommand{\Workbench}[1][]{%
	\tikz \fill [scale=1ex/300,yscale=-1,#1] svg "
	M216,316l-120,-60l0,120c0,22.091 -17.909,40 -40,40c-22.091,0 -40,-17.909 -40,-40l0,-160l-16,-8l0,-48l256,-128l256,128l0,48l-16,8l0,160c0,22.091 -17.909,40 -40,40c-22.091,0 -40,-17.909 -40,-40l0,-120l-120,60l0,156c0,22.091 -17.909,40 -40,40c-22.091,0 -40,-17.909 -40,-40l0,-156Zm48,16l-8,4l-8,-4l0,140c0,4.418 3.582,8 8,8c4.418,0 8,-3.582 8,-8l0,-140Zm184,-92l0,136c0,4.418 3.582,8 8,8c4.418,0 8,-3.582 8,-8l0,-144l-16,8Zm-400,-8l0,144c0,4.418 3.582,8 8,8c4.418,0 8,-3.582 8,-8l0,-136l-16,-8Zm208,-164.223l-194.446,97.223l192.02,91.072c0.803,-0.048 1.611,-0.072 2.426,-0.072c0.815,0 1.623,0.024 2.426,0.072l192.02,-91.072l-194.446,-97.223Z
	";
}
\newcommand{\WorkbenchE}{\Workbench[scale=0.65]}

\newcommand{\Iron}[1][]{%
	\tikz \fill [scale=1ex/300,yscale=-1,#1] svg "
	M420.416,91.584L335.104,6.272c-0.021-0.043-0.064-0.043-0.085-0.085c-1.792-1.749-3.861-3.051-6.037-4.053
	c-0.661-0.32-1.365-0.491-2.048-0.747c-1.813-0.619-3.669-1.003-5.589-1.109C320.875,0.235,320.469,0,320,0H192
	c-0.448,0-0.832,0.235-1.28,0.256c-1.963,0.128-3.84,0.512-5.696,1.152c-0.661,0.235-1.344,0.405-1.984,0.704
	c-2.197,1.024-4.267,2.325-6.059,4.075c-0.021,0.043-0.064,0.043-0.085,0.085L91.584,91.584
	c-4.011,3.989-6.251,9.429-6.251,15.083v384c0,0.512,0.213,0.981,0.235,1.493c0.107,1.472,0.299,2.901,0.704,4.309
	c0.235,0.811,0.533,1.536,0.853,2.304c0.277,0.64,0.384,1.323,0.725,1.963c0.299,0.533,0.725,0.939,1.067,1.451
	c0.448,0.704,0.917,1.365,1.451,2.005c0.853,1.045,1.792,1.92,2.816,2.773c0.597,0.491,1.152,1.003,1.792,1.429
	c1.387,0.917,2.901,1.621,4.48,2.176c0.363,0.149,0.64,0.405,1.024,0.512c0.064,0.021,0.149,0.021,0.213,0.021
	c1.92,0.576,3.925,0.896,5.973,0.896h298.667c2.048,0,4.053-0.32,5.973-0.896c0.064-0.021,0.149,0,0.213-0.021
	c0.384-0.107,0.661-0.363,1.024-0.512c1.579-0.555,3.072-1.259,4.48-2.176c0.64-0.427,1.195-0.939,1.792-1.429
	c1.024-0.853,1.963-1.749,2.816-2.773c0.533-0.64,1.003-1.301,1.451-2.005c0.341-0.512,0.768-0.917,1.067-1.451
	c0.32-0.64,0.448-1.323,0.725-1.963c0.32-0.768,0.619-1.493,0.853-2.304c0.405-1.429,0.597-2.837,0.704-4.309
	c0.021-0.512,0.235-0.981,0.235-1.493v-384C426.667,101.013,424.427,95.573,420.416,91.584z M128,115.499l42.667-42.667v283.371
	l-42.667,64V115.499z M146.517,469.333L203.413,384h105.173l56.896,85.333H146.517z M384,420.203l-42.667-64V72.832L384,115.499
	V420.203z
	";%
}

\newcommand{\Sugarcane}[1][]{%
	\tikz \fill [scale=1ex/300,yscale=-1,#1] svg "
	M511.094,264.722c-1.136-3.307-28.511-81.137-89.171-95.166c-23.855-5.517-48.712-0.478-74.264,14.977v62.921
	c33.771,3.584,91.421,13.025,140.126,37.597c2.383,1.202,4.957,1.79,7.516,1.79c3.953,0,7.87-1.404,10.977-4.113
	C511.396,278.264,513.3,271.144,511.094,264.722z%
	M144.012,30.084C83.06,17.35,25.545,76.5,23.128,79.024c-4.699,4.903-5.959,12.164-3.188,18.365
	c2.771,6.2,9.02,10.103,15.811,9.873c55.339-1.884,112.503,14.251,144.4,25.19V47.257h-0.001
	C169.18,38.572,157.126,32.824,144.012,30.084z%
	M1.419,319.816c-2.739,6.214-1.443,13.469,3.281,18.349c3.208,3.314,7.561,5.083,11.999,5.083
	c2.096,0,4.212-0.395,6.233-1.209c51.663-20.806,111.29-25.11,144.9-25.692c-24.308-12.68-78.121-37.979-133.868-45.914
	C13.29,293.33,2.135,318.191,1.419,319.816z%
	M136.601,205.752c-8.572-2.976-17.327-4.231-25.986-4.231c-44.94,0-87.415,33.739-89.533,35.445
	c-1.771,1.427-2.492,3.756-1.838,5.936c0.654,2.178,2.541,3.724,4.805,3.938c65.199,6.168,129.604,36.853,156.101,50.83v-54.356
	C168.724,224.429,154.173,211.856,136.601,205.752z%
	M208.92,375.079c-2.31,0-4.51-0.47-6.51-1.318v117.115c0,5.763,2.919,10.844,7.361,13.844
	c0.888,0.6,1.838,1.116,2.836,1.539c1.998,0.845,4.194,1.312,6.499,1.312h89.595c1.729,0,3.396-0.263,4.965-0.75
	c2.091-0.65,4.006-1.701,5.655-3.062c3.711-3.062,6.076-7.697,6.076-12.883V373.928c-1.881,0.733-3.92,1.151-6.059,1.151H208.92z%
	M319.321,8.24c-1.65-1.361-3.564-2.412-5.655-3.062c-1.568-0.487-3.236-0.75-4.965-0.75h-89.595
	c-2.305,0-4.501,0.467-6.499,1.312c-0.998,0.423-1.948,0.939-2.836,1.539c-4.44,3.001-7.359,8.082-7.359,13.844v112.783
	c2-0.848,4.201-1.318,6.51-1.318h110.418c2.139,0,4.178,0.417,6.058,1.151V21.124C325.398,15.937,323.032,11.304,319.321,8.24z%
	M319.338,341.688c2.139,0,4.178,0.417,6.058,1.151V164.831c-1.88,0.733-3.919,1.151-6.058,1.151H208.92
	c-2.31,0-4.51-0.47-6.51-1.318v178.341c2-0.848,4.201-1.318,6.51-1.318H319.338z
	";%
}
\newcommand{\SugarcaneE}{\Sugarcane[scale=0.65]}

\newcommand{\Lava}[1][]{%
	\tikz \fill [scale=1ex/300,yscale=-1,#1] svg "
	M398.239,235.891l-0.278-1.111C357.041,106.465,235.114,15.368,229.929,11.572L214.191,0l1.019,19.441
	c0,0.834,2.869,76.193-70.176,138.868c-79.34,68.138-78.878,124.334-78.692,154.515l0.094,4.166
	c0,82.302,60.399,143.839,146.811,153.984c-4.704-0.739-9.509-1.698-14.331-3.172c-27.867-8.61-44.16-28.144-47.308-56.289
	l-0.74-5.554c-2.5-21.756-5.74-48.697,41.845-99.801c17.775-19.071,33.236-39.624,46.012-61.01l6.388-10.738l8.424,9.257
	c25.182,27.495,44.438,55.269,57.399,82.302c11.942,23.886,12.684,61.659,11.757,69.713
	c-4.536,41.105-34.254,71.749-74.063,76.192c-2.595,0.296-6.479,0.583-11.234,0.689c0.418,0.003,0.816,0.051,1.236,0.051
	c94.06,0,167.66-75.637,167.66-172.197C406.293,277.274,398.516,237.557,398.239,235.891z
	";%
}

\newcommand{\Redstone}[1][]{%
	\tikz {
		\fill [scale=1ex/300,yscale=-1,#1] svg "
		M71.312,336.543c0,0-104.503-200.378,188.304-226.92s223.659,190.949,115.622,223.791
		s-76.012,92.573-173.411,91.852C104.427,424.543,71.312,336.543,71.312,336.543z
		";%
		
		\fill [scale=1ex/300,yscale=-1,#1] svg "
		M430.592,137.365c41.36,60.132-10.906,152.517-78.32,173.01
		c-108.038,32.841-76.012,92.574-173.411,91.852c-37.531-0.278-65.509-13.519-85.733-29.687
		c19.674,24.408,54.091,52.32,108.771,52.726c97.399,0.722,65.374-59.01,173.411-91.852
		C452.145,310.057,509.308,193.306,430.592,137.365z
		";%
		
		\fill [color=white,scale=1ex/300,yscale=-1,#1] svg "
		M362.527,173.462c-5.743-0.497-28.246-3.816-48.433-7.01c-4.258-0.674-8.272,2.237-8.948,6.502
		c-0.675,4.266,2.237,8.273,6.503,8.948c11.85,1.875,21.038,3.292,28.213,4.353c-17.344,12.524-28.514,26.932-29.231,27.869
		c-2.622,3.429-1.968,8.331,1.458,10.956c1.418,1.087,3.09,1.612,4.75,1.612c2.349,0,4.672-1.053,6.214-3.06
		c0.212-0.275,18.828-24.222,43.901-35.076c0.075-0.032,0.151-0.066,0.227-0.099c3.567-1.53,7.262-2.796,11.053-3.688
		c35.272-8.299,64.88-34.074,66.123-35.168c0.073-0.064,0.137-0.135,0.206-0.2c-3.519-3.657-7.466-7.145-11.859-10.436
		c-5.885,4.843-30.558,24.107-58.054,30.576C370.476,170.527,366.429,171.871,362.527,173.462z
		";%
		
		\fill [color=white,scale=1ex/300,yscale=-1,#1] svg "
		M176.255,382.055c28.858-5.458,33.872-33.187,34.071-34.364c0.057-0.339,0.073-0.294,0.08-0.536
		c0.341-1.56,4.971-18.379,45.275-23.167c4.289-0.51,7.354-4.4,6.844-8.69c-0.508-4.289-4.398-7.355-8.689-6.844
		c-22.597,2.684-39.254,9.268-49.506,19.566c-7.611,7.646-9.147,14.831-9.453,17.179c-0.403,1.967-4.242,18.214-21.53,21.485
		c-25.517,4.827-42.049,33.179-44.631,37.879c4.289,2.699,8.858,5.228,13.709,7.534
		C146.507,404.469,159.987,385.134,176.255,382.055z
		";%
		
		\fill [color=white,scale=1ex/300,yscale=-1,#1] svg "
		M227.661,391.482c4.145-1.216,6.52-5.562,5.304-9.706c-1.216-4.145-5.562-6.522-9.706-5.304
		c-30.512,8.947-49.571,42.199-50.368,43.61c-0.369,0.654-0.628,1.339-0.793,2.036c5.225,1.108,10.686,1.949,16.394,2.483
		C193.561,416.87,208.044,397.234,227.661,391.482z
		";%
		
		\fill [color=white,scale=1ex/300,yscale=-1,#1] svg "
		M63.771,218.302c0.084-0.016,9.463-0.442,23.649,13.744c2.735,2.737,5.11,5.257,7.205,7.481
		c11.255,11.95,17.831,17.97,37.012,17.97c2.605,0,5.444-0.111,8.557-0.334c4.309-0.308,7.553-4.051,7.245-8.359
		c-0.308-4.308-4.055-7.553-8.359-7.245c-19.727,1.408-22.023-1.032-33.067-12.758c-2.17-2.304-4.629-4.915-7.532-7.817
		c-14.4-14.401-26.343-17.612-32.58-18.179c-2.485,4.971-4.545,9.992-6.229,15.022C60.931,218.332,62.334,218.524,63.771,218.302z
		";%
		
		\fill [color=white,scale=1ex/300,yscale=-1,#1] svg "
		M111.909,162.709c0.211-0.154,21.165-15.322,35.198-4.245c1.188,0.937,1.647,1.648,1.679,1.971
		c0.23,2.291-4.615,8.022-7.51,11.447c-5.697,6.738-11.588,13.707-10.836,22.137c0.4,4.475,2.62,8.41,6.597,11.698
		c18.19,15.033,31.953,2.239,45.26-10.135c4.107-3.819,8.355-7.767,13.283-11.713c3.946-3.16,7.38-4.577,8.958-3.693
		c3.836,2.143,7.988,15.268,5.21,37.74c-0.53,4.287,2.515,8.192,6.803,8.723c0.326,0.041,0.65,0.06,0.97,0.06
		c3.891,0,7.263-2.902,7.753-6.863c0.775-6.27,1.099-12.154,1.002-17.583c1.338-3.137,6.256-12.707,16.602-13.135
		c4.316-0.178,7.67-3.822,7.492-8.138c-0.178-4.315-3.817-7.687-8.138-7.492c-8.262,0.341-14.564,3.522-19.272,7.496
		c-2.388-6.821-6.021-11.799-10.793-14.465c-5.36-2.994-14.284-4.532-26.363,5.137c-5.38,4.306-10.044,8.642-14.158,12.469
		c-15.116,14.054-17.88,15.124-24.643,9.534c-0.843-0.696-0.978-1.045-0.981-1.035c-0.164-1.934,4.626-7.598,7.2-10.643
		c5.636-6.666,12.024-14.222,11.129-23.113c-0.488-4.845-3.029-9.113-7.55-12.684c-6.391-5.045-13.257-7.083-19.886-7.383
		c-14.196,6.608-26.142,13.993-36.149,21.926c0.07,0.105,0.132,0.213,0.209,0.315C103.535,164.522,108.43,165.268,111.909,162.709z
		";%
		
		\fill [color=white,scale=1ex/300,yscale=-1,#1] svg "
		M466.124,211.243c-1.784,0.398-26.811,6.072-52.801,16.156c-0.084-5.832-1.648-10.921-3.853-15.123
		c24.036-4.103,43.301-14.408,54.956-29.366c-1.658-5.411-3.871-10.692-6.696-15.781c-2.037,1.024-3.214,2.768-3.718,3.523
		c-9.172,13.767-26.773,23.219-49.559,26.612c-3.829,0.57-7.456,1.283-10.92,2.089c-0.126,0.024-0.25,0.056-0.376,0.087
		c-38.74,9.144-54.797,31.53-55.519,32.565c-0.051,0.072-0.1,0.145-0.148,0.219c-7.177,11.059-30.613,38.833-52.584,34.635
		c-4.243-0.808-8.34,1.972-9.15,6.214c-0.811,4.242,1.971,8.34,6.214,9.151c2.572,0.491,5.106,0.72,7.611,0.72
		c32.504-0.005,58.669-38.615,60.918-42.029c1.102-1.498,13.38-17.407,42.444-25.301c2.793,3.009,6.696,9.279,3.648,18.923
		c-13.177,6.207-25.317,13.553-33.296,21.92c-2.981,3.127-2.864,8.077,0.263,11.059c1.514,1.444,3.456,2.161,5.397,2.161
		c2.064,0,4.125-0.812,5.662-2.424c20.681-21.689,78.47-36.756,92.039-40.059c0.935-5.376,1.472-10.765,1.576-16.121
		C467.538,211.036,466.833,211.085,466.124,211.243z
		";%
		
		\fill [color=white,scale=1ex/300,yscale=-1,#1] svg "
		M197.877,288.397c-4.024-1.57-8.56,0.422-10.128,4.447c-4.861,12.469-18.213,19.795-18.414,19.903
		c-3.801,2.027-5.248,6.751-3.231,10.56c1.404,2.651,4.114,4.163,6.919,4.163c1.234,0,2.486-0.293,3.653-0.91
		c0.757-0.4,18.627-10.026,25.648-28.035C203.892,294.501,201.903,289.967,197.877,288.397z
		";%
		
		\fill [color=white,scale=1ex/300,yscale=-1,#1] svg "
		M30.959,508.121c-12.632-22.692,9.888-36.435,12.519-37.939c1.729-0.977,2.937-2.516,3.537-4.259
		c-2.996-1.172-5.593-2.043-7.386-2.551c-3.105-0.88-6.89-1.718-10.916-1.79c-3,2.493-6.495,5.881-9.532,10.109
		c-8.253,11.488-9.884,24.667-4.854,37.702c5.094,1.774,11.177,2.614,17.616,2.524C31.944,510.63,31.629,509.325,30.959,508.121z
		";%
	}
}

\newcommand{\Agent}[1][]{%
	\tikz \fill [scale=1ex/300,yscale=-1,#1] svg "
	M245,456.701 490,33.299 0,33.299z
	";%
}

\makeatletter
\newenvironment{customlegend}[1][]{%
	\begingroup
	\pgfplots@init@cleared@structures
	\pgfplotsset{#1}%
}{%
	\pgfplots@createlegend
	\endgroup
}%

\def\addlegendimage{\pgfplots@addlegendimage}

\definecolor{batter}{HTML}{d1bbd7} 
\definecolor{bucket}{HTML}{ae76a3} 
\definecolor{compass}{HTML}{882e72} 
\definecolor{leather}{HTML}{1965b0} 
\definecolor{paper}{HTML}{5289c7} 
\definecolor{quill}{HTML}{7bafde}
\definecolor{sugar}{HTML}{4eb265} 
\definecolor{book}{HTML}{90c987} 
\definecolor{map}{HTML}{cae0ab} 
\definecolor{milkbucket}{HTML}{f7f056} 
\definecolor{bookquill}{HTML}{f4a736} 
\definecolor{milkbucketandsugar}{HTML}{e8601c}
\definecolor{cake}{HTML}{dc050c}

\definecolor{rg}{HTML}{882e72} 
\definecolor{bc}{HTML}{1965b0} 
\definecolor{my}{HTML}{7bafde} 
\definecolor{rgbc}{HTML}{4eb265} 
\definecolor{bcmy}{HTML}{cae0ab} 
\definecolor{rgmy}{HTML}{f7f056}
\definecolor{rgb}{HTML}{ee8026} 
\definecolor{cmy}{HTML}{dc050c} 
\definecolor{rgbcmy}{HTML}{72190e} 

\definecolor{flatcrm}{HTML}{ffa600} 
\definecolor{flathrl}{HTML}{bc5090} 
\definecolor{nonflathrl}{HTML}{003f5c}

\icmltitlerunning{Hierarchies of Reward Machines}

\begin{document}

\twocolumn[
\icmltitle{Hierarchies of Reward Machines}



\icmlsetsymbol{equal}{*}

\begin{icmlauthorlist}
	\icmlauthor{Daniel Furelos-Blanco}{imperial}
	\icmlauthor{Mark Law}{imperial,ilasp}
	\icmlauthor{Anders Jonsson}{upf}
	\icmlauthor{Krysia Broda}{imperial}
	\icmlauthor{Alessandra Russo}{imperial}
\end{icmlauthorlist}

\icmlaffiliation{imperial}{Imperial College London, UK}
\icmlaffiliation{ilasp}{ILASP Limited, UK}
\icmlaffiliation{upf}{Universitat Pompeu Fabra, Spain}

\icmlcorrespondingauthor{Daniel Furelos-Blanco}{d.furelos-blanco18@imperial.ac.uk}

\icmlkeywords{Reinforcement Learning, Reward Machines, Hierarchical Reinforcement Learning}

\vskip 0.3in
]



\printAffiliationsAndNotice{}  

\begin{abstract}
	Reward machines~(RMs) are a recent formalism for representing the reward function of a reinforcement learning task through a finite-state machine whose edges encode subgoals of the task using high-level events. The structure of RMs enables the decomposition of a task into simpler and independently solvable subtasks that help tackle long-horizon and/or sparse reward tasks. We propose a \emph{formalism} for further abstracting the subtask structure by endowing an RM with the ability to call other RMs, thus composing a hierarchy of RMs (HRM). We \emph{exploit} HRMs by treating each call to an RM as an independently solvable subtask using the options framework, and describe a curriculum-based method to \emph{learn} HRMs from traces observed by the agent. Our experiments reveal that exploiting a handcrafted HRM leads to faster convergence than with a flat HRM, and that learning an HRM is feasible in cases where its equivalent flat representation is not.
\end{abstract}

\section{Introduction}
More than a decade ago, \citet{DietterichDGMT08} argued for the need to ``learn at \emph{multiple time scales} simultaneously, and with a rich \emph{structure} of events and durations''. Finite-state machines (FSMs) are a simple yet powerful formalism for abstractly representing temporal tasks in a structured manner. One of the most prominent recent types of FSMs used in reinforcement learning~\citep[RL;][]{SuttonB18} are reward machines \citep[RMs;][]{IcarteKVM18, IcarteKVM22}, which compactly represent state-action histories in terms of high-level events; specifically, each edge is labeled with (i)~a formula over a set of high-level events that capture a task's subgoal, and (ii)~a reward for satisfying the formula. Hence, RMs fulfill the need for structuring events and durations, and keep track of the achieved and pending subgoals.

Hierarchical reinforcement learning \citep[HRL;][]{BartoM03a} frameworks, such as options~\citep{SuttonPS99}, have been used to \emph{exploit} RMs by learning policies at two levels of abstraction: (i)~select a formula (i.e., subgoal) from a given RM state, and (ii)~select an action to (eventually) satisfy the chosen formula \citep{IcarteKVM18,FurelosBlancoLJBR21}. The subtask decomposition powered by HRL enables learning at multiple scales simultaneously, and eases the handling of long-horizon and sparse reward tasks. In addition, several works have considered the problem of \emph{learning} the RMs themselves from interaction \citep[e.g.,][]{IcarteWKVCM19,XuGAMNTW20,FurelosBlancoLJBR21,HasanbeigJAMK21}. A common problem among methods learning minimal RMs is that they scale poorly as the number of states grows.

In this work, we make the following \emph{contributions}:
\begin{enumerate}
	\item Enhance the abstraction power of RMs by defining \emph{hierarchies of RMs (HRMs)}, where constituent RMs can call other RMs (Section~\ref{sec:formalism}). We prove that any HRM can be transformed into an \emph{equivalent} flat HRM that behaves exactly like the original RMs. We show that under certain conditions, the equivalent flat HRM can have exponentially more states and edges.
	\label{contrib:formalism}
	
	\item Propose an HRL algorithm to \emph{exploit} HRMs by treating each call as a subtask (Section~\ref{sec:policy_learning}). Learning policies in HRMs further fulfills the desiderata posed by \citeauthor{DietterichDGMT08} since (i)~there is an arbitrary number of time scales to learn across (not only two), and (ii)~there is a richer range of increasingly abstract events and durations. Besides, hierarchies enable \emph{modularity} and, hence, the \emph{reusability} of the RMs and policies. Empirically, we show that leveraging a handcrafted HRM enables faster convergence than an equivalent flat HRM.
	
	\item Introduce a curriculum-based method for \emph{learning} HRMs from traces given a set of composable tasks (Section~\ref{sec:hierarchy_learning}). In line with the theory (Contribution~\ref{contrib:formalism}), our experiments reveal that decomposing an RM into several is \emph{crucial} to make its learning feasible (i.e., the flat HRM cannot be efficiently learned from scratch) since (i)~the constituent RMs are simpler (i.e., they have fewer states and edges), and (ii)~previously learned RMs can be used to efficiently \emph{explore} the environment in the search for traces in more complex tasks. 
\end{enumerate}

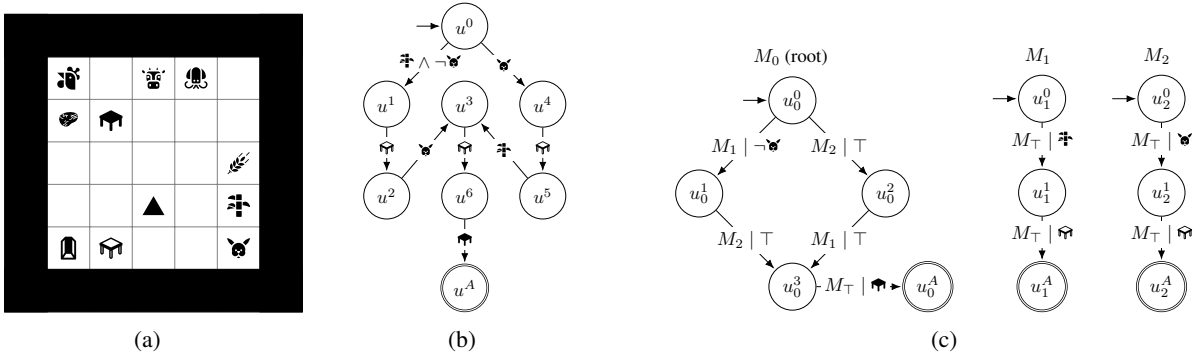
\begin{figure*}
	\hspace{1.4em}
	\begin{subfigure}[b]{0.24\linewidth}
		\centering
		\resizebox{\linewidth}{!}{
			\tikzset{digit/.style = { minimum height = 5mm, minimum width=5mm, anchor=center }}
			\newcommand{\setcell}[3]{\edef\x{#2 - 0.5}\edef\y{6.5 - #1}\node[digit,name={#1-#2}] at (\x, \y) {#3};}
			\centering
			\begin{tikzpicture}[scale=1]
				\draw[gray] (0, 0) grid (7, 7);
				
				\draw[black,fill=black] (0, 0) rectangle (1, 7);
				\draw[black,fill=black] (0, 0) rectangle (7, 1);
				\draw[black,fill=black] (6, 0) rectangle (7, 7);
				\draw[black,fill=black] (0, 6) rectangle (7, 7);
				
				\setcell{1}{2}{\Chicken[scale=2]}
				\setcell{2}{2}{\Redstone[scale=2]}
				\setcell{5}{2}{\Iron[scale=2]}
				
				\setcell{2}{3}{\Table[scale=2]}
				\setcell{5}{3}{\Workbench[scale=2]}
				
				\setcell{1}{4}{\Cow[scale=2]}
				\setcell{4}{4}{\Agent[scale=-2]}
				
				\setcell{1}{5}{\Squid[scale=2]}
				
				\setcell{3}{6}{\Wheat[scale=2]}
				\setcell{4}{6}{\Sugarcane[scale=2]}
				\setcell{5}{6}{\Rabbit[scale=2]}		
		\end{tikzpicture}}
		\caption{}
		\label{fig:craftworld_grid}
	\end{subfigure}
	\begin{subfigure}[b]{0.24\linewidth}
		\centering
		\resizebox{0.67\linewidth}{!}{
			\begin{tikzpicture}[shorten >=1pt,node distance=2.25cm,on grid,auto, every initial by arrow/.style ={-Latex} ]
				\node[state,initial,initial text=] (u_0)   {$u^0$};
				
				\node[state] (u_1) [below left =6em of u_0]  {$u^1$};
				\node[state] (u_2) [below =5em of u_1]  {$u^2$};
				\node[state] (u_3) [below =4.2426em of u_0]  {$u^3$};
				\node[state] (u_4) [below right =6em of u_0]  {$u^4$};
				\node[state] (u_5) [below =5em of u_4]  {$u^5$};
				\node[state] (u_6) [below =5em of u_3]  {$u^6$};
				\node[state,accepting] (u_acc) [below =5em of u_6]  {$u^{A}$};
				
				\path[-Latex] (u_0) edge node[pos=0.4] [in place] {$ \Sugarcane \land \neg \Rabbit$} (u_1);
				\path[-Latex] (u_1) edge node[pos=0.45] [in place] {\Workbench} (u_2);
				\path[-Latex] (u_2) edge node[pos=0.5] [in place] {\Rabbit} (u_3);
				\path[-Latex] (u_3) edge node[pos=0.45] [in place] {\Workbench} (u_6);
				\path[-Latex] (u_0) edge node[pos=0.5] [in place] {\Rabbit} (u_4);
				\path[-Latex] (u_4) edge node[pos=0.45] [in place] {\Workbench} (u_5);
				\path[-Latex] (u_5) edge node[pos=0.5] [in place] {\Sugarcane} (u_3);
				\path[-Latex] (u_6) edge node[pos=0.45] [in place] {\Table} (u_acc);
		\end{tikzpicture}}
		\caption{}
		\label{fig:book_hierarchy_flat}
	\end{subfigure}
	\begin{subfigure}[b]{0.5\linewidth}
		\centering
		\resizebox{0.85\linewidth}{!}{
			\begin{minipage}[b]{0.5\columnwidth}
				\begin{center}
					$\rmname_0$ (root)
				\end{center}
				
				\begin{tikzpicture}[shorten >=1pt,node distance=2.25cm,on grid,auto, every initial by arrow/.style ={-Latex} ] 
					\node[state,initial,initial text=] (u_0)   {$u^0_0$};
					\node[state] (u_1) [below left = 7.01em of u_0]  {$u^1_0$};
					\node[state] (u_2) [below right = 7.01em of u_0]  {$u^2_0$};
					\node[state] (u_3) [below right = 7.01em of u_1]  {$u^3_0$};
					\node[state,accepting] (u_acc) [right =7.2em of u_3]  {$u^{A}_0$};
					
					\path[-Latex] (u_0) edge node[pos=0.5] [in place] {$\rmname_1\mid \neg \Rabbit$} (u_1);
					\path[-Latex] (u_0) edge node[pos=0.5] [in place] {$\rmname_2\mid \top$} (u_2);
					\path[-Latex] (u_1) edge node[pos=0.5] [in place] {$\rmname_2\mid \top$} (u_3);
					\path[-Latex] (u_2) edge node[pos=0.5] [in place] {$\rmname_1\mid \top$} (u_3);
					\path[-Latex] (u_3) edge node[pos=0.45] [in place] {$\leaf\mid\Table$} (u_acc);
				\end{tikzpicture}
			\end{minipage}
			
			\hspace{3.5em}
			
			\begin{minipage}[b]{0.25\columnwidth}
				\begin{center}
					$\rmname_1$
				\end{center}
				
				\begin{tikzpicture}[shorten >=1pt,node distance=2cm,on grid,auto, every initial by arrow/.style ={-Latex} ]
					\node[state,initial,initial text=] (u_0)   {$u^0_1$};
					\node[state] (u_1) [below =5em of u_0]  {$u^1_1$};
					\node[state,accepting] (u_acc) [below =5em of u_1]  {$u^A_1$};
					
					\path[-Latex] (u_0) edge node[pos=0.4] [in place] {$\leaf\mid\Sugarcane$} (u_1);
					\path[-Latex] (u_1) edge node[pos=0.4] [in place] {$\leaf\mid\Workbench$} (u_acc);
				\end{tikzpicture}
			\end{minipage}
			
			\begin{minipage}[b]{0.25\columnwidth}
				\begin{center}
					$\rmname_2$	
				\end{center}
				
				\begin{tikzpicture}[shorten >=1pt,node distance=2cm,on grid,auto, every initial by arrow/.style ={-Latex} ]
					\node[state,initial,initial text=] (u_0)   {$u^0_2$};
					\node[state] (u_1) [below =5em of u_0]  {$u^1_2$};
					\node[state,accepting] (u_acc) [below =5em of u_1]  {$u^A_2$};
					
					\path[-Latex] (u_0) edge node[pos=0.4] [in place] {$\leaf\mid\Rabbit$} (u_1);
					\path[-Latex] (u_1) edge node[pos=0.4] [in place] {$\leaf\mid\Workbench$} (u_acc);
				\end{tikzpicture}
			\end{minipage}
		}
		\caption{}
		\label{fig:book_hierarchy}
	\end{subfigure}
	\caption{A \textsc{CraftWorld} grid~(a), and a flat HRM (b) and a non-flat one (c) for \textsc{Book}. In (b), an edge from $u$ to $u'$ is labeled $\rmtransition(u,u')$. In (c), an edge from $u$ to $u'$ in RM $\rmname_i$ is labeled $\rmname_j \mid \rmtransition_i(u,u',\rmname_j)$. In both cases, accepting states are double circled, and loops are omitted.}
	\label{fig:craftworld_book_example}
\end{figure*}
\begin{table*}
	\caption{List of \textsc{CraftWorld} tasks. Descriptions ``$x$ ; $y$'' express sequential order (observe/do $x$ then $y$), descriptions ``$x$ \& $y$'' express that $x$ and $y$ can be observed/done in any order, and $\rmheight$ is the root RM's height.}
	\label{tab:craftworld_tasks}
	\centering
	\resizebox{0.8\linewidth}{!}{
		\begin{tabular}{lclclclclcl}
			\toprule
			\multicolumn{1}{c}{Task} & $h$   & \multicolumn{1}{c}{Description}         & & \multicolumn{1}{c}{Task}& $h$     & \multicolumn{1}{c}{Description}                 & & \multicolumn{1}{c}{Task}   & $h$   & \multicolumn{1}{c}{Description}\\
			\cmidrule{1-3} \cmidrule{5-7} \cmidrule{9-11}
			\textsc{Batter}          & 1     & ({\Wheat} \& {\Chicken}) ; {\Table}     & &  \textsc{Quill}         & 1       & ({\Squid} \& {\Chicken}) ; {\Table}             & & \textsc{BookQuill}         & 3     & \textsc{Book} \& \textsc{Quill}\\
			\textsc{Bucket}          & 1     & {\Iron} ; \Table                        & & \textsc{Sugar}          & 1       & {\Sugarcane} ; \Table                           & & \textsc{MilkB.Sugar}       & 3     & \textsc{MilkBucket} \& \textsc{Sugar} \\
			\textsc{Compass}         & 1     & ({\Iron} \& {\Redstone}) ; {\Workbench} & & \textsc{Book}           & 2       & (\textsc{Paper} \& \textsc{Leather}) ; {\Table} & & \textsc{Cake}              & 4     & \textsc{Batter} ; \textsc{MilkB.Sugar} ; \Workbench \\
			\textsc{Leather}         & 1     & {\Rabbit} ; {\Workbench}                & & \textsc{Map}            & 2       & (\textsc{Paper} \& \textsc{Compass}) ; \Table   & &                            &       &                                \\
			\textsc{Paper}           & 1     & {\Sugarcane} ; \Workbench               & & \textsc{MilkBucket}     & 2       & \textsc{Bucket} ; \Cow                          & &                            &       &                                \\
			\bottomrule
	\end{tabular}}
\end{table*}

\section{Background}
\label{sec:background}
Given a finite set $\arbitraryset$, we use $\Delta(\arbitraryset)$ to denote the probability simplex over $\arbitraryset$, $\arbitraryset^*$ to denote (possibly empty) sequences of elements from $\arbitraryset$, and $\arbitraryset^+$ to denote non-empty sequences. We use $\bot$ and $\top$ to denote the truth values false and true, respectively. $\mathds{1}[A]$ is the indicator function of event $A$.

\textbf{Reinforcement Learning.}~
We represent RL tasks as \emph{episodic} labeled Markov decision processes~\citep[MDPs;][]{XuGAMNTW20}, each consisting of a set of states $\mdpstates$, a set of actions $\mdpactions$, a transition function $\mdpprob:\mdpstates \times \mdpactions \to \Delta(\mdpstates)$, a reward function $\mdprewfunc:(\mdpstates \times \mdpactions)^+ \times \mdpstates \to \mathbb{R}$, a discount factor $\mdpdiscount\in[0,1)$, a finite set of \emph{propositions} $\propset$ representing high-level events, a \emph{labeling function} $\lfunc:\mdpstates \to 2^\propset$ mapping states to proposition subsets called \emph{labels}, and a \emph{termination function} $\tau: (\mdpstates \times \mdpactions)^* \times \mdpstates \to \{\bot,\top\}\times\{\bot,\top\}$. Hence the transition function $\mdpprob$ is Markovian, but the reward function $\mdprewfunc$ and termination function $\mdptermfunc$ are not. Given a \emph{history} $\history_t=\langle \mdpstate_0,\mdpaction_0,\ldots,\mdpstate_t\rangle\in(\mdpstates\times \mdpactions)^*\times\mdpstates$, a {\em label trace} (or trace, for short) $\trace_t=\langle\lfunc(\mdpstate_0),\ldots,\lfunc(\mdpstate_t)\rangle\in(2^\propset)^+$ assigns labels to all states in $\history_t$. We assume $(\trace_t,\mdpstate_t)$ captures all relevant information about $\history_t$; thus, the reward and transition information can be written $\mdprewfunc(\history_t,\mdpaction_t,\mdpstate_{t+1})=\mdprewfunc(h_{t+1})=\mdprewfunc(\trace_{t+1},\mdpstate_{t+1})$ and $\mdptermfunc(h_t)=\mdptermfunc(\trace_t,\mdpstate_t)$, respectively. We aim to find a \emph{policy} $\pi:(2^\propset)^+ \times \mdpstates \to \mdpactions$, a mapping from traces-states to actions, that maximizes the expected cumulative discounted reward (or \emph{return}) $R_t=\mathbb{E}_\pi[\sum_{k=t}^n\mdpdiscount^{k-t}\mdprewfunc(\trace_{k+1},\mdpstate_{k+1})]$, where $n$ is the last episode's step.

At time $t$, the trace is $\trace_t\in(2^\propset)^+$, and the agent observes a tuple $\obstuple_t=\langle \obs_t, \obsterm_t, \obsgoal_t \rangle$, where $\obs_t\in \mdpstates$ is the state and $(\obsterm_t,\obsgoal_t)=\tau(\trace_t,\obs_t)$ is the termination information, with $\obsterm_t$ and $\obsgoal_t$ indicating whether or not the history $(\trace_t,\obs_t)$ is terminal or a goal, respectively. The agent also observes a label $\proplabel_{t} = \lfunc(\mdpstate_t)$. If the history is non-terminal, the agent runs action $a_t\in \mdpactions$, and the environment transitions to state $\obs_{t+1} \sim \mdpprob(\cdot|\obs_t,a_t)$. The agent then observes tuple $\obstuple_{t+1}$ and label $\proplabel_{t+1}$, extends the trace as $\trace_{t+1}=\trace_t\oplus\proplabel_{t+1}$, and receives reward $r_{t+1}= r(\trace_{t+1},\obs_{t+1})$. A trace $\trace_t$ is a \emph{goal} trace if $(\obsterm_t, \obsgoal_t) = (\top, \top)$, a \emph{dead-end} trace if $(\obsterm_t, \obsgoal_t) = (\top, \bot)$, and an \emph{incomplete} trace if $\obsterm_t = \bot$. We assume that the reward is $r(\trace_{t+1},\obs_{t+1})=\mathds{1}[\tau(\trace_{t+1},\obs_{t+1})=(\top,\top)]$, i.e.~$1$ for goal traces and $0$ otherwise.

	\begin{example}
		The \textsc{CraftWorld} domain (cf.~Figure~\ref{fig:craftworld_grid}) is used as a running example. In this domain, the agent (\Agent[scale=-1]) can move forward or rotate 90$^{\circ}$, staying put if it moves towards a wall. Locations are labeled with propositions from $\propset = \lbrace \Iron, \Table,\Cow,\Sugarcane, \Wheat, \Chicken,\allowbreak \Redstone, \Rabbit, \Squid, \Workbench\rbrace$. The agent observes propositions that it steps on, e.g. 
		$\lbrace\Chicken\rbrace$ in the top-left corner. Table~\ref{tab:craftworld_tasks} lists the tasks we consider, which consist of observing a sequence of propositions.\footnote{The tasks are based on those by \citet{AndreasKL17} and \citet{IcarteKVM18}, but definable in terms of each other.} For the task \textsc{Book}, two goal traces are $\langle \{\Sugarcane\}, \{\Workbench\}, \{\Rabbit\}, \{\Workbench\}, \{\Table\} \rangle$ and $\langle \{\Rabbit\}, \{\Workbench\},\{\Sugarcane\}, \{\Workbench\}, \{\Table\} \rangle$ (they could contain irrelevant labels in between). The traces $\langle\{\Rabbit\},\{\Workbench\}\rangle$ and $\langle\{\Chicken\}\rangle$ are incomplete. There are no dead-end traces in this scenario.
	\end{example}

Options~\citep{SuttonPS99} address temporal abstraction in RL. Given an episodic labeled MDP, an option is a tuple $\opt=\langle \optinit, \pi_\opt, \beta_\opt\rangle$, where $\optinit \subseteq \mdpstates$ is the option's initiation set, $\pi_\opt: \mdpstates \to \mdpactions$ is the option's policy, and $\beta_\opt: \mdpstates \to [0,1] $ is the option's termination condition. An option is available in $\mdpstate \in \mdpstates$ if $\mdpstate \in \optinit$, selects actions according to $\pi_\opt$, and terminates in $\mdpstate \in \mdpstates$ with probability $\beta_\opt(\mdpstate)$.

\textbf{Reward Machines.}~
A \emph{(simple) reward machine} \citep[RM;][]{IcarteKVM18,IcarteKVM22} is a tuple $\rmtuple$, where $\rmstates$ is a finite set of states; $\propset$ is a finite set of propositions; $\rmtransition:\rmstates \times \rmstates \to \dnfprop$ is a state transition function such that $\rmtransition(u,u')$ denotes the disjunctive normal form (DNF) formula over $\propset$ to be satisfied to transition from $u$ to $u'$; $\rmreward:\rmstates \times \rmstates \to \mathbb{R}$ is a reward function such that $\rmreward(u,u')$ is the reward for transitioning from $u$ to $u'$; $\rminitstate\in \rmstates$ is an initial state; $\rmstatesacc \subseteq \rmstates$ is a set of accepting states denoting the task's goal achievement; and $\rmstatesrej \subseteq \rmstates$ is a set of rejecting states denoting the unfeasibility of achieving the goal. The state transition function is \emph{deterministic}, i.e.~at most one formula from each state is satisfied. To verify if a formula is satisfied by a label $\proplabel\subseteq \propset$, $\proplabel$ is used as a truth assignment where propositions in $\proplabel$ are true, and false otherwise (e.g., $\{a\} \models a \land \neg b$). If no transition formula is satisfied, the state remains unchanged.

Ideally, RM states should capture traces, such that (i)~pairs $(u,\mdpstate)$ of an RM state and an MDP state make termination and reward Markovian, (ii)~the reward $\rmreward(u,u')$ matches the underlying MDP's reward, and (iii)~goal traces end in an accepting state, rejecting traces end in a rejecting state, and incomplete traces do not end in accepting or rejecting states. As per the previous reward assumption, the reward transition functions are $r(u,u')=\mathds{1}[u\notin\rmstatesacc \land u'\in \rmstatesacc]$.	

	\begin{example}
		Figure~\ref{fig:book_hierarchy_flat} shows an RM for \textsc{Book}. 
		The state transition function $\rmtransition$ is deterministic since no label satisfies both $\rmtransition(u^0, u^1)=\Sugarcane \land \neg \Rabbit$ and $\rmtransition(u^0,u^4)=\Rabbit$; in contrast, $\rmtransition$ becomes non-deterministic if $\rmtransition(u^0, u^1)=\Sugarcane$ since $\{\Sugarcane, \Rabbit\}$ satisfies both formulas. Note that RMs compactly represent traces in terms of key events, e.g.~$u^2$ indicates that a label satisfying $\Sugarcane \land \neg \Rabbit$ followed by another satisfying $\Workbench$ were observed (everything else is ignored).
	\end{example}

\section{Formalization of HRMs}
\label{sec:formalism}
RMs are the building blocks of our formalism. To constitute a hierarchy of RMs, we need to endow RMs with the ability to call each other. We redefine the \emph{state transition function} as $\rmtransition:\rmstates \times \rmstates \times \machineset \to \dnfprop$, where $\machineset$ is a set of RMs. The expression $\varphi(u,u',\rmname)$ denotes the DNF formula over $\propset$ that must be satisfied to transition from $u\in \rmstates$ to $u'\in \rmstates$ by calling RM $\rmname\in \machineset$. We refer to the formulas $\varphi(u,u',\rmname)$ as \emph{contexts} since they represent conditions under which calls are made. As we shall see later, contexts help preserve determinism and must be satisfied to start a call (a necessary but not sufficient condition). The hierarchies we consider contain an RM $\leaf$ called the \emph{leaf} RM, which solely consists of an accepting state (i.e., $\rmstates_\top = \rmstatesacc_\top = \{\rminitstate_\top\}$), and immediately returns control to the RM that calls it.

\begin{definition}
	A \emph{hierarchy of reward machines~(HRM)} is a tuple $\hrmtuple$, where $\machineset=\{\rmname_0,\ldots, \rmname_{m-1} \} \cup \{\leaf\}$ is a set of $m$ RMs and the leaf RM $\leaf$, $\rmname_r \in \machineset \setminus \{\leaf\}$ is the root RM, and $\propset$ is a finite set of propositions used by all constituent RMs.
\end{definition}

We make the following \emph{assumptions}: (i)~HRMs do not have circular dependencies (i.e., an RM cannot be called back from itself, including recursion), (ii)~rejecting states are global (i.e., cause the root task to fail), (iii)~accepting and rejecting states do not have transitions to other states, and (iv)~the reward function of the root corresponds to the reward obtained in the underlying MDP. Given assumption~(i), each RM $\rmname_i$ has a \emph{height} $\rmheight_i$, which corresponds to the maximum number of nested calls needed to reach the leaf. Formally, if $i=\top$, then $\rmheight_i=0$; otherwise, $\rmheight_i=1+\max_{j}\rmheight_j$, where $j$ ranges over all RMs called by $\rmname_i$ (i.e., there exists $(u,v) \in \rmstates_i \times \rmstates_i$ such that $\rmtransition_i(u,v,\rmname_j)\neq\bot$).

\begin{example}
	Figure~\ref{fig:book_hierarchy} shows \textsc{Book}'s HRM, whose root has height 2. The \textsc{Paper} and \textsc{Leather} RMs, which have height~1 and consist of observing a two-proposition sequence, can be run in any order followed by observing {\Table}. The context $\neg\Rabbit$ in the call to $\rmname_1$ preserves determinism, as detailed later.
\end{example}

In the following paragraphs, we describe how an HRM processes a label trace. To indicate where the agent is in an HRM, we define the notion of \emph{hierarchy states}.
\begin{definition}
	Given an HRM $\hrmtuple$, a \emph{hierarchy state} is a tuple $\langle \rmname_i, u, \Context, \stack \rangle$, where $\rmname_i\in\machineset$ is an RM, $u\in \rmstates_i$ is a state, $\Context\in \dnfprop$ is an accumulated context, and $\stack$ is a call stack.
\end{definition}

\begin{definition}
	Given an HRM $\hrmtuple$, a \emph{call stack} $\stack$ contains tuples $\langle u,v,\rmname_i,\rmname_j,\context,\Context \rangle$, each denoting a call where $u\in \rmstates_i$ is the state from which the call is made; $v\in \rmstates_i$ is the next state in the calling RM $\rmname_i\in\machineset$ after reaching an accepting state of the called RM $\rmname_j \in \machineset$; $\context\in\dnfprop$ are the disjuncts of $\rmtransition_i(u,v,\rmname_j)$ satisfied by a label; and $\Context\in \dnfprop$ is the accumulated context.
\end{definition}

Call stacks determine where to resume the execution. Each RM appears in the stack at most once since, by assumption, HRMs have no circular dependencies. We use $\stack \oplus \langle u,v,\rmname_i,\rmname_j,\context,\Context \rangle$ to denote a stack recursively defined by a stack $\stack$ and a top element $\langle u,v,\rmname_i,\rmname_j,\context,\Context \rangle$, where the \emph{accumulated context} $\Context$ is the condition under which a call from a state $u$ is made. The initial hierarchy state of an HRM $\hrmtuple$ is $\langle \hrmroot, u^0_r, \top, []\rangle$: we are in the initial state of the root, there is no accumulated context, and the stack is empty.

At the beginning of this section, we mentioned that satisfying the context of a call is a necessary but not sufficient condition to start the call. We now introduce a sufficient condition, called \emph{exit condition}.

\begin{definition}
	Given an HRM $\hrmtuple$ and a hierarchy state $\langle \rmname_i, u, \Context, \stack\rangle$, the \emph{exit condition} $\excond_{i,u,\Context}\in\dnfprop$ is the formula that must be satisfied to leave that hierarchy state. Formally, 
	\begin{align*}
		\excond_{i,u,\Context}=\begin{cases}
			\Context & \textnormal{if~}i=\top,\\
			\bigvee_{\substack{\phi = \varphi_i(u,v,\rmname_j),\\\phi \neq \bot, v \in \rmstates_i,\rmname_j\in\machineset}}\excond_{j,u^0_j,\dnf(\Context\land\phi)}  & \textnormal{otherwise},
		\end{cases}
	\end{align*}
	where $\textnormal{DNF}(\Context \land \phi)$ is $\Context \land \phi$ in DNF. The formula is $\Context$ if $\rmname_i=\leaf$ since it always returns control once called. Otherwise, the formula is recursively defined as the disjunction of the exit conditions from the initial state of the called RM. For instance, the exit condition for the initial hierarchy state in Figure~\ref{fig:book_hierarchy} is $(\neg \Rabbit \land \Sugarcane) \vee \Rabbit$.
\end{definition}

We can now define the \emph{hierarchical transition function} $\hrmtrans$, which maps a hierarchy state $\langle \rmname_i, u, \Context, \stack \rangle$ into another given a label $\proplabel$. There are three cases:
\begin{enumerate}
	\item If $u$ is an accepting state of $\rmname_i$ and the stack $\stack$ is non-empty, pop the top element of $\stack$ and return control to the previous RM, recursively applying $\hrmtrans$ in case several accepting states are reached simultaneously. Formally, the next hierarchy state is $\hrmtrans(\langle \rmname_j, u', \top, \stack' \rangle,\bot)$ if $u\in \rmstatesacc_i$, $|\stack| > 0$, where $\stack=\stack'\oplus\langle\cdot,u',\rmname_j,\rmname_i, \cdot,\cdot \rangle$, $\bot$ denotes a label that cannot satisfy any formula, and $\cdot$ denotes something unimportant for the case.
	
	\item If $\proplabel$ satisfies the context of a call and the exit condition from the initial state of the called RM, push the call onto the stack and recursively apply $\hrmtrans$ until $\leaf$ is reached. Formally, the next hierarchy state is $\hrmtrans(\langle \rmname_j, u^0_j,\Context',\Gamma\oplus\langle u,u',\rmname_i,\rmname_j,\phi,\Context \rangle \rangle, \proplabel)$ if $\proplabel\models\excond_{j,u^0_j,\Context'}$, where $\phi = \rmtransition_i(u,u',\rmname_j)(\proplabel)$ and $\Context'=\dnf(\Context\land\phi)$. Here, $\varphi(\proplabel)$ denotes the disjuncts of a DNF formula $\varphi \in \dnfprop$ satisfied by  $\proplabel$.
	
	\item If none of the above holds, the hierarchy state remains unchanged. 
\end{enumerate}

The state transition functions $\rmtransition$ of the RMs must be such that $\hrmtrans$ is \emph{deterministic}, i.e.~a label cannot simultaneously satisfy the contexts and exit conditions associated with two triplets $\langle u,v,\rmname_i\rangle$ and $\langle u,v',\rmname_j\rangle$ such that either (i) $v= v'$ and $i \neq j$, or (ii) $v\neq v'$. Contexts help enforce determinism by making formulas mutually exclusive. For instance, if the call to $\rmname_1$ from the initial state of $\rmname_0$ in Figure~\ref{fig:book_hierarchy} had context $\top$ instead of $\neg \Rabbit$, then $\rmname_1$ and $\rmname_2$ could be both started if $\{\Sugarcane,\Rabbit\}$ was observed, thus making the HRM non-deterministic. Finally, we introduce \emph{hierarchy traversals}, which determine how a label trace is processed by an HRM.

\begin{definition}
	Given a label trace $\trace=\langle \proplabel_0, \ldots, \proplabel_n\rangle$, a \emph{hierarchy traversal} $\hrm(\trace)=\langle v_0,v_1,\ldots,\allowbreak v_{n+1} \rangle$ is a unique sequence of hierarchy states such that (i)~$v_0 = \langle \rmname_r,u_r^0,\top,[] \rangle$, and (ii)~$\hrmtrans(v_i,\proplabel_i)=v_{i+1}$ for $i=0,\ldots,n$. An HRM $\hrm$ \emph{accepts} $\trace$ if $v_{n+1}=\langle \rmname_r,u,\top, [] \rangle$ and $u \in \rmstatesacc_r$ (i.e., an accepting state of the root is reached). Analogously, $\hrm$ \emph{rejects} $\trace$ if $v_{n+1}=\langle \rmname_k,u,\cdot,\cdot\rangle$ and $u \in \rmstatesrej_k$ for any $k \in [0,m-1]$ (i.e., a rejecting state in the HRM is reached).
\end{definition}

\begin{example}
	The HRM in Figure~\ref{fig:book_hierarchy} accepts $\trace=\allowbreak\langle \{\Sugarcane\},\allowbreak \{\Workbench\},\allowbreak \{\},\allowbreak\{\Rabbit\},\allowbreak \{\Workbench\},\allowbreak \{\Table\} \rangle$ since $\hrm(\trace)=\langle\langle\rmname_0, u^0_0, \top, [] \rangle,\allowbreak \langle\rmname_1, u^1_1, \top, [\langle u^0_0, u^1_0, \rmname_0, \rmname_1, \neg\Rabbit, \top \rangle] \rangle,\allowbreak \langle\rmname_0, u^1_0, \top, [] \rangle,\allowbreak \langle\rmname_0, u^1_0, \top, [] \rangle, \langle\rmname_2, u^1_2, \top, [\langle u^1_0, u^3_0, \rmname_0, \rmname_2, \top, \top \rangle] \rangle,\allowbreak \langle\rmname_0, u^3_0, \top, [] \rangle,\allowbreak \langle\rmname_0, u^A_0, \top, [] \rangle\rangle$. Appendix~\ref{app:hierarchy_traversal_full_example} shows the step-by-step application of $\hrmtrans$ omitted here.
	\label{example:hierarchy_traversal_short_example}
\end{example}

The behavior of an HRM $\hrm$ can be reproduced by an \emph{equivalent flat} HRM $\bar{\hrm}$; that is, (i)~the root of $\bar{\hrm}$ has height~1 and, (ii)~$\bar{\hrm}$ accepts a trace iff $\hrm$ accepts it, rejects a trace iff $\hrm$ rejects it, and neither accepts nor rejects a trace iff $\hrm$ does not accept it nor reject it. Flat HRMs thus capture the original RM definition, e.g.~Figure~\ref{fig:book_hierarchy_flat} is a flat HRM for \textsc{Book}. We formally define equivalence and prove the following equivalence theorem by construction in Appendix~\ref{app:flat_equivalent}.

\begin{restatable}{theorem}{flatequivalent}
	\label{thm:flat_equivalent}
	Given an HRM $\hrm$, there exists an equivalent flat HRM $\bar{\hrm}$.
\end{restatable}

Given the construction used in Theorem~\ref{thm:flat_equivalent}, we show that the number of states and edges of the resulting flat HRM can be \emph{exponential} in the height of the root (see Theorem~\ref{thm:flat_size}). We prove this in Appendix~\ref{app:flat_size} through an instance of a general HRM parametrization where the constituent RMs are \emph{highly reused}, hence illustrating the convenience of HRMs to succinctly compose existing knowledge. In line with the theory, learning a non-flat HRM can take a few seconds, whereas learning an equivalent flat HRM is often unfeasible (see Section~\ref{sec:experimental_results}).

\begin{restatable}{theorem}{flatsize}
	Let $\hrmtuple$ be an HRM and $\rootheight$ be the height of its root $\rmname_r$. The number of states and edges in an equivalent flat HRM $\bar{\hrm}$ can be exponential in $\rootheight$.
	\label{thm:flat_size}
\end{restatable}

\section{Policy Learning in HRMs}
\label{sec:policy_learning}

In what follows, we explain how to \emph{exploit} the temporal structure of an HRM $\hrmtuple$ using two types of \emph{options}. We also describe (i)~how to learn the policies of these options, (ii)~when these options terminate, and (iii)~an option selection algorithm that ensures the currently running options and the current hierarchy state are aligned.

\textbf{Option Types.} Given an RM $\rmname_i \in \machineset$, a state $u\in \rmstates_i$ and a context $\Context$, an option $\opt^{j,\context}_{i,u,\Context}$ is derived for each non-false disjunct $\context$ of each transition $\rmtransition_i(u,v,\rmname_j)$, where $v\in \rmstates_i$ and $\rmname_j\in\machineset$. An option is either (i)~a \emph{formula option} if $j=\top$ (i.e., $\leaf$ is called), or (ii)~a \emph{call option} otherwise. A formula option attempts to reach a label that satisfies $\context \land \Context$ through primitive actions, whereas a call option aims to reach an accepting state of the called RM $\rmname_j$ under context $\context \land \Context$ by invoking other options.

\textbf{Policies.}
Policies are $\epsilon$-greedy during training, and greedy during evaluation. A \emph{formula option's policy} is derived from a Q-function $\qfunc_{\context\land\Context}(\obs,a;\nnparams_{\context\land\Context})$ approximated by a deep Q-network \citep[DQN;][]{MnihKSRVBGRFOPB15} with parameters $\nnparams_{\context\land\Context}$, which outputs the Q-value of each action given an MDP state. We store all options' experiences $(\obstuple_t,a,\obstuple_{t+1}, \proplabel_{t+1})$ in a single replay buffer $\mathcal{D}$, thus performing intra-option learning~\citep{SuttonPS98}. The Q-learning update uses the following loss function:
\begin{align}
	\mathbb{E}_{(\obstuple_t,a,\obstuple_{t+1},\proplabel_{t+1})\sim\mathcal{D}}\left[\left(y_{\context\land\Context} - \qfunc_{\context\land\Context}(\obs_t,a;\nnparams_{\context\land\Context})\right)^2 \right],
	\label{eq:q_learning}
\end{align}
where $y_{\context\land\Context}=r_{\context\land\Context}(\proplabel_{t+1})+\mdpdiscount\max\limits_{a'}\qfunc_{\context\land\Context}(\obs_{t+1},a';\nnparams^-_{\context\land\Context})$. The reward $r_{\context\land\Context}(\proplabel_{t+1})$ is 1 if $\context\land \Context$ is satisfied by $\proplabel_{t+1}$ and 0 otherwise; the term $\qfunc_{\context\land\Context}(\obs_{t+1},a';\nnparams^-_{\context\land\Context})$ is 0 when $\context\land\Context$ is satisfied or a dead-end is reached (i.e.,~$\obsterm_{t+1}=\top$ and $\obsgoal_{t+1}=\bot$); and $\nnparams^-_{\context\land\Context}$ are the parameters of a fixed target network.

A \emph{call option's policy} is induced by a Q-function $\qfunc_i(\obs,u,\Context,\langle\rmname_j,\context \rangle ; \nnparams_i)$ associated with the called RM $\rmname_i$ and approximated by a DQN with parameters $\nnparams_i$ that outputs the Q-value of each call in the RM given an MDP state, an RM state and a context. We store experiences $(\obstuple_t,\omega^{j,\context}_{i,u,\Context},\obstuple_{t+k})$ in a replay buffer $\mathcal{D}_i$ associated with $\rmname_i$, and perform SMDP Q-learning using the following loss:
\begin{align*}
	\mathbb{E}_{(\obstuple_t,\omega^{j,\context}_{i,u,\Context},\obstuple_{t+k})\sim\mathcal{D}_i}\left[\left(y_i - \qfunc_{i}(\obs_t,u,\Context,\langle \rmname_j, \context \rangle; \nnparams_i)  \right)^2\right],
\end{align*}
where $y_i=r + \mdpdiscount^k\max_{j',\context'}\qfunc_{i}(\obs_{t+k},u',\Context',\langle\rmname_{j'},\context' \rangle;\nnparams^-_i)$; $k$ is the number of steps between $\obstuple_t$ and $\obstuple_{t+k}$; $r$ is the sum of discounted rewards during this time; $u'$ and $\Context'$ are the RM state and context after running the option; $\rmname_{j'}$ and $\context'$ correspond to an outgoing transition from $u'$, i.e.~$\context' \in \rmtransition_i(u',\cdot,\rmname_{j'})$; and $\nnparams^-_i$ are the parameters of a fixed target network. The term $\qfunc_i(\obs_{t+k}, u', \ldots)$ is 0 if $u'$ is accepting or rejecting. Following the definition of $\hrmtrans$, $\Context'$ is $\top$ if the hierarchy state changes; thus, $\Context'=\top$ if $u'\neq u$, and $\Context'=\Context$ otherwise. Given our assumption on the MDP reward, we define reward transition functions as $\rmreward_i(u,u')=\mathds{1}[u\notin \rmstatesacc_i \land u' \in \rmstatesacc_i]$. Learning a call option's policy and lower-level option policies at once can be unstable due to \emph{non-stationarity}~\citep{LevyKPS19}, e.g.~a lower-level option may not achieve its goal at times. To relax the issue, experiences are added to the buffer only when options achieve their goal (i.e., call options assume lower-level options terminate successfully). The policies will be \emph{recursively optimal}~\citep{Dietterich00} as each subtask is optimized individually; however, since the Q-functions are approximated, policies may only be approximately optimal. The implementation details are discussed in Appendix~\ref{app:policy_learning_policy_details}.

\textbf{Termination.} 
An option terminates in two cases. First, if the episode ends in a goal or dead-end state. Second, if the hierarchy state changes and either successfully completes the option or interrupts the option. Concretely, a formula option $\opt^{\top,\phi}_{i,u,\Context}$ is only applicable in a hierarchy state $\langle \rmname_i, u, \Context, \stack \rangle$, while a call option $\opt^{j,\phi}_{i,u,\Context}$ always corresponds to a stack item $\langle u,\cdot,\rmname_i,\rmname_j,\phi,\Context\rangle$. We can thus analyze the hierarchy state to see if an option is still executing or should terminate.

\textbf{Algorithm.}
An \emph{option stack} $\opth$ stores the currently executing options. Initially, $\opth$ is empty. At each step, $\opth$ is filled (if needed) by repeatedly choosing options starting from the current hierarchy state using call option policies until a formula option is selected. Since HRMs have, by assumption, no circular dependencies, a formula option will eventually be chosen. An action is then selected using the formula option's policy. Once the action is applied, the DQNs associated with formula options are updated. The new hierarchy state is then used to determine which options in $\opth$ have terminated. Experiences for the terminated options that achieved their goal are pushed into the corresponding buffers, and the DQNs associated with the call options are updated. Finally, $\opth$ is updated to \emph{match the call stack} of the new hierarchy state (if needed) by mapping each call stack item into an option, and adding it to $\opth$ if it is not already there. By aligning the option stack with the call stack, we can update DQNs for options that ended up being run in \emph{hindsight} and which would have been otherwise ignored. We refer the reader to Appendix~\ref{app:option_selection_algorithm} for the pseudo-code and step-by-step examples.

	\begin{example}
		Given the HRM in Figure~\ref{fig:book_hierarchy}, let us assume the agent had chosen to run $\rmname_1$ from $u^0_0$. The option running $\rmname_1$ is interrupted if the agent observes $\{\Rabbit\}$ before $\{\Sugarcane\}$ since $\rmname_2$ is started; thus, $\opth$ is updated and indicates that the agent is now acting according to $\rmname_2$. In contrast, if $\{\Sugarcane\}$ is observed, the agent gets into $\rmname_1$ as originally decided and, hence, the corresponding option is not interrupted. 	
	\end{example}

\begin{figure*}
	\centering
	\tikzstyle{none}=[inner sep=0mm]
	\tikzstyle{input} = [inner sep=0mm, outer sep=-0.3]
	\tikzstyle{output} = [inner sep=0mm, outer sep=-0.3]
	\tikzstyle{arrow}=[->,>=stealth]  
	\resizebox{0.7\linewidth}{!}{
		\begin{tikzpicture}	
			\node [fill=white, draw=black, shape=rectangle, minimum width=1cm, minimum height=0.6cm] (task_instance) at (10, 10){Task -- Instance};
			
			\node [fill=white, draw=black, shape=rectangle, minimum width=1cm, minimum height=0.6cm, right=2 of task_instance] (agent) {Agent};
			\draw ([yshift=1]task_instance.east) edge[arrow, bend left] node[midway, fill=white] {$\obstuple_t, \proplabel_t$} ([yshift=1]agent.west) ;
			\draw ([yshift=-1]agent.west) edge[arrow, bend left] node[midway,fill=white] {$a_{t+1}$} ([yshift=-1]task_instance.east);
			
			\node [cascaded, fill=white, draw=black, shape=rectangle, minimum width=1cm, minimum height=0.6cm, above=1.65 of task_instance] (hrms) {HRMs};
			\draw[arrow] (task_instance) -- (hrms) node[midway, fill=white] {\footnotesize task};
			
			\path let \p1 = (hrms), \p2=(agent) in coordinate (task_hrm_coords) at (\x2,\y1);
			\node [fill=white, draw=black, shape=rectangle, minimum width=1cm, minimum height=0.6cm] (task_hrm) at (task_hrm_coords) {Task HRM};
			\draw[arrow] (hrms) -- (task_hrm);
			\draw ([xshift=-1]agent.north) edge[arrow, bend left=50] node[midway, fill=white]  {$\mathcal{L}_t$} ([xshift=-1]task_hrm.south);
			\draw ([xshift=1]task_hrm.south) edge[arrow, bend left=50] node[right,xshift=-10,yshift=-1.5, fill=white] {$\mdprewfunc_{t+1}, \hrmstate_{t+1}$} ([xshift=1]agent.north);
			
			\node [fill=white, draw=black, shape=rectangle, minimum width=1cm, minimum height=0.6cm, right=2.5 of task_hrm] (ilasp) {ILASP};
			\draw [style=arrow] ([yshift=4.3]hrms.north) -- +(0,+0.55) -| node[xshift=-110, fill=white] {\footnotesize Callable RMs ($\machineset_c$)} ([xshift=5]ilasp.north);
			\draw [style=arrow] (ilasp.north) -- +(0,+0.3) -| node[xshift=113, fill=white] {\footnotesize Updated HRM} ([xshift=5,yshift=4.3]hrms.north);
			
			\node [cascaded, fill=white, draw=black, shape=rectangle, minimum width=1cm, minimum height=0.6cm, below=0.3 of ilasp] (traces) {Traces};
			\draw [style=arrow] ([yshift=4]traces.north) -- (ilasp);
			
			\node[regular polygon, regular polygon sides=6, minimum width=1cm, xscale=3,draw,label=center:\footnotesize Counterexample?, below=0.6 of traces] (decision) {};
			
			\draw [style=arrow] (decision) -- (traces) node [fill=white, midway, pos=0.435] {\footnotesize yes};
			\draw [style=arrow] (agent) -- (decision) node [fill=white, midway] {$\obstuple_t, \hrmstate_{t+1}$};
			
			\node [fill=white, draw=black, shape=rectangle, minimum width=1cm, minimum height=0.6cm, right =0.2 of decision] (trace) {$\proplabel_0, \ldots, \proplabel_t$};
			\node [above=0.0 of trace] (trace_label) {\footnotesize Trace};
			\draw[style=arrow] (trace) -- (decision);
			\draw (agent.south) edge[arrow, bend right=17] node[fill=white] {$\proplabel_t$} (trace.south);
			
			\node [fill=white, draw=black, shape=rectangle, minimum width=1cm, minimum height=0.6cm, left =0.2 of task_instance] (curriculum) {Curriculum};
			\draw [style=arrow] (curriculum) -- (task_instance);
			\node [cascaded, fill=white, draw=black, shape=rectangle, minimum width=1cm, minimum height=0.6cm, above =0.2 of curriculum] (tasks) {Tasks};
			\node [cascaded, fill=white, draw=black, shape=rectangle, minimum width=1cm, minimum height=0.6cm, below =0.2 of curriculum] (instances) {Instances};
			\draw (tasks) -- (curriculum) -- ([yshift=4pt]instances.north);
			
			\draw (agent.south) edge[arrow, bend left=20] node[fill=white] {$R_{t+1}$} (curriculum.south east);
		\end{tikzpicture}
	}
	\caption{Overview of the interleaving algorithm. Given a set of tasks and a set of instances, the curriculum selects a task-instance pair at the start of an episode, and the HRM for the chosen task is taken from the bank of HRMs. At each step, the agent observes a tuple $\obstuple_t$ and a label $\proplabel_t$ from the task-instance environment, and performs an action $\mdpaction_{t+1}$. The label is used to (i)~determine the next hierarchy state $\hrmstate_{t+1}$ and the reward $\mdprewfunc_{t+1}$, and (ii)~update the trace $\langle \proplabel_0, \ldots, \proplabel_t\rangle$. If the trace is a counterexample, it is added to the task's counterexample set and ILASP learns a new HRM (perhaps using previously learned RMs). The learned HRM replaces the old one in the bank of HRMs. If no counterexample is observed during the episode, the curriculum is updated using the undiscounted return $R_{t+1}$. Further details are described in the main text.
	}
	\label{fig:lhrm_general}
\end{figure*}
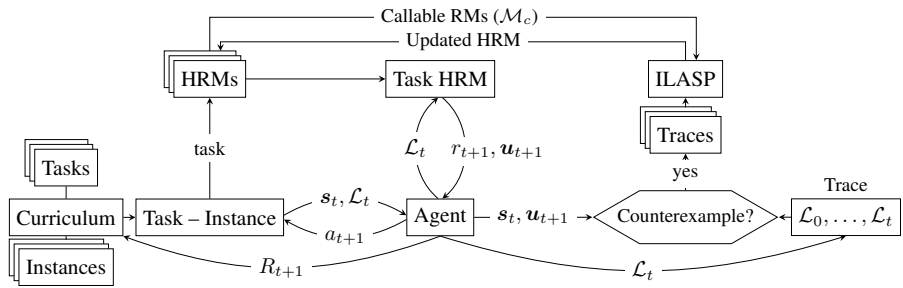

\section{Learning HRMs from Traces}
\label{sec:hierarchy_learning}
In the previous section, we explained how a \emph{given} HRM can be exploited using options; however, engineering an HRM is impractical. We here describe \learningmethod, a method that \emph{interleaves} policy learning with HRM learning from interaction. We consider a \emph{multi-task} setting. Given $\numtasks$ tasks and $\numinstances$ instances (e.g., grids) of an environment, the agent learns (i)~an HRM for each task using traces from several instances for better accuracy, and (ii)~general policies to reach the goal in each task-instance pair. Namely, the agent interacts with $\numtasks \times \numinstances$ MDPs $\mdp_{ij}$, where $i\in[1,\numtasks]$ and $j\in[1,\numinstances]$. The learning proceeds from simpler to harder tasks such that HRMs for the latter build on the former.

In what follows, we detail \learningmethod's components. We \emph{assume} that (i)~all MDPs share propositions $\propset$ and actions $\mdpactions$, and those defined on a given instance share states $\mdpstates$ and labeling function $\lfunc$; (ii)~to stabilize policy learning, dead-end traces are shared across tasks;\footnote{The term $\qfunc_{\context\land\Context}(\obs_{t+1},\ldots)$ in Equation~\ref{eq:q_learning} is 0 if $(\obsterm_{t+1}, \obsgoal_{t+1})=(\top, \bot)$. Since experiences $(\obstuple_t, a, \obstuple_{t+1}, \proplabel_{t+1})$ are shared through the buffer, evaluating the condition differently causes instabilities.} (iii)~the root's height of a task's HRM (or \emph{task level}, for brevity) is known (see Table~\ref{tab:craftworld_tasks} for \cw); and (iv)~without loss of generality, each RM has a single accepting state and a single rejecting state.

\textbf{Curriculum Learning \citep{BengioLCW09}.} 
\learningmethod learns the tasks' HRMs from lower to higher levels akin to  \citet{PierrotLRS0LKBF19}. Before starting an episode, \learningmethod selects an MDP $\mdp_{ij}$, where $i\in[1,\numtasks]$ and $j\in[1,\numinstances]$. The probability of selecting an MDP $\mdp_{ij}$ is determined by an estimate of its average undiscounted return $\avgreturn_{ij}$ such that lower returns are mapped into higher probabilities (see details in Appendix~\ref{app:curriculum_learning}). Initially, only level 1 MDPs can be chosen. When the minimum average return across MDPs up to the current level surpasses a given threshold, the current level increases by 1, hence ensuring the learned HRMs and their associated policies are reusable in higher level tasks.

\textbf{Learning an HRM.} The learning of an HRM is analogous to the learning of a flat RM ~\citep{IcarteWKVCM19,XuGAMNTW20,FurelosBlancoLJBR21,HasanbeigJAMK21}. The objective is to learn the state transition function $\rmtransition_r$ of the root $\rmname_r$ with height $\rmheight_r$ given (i)~a set of states $\rmstates_r$, (ii)~a set of label traces $\traceset=\traceset^G \cup \traceset^D \cup \traceset^I$, (iii)~a set of propositions $\propset$, (iv)~a set of RMs $\machineset$ with lower heights than $\rmheight_r$, (v)~a set of callable RMs $\machinesetcall \subseteq \machineset$ (by default, $\machinesetcall = \machineset$), and (vi)~the maximum number of disjuncts $\maxdisjuncts$ in the DNF formulas labeling the edges. The learned state transition function $\rmtransition_r$ is such that the resulting HRM $\hrm=\langle \machineset \cup \{\rmname_r\},\rmname_r,\propset \rangle$ accepts all goal traces $\traceset^G$, rejects all dead-end traces $\traceset^D$, and neither accepts or rejects incomplete traces $\traceset^I$. The transition functions can be represented as sets of logic rules, which are learned using the ILASP~\citep{ILASP_system} inductive logic programming system (see Appendix~\ref{app:learning_hrms_ilasp} for details on the ILASP encoding).

\begin{figure*}[h]
	\begin{minipage}[t]{0.3\linewidth}
		\vskip 0pt
		\includegraphics[width=\linewidth]{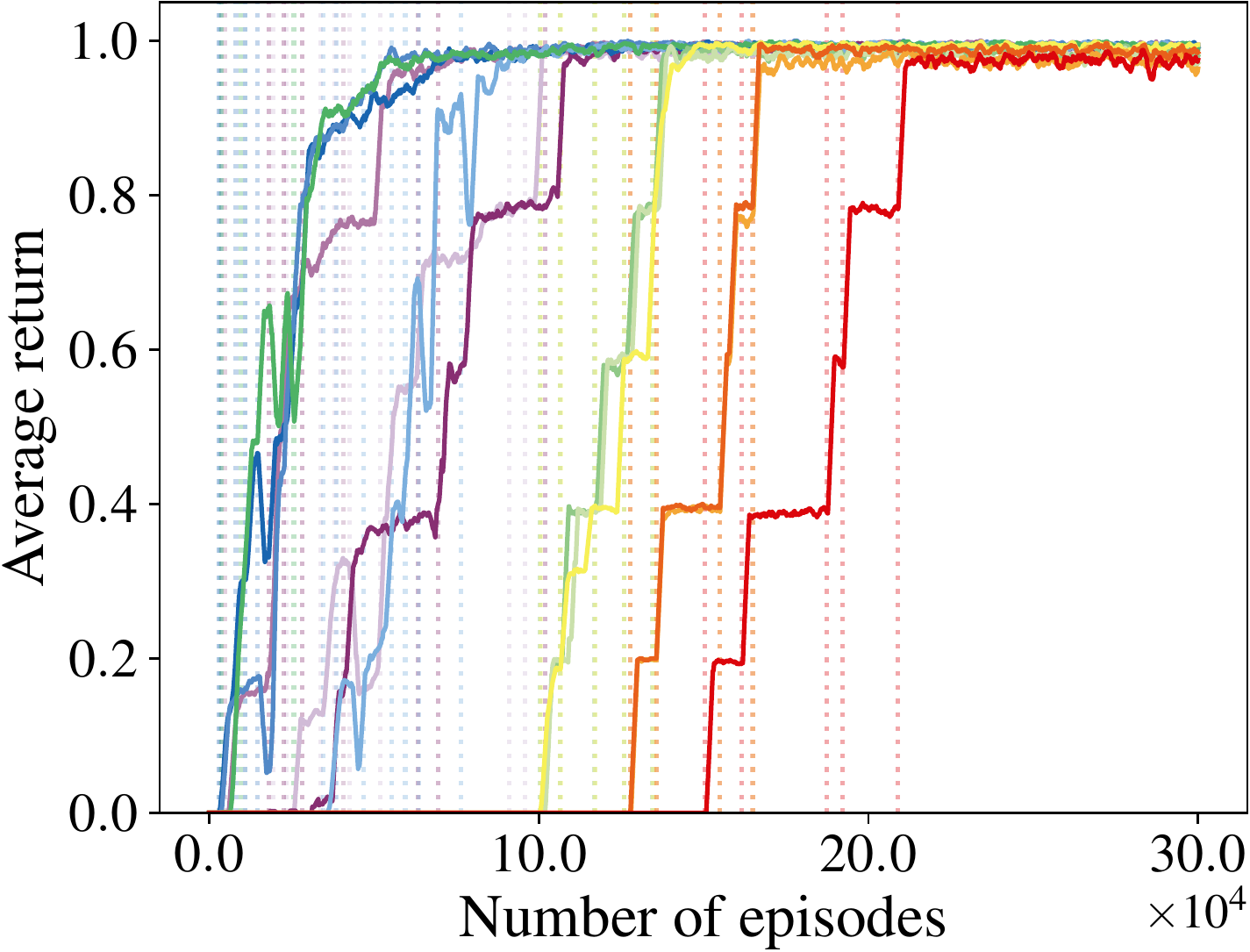}
	\end{minipage}
	\begin{minipage}[t]{0.1059375\linewidth}
		\vskip 0pt
		\resizebox{\linewidth}{!}{
			\begin{tikzpicture}
				\begin{customlegend}[legend columns=1, legend style={column sep=1ex},legend cell align={left}, legend entries={
						\textsc{Batter}, 
						\textsc{Bucket}, 
						\textsc{Compass},
						\textsc{Leather}, 
						\textsc{Paper}, 
						\textsc{Quill}, 
						\textsc{Sugar},  
						------------,
						\textsc{Book}, 
						\textsc{Map}, 
						\textsc{MilkBucket}, 
						------------,
						\textsc{BookQuill}, 
						\textsc{MilkB.Sugar}, 
						------------,
						\textsc{Cake}
				}]
					\addlegendimage{batter,line width=2pt}
					\addlegendimage{bucket,line width=2pt}
					\addlegendimage{compass,line width=2pt}
					\addlegendimage{leather,line width=2pt}
					\addlegendimage{paper,line width=2pt}
					\addlegendimage{quill,line width=2pt}
					\addlegendimage{sugar,line width=2pt}
					\addlegendimage{white,line width=2pt}
					\addlegendimage{book,line width=2pt}
					\addlegendimage{map,line width=2pt}
					\addlegendimage{milkbucket,line width=2pt}
					\addlegendimage{white,line width=2pt}
					\addlegendimage{bookquill,line width=2pt}
					\addlegendimage{milkbucketandsugar,line width=2pt}
					\addlegendimage{white,line width=2pt}
					\addlegendimage{cake,line width=2pt}
				\end{customlegend}
		\end{tikzpicture}}
	\end{minipage}
	\hfill
	\begin{minipage}[t]{0.3\linewidth}
		\vskip 0pt
		\includegraphics[width=\linewidth]{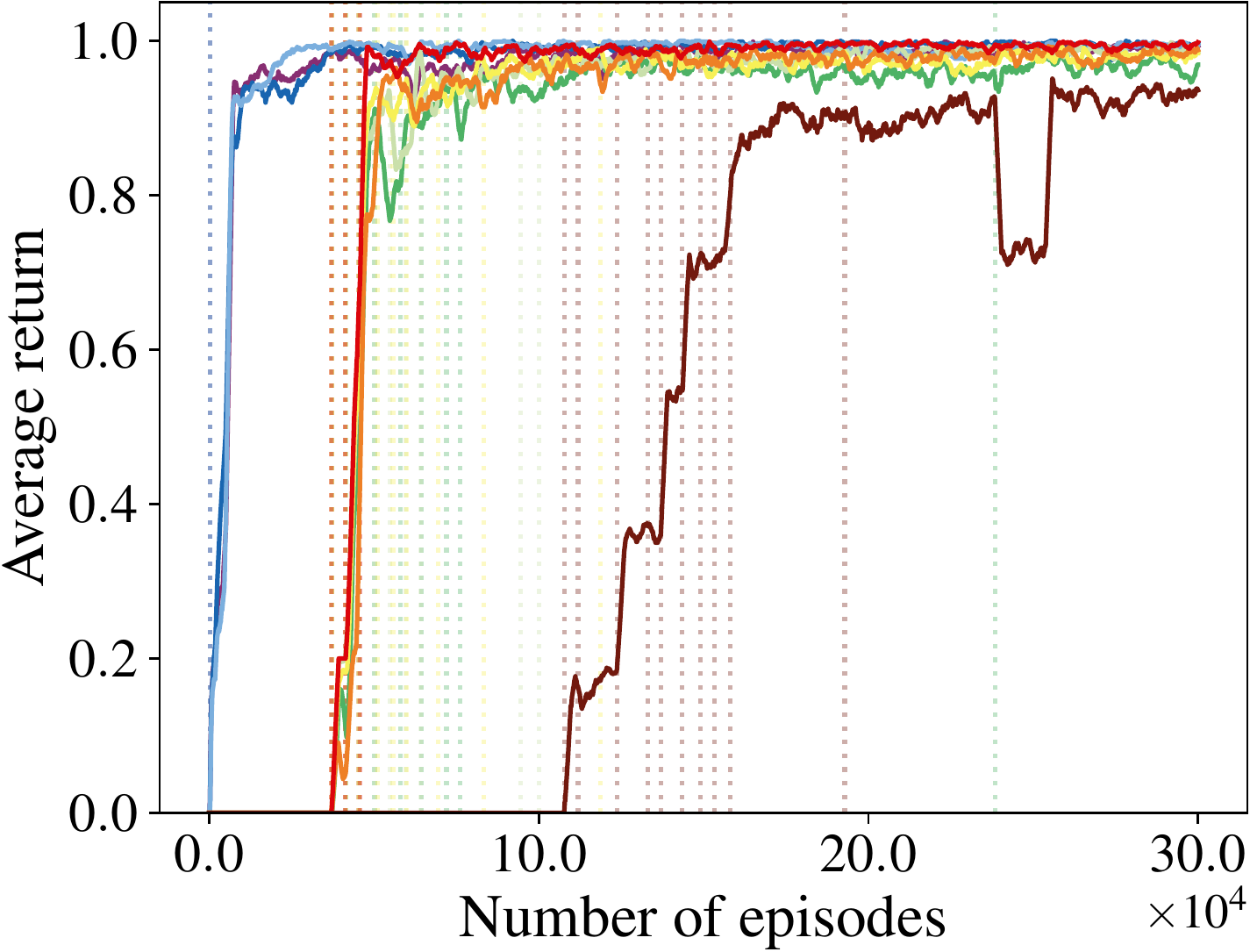}
	\end{minipage}
	\begin{minipage}[t]{0.162\linewidth}
		\vskip 0pt
		\resizebox{\linewidth}{!}{
			\begin{tikzpicture}
				\begin{customlegend}[legend columns=1, legend style={column sep=1ex},legend cell align={left}, legend entries={
						\textsc{rg} [$r$ ; $g$], 
						\textsc{bc} [$b$ ; $c$], 
						\textsc{my} [$m$ ; $y$], 
						--------------------------,
						\textsc{rg\&bc} [\textsc{rg} \& \textsc{bc}], 
						\textsc{bc\&my} [\textsc{bc} \& \textsc{my}], 
						\textsc{rg\&my} [\textsc{rg} \& \textsc{my}],
						\textsc{rgb} [\textsc{rg} ; $b$],
						\textsc{cmy} [$c$ ; \textsc{my}],
						--------------------------,
						\textsc{rgb\&cmy} [\textsc{rgb} \& \textsc{cmy}]
					}]
					\addlegendimage{rg,line width=2pt}
					\addlegendimage{bc,line width=2pt}
					\addlegendimage{my,line width=2pt}
					\addlegendimage{white,line width=2pt}
					\addlegendimage{rgbc,line width=2pt}
					\addlegendimage{bcmy,line width=2pt}
					\addlegendimage{rgmy,line width=2pt}
					\addlegendimage{rgb,line width=2pt}
					\addlegendimage{cmy,line width=2pt}
					\addlegendimage{white,line width=2pt}
					\addlegendimage{rgbcmy,line width=2pt}
				\end{customlegend}
		\end{tikzpicture}}
	\end{minipage}
	\caption{\learningmethod learning curves for \cw (FRL) and \ww (WD). The legends separate tasks by level. The \ww legend describes the subtask order in brackets following the specification introduced in Table~\ref{tab:craftworld_tasks}. The dotted vertical lines correspond to episodes in which an HRM is learned.}
	\label{fig:learning_hrm_and_policy}
\end{figure*}

\textbf{Interleaving  Algorithm.} \learningmethod \emph{interleaves} the induction of HRMs with policy learning akin to \citet{FurelosBlancoLJBR21}. Figure~\ref{fig:lhrm_general} illustrates the core blocks of the algorithm. Initially, the HRM's root of each task $i \in [1, \numtasks]$ consists of 3 states (the initial, accepting, and rejecting states) and neither accepts nor rejects anything. A new HRM is learned when an episode's label trace is not correctly recognized by the current HRM (i.e., if a goal trace is not accepted, a dead-end trace is not rejected, or an incomplete trace is accepted or rejected). The number of states in $\rmstates_r$ increases by 1 when an HRM that covers the examples cannot be learned, hence guaranteeing that the root has the smallest possible number of states (i.e.,~it is \emph{minimal}) for a specific value of $\maxdisjuncts$. When an HRM for task $i$ is learned, the returns $\avgreturn_{ij}$ in the curriculum are set to 0 for all $j\in[1,\numinstances]$. Analogously to some  RM learning methods~\citep{IcarteWKVCM19,XuGAMNTW20,HasanbeigJAMK21}, the first HRM for a task is learned using a set of traces; in our case, the $\numshortestgoaltracesexp$ shortest traces from a set of $\numgoaltracesexp$ goal traces are used (empirically, short traces speed up learning). \learningmethod leverages learned options to \emph{explore} the environment during the goal trace collection, accelerating the process when labels are sparse; specifically, options from lower height RMs are sequentially selected uniformly at random, and their greedy policy is run until termination. We describe other details in Appendix~\ref{app:interleaving_algorithm}.

\section{Experimental Results}
\label{sec:experimental_results}
	We evaluate the policy and HRM learning components in two \emph{domains}: \cw and \ww. We consider four grid types for \cw (see Section~\ref{sec:background}): an open plan $7\times 7$ grid (OP, Figure~\ref{fig:craftworld_grid}), an open plan $7\times 7$ grid with a lava location (OPL), a $13\times 13$ four rooms grid \citep[FR;][]{SuttonPS99}, and a $13\times 13$ four rooms grid with a lava location per room (FRL). The lava proposition must be avoided. \ww~\citep{Karpathy15,Sidor16,IcarteKVM18} consists of a 2D box containing 12 balls of 6 different colors (2 per color) each moving at a constant speed in a fixed direction. The agent ball can change its velocity in any cardinal direction. The propositions $\propset=\{r,g,b,c,m,y\}$ are the balls' colors. Labels consist of the color of the balls the agent overlaps with and, unlike \cw, they may contain multiple propositions. The tasks consist in observing color sequences. We consider two settings: without dead-ends (WOD) and with dead-ends (WD). In WD, the agent must avoid 2 balls of an extra color. Further details are described in Appendix~\ref{app:experimental_details_domains}.

We report the average performance across 5 runs, each using a different set of 10 random instances. The learning curves show the average undiscounted return obtained by the greedy policy every 100 episodes across instances. For other metrics (e.g., learning times), we present the average and the standard error (the latter in brackets). In HRM learning experiments, we set a 2-hour limit to learn the HRMs. The code is available at \url{https://github.com/ertsiger/hrm-learning}.

\begin{figure*}[h]
	\centering
	\begin{subfigure}[b]{0.29\linewidth}
		\centering
		\includegraphics[width=\linewidth]{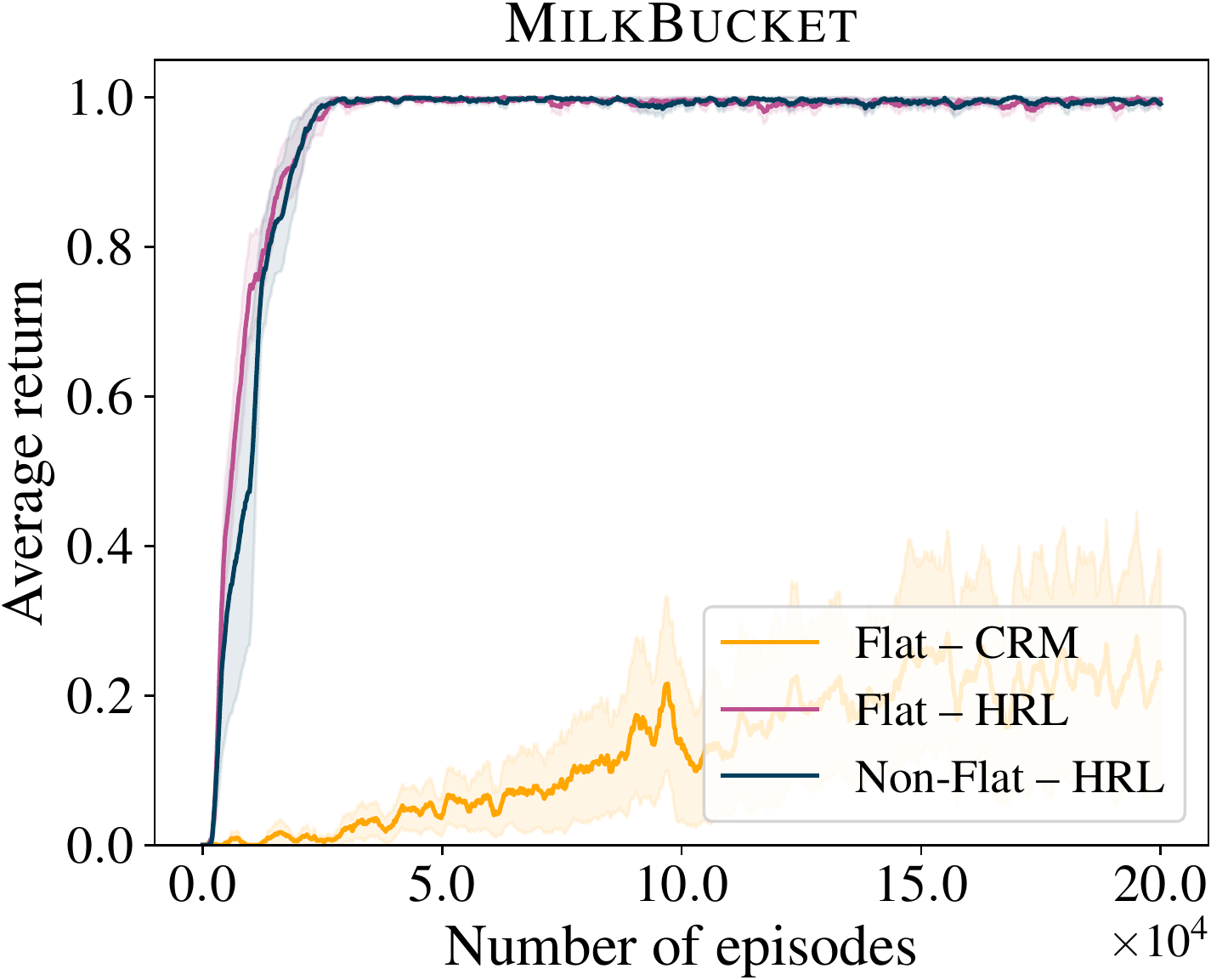}
	\end{subfigure}
	\hfill
	\begin{subfigure}[b]{0.29\linewidth}
		\centering
		\includegraphics[width=\linewidth]{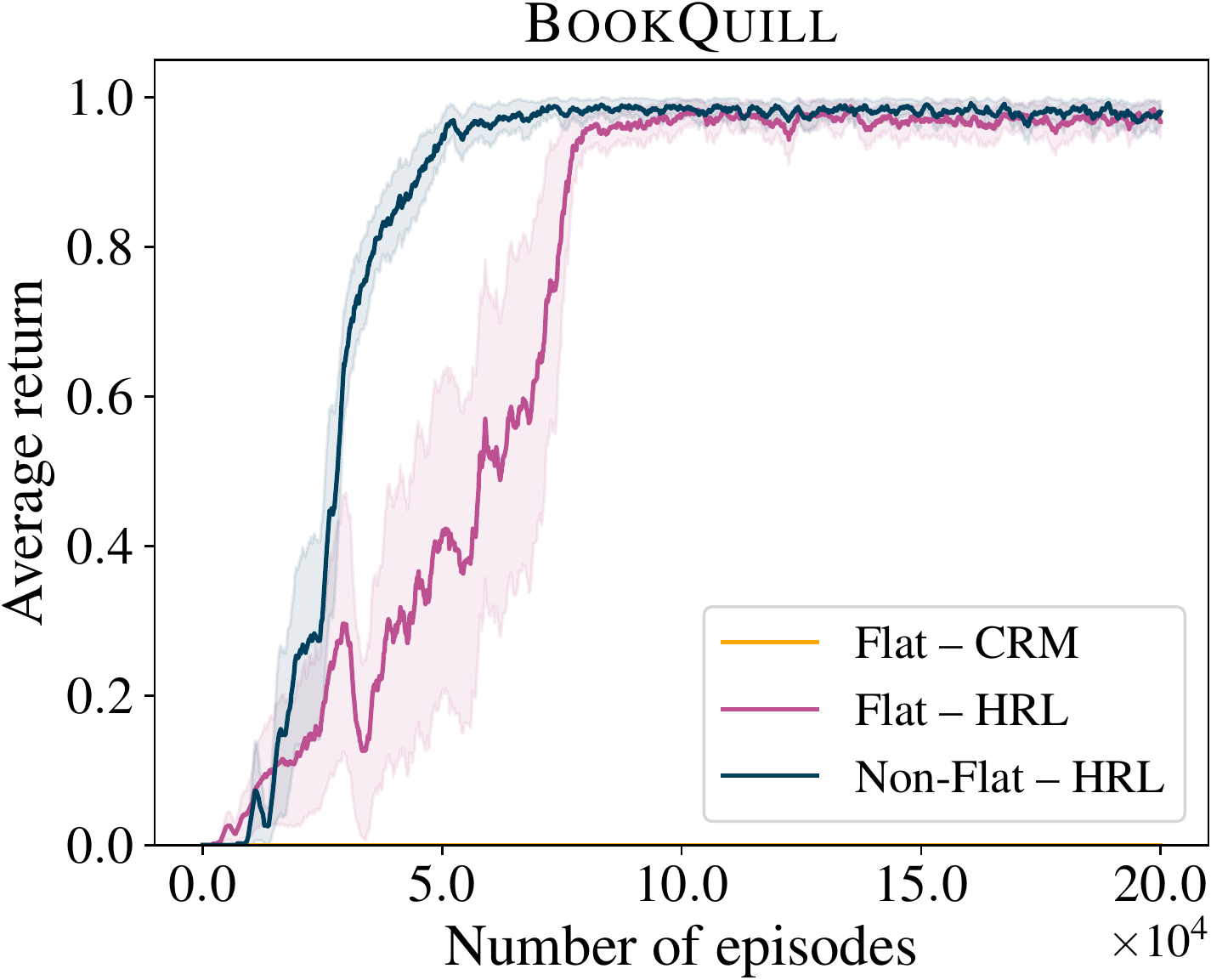}
	\end{subfigure}	
	\hfill
	\begin{subfigure}[b]{0.29\linewidth}
		\centering
		\includegraphics[width=\linewidth]{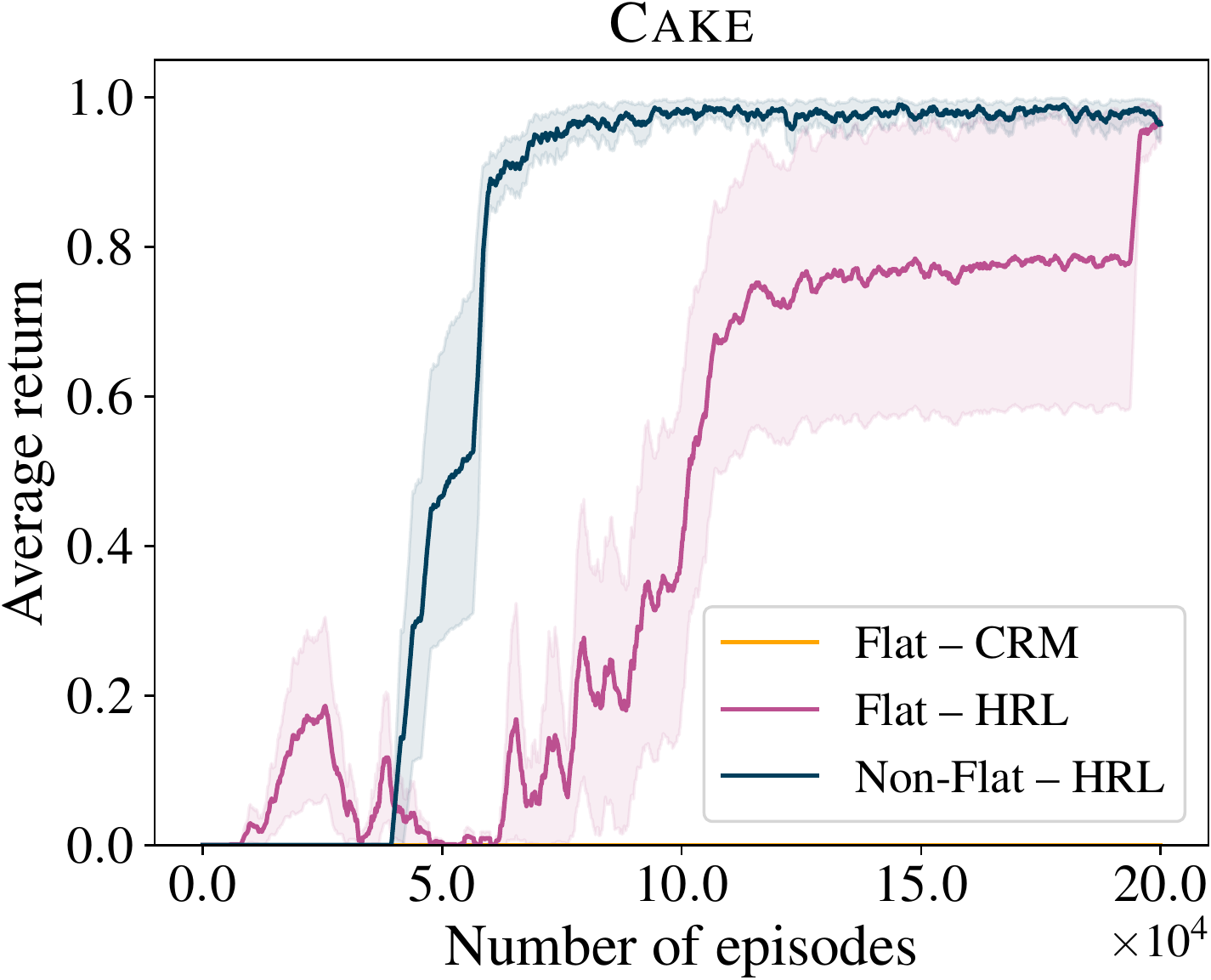}
	\end{subfigure}
	\caption{Learning curves for three \cw (FRL) tasks using handcrafted HRMs.}
	\label{fig:flat_vs_hrm_policy}
\end{figure*}

\subsection{Learning of Non-Flat HRMs}
\label{sec:learning_nonflat_hrms}
Figure~\ref{fig:learning_hrm_and_policy} shows the \learningmethod learning curves for \cw (FRL) and \ww (WD). These settings are the most challenging due to the inclusion of dead-ends since (i)~they hinder the observation of goal examples in level~1 tasks using random walks, (ii)~the RMs must include rejecting states, (iii)~formula options must avoid dead-ends, and (iv)~call options must avoid invoking options leading to rejecting states. In line with the curriculum method, \learningmethod does not start learning a level $\rmheight$ task until performance in tasks from levels $1,\ldots,\rmheight-1$ is sufficiently good. The convergence for high-level tasks is often fast due to the reuse of lower level HRMs and policies.

The average time (in seconds) exclusively spent on learning \emph{all} HRMs is 1009.8~(122.3) for OP, 1622.6~(328.7) for OPL, 1031.6~(150.3) for FR, 1476.8~(175.3) for FRL, 35.4~(2.0) for WOD, and 67.0~(6.2) for WD (see Tables~\ref{tab:non_flat_hrm_default_cw} and \ref{tab:non_flat_hrm_default_ww} in Appendix~\ref{app:extended_results_non_flat_hrm_learning}). Dead-ends (OPL, FRL, WD) incur longer times since (i)~there is one more proposition, (ii)~there are edges to the rejecting state(s), and (iii)~there are dead-end traces to cover. We observe that the complexity of learning an HRM does not necessarily correspond with the task complexity (e.g.,~the times for OP and FRL are close). Learning in \ww is faster than in \cw since the RMs have fewer states and there are fewer callable RMs.

\textbf{Ablations.} By \emph{restricting the callable RMs} to those required by the HRM (e.g., \emph{just} using \textsc{Paper} and \textsc{Leather} RMs to learn \textsc{Book}'s), there are fewer ways to label the edges of the induced RM. Learning is {5-7\texttimes} faster using 20\% fewer calls to the learner (i.e., fewer examples) in \cw, and {1.5\texttimes} faster in \ww (see Tables~\ref{tab:non_flat_hrm_restricted_cw} and \ref{tab:non_flat_hrm_restricted_ww} in Appendix~\ref{app:extended_results_non_flat_hrm_learning}); thus, HRM learning becomes less scalable as the number of tasks and levels grows. This is an instance of the \emph{utility} problem \citep{Minton88}. Refining the callable RM set prior to HRM learning is an avenue for future work.

We evaluate the performance of \emph{exploration with options} using the number of episodes needed to collect the $\numgoaltracesexp$ goal traces for a given task since the activation of its level. Intuitively, the agent rarely moves far from a region of the state space using primitive actions only, thus taking longer to collect the traces; in contrast, options enable the agent to explore the state space efficiently. In \cw's FRL setting, using primitive actions requires {128.1\texttimes} more episodes than options in \textsc{MilkBucket}, the only level~2 task for which $\numgoaltracesexp$ traces are collected. Likewise, primitive actions take {53.1\texttimes} and {10.1\texttimes} more episodes in OPL and WD respectively. In OP and WOD, options are not as beneficial since episodes are relatively long (1000 steps), there are no dead-ends and it is easy to observe the different propositions. See Tables~\ref{tab:non_flat_hrm_action_explore_cw} and \ref{tab:non_flat_hrm_action_explore_ww} in Appendix~\ref{app:extended_results_non_flat_hrm_learning} for detailed results.

Learning the first HRMs using a \emph{single goal trace} ($\numgoaltracesexp=\numshortestgoaltracesexp=1$) incurs timeouts in all \cw settings, thus showing the value of using many short traces instead.

\subsection{Learning of Flat HRMs}
\label{sec:learning_flat_hrms}
Learning a flat HRM is often less scalable than learning a non-flat equivalent since (i)~already learned HRMs cannot be reused, and (ii)~a flat HRM usually has more states and edges (as shown in Theorem~\ref{thm:flat_size}, growth can be exponential). We compare the performance of learning (from interaction) a non-flat HRM using \learningmethod with that of an equivalent flat HRM using \learningmethod, DeepSynth~\citep{HasanbeigJAMK21}, JIRP~\citep{XuGAMNTW20} and LRM~\citep{IcarteWKVCM19}.  \learningmethod and JIRP induce RMs with explicit accepting states, while DeepSynth and LRM do not. We use OP and WOD instances in \cw and \ww respectively.

A non-flat HRM for \textsc{MilkBucket} (level~2) is learned in 1.5~(0.2) seconds, whereas flat HRMs take longer: 3.2~(0.6) w/\learningmethod, 325.6~(29.7) w/DeepSynth, 17.1~(5.5) w/JIRP and 347.5~(64.5) w/LRM. \learningmethod and JIRP learn minimal RMs, hence producing the same RM consisting of 4 states and 3 edges. DeepSynth and LRM do not learn a minimal RM but one that is good at predicting the next possible label given the current one. In domains like ours where propositions can be observed anytime (i.e., without temporal dependencies between them), these methods tend to `overfit' the input traces and output large RMs that barely reflect the task's structure, e.g. DeepSynth learns RMs with 13.4~(0.4) states and 93.2~(1.7) edges. In contrast, methods learning minimal RMs only from observable traces may suffer from \emph{overgeneralization}~\citep{Angluin80} in other domains (e.g.,~with temporally-dependent propositions). These observations apply to more complex tasks (i.e., involving more high-level temporal steps and multiple paths to the goal), such as \textsc{Book} (level~2), \textsc{BookQuill} (level~3) and \textsc{Cake} (level~4). \learningmethod learns non-flat HRMs (e.g., see Figure~\ref{fig:book_hierarchy}) for these tasks in (at most) a few minutes, while learning an informative flat HRM (e.g., see Figure~\ref{fig:book_hierarchy_flat}) is unfeasible. We refer the reader to Table~\ref{tab:flat_hrms_table} in Appendix~\ref{app:extended_results_flat_hrm_learning} for details.

DeepSynth, JIRP and LRM perform poorly in \ww. Unlike \learningmethod, these learn RMs whose edges are not labeled by formulas but proposition sets; hence, the RMs may have exponentially more edges (e.g., 64  instead of 2 for \textsc{rg}), and become unfeasible to learn. Indeed, flat HRM learners time out in \textsc{rg\&bc} and \textsc{rgb\&cmy}, while \learningmethod needs a few seconds (see Table~\ref{tab:flat_hrms_table} in Appendix~\ref{app:extended_results_flat_hrm_learning}).

\subsection{Policy Learning in Handcrafted HRMs}
\label{sec:experimental_results_policy_learning_handcrafted}
We compare the performance of policy learning in handcrafted non-flat HRMs against in flat equivalents, which are guaranteed to exist by Theorem~\ref{thm:flat_equivalent}. For fairness, the flat HRMs are minimal. To exploit the flat HRMs, we apply our HRL algorithm (Section~\ref{sec:policy_learning}) and CRM~\citep{IcarteKVM22}, which learns a Q-function over $\mdpstates \times \rmstates$ using synthetic counterfactual experiences for each RM state. Figure~\ref{fig:flat_vs_hrm_policy} shows the learning curves for some \cw tasks in the FRL setting. The convergence rate is similar in the simplest task (\textsc{MilkBucket}), but higher for non-flat HRMs in the hardest ones. Unlike the HRL approaches, CRM does not decompose the subtask into independently solvable subtasks and, hence, deals with sparser rewards that result in a slower convergence. In the case of the HRL approaches, since both use the same set of formula option policies, differences arise from flat HRMs' lack of modularity. Call options, which are not present in flat HRMs, form independent modules that reduce reward sparsity. \textsc{MilkBucket} involves fewer high-level steps than \textsc{BookQuill} and \textsc{Cake}, thus reward is less sparse and non-flat HRMs are not as beneficial. The efficacy of non-flat HRMs is also limited when (i)~the task's goal is reachable regardless of the chosen options (e.g., if there are no rejecting states, like in OP and FR), and (ii)~the reward is not sparse, like in OPL (the grid is small) or \ww (the balls easily get near the agent). See Appendix~\ref{app:extended_results_policy_learning} for additional results.

\section{Related Work}
\label{sec:related_work}

\textbf{RMs and Composability.} 
Our RMs differ from the original~\citep{IcarteKVM18,IcarteKVM22} in that (i)~an RM can call other RMs, (ii)~there are explicit accepting and rejecting states~\citep{XuGAMNTW20,FurelosBlancoLJBR21}, and (iii)~transitions are labeled with propositional logic formulas instead of proposition sets~\citep{FurelosBlancoLJBR21}. Recent works \emph{derive} RMs (or similar FSMs) from formal language specifications~\citep{CamachoIKVM19,Araki0VDFR21} and expert demonstrations~\citep{CamachoVZJIK21}, or \emph{learn} them from experience using discrete optimization~\citep{IcarteWKVCM19,ChristoffersenLTM20}, SAT solving~\citep{XuGAMNTW20,CorazzaGN22}, active learning~\citep{GaonB20,XuWONT21,DohmenTAB0V22}, state-merging~\citep{XuGAMNTW19,GaonB20}, program synthesis~\citep{HasanbeigJAMK21} or inductive logic programming~\citep{FurelosBlancoLJBR21,ArdonFR23}. A prior way of composing RMs consists in merging the state and reward transition functions~\citep{DeGiacomo20}. Other works have considered settings where the labeling function is noisy~\citep{LiCVKTM22,VerginisKCT22}, the RM transitions and/or rewards are stochastic~\citep{CorazzaGN22,DohmenTAB0V22} or defined over predicates~\citep{ZhouL22}, and multiple agents interact with the world~\citep{Neary00T21,Dann0A0T22,ArdonFR23}. High-probability regret bounds have been derived for RMs~\citep{BourelJMT23}.

Alternative methods for modeling task composability include subtask sequences~\citep{AndreasKL17}, context-free grammars~\citep{Chevalier-Boisvert19}, formal languages~\citep{JothimuruganAB19,IllanesYIM20,LeonSB20,WangTLHL20} and logic-based algebras~\citep{TasseJR20}.

\textbf{Hierarchical RL.} 
Our method for exploiting HRMs resembles a hierarchy of DQNs \citep{KulkarniNST16}. Akin to option discovery methods, \learningmethod induces a set of options from experience. While \learningmethod's options are a byproduct of finding an HRM that compactly captures label traces, usual option discovery methods explicitly look for them (e.g., options that reach novel states). \learningmethod requires a set of propositions and tasks, which bound the number of discoverable options; similarly, some of these methods impose an explicit bound~\citep{BaconHP17, MachadoBB17}. \learningmethod requires each task to be solved at least once before learning an HRM (and, hence, options), just like other methods~\citep{McGovernB01,StolleP02}. The problem of discovering options for exploration has been considered before \citep{BellemareSOSSM16,MachadoBB17,JinnaiPAK19,DabneyOB21}. While our options are not discovered for exploration, we leverage them to find goal traces in new tasks. \citet{LevyKPS19} learn policies from multiple hierarchical levels in parallel by training each level as if the lower levels were optimal; likewise, we train call option policies from experiences where invoked options achieve their goal.

HRMs are close to hierarchical abstract machines \citep[HAMs;][]{ParrR97} since both are hierarchies of FSMs, but there are two core differences. First, HAMs do not have reward transition functions. Second, (H)RMs decouple the traversal from the policies, i.e.~independently of the agent's choices, the (H)RM is followed; thus, an agent using an (H)RM must be able to interrupt its choices (see Section~\ref{sec:policy_learning}). While HAMs do not support interruption, Programmable HAMs~\citep{AndreR00} extend them to support it along with other program-like features. Despite the similarity, there are few works on learning HAMs~\citep{LeonettiIP12} and many on learning RMs, as outlined before.

\textbf{Curriculum Learning.} \citet{PierrotLRS0LKBF19} learn hierarchies of neural programs given the level of each program, akin to our RMs' height; likewise, \citet{AndreasKL17} prioritize tasks consisting of fewer high-level steps. The `online' method by \citet{MatiisenOCS20} also keeps an estimate of each task's average return, but it is not applied in an HRL scenario. \citet{WangTLHL20} learn increasingly complex temporal logic formulas leveraging previously learned formulas using a set of templates.

\section{Conclusions and Future Work}
\label{sec:conclusions}
We have here proposed (1)~HRMs, a \emph{formalism} that composes RMs in a hierarchy by enabling them to call each other, (2)~an HRL method that \emph{exploits} the structure of an HRM, and (3)~a curriculum-based method for \emph{learning} a set of HRMs from traces. Non-flat HRMs have significant advantages over their flat equivalents. Theoretically, a flat equivalent of a given HRM can have exponentially more states and edges. Empirically, (i)~our HRL method converges faster given a non-flat HRM instead of a flat equivalent one, and (ii)~in line with the theory, learning an HRM is feasible in cases where a flat equivalent is not.

\learningmethod \emph{assumes} the proposition set is known, shared dead-end indicators across tasks, and a fixed set of tasks. Relaxing these assumptions by forming the propositions from raw data, conditioning policies to dead-ends, and letting the agent propose its own composable tasks are promising directions for future work. Other  directions include \emph{non-episodic} settings and learning \emph{globally optimal} policies over HRMs.

\section*{Acknowledgements}
We thank the reviewers, as well as Hadeel Al-Negheimish, Nuri Cingillioglu, and Alex F. Spies for their comments. Anders Jonsson is partially funded by TAILOR, AGAUR SGR and Spanish grant PID2019-108141GB-I00.

\bibliography{icml2023}
\bibliographystyle{icml2023}

\newpage
\appendix
\onecolumn
\section{Formalism Details}
In this appendix, we extend Example~\ref{example:hierarchy_traversal_short_example} by showing all the intermediate steps (Appendix~\ref{app:hierarchy_traversal_full_example}), and provide the proofs for Theorems~\ref{thm:flat_equivalent} and \ref{thm:flat_size} (Appendix~\ref{app:equivalente_theory_results}).

\subsection{Hierarchy Traversal Example}
\label{app:hierarchy_traversal_full_example}
The HRM in Figure~\ref{fig:book_hierarchy} accepts trace $\trace=\allowbreak\langle \{\Sugarcane\},\allowbreak \{\Workbench\},\allowbreak \{\},\allowbreak\{\Rabbit\},\allowbreak \{\Workbench\},\allowbreak \{\Table\} \rangle$, whose traversal is $\hrm(\trace)=\langle v_0,\allowbreak v_1,\allowbreak v_2,\allowbreak v_3,\allowbreak v_4,\allowbreak v_5,\allowbreak v_6\rangle$, where:
\begingroup
\allowdisplaybreaks
\begin{align*}
	v_0 &= \langle\rmname_0, u^0_0, \top, [] \rangle, \\
	v_1 &= \hrmtrans(v_0, \{\Sugarcane\})\\ &=\hrmtrans(\langle\rmname_0, u^0_0, \top, [] \rangle, \{\Sugarcane\})\\
	&= \hrmtrans(\langle\rmname_1, u^0_1, \neg\Rabbit, [\langle u^0_0, u^1_0, \rmname_0, \rmname_1, \neg\Rabbit, \top \rangle] \rangle, \{\Sugarcane\})\\
	&= \hrmtrans(\langle\leaf, u^0_\top, \neg\Rabbit \land \Sugarcane, [\langle u^0_0, u^1_0, \rmname_0, \rmname_1, \neg\Rabbit, \top \rangle, \langle u^0_1, u^1_1, \rmname_1, \leaf, \Sugarcane, \neg\Rabbit \rangle] \rangle, \{\Sugarcane\})\\
	&= \hrmtrans(\langle\rmname_1, u^1_1, \top, [\langle u^0_0, u^1_0, \rmname_0, \rmname_1, \neg\Rabbit, \top \rangle] \rangle, \bot)\\
	&= \langle\rmname_1, u^1_1, \top, [\langle u^0_0, u^1_0, \rmname_0, \rmname_1, \neg\Rabbit, \top \rangle] \rangle,\\
	v_2 &= \hrmtrans(v_1, \{\Workbench\})\\
	&= \hrmtrans(\langle\rmname_1, u^1_1, \top, [\langle u^0_0, u^1_0, \rmname_0, \rmname_1, \neg\Rabbit, \top \rangle] \rangle, \{\Workbench\})\\
	&= \hrmtrans(\langle\leaf, u^0_\top, \Workbench, [\langle u^0_0, u^1_0, \rmname_0, \rmname_1, \neg\Rabbit, \top \rangle, \langle u^1_1, u^A_1, \rmname_1, \rmname_\top, \Workbench,\top \rangle] \rangle, \{\Workbench\})\\
	&= \hrmtrans(\langle\rmname_1, u^A_1, \top, [\langle u^0_0, u^1_0, \rmname_0, \rmname_1, \neg\Rabbit, \top \rangle] \rangle, \bot)\\
	&= \hrmtrans(\langle\rmname_0, u^1_0, \top, [] \rangle, \bot)\\
	&= \langle\rmname_0, u^1_0, \top, [] \rangle,\\
	v_3 &= \hrmtrans(v_2, \{\})\\
	&= \hrmtrans(\langle\rmname_0, u^1_0, \top, [] \rangle, \{\})\\
	&= \langle\rmname_0, u^1_0, \top, [] \rangle,\\
	v_4 &= \hrmtrans(v_3, \{\Rabbit\})\\
	&= \hrmtrans(\langle\rmname_0, u^1_0, \top, [] \rangle, \{\Rabbit\})\\
	&= \hrmtrans(\langle\rmname_2, u^0_2, \top, [\langle u^1_0, u^3_0, \rmname_0, \rmname_2, \top, \top \rangle] \rangle, \{\Rabbit\})\\
	&= \hrmtrans(\langle\leaf, u^0_\top, \Rabbit, [\langle u^1_0, u^3_0, \rmname_0, \rmname_2, \top, \top \rangle, \langle u^0_2, u^1_2, \rmname_2, \leaf, \Rabbit, \top \rangle] \rangle, \{\Rabbit\})\\
	&= \hrmtrans(\langle\rmname_2, u^1_2, \top, [\langle u^1_0, u^3_0, \rmname_0, \rmname_2, \top, \top \rangle] \rangle, \bot)\\
	&= \langle\rmname_2, u^1_2, \top, [\langle u^1_0, u^3_0, \rmname_0, \rmname_2, \top, \top \rangle] \rangle,\\
	v_5 &= \hrmtrans(v_4, \{\Workbench\})\\
	&= \hrmtrans(\langle\rmname_2, u^1_2, \top, [\langle u^1_0, u^3_0, \rmname_0, \rmname_2, \top, \top \rangle] \rangle, \{\Workbench\})\\
	&= \hrmtrans(\langle\leaf, u^0_\top, \Workbench, [\langle u^1_0, u^3_0, \rmname_0, \rmname_2, \top, \top \rangle, \langle u^1_2, u^A_2, \rmname_2, \leaf, \Workbench, \top \rangle]\rangle, \{\Workbench\})\\
	&= \hrmtrans(\langle\rmname_2, u^A_2, \top, [\langle u^1_0, u^3_0, \rmname_0, \rmname_2, \top, \top \rangle] \rangle, \bot)\\
	&= \hrmtrans(\langle\rmname_0, u^3_0, \top, [] \rangle, \bot)\\
	&= \langle\rmname_0, u^3_0, \top, [] \rangle,\\
	v_6 &= \hrmtrans(v_5, \{\Table\})\\
	&= \hrmtrans(\langle\rmname_0, u^3_0, \top, [] \rangle, \{\Table\})\\
	&= \hrmtrans(\langle \leaf, u^0_\top, \Table, [\langle u^3_0, u^A_0, \rmname_0, \leaf, \Table, \top \rangle] \rangle, \{\Table\})\\
	&= \hrmtrans(\langle\rmname_0, u^A_0, \top, [] \rangle, \bot)\\
	&=\langle\rmname_0, u^A_0, \top, [] \rangle.
\end{align*}
\endgroup

\subsection{Equivalence to Flat Hierarchies of Reward Machines}
\label{app:equivalente_theory_results}
In this section, we prove the theorems introduced in Section~\ref{sec:formalism} regarding the equivalence of an arbitrary HRM to a flat HRM.

\subsubsection{Proof of Theorem~\ref{thm:flat_equivalent}}
\label{app:flat_equivalent}
We formally show that any HRM can be transformed into an equivalent one consisting of a single non-leaf RM. The latter HRM type is called \emph{flat} since there is a single hierarchy level.
\begin{definition}
	Given an HRM $\hrmtuple$, a constituent RM $\rmname_i \in \machineset$ is \emph{flat} if its height $h_i$ is 1.
\end{definition}
\begin{definition}
	An HRM $\hrmtuple$ is \emph{flat} if the root RM $\hrmroot$ is flat.
\end{definition}

We now define what it means for two HRMs to be equivalent. This definition is based on that used in automaton theory~\citep{Sipser97}.
\begin{definition}
	Given a set of propositions $\propset$ and a labeling function $\lfunc$, two HRMs $\hrmtuple$ and $\hrmtuplep$ are \emph{equivalent} if for any label trace $\trace$ one of the following conditions holds: (i)~both HRMs accept $\trace$, (ii)~both HRMs reject $\trace$, or (iii)~neither of the HRMs accepts or rejects $\trace$.
	\label{def:hierarchy_equiv}
\end{definition}

We now have all the required definitions to prove Theorem~\ref{thm:flat_equivalent}, which is restated below.
\flatequivalent*

To prove the theorem, we introduce an algorithm for flattening any HRM. Without loss of generality, we work on the case of an HRM with two hierarchy levels; that is, an HRM consisting of a root RM that calls flat RMs. Note that an HRM with an arbitrary number of levels can be flattened by considering the RMs in two levels at a time. We start flattening RMs in the second level (i.e., with height~2), which use RMs in the first level (by definition, these are already flat), and once the second level RMs are flat, we repeat the process with the levels above until the root is reached. This process is applicable since, by assumption, the hierarchies do not have cyclic dependencies (including recursion). For simplicity, we use the MDP reward assumption made in Section~\ref{sec:background}, i.e.~the reward transition function of any RM $\rmname_i$ is $\rmreward_i(u,u')=\mathds{1}[u\notin \rmstatesacc_i \land u' \in \rmstatesacc_i]$ like in Section~\ref{sec:policy_learning}. However, the proof below could be adapted to arbitrary definitions of $\rmreward_i(u,u')$.

\textbf{Preliminary Transformation Algorithm.}~
Before proving Theorem~\ref{thm:flat_equivalent}, we introduce an intermediate step that transforms a flat HRM into an equivalent one that takes contexts with which it may be called into account. Remember that a call to an RM is associated with a context. In the case of two-level HRMs such as the ones we are considering in this flattening process, the context and the exit condition from the called flat RM must be satisfied. Crucially, the context must only be satisfied at the time of the call; that is, it only lasts for a single transition. Therefore, if we revisit the initial state of the called RM by taking an edge to it, the context should not be checked anymore. 

To make the need for this transformation clearer, we use the HRM illustrated in Figure~\ref{fig:transformation_example_a}. The flattening algorithm described later embeds the called RM into the calling one; crucially, the context of the call is taken into account by putting it in conjunction with the outgoing edges from the initial state of the called RM.\footnote{We refer the reader to the `Flattening Algorithm' description introduced later for specific details.} Figure~\ref{fig:transformation_example_b} is a flat HRM obtained using the flattening algorithm; however, it does not behave like the HRM in Figure~\ref{fig:transformation_example_a}. Following the definition of the hierarchical transition function $\hrmtrans$, the context of a call only lasts for a single transition in the called RM in Figure~\ref{fig:transformation_example_a} (i.e.,~$a\land\neg c$ is only checked when $\rmname_{1}$ is started), but the context is kept permanently in Figure~\ref{fig:transformation_example_b}, which is problematic if we go back to the initial state at some point. We later come back to this example after presenting the transformation algorithm.

\begin{figure}[h]
	\begin{subfigure}[b]{0.5\linewidth}
		\centering
		\begin{minipage}{0.29\linewidth}
			\centering
			\resizebox{0.8\linewidth}{!}{
				\begin{tikzpicture}[shorten >=1pt,node distance=2cm,on grid,auto,every initial by arrow/.style ={-Latex} ]
					\node[state,initial, initial text=] (u_0)   {$u^0_0$};
					\node[state,accepting] (u_1) [below =of u_0]   {$u^A_0$};
					\path[-Latex] (u_0) edge node [in place] {$\rmname_1\mid \neg c$} (u_1);
				\end{tikzpicture}
			}
			
			$\rmname_{0}$
		\end{minipage}
		\hfill
		\begin{minipage}{0.69\linewidth}
			\centering
			\resizebox{0.9\linewidth}{!}{
				\begin{tikzpicture}[shorten >=1pt,node distance=2cm,on grid,auto,every initial by arrow/.style ={-Latex} ]
					\node[state,initial, initial text=] (u_0)   {$u^0_1$};
					\node[state] (u_1) [below =of u_0]   {$u^1_1$};
					\node[state,accepting] (u_A) [right =21ex of u_1]   {$u^A_1$};
					\path[-Latex] (u_0) edge[bend right=70] node [in place] {$\leaf \mid a$} (u_1);
					\path[-Latex] (u_1) edge[bend right=70] node [in place] {$\leaf\mid b$} (u_0);
					\path[-Latex] (u_1) edge node [in place] {$\leaf\mid c \land \neg b$} (u_A);
				\end{tikzpicture}
			}
			
			$\rmname_{1}$
		\end{minipage}	
		\caption{Original HRM with root $\rmname_{0}$.}
		\label{fig:transformation_example_a}
	\end{subfigure}
	\begin{subfigure}[b]{0.5\linewidth}
		\centering
		\resizebox{0.7\linewidth}{!}{
			\begin{tikzpicture}[shorten >=1pt,node distance=2cm,on grid,auto,every initial by arrow/.style ={-Latex} ]
				\node[state,initial, initial text=] (u_0)   {$u^0_0$};
				\node[state] (u_1) [below =of u_0]   {$u^1_1$};
				\node[state,accepting] (u_A) [right =21ex of u_1]   {$u^A_0$};
				\path[-Latex] (u_0) edge[bend right=70] node [in place] {$\bm{\leaf \mid a\land \neg c}$} (u_1);
				\path[-Latex] (u_1) edge[bend right=70] node [in place] {$\leaf\mid b$} (u_0);
				\path[-Latex] (u_1) edge node [in place] {$\leaf\mid c \land \neg b$} (u_A);
			\end{tikzpicture}
		}
		\caption{Flattened HRM without transforming $\rmname_1$.}
		\label{fig:transformation_example_b}
	\end{subfigure}
	~\\
	\begin{subfigure}{0.45\linewidth}
		\centering
		\resizebox{0.9\linewidth}{!}{
			\begin{tikzpicture}[shorten >=1pt,node distance=2cm,on grid,auto,every initial by arrow/.style ={-Latex} ]
				\node[state,initial, initial text=] (u_0)   {$u^0_1$};
				\node[state] (u_1) [below =of u_0]   {$u^1_1$};
				\node[state] (u_0d) [left =7em of u_1]   {$\hat{u}^0_1$};
				\node[state,accepting] (u_A) [right =21ex of u_1]   {$u^A_1$};
				\path[-Latex] (u_0) edge node [in place] {$\leaf \mid a$} (u_1);
				\path[-Latex] (u_1) edge[bend right=40] node [in place] {$\leaf\mid b$} (u_0d);
				\path[-Latex] (u_0d) edge[bend right=40] node [in place] {$\leaf\mid a$} (u_1);
				\path[-Latex] (u_1) edge node [in place] {$\leaf\mid c \land \neg b$} (u_A);
		\end{tikzpicture}}
		\caption{Transformed $\rmname_1$ from (a).}
		\label{fig:transformation_example_c}
	\end{subfigure}
	\hfill
	\begin{subfigure}{0.45\linewidth}
		\centering
		\resizebox{0.9\linewidth}{!}{
			\begin{tikzpicture}[shorten >=1pt,node distance=2cm,on grid,auto,every initial by arrow/.style ={-Latex} ]
				\node[state,initial, initial text=] (u_0)   {$u^0_0$};
				\node[state] (u_1) [below =of u_0]   {$u^1_1$};
				\node[state] (u_0d) [left =7em of u_1]   {$\hat{u}^0_1$};
				\node[state,accepting] (u_A) [right =21ex of u_1]   {$u^A_0$};
				\path[-Latex] (u_0) edge node [in place] {$\bm{\leaf \mid a \land \neg c}$} (u_1);
				\path[-Latex] (u_1) edge[bend right=40] node [in place] {$\leaf\mid b$} (u_0d);
				\path[-Latex] (u_0d) edge[bend right=40] node [in place] {$\leaf\mid a$} (u_1);
				\path[-Latex] (u_1) edge node [in place] {$\leaf\mid c \land \neg b$} (u_A);
		\end{tikzpicture}}
		\caption{Flattened HRM after transforming $\rmname_1$.}
		\label{fig:transformation_example_d}
	\end{subfigure}
	\caption{Example to justify the need for the preliminary transformation algorithm.}
	\label{fig:transformation_example}
\end{figure}
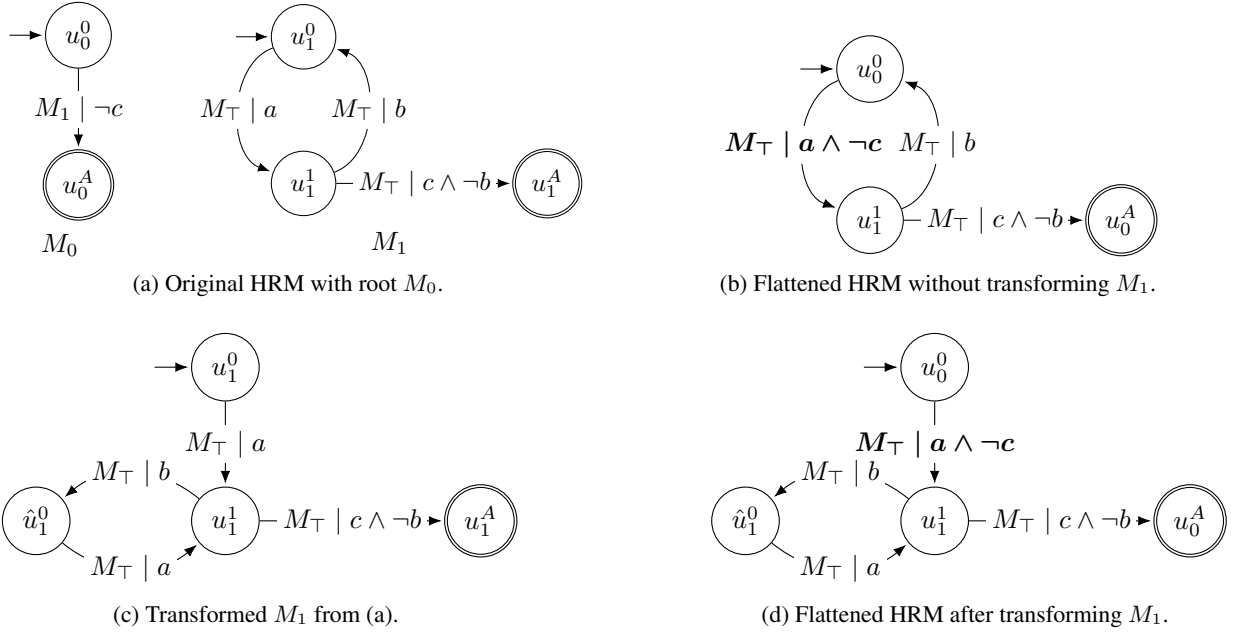

To deal with the situation above, we need to transform an RM to ensure that contexts are only checked once from the initial state. We describe this transformation as follows. Given a flat HRM $\hrmtuple$ with root $\rmname_r=\langle \rmstates_r, \propset, \rmtransition_r, \rmreward_r, \rminitstate_r, \rmstatesacc_r, \rmstatesrej_r \rangle$, we construct a new HRM $\hrmtuplep$ with root $\rmname'_r=\langle \rmstates'_r, \propset, \rmtransition'_r, \rmreward'_r, \rminitstate_r, \rmstatesacc_r, \rmstatesrej_r \rangle$ such that:
\begin{itemize}
	\item $\rmstates_r' = \rmstates_r \cup \{\hat{u}_r^0\}$, where $\hat{u}_r^0$ plays the role of the initial state after the first transition is taken.
	\item The state transition function $\rmtransition'_r$ is built by copying $\rmtransition_r$ and applying the following changes:
	\begin{enumerate}
		\item Remove the edges to the actual initial state from any state $v\in \rmstates_r'$: $\rmtransition'_r(v,\rminitstate_r,\leaf)=\bot$. Note that since the RM is flat, the only callable RM is the leaf $\leaf$.
		\item Add edges to the dummy initial state $\hat{u}^0_r$ from all states $v\in \rmstates_r'$ that had an edge to the actual initial state:  $\rmtransition_r'(v, \hat{u}^0_r, \leaf)=\rmtransition_r(v, u^0_r, \leaf)$. 
		\item Add edges from the dummy initial state $\hat{u}^0_r$ to all those states $v\in \rmstates_r'$ that the actual initial state $u^0_r$ points to: $\rmtransition'_r(\hat{u}^0_r,v, \leaf)=\rmtransition'_r(u^0_r,v, \leaf)$.
	\end{enumerate}
	\item The reward transition function $\rmreward'_r(u,u')=\mathds{1}[u\notin \rmstatesacc_r \land u' \in \rmstatesacc_r]$ is defined as stated at the beginning of the section.
\end{itemize}
The HRM $\hrm'$ is such that $\machineset'=\{\rmname'_r, \leaf \}$. Note that this transformation is only required in HRMs where the RMs have initial states with incoming edges.

We now prove that this transformation is correct; that is, the HRMs are equivalent. There are two cases depending on whether the initial state has incoming edges or not. First, if the initial state $u^0_r$ does not have incoming edges, step 1 does not remove any edges going to $u^0_r$, and step 2 does not add any edges going to $\hat{u}^0_r$, making it unreachable. Even though edges from $\hat{u}^0_r$ to other states may be added, it is irrelevant since it is unreachable. Therefore, we can safely say that in this case, the transformed HRM is equivalent to the original one. Second, if the initial state has incoming edges, we prove equivalence by examining the traversals $\hrm(\trace)$ and $\hrm'(\trace)$ for the original HRM $\hrmtuple$ and the transformed one $\hrmtuplep$ given a generic label trace $\trace=\langle \proplabel_0, \ldots, \proplabel_n\rangle$. By construction, both $\hrm(\trace)$ and $\hrm'(\trace)$ will be identical until reaching a state $w$ with an outgoing transition to $u^0_r$ in the case of $\hrm$ and the dummy initial state $\hat{u}^0_r$ in the case of $\hrm'$. More specifically, upon reaching $w$ and satisfying an outgoing formula to the aforementioned states, the traversals are:
\begin{align*}
	\hrm(\trace)&=\langle \langle \rmname_r, u^0_r, \top, [] \rangle, \ldots, \langle \rmname_r, w, \top, [] \rangle \rangle,\\
	\hrm'(\trace)&=\langle \langle \rmname'_r, u^0_r, \top, [] \rangle, \ldots, \langle \rmname'_r, w, \top, [] \rangle \rangle.
\end{align*}
By construction, state $w$ is in both HRMs, and both of the aforementioned transitions from this state are associated with the same formula, i.e. $\rmtransition_r(w,u^0_r,\leaf)=\rmtransition'_r(w,\hat{u}^0_r,\leaf)$. Therefore, if one of them is satisfied, the other will be too, and the traversals will become:
\begin{align*}
	\hrm(\trace)&=\langle \langle \rmname_r, u^0_r, \top, [] \rangle, \ldots, \langle \rmname_r, w, \top, [] \rangle, \langle \rmname_r, u^0_r,  \top, [] \rangle \rangle,\\
	\hrm'(\trace)&=\langle \langle \rmname'_r, u^0_r, \top, [] \rangle, \ldots, \langle \rmname'_r, w, \top, [] \rangle, \langle \rmname'_r, \hat{u}^0_r, \top, [] \rangle \rangle.
\end{align*}
We stay in $u^0_r$ and $\hat{u}^0_r$ until a transition to a state $w'$ is satisfied. By construction, $w'$ is in both HRMs and the same formula is satisfied, i.e., $\rmtransition_r(u^0_r, w', \leaf)=\rmtransition'_r(\hat{u}^0, w', \leaf)$. The hierarchy traversals then become:
\begin{align*}
	\hrm(\trace)&=\langle \langle \rmname_r, u^0_r, \top, [] \rangle, \ldots, \langle \rmname_r, w, \top, [] \rangle, \langle \rmname_r, u^0_r, \top, [] \rangle, \ldots, \langle\rmname_r, u^0_r, \top, [] \rangle, \langle \rmname_r, w', \top, [] \rangle \rangle,\\
	\hrm'(\trace)&=\langle \langle \rmname_r', u^0_r, \top, [] \rangle, \ldots, \langle \rmname_r', w, \top, [] \rangle, \langle \rmname'_r, \hat{u}^0_r, \top, [] \rangle, \ldots, \langle \rmname'_r, \hat{u}^0_r, \top, [] \rangle, \langle \rmname'_r, w', \top, [] \rangle \rangle.
\end{align*}
From here both traversals will be the same until transitions to $u^0_r$ and $\hat{u}^0_r$ are respectively satisfied again (if any) in $\hrm$ and $\hrm'$. Clearly, the only change in $\hrm(\trace)$ with respect to $\hrm'(\trace)$ (except for the different roots) is that the hierarchy states of the form $\langle \rmname'_r, \hat{u}^0_r, \top, []\rangle$ in the latter appear as $\langle \rmname_r, u^0_r, \top, []\rangle$ in the former. We now check if the equivalence conditions from Definition~\ref{def:hierarchy_equiv} hold:
\begin{itemize}
	\item If $\hrm(\trace)$ ends with state $u^0_r$, $\hrm'(\trace)$ ends with state $\hat{u}^0_r$ following the reasoning above. By construction, neither of these states is accepting or rejecting; therefore, neither of these HRMs accepts or rejects $\trace$.
	\item If $\hrm(\trace)$ ends with state $w$, $\hrm'(\trace)$ will also end with this state following the reasoning above. Therefore, if $w$ is an accepting state, both HRMs accept $\trace$; if $w$ is a rejecting state, both HRMs reject $\trace$; and if $w$ is not an accepting or rejecting state, neither of the HRMs accepts or rejects $\trace$.
\end{itemize}
Since all equivalence conditions are satisfied for any trace $\trace$, $\hrm$ and $\hrm'$ are equivalent.

Figure~\ref{fig:transformation_example_c} exemplifies the output of the transformation algorithm given $\rmname_{1}$ in Figure~\ref{fig:transformation_example_a} as input, whereas Figure~\ref{fig:transformation_example_d} is the output of the flattening algorithm discussed next, which properly handles the context unlike the HRM in Figure~\ref{fig:transformation_example_b}.

\textbf{Flattening Algorithm.}~
We describe the algorithm for flattening an HRM. As previously stated, we assume without loss of generality that the HRM to be flattened consists of two hierarchy levels (i.e., the root calls flat RMs). We also assume that the flat RMs have the form produced by the previously presented transformation algorithm. 

Given an HRM $\hrmtuple$ with root $\rmname_r=\langle \rmstates_r,\propset, \rmtransition_r, \rmreward_r, \rminitstate_r, \rmstatesacc_r, \rmstatesrej_r\rangle$, we build a flat RM $\bar{\rmname}_r=\langle \bar{\rmstates}_r,\propset, \bar{\rmtransition}_r,\allowbreak \bar{\rmreward}_r, \bar{u}^0_r, \bar{\rmstates}^A_r, \bar{\rmstates}^R_r \rangle$ using the following steps:
\begin{enumerate}
	\item Copy the sets of states and initial state from $\rmname_r$ (i.e., $\bar{\rmstates}_r=\rmstates_r, \bar{u}^0_r=u_r^0, \bar{\rmstates}^A_r=\rmstates^A_r,\bar{\rmstates}^R_r=\rmstates^R_r$).
	\item Loop through the non-false entries of the transition function $\rmtransition_r$ and decide what to copy. That is, for each triplet $(u,u',\rmname_j)$ where $u,u'\in \rmstates_r$ and $\rmname_j \in \machineset$ such that $\rmtransition_r(u,u',\rmname_j) \neq \bot$:
	\begin{enumerate}
		\item If $\rmname_j = \leaf$ (i.e., the called RM is the leaf), we copy the transition: $\bar{\rmtransition}_r(u,u',\leaf)=\rmtransition_r(u,u',\leaf)$.
		\item If $\rmname_j \neq \leaf$, we embed the transition function of $\rmname_j=\langle \rmstates_j, \propset, \rmtransition_j, \rmreward_j, u^0_j, \rmstatesacc_j, \rmstatesrej_j \rangle$ into $\bar{\rmname}_r$. Remember that $\rmname_j$ is flat. To do so, we run the following steps:
		\begin{enumerate}
			\item Update the set of states by adding all non-initial and non-accepting states from $\rmname_j$. Similarly, the set of rejecting states is also updated by adding all rejecting states of the called RM. The initial and accepting states from $\rmname_j$ are unimportant: their roles are played by $u$ and $u'$ respectively. In contrast, the rejecting states are important since, by assumption, they are global. Note that the added states $v$ are renamed to $v_{u,u',j}$ in order to take into account the edge being embedded: if the same state $v$ was reused for another edge, then we would not be able to distinguish them.
			\begin{align*}
				\bar{\rmstates}_r   &= \bar{\rmstates}_r \cup \left\lbrace v_{u,u',j} \mid v \in \left(\rmstates_j \setminus \left(\{u^0_j\} \cup \rmstatesacc_j \right) \right) \right\rbrace,\\
				\bar{\rmstates}^R_r &= \bar{\rmstates}^R_r \cup \left\lbrace v_{u,u',j} \mid v \in \rmstates^R_j \right\rbrace.
			\end{align*}
			\item Embed the transition function $\rmtransition_j$ of $\rmname_j$ into $\bar{\rmtransition}_r$. Since $\rmname_j$ is flat, we can make copies of the transitions straightaway: the only important thing is to check whether these transitions involve initial or accepting states which, as stated before, are going to be replaced by $u$ and $u'$ accordingly. Given a triplet $(v,w,\leaf)$ such that $v,w\in \rmstates_j$ and for which $\rmtransition_j(v,w,\leaf)=\phi$ and $\phi \neq \bot$ we update $\bar{\rmtransition}_r$ as follows:\footnote{We do not to cover the case where $v$ is an accepting state since, by assumption, there are no outgoing transitions from it. In the case of rejecting states, we keep all of them as explained in the previous case and, therefore, there are no substitutions to be made. We also do not cover the case where $w=u^0_j$ since the input flat machines never have edges to their initial states, but to the dummy initial state.}
			\begin{enumerate}
				\item If $v=u^0_j$ and $w\notin \rmstatesacc_j$, then $\bar{\rmtransition}_r(u,w_{u,u',j},\leaf)= \dnf(\phi \land \rmtransition_r(u,u',\rmname_j))$. The initial state of $\rmname_j$ has been substituted by $u$, we use the clone of $w$ associated with the call ($w_{u,u',j}$), and append the context of the call to $\rmname_j$ to the formula $\phi$.
				\item If $v=u^0_j$ and $w\in \rmstatesacc_j$, then $\bar{\rmtransition}_r(u,u',\leaf)= \textnormal{DNF}(\phi \land \rmtransition_r(u,u',\rmname_j))$. Like the previous case but performing two substitutions: $u$ replaces $v$ and $u'$ replaces $w$. The context is appended since it is a transition from the initial state of $\rmname_j$.
				\item If $v\neq u^0_j$ and $w\in \rmstatesacc_j$, then $\bar{\rmtransition}_r(v_{u,u',j},u',\leaf)= \phi$. We substitute the accepting state $w$ by $u'$, and use the clone of $v$ associated with the call ($v_{u,u',j}$). This time the call's context is not added since $v$ is not the initial state of $\rmname_j$.
				\item If none of the previous cases holds, there are no substitutions to be made nor contexts to be taken into account. Hence, $\bar{\rmtransition}_r(v_{u,u',j},w_{u,u',j},\leaf)= \phi$. We just use the clones of $v$ and $w$ corresponding to the call ($v_{u,u',j}$ and $w_{u,u',j}$).
			\end{enumerate}
		\end{enumerate}
	\end{enumerate}
	\item We apply the transformation algorithm we described before, and form a new flat HRM $\bar{\hrm}=\langle \{\bar{\rmname}_r,\leaf\}, \bar{\rmname}_r, \propset\rangle$ with the flattened (and transformed) root $\bar{\rmname}_r$.
\end{enumerate}
The reward transition function $\rmreward'_r(u,u')=\mathds{1}[u\notin \bar{\rmstates}^A_r \land u' \in \bar{\rmstates}^A_r]$ is defined as stated at the beginning of the section. Note that $u$ might not necessarily be a state of the non-flat root, but derived from an RM with lower height.

We now have everything to prove the previous theorem. Without loss of generality and for simplicity, we assume that the transformation algorithm has not been applied over the flattened root (we have already shown that the transformation produces an equivalent flat machine).

\flatequivalent*
\begin{proof}
	Let us assume that an HRM $\bar{\hrm}=\langle \bar{\machineset},\bar{\rmname}_r, \propset \rangle$, where $\bar{\rmname}_r=\langle \bar{\rmstates}_r, \propset, \bar{\rmtransition}_r, \bar{\rmreward}_r, \bar{u}^0_r, \bar{\rmstates}_r^A, \bar{\rmstates}_r^R \rangle$, is a flat HRM that results from applying the flattening algorithm on an HRM $\hrmtuple$, where $\rmname_r=\langle \rmstates_r, \propset, \rmtransition_r, \rmreward_r, u^0_r, \rmstatesacc_r, \rmstatesrej_r \rangle$. For these HRMs to be equivalent, any label trace $\trace=\langle \proplabel_0, \ldots, \proplabel_n\rangle$ must satisfy one of the conditions in Definition~\ref{def:hierarchy_equiv}. To prove the equivalence, we examine the hierarchy traversals $\hrm(\trace)$ and $\bar{\hrm}(\trace)$ given a generic label trace $\trace$.
	
	Let $u \in \rmstates_r$ be a state in the root $\rmname_r$ of $\hrm$ and let $\rmtransition_r(u,u',\leaf)$  be a satisfied transition from that state.  By construction, $u$ is also in the root $\bar{\rmname}_r$ of the flat hierarchy $\bar{\hrm}$, and $\bar{\rmname}_r$ has an identical transition $\bar{\rmtransition}_r(u,u',\leaf)$, which must also be satisfied. If the hierarchy states are $\langle \rmname_r, u, \top, []\rangle$ and $\langle \bar{\rmname}_r, u, \top, []\rangle$ for  $\hrm$ and $\bar{\hrm}$ respectively, then the next hierarchy states upon application of $\hrmtrans$ will be $\langle\rmname_r, u', \top, []\rangle$ and $\langle \bar{\rmname}_r, u', \top, []\rangle$. Therefore, both HRMs behave equivalently when calls to the leaf RM are made.
	
	We now examine what occurs when a non-leaf RM is called in $\hrm$. Let $\rmtransition_r(u,u',\rmname_j)$ be a satisfied transition in $\rmname_r$, and let $\varphi_j(u^0_j,w,\leaf)$ be a satisfied transition from $\rmname_j$'s initial state. By construction, $\bar{\rmname}_r$ contains a transition whose associated formula is the conjunction of the previous two, i.e.~$\rmtransition_r(u,u',\rmname_j) \land \rmtransition_j(u^0_j,w,\leaf)$. Now, the hierarchy traversals will be different depending on $w$:
	\begin{itemize}
		\item If $w \notin \rmstatesacc_j$ (i.e., $w$ is not an accepting state of $\rmname_j$), by construction, $\bar{\rmname}_r$ contains the transition $\bar{\rmtransition}_r(u,w_{u,u',j},\leaf)=\rmtransition_r(u,u',\rmname_j) \land \rmtransition_j(u^0_j,w,\leaf)$. If the current hierarchy states are (the equivalent) $\langle \rmname_r, u,\top,[]\rangle$ and $\langle \bar{\rmname}_r,u,\top,[]\rangle$ for $\hrm$ and $\bar{\hrm}$, then the next hierarchy states upon application of $\hrmtrans$ are $\langle \rmname_j, w, \top, [\langle u, u', \rmname_r, \rmname_j, \rmtransition_r(u,u',\rmname_j), \top \rangle] \rangle$ and $\langle \bar{\rmname}_r, w_{u,u',j}, \top, [] \rangle$. These hierarchy states are equivalent since $w_{u,u',j}$ is a clone of $w$ that saves all the call information (i.e., a call to machine $\rmname_j$ for transitioning from $u$ to $u'$).
		\item If $w \in \rmstatesacc_j$ (i.e., $w$ is an accepting state of $\rmname_j$), by construction, $\bar{\rmname}_r$ contains the transition $\bar{\rmtransition}_r(u,u',\leaf)=\rmtransition_r(u,u',\rmname_j) \land \rmtransition_j(u^0_j,w,\leaf)$. If the current hierarchy states are (the equivalent) $\langle \rmname_r, u,\top,[]\rangle$ and $\langle \bar{\rmname}_r,u,\top,[]\rangle$ for $\hrm$ and $\bar{\hrm}$, then the next hierarchy states upon application of $\hrmtrans$ are $\langle \rmname_r, u', \top, [] \rangle$ and $\langle \bar{\rmname}_r, u', \top, [] \rangle$. These hierarchy states are clearly equivalent since the machine states are exactly the same.
	\end{itemize}
	
	We now check the case in which we are inside a called RM. Let $\rmtransition_r(u,u',\rmname_j)$ be the transition that caused $\hrm$ to start running $\rmname_j$, and let $\rmtransition_j(v,w,\leaf)$ be a satisfied transition within $\rmname_j$ such that $v\neq u^0_j$. By construction, $\bar{\rmname}_r$ contains a transition associated with the same formula $\rmtransition_j(v,w,\leaf)$. The hierarchy traversals vary depending on $w$:
	\begin{itemize}
		\item If $w \notin \rmstatesacc_j$ (i.e., $w$ is not an accepting state of $\rmname_j$), by construction, $\bar{\rmname}_r$ contains the transition $\bar{\rmtransition}_r(v_{u,u',j},w_{u,u',j},\leaf)=\rmtransition_j(v,w,\leaf)$. For the transition to be taken in $\hrm$, the hierarchy state must be $\langle \rmname_j, v, \top, [\langle u, u', \rmname_r, \rmname_j, \rmtransition_r(u,u',\rmname_j), \top \rangle] \rangle$, whereas in $\bar{\hrm}$ it will be $\langle \bar{\rmname}_r, v_{u,u',j}, \top, [] \rangle$. These hierarchy states are clearly equivalent: $v_{u,u',j}$ is a clone of $v$ that saves all information related to the call being made (the called machine, and the starting and resulting states in the transition). Upon application of $\hrmtrans$, the hierarchy states will remain equivalent: $\langle \rmname_j, w, \top, [\langle u, u', \rmname_r, \rmname_j, \rmtransition_r(u,u',\rmname_j), \top \rangle] \rangle$ and $\langle \bar{\rmname}_r, w_{u,u',j}, \top, [] \rangle$ (again $w_{u,u',j}$ saves all the call information, just like the stack).
		\item If $w \in \rmstatesacc_j$ (i.e., $w$ is an accepting state of $\rmname_j$), by construction, $\bar{\rmname}_r$ contains the transition $\bar{\rmtransition}_r(v_{u,u',j},u',\leaf)=\rmtransition_j(v,w,\leaf)$. This case corresponds to that where control is returned to the calling RM. Like in the previous case, for the transition to be taken in $\hrm$, the hierarchy state must be $\langle \rmname_j, v, \top, [\langle u, u', \rmname_r, \rmname_j, \rmtransition_r(u,u',\rmname_j), \top \rangle] \rangle$, whereas in $\bar{\hrm}$ it will be $\langle \bar{\rmname}_r, v_{u,u',j},  \top, [] \rangle$. The resulting hierarchy states then become $\langle \rmname_r, u',  \top, [] \rangle$ and $\langle \bar{\rmname}_r, u',  \top, [] \rangle$ respectively, which are clearly equivalent (the state is exactly the same and both come from equivalent hierarchy states).
	\end{itemize}
	
	We have shown both HRMs have equivalent traversals for any given trace, implying that both will accept, reject, or not accept nor reject a trace. Therefore, the HRMs are equivalent.
\end{proof}

Figure~\ref{fig:flattening_examples_wo_comp} shows the result of applying the flattening algorithm on the \textsc{Book} HRM shown in Figure~\ref{fig:book_hierarchy}. Note that the resulting HRM can be compressed: there are two states having an edge with the same label to a specific state. Indeed, the presented algorithm might not produce the smallest possible flat equivalent. Figure~\ref{fig:flattening_examples_w_comp} shows the resulting compressed HRM, which is like Figure~\ref{fig:book_hierarchy_flat} but naming the states following the algorithm for clarity. Estimating how much a flat HRM (or any HRM) can be compressed and designing an algorithm to perform such compression are left as future work. 

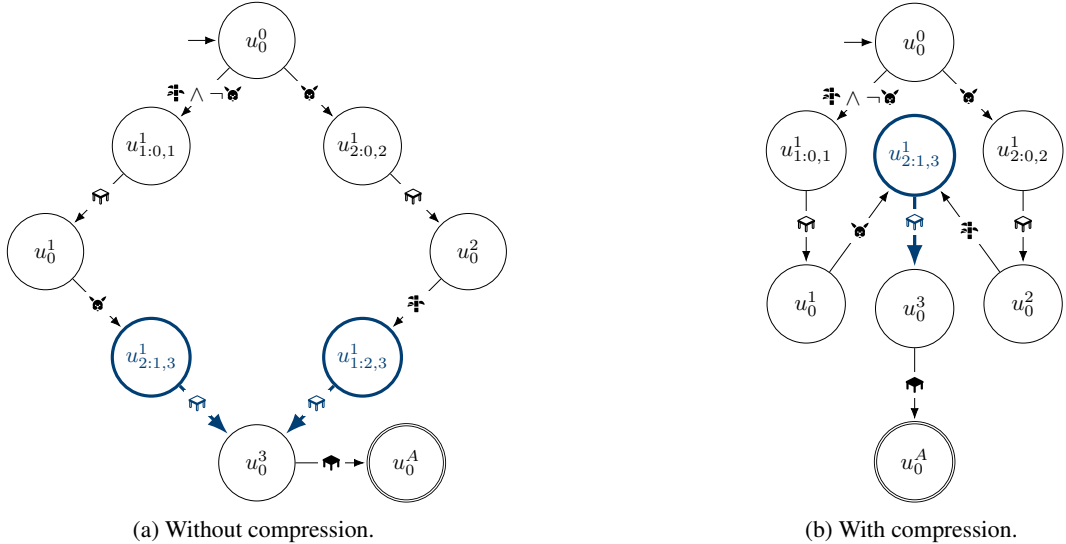
\begin{figure}[h]
	\begin{subfigure}[b]{0.49\linewidth}
		\centering
		\resizebox{0.8\linewidth}{!}{
			\begin{tikzpicture}[shorten >=1pt,node distance=2.35cm,on grid,auto,  state/.style={circle, draw, minimum size=1.2cm}, every initial by arrow/.style ={-Latex}]
				\node[state,initial,initial text=] (u_0)   {$u^0_0$};
				\node[state] (u_1) [below left = of u_0]  {$u^1_{1:0,1}$};
				\node[state] (u_2) [below left = of u_1]  {$u^1_0$};
				\node[state,color=imperial,line width=0.5mm] (u_3) [below right = of u_2]  {$u^1_{2:1,3}$};
				\node[state] (u_4) [below right = of u_0]  {$u^1_{2:0,2}$};
				\node[state] (u_5) [below right = of u_4]  {$u^2_0$};
				\node[state,color=imperial,line width=0.5mm] (u_6) [below left = of u_5]  {$u^1_{1:2,3}$};
				\node[state] (u_7) [below right = of u_3]  {$u^3_0$};
				\node[state,accepting] (u_acc) [right = of u_7]  {$u^{A}_0$};
				
				\path[-Latex] (u_0) edge node[pos=0.5] [in place] {$\Sugarcane \land \neg \Rabbit$} (u_1);
				\path[-Latex] (u_1) edge node[pos=0.45] [in place] {\Workbench} (u_2);
				\path[-Latex] (u_2) edge node[pos=0.5] [in place] {\Rabbit} (u_3);
				\path[-Latex,color=imperial,line width=0.5mm] (u_3) edge node[pos=0.35] [in place] {\Workbench} (u_7);
				
				\path[-Latex] (u_0) edge node[pos=0.5] [in place] {$\Rabbit$} (u_4);
				\path[-Latex] (u_4) edge node[pos=0.45] [in place] {\Workbench} (u_5);
				\path[-Latex] (u_5) edge node[pos=0.5] [in place] {\Sugarcane} (u_6);	
				\path[-Latex,color=imperial,line width=0.5mm] (u_6) edge node[pos=0.35] [in place] {\Workbench} (u_7);
				\path[-Latex] (u_7) edge node[pos=0.5] [in place] {\Table} (u_acc);
		\end{tikzpicture}}
		\caption{Without compression.}
		\label{fig:flattening_examples_wo_comp}
	\end{subfigure}
	\hfill
	\begin{subfigure}[b]{0.49\linewidth}
		\centering
		\resizebox{0.48\linewidth}{!}{
			\begin{tikzpicture}[shorten >=1pt,node distance=2.35cm,on grid,auto, state/.style={circle, draw, minimum size=1.2cm}, every initial by arrow/.style ={-Latex}]
				\node[state,initial,initial text=] (u_0)   {$u^0_0$};
				\node[state] (u_1) [below left = of u_0]  {$u^1_{1:0,1}$};
				\node[state] (u_2) [below = of u_1]  {$u^1_0$};
				\node[state] (u_3) [below right = of u_0]  {$u^1_{2:0,2}$};
				\node[state] (u_4) [below = of u_4]  {$u^2_0$};
				\node[state,color=imperial,line width=0.5mm] (u_5) [below =1.73241cm of u_0]  {$u^1_{2:1,3}$};
				\node[state] (u_6) [below = of u_5]  {$u^3_0$};
				\node[state,accepting] (u_acc) [below = of u_6]  {$u^{A}_0$};
				
				\path[-Latex] (u_0) edge node[pos=0.5] [in place] {$\Sugarcane \land \neg \Rabbit$} (u_1);
				\path[-Latex] (u_1) edge node[pos=0.45] [in place] {\Workbench} (u_2);
				\path[-Latex] (u_2) edge node[pos=0.5] [in place] {\Rabbit} (u_5);
				\path[-Latex,color=imperial,line width=0.5mm] (u_5) edge node[pos=0.35] [in place] {\Workbench} (u_6);
				\path[-Latex] (u_0) edge node[pos=0.5] [in place] {$\Rabbit$} (u_3);
				\path[-Latex] (u_3) edge node[pos=0.45] [in place] {\Workbench} (u_4);
				\path[-Latex] (u_4) edge node[pos=0.5] [in place] {\Sugarcane} (u_5);	
				\path[-Latex] (u_6) edge node[pos=0.5] [in place] {\Table} (u_acc);
		\end{tikzpicture}}
		\caption{With compression.}
		\label{fig:flattening_examples_w_comp}
	\end{subfigure}
	\caption{Results of flattening the HRM in Figure~\ref{fig:book_hierarchy}. The notation $u^i_{j:x,y}$ denotes the $i$-th state of RM $j$ in the call between states $x$ and $y$ in the parent RM. Note that $x$ and $y$ appear only if that state comes from a called RM. The blue states and edges in (a) can be compressed as shown in (b).}
	\label{fig:flattening_examples}
\end{figure}

\subsubsection{Proof of Theorem~\ref{thm:flat_size}}
\label{app:flat_size}
We prove Theorem~\ref{thm:flat_size} by first characterizing an HRM $\hrm$ using a set of abstract parameters. Then, we describe how the number of states and edges in an HRM and its corresponding flat equivalent are computed, and use these quantities to give an example for which the theorem holds. The parameters are the following:
\begin{itemize}
	\item The height of the root $\rootheight$.
	\item The number of RMs with height $i$, $\numrmslevel^{(i)}$.
	\item The number of states in RMs with height $i$, $\numstateslevel^{(i)}$.
	\item The number of edges from each state in RMs with height $i$, $\numedgeslevel^{(i)}$.
\end{itemize}

We assume that (i)~RMs with height $i$ only call RMs with height $i-1$; (ii)~all RMs have a single accepting state and no rejecting states; (iii)~all RMs except for the root are called; and (iv)~the HRM is well-formed (i.e., it behaves deterministically and there are no cyclic dependencies). Note that $\numrmslevel^{(\rootheight)}=1$ since there is a single root. Assumption~(i) can be made since for the root to have height $\rootheight$ we need it to call at least one RM with height $\rootheight-1$. Considering that all called RMs have the same height simplifies the analysis since we can characterize the RMs at each height independently. Assumption~(ii) is safe to make since a single accepting state is enough, and helps simplify the counting since only some RMs could have rejecting states. Assumption~(iii) ensures that the flat HRM will comprise all RMs in the original HRM. This is also a fair assumption: if a given RM is not called by any RM in the hierarchy, we could remove it beforehand.

The \emph{number of states $|\hrm|$ in the HRM $\hrm$} is obtained by summing the number of states of each RM:
\begin{equation*}
	|\hrm| = \sum^{\rootheight}_{i=1} \numrmslevel^{(i)}\numstateslevel^{(i)}.
\end{equation*}

The \emph{number of states $|\bar{\hrm}|$ in the flat HRM $\bar{\hrm}$} is given by the number of states in the flattened root RM
\begin{equation*}
	|\bar{\hrm}| = \bar{\numstateslevel}^{(\rootheight)},
\end{equation*}
where $\bar{\numstateslevel}^{(i)}$ is the number of states in the flattened representation of an RM with height $i$, which is recursively defined as:
\begin{equation*}
	\bar{\numstateslevel}^{(i)} = \begin{cases}
		\numstateslevel^{(i)} & \textnormal{if~} i=1,\\
		\numstateslevel^{(i)} + \left(\bar{\numstateslevel}^{(i-1)} -2 \right)\left( \numstateslevel^{(i)} -1 \right) \numedgeslevel^{(i)} & \textnormal{if~} i>1.
	\end{cases}
\end{equation*}
That is, the number of states in a flattened RM with height $i$ has all states that the non-flat HRM had. In addition, for each of the $\numstateslevel^{(i)}-1$ non-accepting states in the non-flat RM, there are $\numedgeslevel^{(i)}$ edges, each of which calls an RM with height $i-1$ whose number of states is $\bar{\numstateslevel}^{(i-1)}$. These edges are replaced by the called RM except for the initial and accepting states, whose role is now played by the states involved in the substituted edge (hence the $-2$). This construction process corresponds to the one used to prove Theorem~\ref{thm:flat_equivalent}.

The \emph{total of number of edges} in an HRM is given by:
\begin{equation*}
	\sum_{i=1}^{\rootheight}\numrmslevel^{(i)}(\numstateslevel^{(i)}-1)\numedgeslevel^{(i)},
\end{equation*}
where $(\numstateslevel^{(i)}-1)\numedgeslevel^{(i)}$ is the total number of edges in an RM with height $i$ (the $-1$ is because the accepting state is discarded), so $\numrmslevel^{(i)}(\numstateslevel^{(i)}-1)\numedgeslevel^{(i)}$ determines how many edges there are across RMs with height $i$.

The \emph{total number of edges in the flat HRM} is given by the total number of edges in the flattened root RM, $\numedgestotal^{(\rootheight)}$, where $\numedgestotal^{(i)}$ is the total number of edges in the flattened representation of an RM with height $i$, which is recursively defined as follows:
\begin{equation*}
	\numedgestotal^{(i)} = \begin{cases}
		(\numstateslevel^{(i)}-1)\numedgeslevel^{(i)} & \textnormal{if~}i=1,\\
		(\numstateslevel^{(i)}-1)\numedgeslevel^{(i)}\numedgestotal^{(i-1)} & \textnormal{if~}i>1.
	\end{cases}
\end{equation*}
That is, each of the $(\numstateslevel^{(i)}-1)\numedgeslevel^{(i)}$ edges in an RM with height $i$ is replaced by $\numedgestotal^{(i-1)}$ edges given by an RM with height $i-1$ (if any).

The \emph{key intuition} is that an HRM with root height $\rootheight>1$ is beneficial representation-wise if the number of calls across RMs with height $i$ is higher than the number of RMs with height $i-1$; in other words, RMs with lower heights are being reused. Numerically, the total number of edges/calls in an RM with height $i$ is $(\numstateslevel^{(i)}-1)\numedgeslevel^{(i)}$ and, therefore, the total number of calls across RMs with height $i$ is $(\numstateslevel^{(i)}-1)\numedgeslevel^{(i)}\numrmslevel^{(i)}$. If this quantity is higher than $\numrmslevel^{(i-1)}$, then RMs with lower heights are reused, and therefore having RMs with different heights is beneficial.

\flatsize*
\begin{proof}
	By example. Let $\hrmtuple$ be an HRM whose root $\rmname_r$ has height $\rootheight$ and is parameterized by $\numrmslevel^{(i)}=1$, $\numstateslevel^{(i)}=3$, $\numedgeslevel^{(i)}=1$ for $i=1,\ldots,\rootheight$. Figure~\ref{fig:hrm_instance_exp_states} shows an instance of this hierarchy. Let us write the number of states in the flat RMs of each level:
	\begin{align*}
		\bar{\numstateslevel}^{(1)}&=U^{(1)}=3,\\
		\bar{\numstateslevel}^{(2)}&=U^{(2)}  +\left(\bar{\numstateslevel}^{(1)} -2 \right)\left( \numstateslevel^{(2)} -1 \right) E^{(2)}=3 + \left(3 -2 \right)\left(3 -1 \right) 1 = 5,\\
		\bar{\numstateslevel}^{(3)}&=U^{(3)}  +\left(\bar{\numstateslevel}^{(2)} -2 \right)\left( \numstateslevel^{(3)} -1 \right) E^{(3)}=3 + \left(5 -2 \right)\left(3 -1 \right) 1 = 9,\\
		&\vdots\\
		\bar{\numstateslevel}^{(i)}&= 2\bar{\numstateslevel}^{(i-1)}-1 =  2^{i}+1.
	\end{align*}
	Hence, the number of states in the flat HRM is $|\bar{\hrm}|=\bar{\numstateslevel}^{(\rootheight)}=2^{\rootheight}+1$, showing that the number of states in the flat HRM grows exponentially with the height of the root. In contrast, the number of states in the HRM grows linearly with the height of the root, $|\hrm| = \sum_{i=1}^{\rootheight}\numrmslevel^{(i)}\numstateslevel^{(i)}=\sum_{i=1}^{\rootheight}1\cdot3=3\rootheight$.
	
	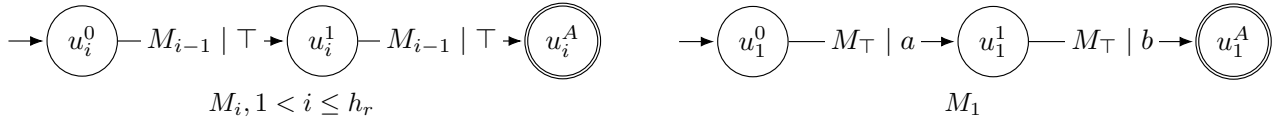
\begin{figure}[h]
		\begin{minipage}{0.48\linewidth}
			\centering
			\resizebox{\columnwidth}{!}{
				\begin{tikzpicture}[shorten >=1pt,node distance=3cm,on grid,auto, every initial by arrow/.style ={-Latex}]
					\node[state,initial, initial text=] (u_0)   {$u^0_i$};
					\node[state] (u_1) [right =of u_0]   {$u^1_i$};
					\node[state,accepting] (u_acc) [right =of u_1]   {$u^A_i$};
					\path[-Latex] (u_0) edge node [in place] {$\rmname_{i-1}\mid \top$} (u_1);
					\path[-Latex] (u_1) edge node [in place] {$\rmname_{i-1}\mid \top$} (u_acc);
			\end{tikzpicture}}
			
			$\rmname_i, 1<i\leq \rootheight$
		\end{minipage} \hfill
		\begin{minipage}{0.48\linewidth}
			\centering
			\resizebox{\columnwidth}{!}{
				\begin{tikzpicture}[shorten >=1pt,node distance=3cm,on grid,auto, every initial by arrow/.style ={-Latex}]
					\node[state,initial, initial text=] (u_0)   {$u^0_1$};
					\node[state] (u_1) [right =of u_0]   {$u^1_1$};
					\node[state,accepting] (u_acc) [right =of u_1]   {$u^A_1$};
					\path[-Latex] (u_0) edge node [in place] {$\leaf\mid a$} (u_1);
					\path[-Latex] (u_1) edge node [in place] {$\leaf\mid b$} (u_acc);
			\end{tikzpicture}}
			
			$\rmname_1$
		\end{minipage}
		\caption{Example of an HRM whose root has height $\rootheight$ used in the proof of Theorem~\ref{thm:flat_size}.}
		\label{fig:hrm_instance_exp_states}
	\end{figure}
	
	In the case of the total number of edges, we again write some iterations to derive a general expression:
	\begin{align*}
		\numedgestotal^{(1)} &= (\numstateslevel^{(1)}-1)\numedgeslevel^{(1)} = (3-1)1=2,\\
		\numedgestotal^{(2)} &= (\numstateslevel^{(2)}-1)\numedgeslevel^{(2)}\numedgestotal^{(1)} = (3-1)\cdot1\cdot2=4,\\
		\numedgestotal^{(3)} &= (\numstateslevel^{(3)}-1)\numedgeslevel^{(3)}\numedgestotal^{(2)} = (3-1)\cdot1\cdot4=8,\\
		& \vdots\\
		\numedgestotal^{(i)} &= 2\numedgestotal^{(i-1)} = 2^{i}.
	\end{align*}
	Therefore, the total number of edges in the flat HRM is $\numedgestotal^{(\rootheight)}=2^{\rootheight}$. In contrast, the total number of edges in the HRM grows linearly: $\sum_{i=1}^{\rootheight}N^{(i)}(U^{(i)}-1)E^{(i)}=\sum_{i=1}^{\rootheight}1(3-1)1=2\rootheight$.
	
	Finally, we emphasize that the resulting flat HRM cannot be compressed, unlike the HRM in Figure~\ref{fig:flattening_examples}: each state has at most one incoming edge, so there are not multiple paths that can be merged. We have thus shown that there are HRMs whose equivalent flat HRM has a number of states and edges that grows exponentially with the height of the root.
\end{proof}

Using the aforementioned intuition, we observe that the hierarchical structure is actually expected to be useful: the number of calls across RMs with height $i$ is $(\numstateslevel^{(i)}-1)\numedgeslevel^{(i)}=(3-1)1=2$, which is greater than the number of RMs with height $i-1$ (only 1).

There are cases where having a multi-level hierarchy (i.e., with $\rootheight>1$) is not beneficial. For instance, given an HRM whose root has height $\rootheight$ and parameterized by  $\numrmslevel^{(i)}=1$, $\numstateslevel^{(i)}=2$ and $\numedgeslevel^{(i)}=1$, the number of states in the equivalent flat HRM is constant (2), whereas in the HRM itself it grows linearly with $\rootheight$. The same occurs with the number of edges. By checking the previously introduced intuition, we observe that $(\numstateslevel^{(i)}-1)\numedgeslevel^{(i)}\numrmslevel^{(i)}=(2-1)\cdot 1 \cdot 1=1 \not> \numrmslevel^{(i-1)}=1$, which verifies that having non-reused RMs with multiple heights is not useful.

\section{Policy Learning Implementation Details}
\label{app:option_hierarchy_selection}
In this appendix, we describe some implementation details that were omitted in Section~\ref{sec:policy_learning} for simplicity. First, we start by describing some methods used in policy learning. Second, we explain the option selection algorithm step-by-step and provide examples to ease its understanding.

\subsection{Policies}
\label{app:policy_learning_policy_details}
\textbf{Deep Q-networks \citep[DQNs;][]{MnihKSRVBGRFOPB15}.~}
We use Double DQNs~\citep{HasseltGS16} for both formula and call options. The DQNs associated with formula options simply take an MDP state and output a Q-value for each action. In contrast, the DQNs associated with call options also take an RM state and a context, which are encoded as follows:
\begin{itemize}
	\item The RM state is encoded using a one-hot vector. The size of the vector is given by the number of states in the RM.
	\item The context, which is either $\top$ or a DNF formula with a single disjunct/conjunction, is encoded using a vector whose size is the number of propositions $|\propset|$. Each vector position corresponds to a proposition $p \in \propset$ whose value depends on how $p$ appears in the context: (i)~+1 if $p$ appears positively, (ii)~-1 if $p$ appears negatively, or (iii)~0 if $p$ does not appear. Note that if the context is $\top$, the vector solely consists of zeros.
\end{itemize}
These DQNs output a value for each possible call in the RM; however, some of these values must be masked if the corresponding calls are not available from the RM state-context used as input. For instance, the DQN for $\rmname_0$ in Figure~\ref{fig:book_hierarchy} outputs a value for $\langle\rmname_1,\neg\Rabbit\rangle$, $\langle\rmname_2,\top\rangle$, $\langle\rmname_1,\top\rangle$,  and $\langle \leaf,\Table\rangle$. If the RM state was $u^0_0$ and the context was $\top$, only the values for the first two calls are relevant. Just like unavailable calls, we also mask unsatisfiable calls (i.e., calls whose context cannot be satisfied in conjunction with the accumulated context used as input).

To speed up learning, a subset of the Q-functions associated with formula options is updated after each step. Updating all the Q-functions after each step is costly and we observed that similar performance could be obtained with this strategy. To determine the subset, we keep an update counter $c_{\phi}$ for each Q-function $\qfunc_{\phi}$, and a global counter $c$ (i.e., the total number of times Q-functions have been updated). The probability of updating $\qfunc_{\phi}$ is:
\begin{equation*}
	p_{\phi} = \frac{s_{\phi}}{\sum_{\phi'} s_{\phi'}}, \textnormal{where}~s_\phi=c-c_\phi-1.
\end{equation*}
A subset of Q-functions is chosen using this probability distribution without replacement.

\textbf{Exploration.}~
During training, the formula and call option policies are $\epsilon$-greedy. In the case of formula options, akin to Q-functions, each option $\opt^{j,\phi}_{i,u,\Context}$ performs exploration with an exploration factor $\epsilon_{\phi\land\Context}$, which linearly decreases with the number of steps performed using the policy induced by $\qfunc_{\phi\land\Context}$. Likewise, \citet{KulkarniNST16} keep an exploration factor for each subgoal, but vary it depending on the option's success rather than on the number of performed steps. In the case of call options, each RM state-context pair is associated with its own exploration factor, which linearly decreases as options started from that pair terminate.

\textbf{The Formula Tree.}
As explained in Section~\ref{sec:policy_learning}, each formula option's policy is induced by a Q-function associated with a formula. In domains where certain proposition sets cannot occur, it is unnecessary to consider formulas that cover some of these sets. For instance, in a domain where two propositions $a$ and $b$ cannot be simultaneously observed (i.e., it is impossible to observe $\{a,b\}$), formulas such as $a\land\neg b$ or $b\land\neg a$ could instead be represented by the more abstract formulas $a$ or $b$; therefore, $a\land\neg b$ and $a$ could be both associated with a Q-function $\qfunc_a$, whereas $b\land\neg a$ and $b$ could be both associated with a Q-function $\qfunc_b$. By reducing the number of Q-functions, learning naturally becomes more efficient.

We represent relationships between formulas using a \emph{formula tree} which, as the name suggests, arranges a set of formulas in a tree structure. Formally, given a set of propositions $\propset$, a formula tree is a tuple $\langle \ftreenodes, \ftreeroot, \proplabelset\rangle$, where $\ftreenodes$ is a set of nodes, each associated with a formula; $\ftreeroot\in\ftreenodes$ is the root of the tree and it is associated with the formula $\top$; and $\proplabelset\subseteq (2^\propset)^\ast$ is a set of labels. All the nodes in the tree except for the root are associated with conjunctions. Let $\formtolits(X)\subseteq 2^{2\propset}$ denote the set of literals of a formula $X$, e.g. if $X=a\land\neg b$, then $\nu(X)=\{a,\neg b\}$. A formula $X$ \emph{subsumes} a formula $Y$ if (1)~$X=\top$, or (2.i)~$\formtolits(X)\subseteq \formtolits(Y)$ and (2.ii)~for all labels $\proplabel\in \proplabelset$, either $\proplabel \models X$ and $\proplabel\models Y$, or $\proplabel \not\models X$ and $\proplabel\not\models Y$. Case~(2) indicates that $Y$ is a special case of $X$ (it adds literals but it is satisfied by exactly the same labels). The tree is organized such that the formula at a given node subsumes all its descendants. The set of Q-functions is determined by the children of the root.

During the agent-environment interaction, the formula tree is updated if (i)~a new formula appears in the learned HRMs, or (ii)~a new label is observed. Algorithm~\ref{alg:formula_tree} contains the pseudo-code for updating the tree in these two cases. When a new formula is added (line~\ref{algline:tree:add_formula}), we create a node for the formula (line~\ref{algline:tree:create_node}) and add it to the tree. The insertion place is determined by exploring the tree top-down from the root $\ftreeroot$ (lines~\ref{algline:tree:add_node_start}--\ref{algline:tree:add_node_end}). First, we check whether a child of the current node subsumes the new node (line~\ref{algline:tree:find_subsume_call}). If such a node exists, then we go down this path (lines~\ref{algline:tree:add_node_non_nil_1}--\ref{algline:tree:add_node_non_nil_2}); otherwise, the new node is going to be a child of the current node (lines~\ref{algline:tree:current_node_children_update}--\ref{algline:tree:new_node_set_parent}). In the latter case, in addition, all those children nodes of the current node that are subsumed by the new node need to become children of the new node (lines~\ref{algline:tree:insert_node_start}--\ref{algline:tree:insert_node_end}). The other core case in which the tree may need an update occurs when a new label is observed (lines~\ref{algline:tree:on_label_start}--\ref{algline:tree:on_label_end}) since we need to make sure that parenting relationships comply with the set of labels $\proplabelset$. First, we find nodes inconsistent with the new label: a parenting relationship is broken (line~\ref{algline:tree:inconsistency_found}) when the formula of the parent non-root node is satisfied by the label but the formula of the child node is not (or vice versa). Once the inconsistent nodes are found, we remove their current parenting relationship (lines~\ref{algline:tree:remove_parenting_1}--\ref{algline:tree:remove_parenting_2}) and reinsert them in the tree (line~\ref{algline:tree:reinsertion}). Figure~\ref{fig:formula_tree_example} shows two simple examples of formula trees, where the Q-functions are $\qfunc_a$ in (a), and $\qfunc_a$ and $\qfunc_{a\land\neg c}$ in (b).

\begin{algorithm}[!htbp]
	\caption{Formula tree operations}
	\label{alg:formula_tree}
	\begin{algorithmic}[1]
		\REQUIRE a formula tree $\langle \ftreenodes, \ftreeroot,  \proplabelset\rangle$, where $\ftreenodes$ is a set of nodes, $\ftreeroot\in\ftreenodes$ is the root node (associated with the formula $\top$), and $\proplabelset$ is a set of labels.
		\FUNCTION{\textsc{AddFormula}($f$)} \alglinelabel{algline:tree:add_formula}
		\STATE \textsc{AddNode}(\textsc{CreateNode}($f$)) \alglinelabel{algline:tree:create_node}
		\ENDFUNCTION
		\FUNCTION{\textsc{AddNode}(new\_node)} \alglinelabel{algline:tree:add_node_start}
		\STATE $\textnormal{current\_node} \leftarrow \ftreeroot$
		\STATE $\textnormal{added\_node} \leftarrow \bot$
		\WHILE{$\textnormal{added\_node}=\bot$}
		\STATE 	$\textnormal{child\_node} \leftarrow \textsc{FindSubsumingChild}(\textnormal{current\_node}, \textnormal{new\_node})$ \alglinelabel{algline:tree:find_subsume_call}
		\IF[Keep exploring down this path]{$\textnormal{child\_node} \neq \textnormal{nil}$} \alglinelabel{algline:tree:add_node_non_nil_1}
		\STATE $\textnormal{current\_node} \leftarrow \textnormal{child\_node}$ \alglinelabel{algline:tree:add_node_non_nil_2}
		\ELSE[Insert the node]
		\STATE $\textnormal{subsumed\_children} \leftarrow \textsc{GetSubsumedChildren}(\textnormal{current\_node}, \textnormal{new\_node})$ \alglinelabel{algline:tree:insert_node_start}
		\STATE $\textnormal{new\_node.children} \leftarrow \textnormal{new\_node.children} \cup \textnormal{subsumed\_children}$
		\FOR{$\textnormal{child} \in \textnormal{subsumed\_children}$}
		\STATE $\textnormal{current\_node.children} \leftarrow\textnormal{current\_node.children} \setminus \{\textnormal{child}\}$
		\STATE $\textnormal{child.parent} \leftarrow \textnormal{new\_node}$ \alglinelabel{algline:tree:insert_node_end}
		\ENDFOR
		\STATE $\textnormal{current\_node.children} \leftarrow \textnormal{current\_node.children} \cup \{\textnormal{new\_node}\}$ \alglinelabel{algline:tree:current_node_children_update}
		\STATE $\textnormal{new\_node.parent} \leftarrow \textnormal{current\_node}$ \alglinelabel{algline:tree:new_node_set_parent}
		\STATE $\textnormal{added\_node}\leftarrow \top$
		\ENDIF
		\ENDWHILE
		\STATE $\ftreenodes \leftarrow \ftreenodes \cup \{\textnormal{new\_node}\}$ \alglinelabel{algline:tree:add_node_end}
		\ENDFUNCTION
		\FUNCTION{\textsc{OnLabel}($\proplabel$)} \alglinelabel{algline:tree:on_label_start}
		\STATE $\proplabelset \leftarrow \proplabelset \cup \{\proplabel\}$
		\STATE $\textnormal{inconsistent\_nodes} \leftarrow \{\}$
		\FOR{$\textnormal{child} \in \ftreeroot.\textnormal{children}$}
		\STATE \textsc{FindInconsistentNodes}($\textnormal{child}, \proplabel, \textnormal{inconsistent\_nodes}$) 
		\ENDFOR
		\STATE \textsc{ReinsertInconsistentNodes}(inconsistent\_nodes) \alglinelabel{algline:tree:on_label_end}
		\ENDFUNCTION
		\FUNCTION{\textsc{FindSubsumingChild}(current\_node, new\_node)}
		\FOR{$\textnormal{child} \in \textnormal{current\_node.children}$}
		\IF{child.formula \underline{subsumes} new\_node.formula}
		\STATE \textbf{return} child
		\ENDIF
		\ENDFOR
		\STATE \textbf{return} nil
		\ENDFUNCTION
		\FUNCTION{\textsc{GetSubsumedChildren}(current\_node, new\_node)}
		\STATE $\textnormal{subsumed\_children} \leftarrow \{\}$
		\FOR{$\textnormal{child} \in \textnormal{current\_node.children}$}
		\IF{new\_node.formula \underline{subsumes} child.formula}
		\STATE $\textnormal{subsumed\_children} \leftarrow \textnormal{subsumed\_children} \cup \{\textnormal{new\_node}\}$
		\ENDIF
		\ENDFOR
		\STATE \textbf{return} subsumed\_children
		\ENDFUNCTION
		\FUNCTION{\textsc{FindInconsistentNodes}(current\_node, $\proplabel$, inconsistent\_nodes)}
		\FOR{$\textnormal{child} \in \textnormal{current\_node.children}$}
		\IF{$\proplabel \models \textnormal{current\_node.formula} \oplus \proplabel \models \textnormal{child.formula}$} \alglinelabel{algline:tree:inconsistency_found}
		\STATE $\textnormal{inconsistent\_nodes} \leftarrow \textnormal{inconsistent\_nodes} \cup \{\textnormal{child}\}$
		\ELSE
		\STATE \textsc{FindInconsistentNodes}(child, $\proplabel$, inconsistent\_nodes)
		\ENDIF
		\ENDFOR
		\ENDFUNCTION
		\FUNCTION{\textsc{ReinsertInconsistentNodes}(inconsistent\_nodes)}
		\FOR{$\textnormal{node}\in \textnormal{inconsistent\_nodes}$}
		\STATE $\textnormal{node.parent.children} \leftarrow \textnormal{node.parent.children} \setminus \{\textnormal{node}\}$ \alglinelabel{algline:tree:remove_parenting_1}
		\STATE $\textnormal{node.parent}\leftarrow\textnormal{nil}$ \alglinelabel{algline:tree:remove_parenting_2}
		\STATE \textsc{AddNode}(node) \alglinelabel{algline:tree:reinsertion}
		\ENDFOR
		\ENDFUNCTION
	\end{algorithmic}
\end{algorithm}

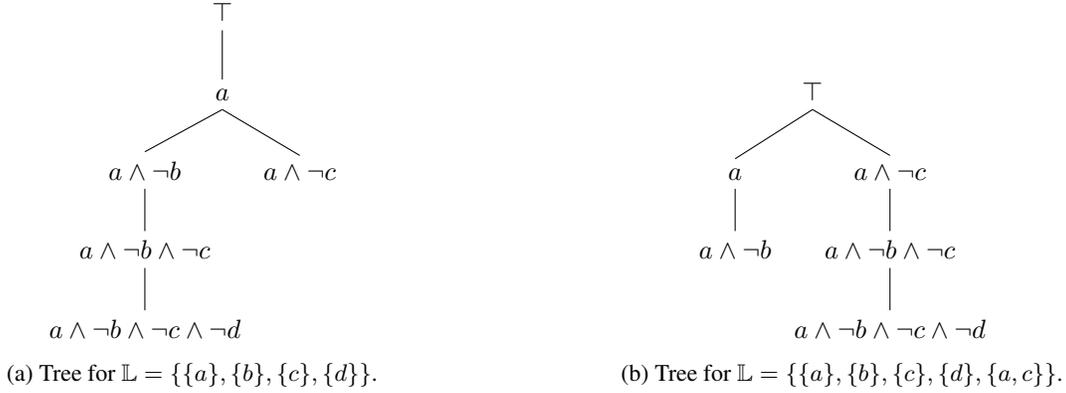
\begin{figure}
	\begin{subfigure}[b]{0.5\linewidth}
		\centering
		\begin{tikzpicture}
			\Tree [.$\top$ 
			[.$a$ 
			[.$a\land\neg b$ 
			[.$a\land\neg b\land \neg c$ $a\land\neg b\land \neg c\land \neg d$ ]
			]
			[.$a\land\neg c$ ]
			] 
			]
		\end{tikzpicture}
		\caption{Tree for $\proplabelset=\left\lbrace\{a\},\{b\},\{c\},\{d\} \right\rbrace$.}
	\end{subfigure}
	\begin{subfigure}[b]{0.5\linewidth}
		\centering
		\begin{tikzpicture}
			\Tree [.$\top$ 
			[.$a$ 
			[.$a\land\neg b$ ]
			]
			[.$a\land\neg c$ 
			[.$a\land\neg b\land \neg c$ $a\land\neg b\land \neg c\land \neg d$ ]
			] 
			]
		\end{tikzpicture}
		\caption{Tree for $\proplabelset=\left\lbrace\{a\},\{b\},\{c\},\{d\}, \{a,c\} \right\rbrace$.}
	\end{subfigure}
	\caption{Examples of formula trees for different sets of literals. Note that the node $a\land\neg b\land\neg c$ in (a) could also be a child of $a\land\neg c$ (the parent depends on the insertion order).}
	\label{fig:formula_tree_example}
\end{figure}

\begin{figure*}
	
	\begin{subfigure}[t]{0.55\linewidth}
		\centering
		\tikzstyle{none}=[inner sep=0mm]
		\tikzstyle{rm_qi_input} = [inner sep=0mm, outer sep=-0.3]
		\tikzstyle{rm_qi_output} = [inner sep=0mm, outer sep=-0.3]
		\tikzstyle{arrow}=[->,>=stealth]  
		\begin{tikzpicture}	
			\node [fill=white, draw=black, shape=rectangle, minimum width=0.55cm, minimum height=1.5cm] (rm_q) at (10, 10) {$\qfunc_i$};
			\node [style=rm_qi_input] (rm_s_qi) at (9.725, 10.45) {};
			\node [left=0.65 of rm_s_qi] (rm_s) {$\mdpstate$};
			\draw [style=arrow] (rm_s) -- (rm_s_qi);
			
			\node [style=rm_qi_input] (rm_m_qi) at (9.725, 10.15) {};
			\node [left=0.65 of rm_m_qi] (rm_m) {$\rmname_i$};
			\draw [style=arrow] (rm_m) -- (rm_m_qi);
			
			\node [style=rm_qi_input] (rm_u_qi) at (9.725, 9.85) {};
			\node [left=0.65 of rm_u_qi] (rm_u) {$\rmstate$};
			\draw [style=arrow] (rm_u) -- (rm_u_qi);
			
			\node [style=rm_qi_input] (rm_phi_qi) at (9.725, 9.55) {};
			\node [left=0.65 of rm_phi_qi] (rm_phi) {$\Context$};
			\draw [style=arrow] (rm_phi) -- (rm_phi_qi);
			
			\node [style=rm_qi_output] (rm_out1) at (10.275, 10.45) {};
			\node [right=0.2 of rm_out1] (rm_out1_res) {$\qfunc_i(\mdpstate,\rmstate,\Context, \langle \rmname_j, \context\rangle)$};
			\draw [style=arrow] (rm_out1) -- (rm_out1_res);
			
			\node [style=rm_qi_output] (rm_out2) at (10.275, 9.55) {};
			\node [right=0.2 of rm_out2] (rm_out2_res) {$\qfunc_i(\mdpstate,\rmstate,\Context, \langle \rmname_{j'}, \context'\rangle)$};
			\draw [style=arrow] (rm_out2) -- (rm_out2_res);
			
			\node at ($(rm_out1_res)!0.4!(rm_out2_res)$) {$\vdots$};
			
			\node [fill=white, draw=black, shape=rectangle, minimum width=0.55cm, minimum height=1.5cm] (rm_pi) at (14.1, 10) {$\pi_i$};
			
			\node [style=rm_qi_input] (pi_q1) at (13.825, 10.45) {};
			\draw [style=arrow] (rm_out1_res) -- (pi_q1);`
			
			\node [style=rm_qi_input] (pi_q2) at (13.825, 9.55) {};
			\draw [style=arrow] (rm_out2_res) -- (pi_q2);
			
			\node [right=0.2 of rm_pi] (option) {$\opt^{j,\context}_{i,\rmstate,\Context}$};
			\draw [style=arrow] (rm_pi) -- (option);
			
			\node [shape=rectangle, draw=black, right=0.3 of option] (option_stack) {$\opth$};
			\draw [style=arrow, dotted] (option) -- (option_stack); 
			
			\node[regular polygon, regular polygon sides=6, minimum width=1cm, xscale=1.5,draw,label=center: {$j = \top$}, below=0.6 of option] (decision) {};
			\draw[style=arrow] (option) -- (decision); 
			
			\node [shape=circle, draw=black, fill=black, right=0.9 of decision, scale=0.5] (algo_end) {};
			\draw[style=arrow] (decision) -- (algo_end) node [midway, fill=white] {\footnotesize yes};
			
			\node [left=1.2 of decision, inner sep=0, outer sep=-0.3] (bifurc_state) {};
			\node [above=0.5 of bifurc_state, inner sep=0, outer sep=-0.6] (bifurc_machine) {};
			\node [below=0.5 of bifurc_state, inner sep=0, outer sep=-0.6] (bifurc_ctx) {};
			
			\draw[-] (decision) -| (bifurc_machine);
			\draw[-] (decision) -| (bifurc_ctx);
			\draw[-] (decision) -- (bifurc_state) node [midway, fill=white] {\footnotesize no};

			\node [left= 1 of bifurc_state] (bifurc_state_node) {$u^0_j$};
			\draw[-] (bifurc_state) -- (bifurc_state_node);
			
			\node [left= 1 of bifurc_machine] (bifurc_machine_node) {$\rmname_j$};
			\draw[-] (bifurc_machine) -- (bifurc_machine_node);
			
			\node [left= 1 of bifurc_ctx] (bifurc_ctx_node) {$\context \land \Context$};
			\draw[-] (bifurc_ctx) -- (bifurc_ctx_node);
			
			%
			\node [inner sep=0, outer sep=-0.25] at ($(rm_m)!0.8!(rm_m_qi)$) (mi_connect) {}; 
			\draw[style=arrow] (bifurc_machine_node) -| (mi_connect);
			
			\node [inner sep=0, outer sep=-0.25] at ($(rm_u)!0.575!(rm_u_qi)$) (u_connect) {}; 
			\draw[style=arrow] (bifurc_state_node) -| (u_connect);
			
			\node [inner sep=0, outer sep=-0.25] at ($(rm_phi)!0.4!(rm_phi_qi)$) (ctx_connect) {}; 
			\draw[style=arrow] (bifurc_ctx_node) -| (ctx_connect);
		\end{tikzpicture}
		\caption{\textbf{Fill the Option Stack $\bm\opth$.} Call options are iteratively selected until reaching a formula option. When a call option is chosen, the next option is selected from the initial state $u^0_j$ of the called RM $\rmname_j$ using the accumulated context $\context\land\Context$.}
		\label{fig:policy_learning_algorithm_fill_stack}
	\end{subfigure}
	\hfill
	\begin{subfigure}[t]{0.4\linewidth}
		\centering
		\tikzstyle{none}=[inner sep=0mm]
		\tikzstyle{input} = [inner sep=0mm, outer sep=-0.3]
		\tikzstyle{output} = [inner sep=0mm, outer sep=-0.3]
		\tikzstyle{arrow}=[->,>=stealth]  
		\begin{tikzpicture}	
			\node [fill=white, draw=black, shape=rectangle, minimum width=1cm, minimum height=1.5cm] (qfunc) at (10, 10) {$\qfunc_{\context \land \Context}$};
			
			\node [style=input] (opt_in) at (9.5, 10.375) {};
			\node [left=0.3 of opt_in] (opt) {$\opt^{\top,\context}_{i,\rmstate,\Context}$};
			\draw [style=arrow] (opt) -- (opt_in);
			
			\node [style=input] (s_in) at (9.5, 9.625) {};
			\node [left=0.3 of s_in] (state) {$\mdpstate$};
			\draw [style=arrow] (state) -- (s_in);
			
			\node [shape=rectangle, draw=black, above=0.3 of opt] (opt_stack) {$\opth$};
			\draw [style=arrow, dotted] (opt_stack) -- (opt);
			
			\node [style=output] (out1) at (10.5, 10.375) {};
			\node [right=0.15 of out1] (rm_out1_res) {$\qfunc_{\context\land\Context}(\mdpstate,\mdpaction)$};
			\draw [style=arrow] (out1) -- (rm_out1_res);
			
			\node [style=output] (out2) at (10.5, 9.625) {};
			\node [right=0.15 of out2] (rm_out2_res) {$\qfunc_{\context\land\Context}(\mdpstate,\mdpaction')$};
			\draw [style=arrow] (out2) -- (rm_out2_res);
			
			\node at ($(rm_out1_res)!0.4!(rm_out2_res)$) {$\vdots$};
			
			\node [fill=white, draw=black, shape=rectangle, minimum width=1.05cm, minimum height=1.5cm] (qfunc) at (13.20, 10) {$\pi_{\context\land\Context}$};
			\node [style=input] (pi_in1) at (12.675, 10.375) {};
			\draw [style=arrow] (rm_out1_res) -- (pi_in1);
			\node [style=input] (pi_in2) at (12.675, 9.625) {};
			\draw [style=arrow] (rm_out2_res) -- (pi_in2);
			
			\node [right=0.3 of qfunc] (action) {$\mdpaction$};
			\draw [style=arrow] (qfunc) -- (action);		
		\end{tikzpicture}
		\caption{\textbf{Select an Action.} The formula option on the option stack $\opth$ determines the action $\mdpaction$ to execute in state $\mdpstate$.}
		\label{fig:policy_learning_algorithm_choose_action}
	\end{subfigure}
	
	\begin{subfigure}{\linewidth}
		\centering
		\tikzstyle{none}=[inner sep=0mm]
		\tikzstyle{input} = [inner sep=0mm, outer sep=-0.3]
		\tikzstyle{output} = [inner sep=0mm, outer sep=-0.3]
		\tikzstyle{arrow}=[->,>=stealth]  
		\begin{tikzpicture}	
			\node (action) at (10, 10) {$\obstuple, \mdpaction$};
			
			\node [fill=white, draw=black, shape=rectangle, minimum width=1cm, minimum height=0.6cm, right=0.3 of action] (env) {\textsc{mdp}};
			\draw [style=arrow] (action) -- (env);
			
			\node [right=0.3 of env] (new_state) {$\obstuple',\proplabel$};
			\draw [style=arrow] (env) -- (new_state);
			
			
			
			\node [fill=white, draw=black, shape=rectangle, minimum width=1cm, minimum height=0.6cm, right=0.3 of new_state] (hrm) {\textsc{hrm}};
			\draw [style=arrow] (new_state) -- (hrm);
			
			\node [below=0.3 of hrm] (h_state) {$\hrmstate$}; 
			\draw [style=arrow] (h_state) -- (hrm);
			
			\node [right=0.3 of hrm] (h_state_new) {$\hrmstate'$}; 
			\draw[style=arrow] (hrm) -- (h_state_new);
			
			\node [fill=white, draw=black, shape=rectangle, minimum width=2cm, minimum height=0.5cm, right=0.3 of h_state_new] (update) {Update $\opth$};
			\draw[style=arrow] (h_state_new) -- (update);
			\draw[style=arrow] (h_state) -| (update);
			\draw [style=arrow] ([xshift=-6.0]new_state.south) -- +(0,-1.05) -| ([xshift=10]update.south);
			
			\node [fill=white, draw=black, shape=rectangle, minimum width=1.6cm, minimum height=0.5cm, above=0.2 of hrm] (buffer) {$\mathcal{D}$};
			\node [fill=white, draw=black, shape=rectangle, minimum width=1.6cm, minimum height=0.5cm, above=-0.02 of buffer] (q_functions) {};
			\node [shape=rectangle, minimum height=0.5cm, anchor=north west] (qphi1) at (q_functions.north west) {$\qfunc_\context$};
			\node [shape=rectangle, minimum height=0.5cm, anchor=north east] (qphi2) at (q_functions.north east) {$\qfunc_{\context'}$};
			\node at ($(qphi1)!0.5!(qphi2)$) {$\cdots$};
			
			\draw [style=arrow] (action.north) |- coordinate[pos=0.89] (random_point_edge) (buffer);
			
			\path let \p1 = (random_point_edge), \p2=([xshift=-6.0]new_state.north) in coordinate (new_state_intersect) at (\x2,\y1);
			\draw [style=arrow] ([xshift=-6.0]new_state.north) -- (new_state_intersect);
			
			\path let \p1 = (random_point_edge), \p2=([xshift=7.0]new_state.north) in coordinate (label_intersect) at (\x2,\y1);
			\draw [style=arrow] ([xshift=7.0]new_state.north) -- (label_intersect);
			
			
			\node [fill=white, draw=black, shape=rectangle, minimum width=2.2cm, minimum height=0.5cm, above right=0.3 of update] (meta_q_buff) {};
			\node [fill=white, draw=black, shape=rectangle, minimum width=2.2cm, minimum height=0.5cm, above=-0.02 of meta_q_buff] (meta_q_functions) {};
			\node [shape=rectangle, minimum height=0.5cm, anchor=north west] (metaq1) at (meta_q_functions.north west) {$\qfunc_0$};
			\node [shape=rectangle, minimum height=0.5cm, anchor=north east] (metaq2) at (meta_q_functions.north east) {$\qfunc_{m-1}$};
			\node [shape=rectangle, minimum height=0.5cm, anchor=north west] (metabuff1) at (meta_q_buff.north west) {$\mathcal{D}_0$};
			\node [shape=rectangle, minimum height=0.5cm, anchor=north east] (metabuff2) at (meta_q_buff.north east) {$\mathcal{D}_{m-1}$};
			\node at ($(metaq1)!0.45!(metaq2)$) {$\cdots$};
			\node (dots_buffer) at ($(metabuff1)!0.45!(metabuff2)$) {$\cdots$};
			
			\draw [style=arrow] (update) -| ([yshift=0.7]metabuff1.south);
			\draw [style=arrow] (update) -| ([yshift=1.5]metabuff2.south);
			
			\node [fill=white, below=0.35of dots_buffer] (dots_arrows) {$\cdots$};
			
			\node [fill=white, below=-0.1of dots_arrows] (terminated_text) {\footnotesize terminated};
			\node [fill=white, below=-0.1of terminated_text] {\footnotesize options};
		\end{tikzpicture}
		\caption{\textbf{Apply the Action.} The selected action $\mdpaction$ is applied in the environment, producing as a result a tuple $\obstuple'$ and a label $\proplabel$. The tuple $(\obstuple,\mdpaction,\obstuple',\proplabel)$ is added to the buffer $\mathcal{D}$ and the DQNs $q_\context,\ldots,q_{\context'}$ for the formulas in the HRM are updated. Given the new label and the current hierarchy state $\hrmstate$, the HRM determines the new hierarchy state $\hrmstate'$. The hierarchy states are used to update the option stack $\opth$, the terminated options' experiences are pushed into the corresponding buffers, and the associated DQNs $q_0,\ldots,q_{m-1}$ are updated.}
		\label{fig:policy_learning_algorithm_apply_action}
	\end{subfigure}
	\caption{The core procedures involved in the policy learning algorithm that exploits HRMs.}
	\label{fig:policy_learning_algorithm}
\end{figure*}
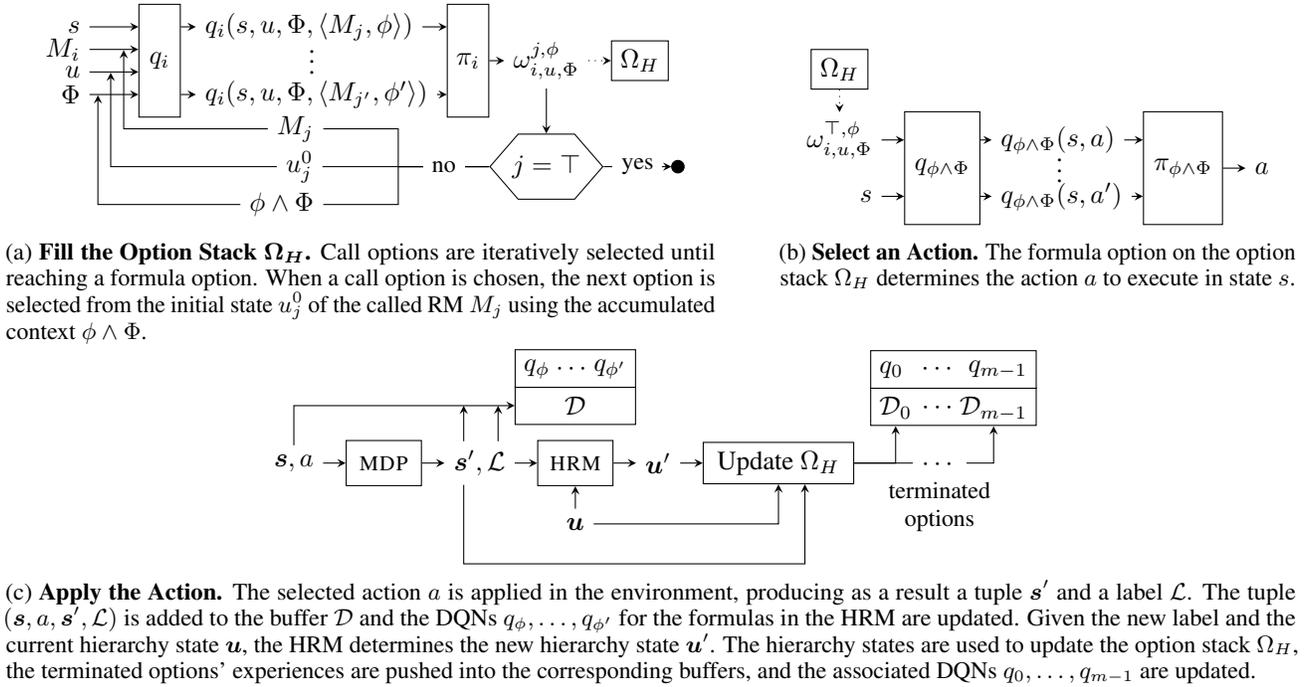

\subsection{Option Selection Algorithm}
\label{app:option_selection_algorithm}
Algorithm~\ref{alg:hrm_episode} shows how options are selected, updated, and interrupted during an episode. Lines \ref{algline:opt:init_start}--\ref{algline:opt:init_end} correspond to the algorithm's initialization. The initial state is that of the environment, while the initial hierarchy state is formed by the root RM $\rmname_r$, its initial state $u_r^0$, an empty context (i.e., $\Context=\top$), and an empty call stack. The option stack $\opth$ contains the options we are currently running, where options at the front are the shallowest ones (e.g.,~the first option in the list is taken in the root RM). The steps taken during an episode are shown in lines \ref{algline:opt:main_loop_start}--\ref{algline:opt:main_loop_end}, which are grouped as follows:
\begin{enumerate}
	\item The agent fills the option stack $\opth$ by selecting options in the HRM from the current hierarchy state until a formula option is chosen (lines~\ref{algline:opt:fill_option_stack_start}--\ref{algline:opt:fill_option_stack_end}). The context is propagated and augmented through the HRM (i.e.,~the context of the calls is conjuncted with the propagating context and converted into DNF form). Note that the context is initially $\top$ (true), and not that of the hierarchy state. It is possible that no new options are selected if the formula option chosen in a previous step has not terminated yet.
	\item The agent chooses an action according to the last option in the option stack (line~\ref{algline:opt:select_action}), which will always be a formula option whose policy maps states into actions. The action is applied, and the agent observes the next state and label (line~\ref{algline:opt:apply_action}). The next hierarchy state is obtained by applying the hierarchical transition function $\hrmtrans$ using the observed label (line~\ref{algline:opt:apply_transition_function}). The Q-functions associated with formula options' policies are updated after this step (line~\ref{algline:opt:update_formula_q}).
	\item The option stack $\opth$ is updated by removing those options that have terminated (lines~\ref{algline:opt:terminate_options_call}, \ref{algline:opt:terminate_options_start}--\ref{algline:opt:terminate_options_end}). The terminated options are saved in a different list $\optterm$ to update the Q-functions of the RMs where they were initiated later on (line~\ref{algline:opt:update_call_q}). The termination of the options is performed as described in Section~\ref{sec:policy_learning}. All options terminate if a terminal state is reached (lines~\ref{algline:opt:terminate_terminal_state_1}--\ref{algline:opt:terminate_terminal_state_2}). Otherwise, we check options in $\opth$ from deeper to shallower levels. The first checked option is always a formula option, which terminates if the hierarchy state has changed (line~\ref{algline:opt:hierarchy_state_change_if}). In contrast, a call option terminates if it does not appear in the stack (lines~\ref{algline:opt:option_stack_call}, \ref{algline:opt:option_stack_start}--\ref{algline:opt:option_stack_end}).\footnote{We denote by $\phi_1 \subseteq \phi_2$, where $\phi_1,\phi_2\in\dnfprop$, the fact that all the disjuncts of $\phi_1$ appear in $\phi_2$. This containment relationship also holds if both formulas are $\top$. For instance, $(a \land \neg c) \subseteq (a \land \neg c) \vee d$.} When an option is found to terminate, it is added to $\optterm$ and removed from $\opth$ (lines~\ref{algline:opt:update_terminated_options}--\ref{algline:opt:remove_last_option}, \ref{algline:opt:update_terminated_options_2}--\ref{algline:opt:remove_last_option_2}). If a non-terminating option is found (lines~\ref{algline:opt:non_terminating_option_found}, \ref{algline:opt:non_terminating_option_found_2}), we stop checking for termination (no higher level options can have terminated in this case).
	\item If at least one option has terminated (line~\ref{algline:opt:term_options_non_empty}), the option stack is updated such that it contains all options appearing in the call stack (lines~\ref{algline:opt:align_stack_call}, \ref{algline:opt:align_stack_start}--\ref{algline:opt:align_stack_end}). Options are derived for the full stack if $\opth$ is empty (lines~\ref{algline:opt:opth_empty_1}--\ref{algline:opt:opth_empty_2}), or for the part of the stack not appearing in $\opth$ (lines~\ref{algline:opt:opth_nonempty_start}--\ref{algline:opt:opth_nonempty_end}). The new derived options (lines~\ref{algline:opt:align_stack_helper_start}--\ref{algline:opt:align_stack_helper_end}) from the call stack are assumed to start in the same state as the last terminated option (i.e.,~the shallowest terminated option, line~\ref{alg:opt:get_shallowest_option}) and to have been run for the same number of steps too. Crucially, the contexts should be propagated accordingly, starting from the context of the last terminated option (line~\ref{alg:opt:new_context_formed}).
	
	As a result of the definition of the hierarchical transition function $\hrmtrans$, the contexts in the stack may be DNF formulas with more than one disjunct. In contrast, the contexts associated with options are either $\top$ or DNFs with a single disjunct (remember that an option is formed for each disjunct). For instance, this occurs if the context is $a\vee b$ and $\{a,b\}$ is observed: since both disjuncts are satisfied, the context shown in the call stack will be the full disjunction $a \vee b$. In the simplest case, the derived option (which as said before is associated with a DNF with a single disjunct or $\top$) can include one of these disjuncts chosen uniformly at random (line~\ref{alg:opt:get_disjunct}). Alternatively, we could memorize all the derived options and perform identical updates for both later on once terminated.
\end{enumerate}
Figure~\ref{fig:policy_learning_algorithm} illustrates the core procedures that constitute the option selection algorithm: (i)~filling the option stack, (ii)~selecting an action using the formula option in the option stack, and (iii)~applying the action and updating the Q-functions and the option stack accordingly.

\begin{algorithm}[!htbp]
	\caption{Episode execution using an HRM (continues on the next page)}
	\label{alg:hrm_episode}
	\begin{algorithmic}[1]
		\REQUIRE an HRM $\hrmtuple$ and an environment $\textsc{Env}=\langle \mdpstates,\mdpactions,\mdpprob,\mdprewfunc,\mdpdiscount,\propset, \lfunc, \mdptermfunc \rangle$.
		\STATE $\obstuple_0 \leftarrow \textsc{Env.Init()}$ \COMMENT{Initial MDP tuple} \alglinelabel{algline:opt:init_start}
		\STATE $\langle \rmname_i, u,\Context,\stack \rangle \leftarrow \langle \rmname_r, u^0_r,\top, []\rangle$ \COMMENT{Initial hierarchy state}
		\STATE $\opth \leftarrow []$ \COMMENT{Initial option stack} \alglinelabel{algline:opt:init_end}
		\FOR{each step $t=0,\ldots,$} \alglinelabel{algline:opt:main_loop_start}
		\STATE $\opth \leftarrow \textsc{FillOptionStack}(\obs_t,\langle \rmname_i, u,\Context,\stack \rangle, \opth)$ \COMMENT{Expand the option stack}
		\STATE $a \leftarrow \textsc{SelectAction}(\obs_t,\opth)$ \COMMENT{Choose $a$ according to the last option in $\opth$} \alglinelabel{algline:opt:select_action}
		\STATE $\obstuple_{t+1}, \proplabel_{t+1} \leftarrow \textsc{Env.ApplyAction}(a)$ \alglinelabel{algline:opt:apply_action}
		\STATE $\langle \rmname_j, u', \Context', \stack'\rangle \leftarrow \hrmtrans(\langle \rmname_i, u,\Context,\stack \rangle, \proplabel_{t+1})$ \COMMENT{Apply transition function} \alglinelabel{algline:opt:apply_transition_function}
		\STATE $\textsc{UpdateFormulaQFunctions}(\obstuple_t,a,\obstuple_{t+1},\proplabel_{t+1})$ \alglinelabel{algline:opt:update_formula_q}
		\STATE $\optterm,\opth \leftarrow \textsc{TerminateOptions}(\opth,\obstuple, \langle \rmname_i,u,\Context,\stack\rangle, \langle \rmname_j,u',\Context',\stack'\rangle)$ \alglinelabel{algline:opt:terminate_options_call}
		\STATE $\textsc{UpdateCallQFunctions}(\optterm,\obs_{t+1},\proplabel_{t+1})$  \alglinelabel{algline:opt:update_call_q}
		\IF {$|\optterm| > 0$} \alglinelabel{algline:opt:term_options_non_empty}
		\STATE $\opth \leftarrow \textsc{AlignOptionStack}(\opth,\stack',\optterm)$ \alglinelabel{algline:opt:align_stack_call}
		\ENDIF
		\STATE $\langle \rmname_i, u,\Context,\stack\rangle \leftarrow \langle\rmname_j, u', \Context',\stack'\rangle$ \alglinelabel{algline:opt:main_loop_end}
		\ENDFOR
		
		\FUNCTION{\textsc{FillOptionStack}($\obs,\langle \rmname_i, u,\cdot,\stack \rangle, \opth$)} \alglinelabel{algline:opt:fill_option_stack_start}
		\STATE $\opth' \leftarrow \opth$
		\STATE $\Context \leftarrow \top$ \COMMENT{The context is initially true}
		\STATE $\rmname_j \leftarrow \rmname_i;v \leftarrow u$ \COMMENT{The state-automaton pair in which an option is selected}
		\WHILE{the last option in $\opth'$ is not a formula option}
		\STATE $\opt^{x,\phi}_{j,v,\Context} \leftarrow \textsc{SelectOption}(\obs,\rmname_j,v,\Context)$ \COMMENT{Select an option (e.g., with $\epsilon$-greedy)}
		\IF[If the option is a call option]{$x \neq \top$}
		\STATE $\rmname_j \leftarrow \rmname_x; v \leftarrow u^0_x$ \COMMENT{Next option is chosen on the called RM's initial state}
		\STATE $\Context \leftarrow \dnf(\Context \land \phi)$ \COMMENT{Update the context}
		\ENDIF
		\STATE $\opth' \leftarrow \opth' \oplus\opt^{x,\phi}_{j,v,\Context}$ \COMMENT{Update the option stack (concatenate new option)}
		\ENDWHILE
		\STATE \textbf{return} $\opth'$ \alglinelabel{algline:opt:fill_option_stack_end}
		\ENDFUNCTION
		
		\FUNCTION{\textsc{TerminateOptions}($\opth,\obstuple,\langle \rmname_i,u,\Context,\stack\rangle,\langle \rmname_j, u',\Context',\stack'\rangle$)} \alglinelabel{algline:opt:terminate_options_start}
		\IF{$\obsterm = \top$} \alglinelabel{algline:opt:terminate_terminal_state_1}
		\STATE \textbf{return} $\opth, []$ \COMMENT{All options terminate} \alglinelabel{algline:opt:terminate_terminal_state_2}
		\ENDIF
		\STATE $\optterm \leftarrow []; \opth' \leftarrow \opth$ \COMMENT{Initialize structures}
		\WHILE[While the option stack is not empty]{$|\opth'| > 0$}
		\STATE $\opt^{x,\phi}_{k,v,\ContextAlt} \leftarrow$ last option in $\opth'$
		\IF[If the option is a call option]{$x \neq \top$}
		\STATE $\textnormal{in\_stack, \_} \leftarrow \textsc{OptionInStack}(\opt^{x,\phi}_{k,v,\ContextAlt},\stack')$ \alglinelabel{algline:opt:option_stack_call}
		\IF{$\neg\textnormal{in\_stack}$}
		\STATE $\optterm \leftarrow \optterm \oplus \opt^{x,\phi}_{k,v,\ContextAlt}$ \COMMENT{Update the list of terminated options} \alglinelabel{algline:opt:update_terminated_options}
		\STATE $\opth' \leftarrow \opth'\ominus \opt^{x,\phi}_{k,v,\ContextAlt}$ \COMMENT{Remove the last option from the option stack} \alglinelabel{algline:opt:remove_last_option}
		\ELSE \alglinelabel{algline:opt:non_terminating_option_found}
		\STATE \textbf{break} \COMMENT{Stop terminating}
		\ENDIF
		\ELSE 
		\IF[If the hierarchy state has changed\ldots]{$\langle \rmname_i, u, \Context,\stack\rangle \neq \langle \rmname_j, u', \Context',\stack'\rangle$} \alglinelabel{algline:opt:hierarchy_state_change_if}
		\STATE $\optterm \leftarrow \optterm \oplus \opt^{x,\phi}_{k,v,\ContextAlt}$ \COMMENT{Update the list of terminated options} \alglinelabel{algline:opt:update_terminated_options_2}
		\STATE $\opth' \leftarrow \opth'\ominus\opt^{x,\phi}_{k,v,\ContextAlt}$ \COMMENT{Remove the last option from the option stack} \alglinelabel{algline:opt:remove_last_option_2}
		\ELSE \alglinelabel{algline:opt:non_terminating_option_found_2}
		\STATE \textbf{break} \COMMENT{Stop terminating}
		\ENDIF
		\ENDIF
		\ENDWHILE 
		\STATE \textbf{return} $\optterm, \opth'$ \alglinelabel{algline:opt:terminate_options_end}
		\ENDFUNCTION
		
		\FUNCTION{\textsc{OptionInStack}($\opt^{x,\phi}_{k,v,\Context}, \stack$)} \alglinelabel{algline:opt:option_stack_start}
		\FOR{$l=0\ldots|\stack|-1$}
		\STATE $\langle u_f, \cdot, \rmname_i, \rmname_j,\phi',\Context'\rangle \leftarrow \stack_l$
		\IF[The call option is in the call stack]{$u_f = v \land i = k \land j = x \land \phi \subseteq \phi' \land \Context \subseteq \Context'$}
		\STATE \textbf{return} $\top, l$ \COMMENT{Return whether it appears in the stack and the index}
		\ENDIF
		\ENDFOR
		\STATE \textbf{return} $\bot, -1$ \alglinelabel{algline:opt:option_stack_end}
		\ENDFUNCTION
	\end{algorithmic}
\end{algorithm}

\begin{algorithm}[!htbp]
	\label{alg:hrm_episode_2}
	\begin{algorithmic}[1]
		\setcounter{ALC@line}{51}
		\FUNCTION{\textsc{AlignOptionStack}($\opth,\stack,\optterm$)} \alglinelabel{algline:opt:align_stack_start}
		\IF{$|\opth|=0$} \alglinelabel{algline:opt:opth_empty_1}
		\STATE \textbf{return} $\textsc{AlignOptionStackHelper}(\opth,\stack,\optterm,0)$ \alglinelabel{algline:opt:opth_empty_2}
		\ELSE
		\STATE $\opt^{x,\phi}_{k,v,\Context} \leftarrow $ last option in $\opth$ \alglinelabel{algline:opt:opth_nonempty_start}
		\STATE $\textnormal{in\_stack},\textnormal{stack\_index} \leftarrow \textsc{OptionInStack}(\opt^{x,\phi}_{k,v,\Context}, \stack)$
		\IF{in\_stack}
		\STATE \textbf{return} $\textsc{AlignOptionStackHelper}(\opth,\stack,\optterm,\textnormal{stack\_index})$ \alglinelabel{algline:opt:opth_nonempty_end}
		\ENDIF
		\ENDIF
		\STATE \textbf{return} $\opth$
		\ENDFUNCTION
		
		\FUNCTION{\textsc{AlignOptionStackHelper}($\opth,\stack,\optterm,\textnormal{stack\_index}$)} \alglinelabel{algline:opt:align_stack_helper_start}
		\STATE $\opth' \leftarrow \opth$
		\STATE $\opt^{\cdot,\cdot}_{\cdot,\cdot,\Context} \leftarrow \textnormal{last option in } \optterm$ \COMMENT{Shallowest terminated option} \alglinelabel{alg:opt:get_shallowest_option}
		\STATE $\Context' \leftarrow \Context$ \COMMENT{Context initialized from last option}
		\FOR{$l=\textnormal{stack\_index}\ldots|\stack|-1$}
		\STATE $\langle u_f, \cdot, \rmname_i, \rmname_j,\phi,\cdot\rangle \leftarrow \stack_l$
		\STATE $\phi_{sel} \leftarrow$ Select disjunct from $\phi$ (e.g., randomly) \alglinelabel{alg:opt:get_disjunct}
		\STATE $\opth' \leftarrow \opth' \oplus \opt^{j,\phi_{sel}}_{i,u_f, \Context'}$ \COMMENT{Append new option to the option stack}
		\STATE $\Context' \leftarrow \dnf(\Context' \land \phi_{sel})$ \alglinelabel{alg:opt:new_context_formed}
		\ENDFOR
		\STATE \textbf{return} $\opth'$ \alglinelabel{algline:opt:align_stack_end} \alglinelabel{algline:opt:align_stack_helper_end}
		\ENDFUNCTION
	\end{algorithmic}
\end{algorithm}

\paragraph{Examples.}
We briefly describe some examples of how policy learning is performed in the HRM of Figure~\ref{fig:book_hierarchy}. We first enumerate the options in the hierarchy. The formula options are $\opt^{\top,\SugarcaneE}_{1,0,\neg\RabbitE}$, $\opt^{\top,\RabbitE}_{2,0,\top}$, $\opt^{\top,\SugarcaneE}_{1,0,\top}$, $\opt^{\top,\WorkbenchE}_{1,1,\top}$, $\opt^{\top,\WorkbenchE}_{2,1,\top}$, and $\opt^{\top,\TableE}_{0,3,\top}$.  The first option should lead the agent to observe the label $\{\Sugarcane\}$ to satisfy $\Sugarcane \land \neg\Rabbit$.  The Q-functions associated with this set of options are $\qfunc_{\SugarcaneE\land\neg\RabbitE}$, $\qfunc_{\RabbitE}$, $\qfunc_{\SugarcaneE}$, $\qfunc_{\WorkbenchE}$ and $\qfunc_{\TableE}$. Note that $\opt^{\top,\WorkbenchE}_{1,1,\top}$ and $\opt^{\top,\WorkbenchE}_{2,1,\top}$ are both associated with $\qfunc_{\WorkbenchE}$. Conversely, the call options are $\opt^{1,\neg\RabbitE}_{0,0,\top}$, $\opt^{2,\top}_{0,0,\top}$, $\opt^{2,\top}_{0,1,\top}$, and $\opt^{1,\top}_{0,2,\top}$, where the first one achieves its local goal if formula options $\opt^{\top,\SugarcaneE}_{1,0,\neg\RabbitE}$ and $\opt^{\top,\WorkbenchE}_{1,1,\top}$ sequentially achieve theirs. The associated Q-functions are $\qfunc_0$, $\qfunc_1$ and $\qfunc_2$. Note that $\opt^{2,\top}_{0,0,\top}$ and $\opt^{2,\top}_{0,1,\top}$ are both associated with $\qfunc_2$.

We now describe a few steps of the aforementioned option selection algorithm in two scenarios. First, we consider the scenario where all chosen options are run to completion (i.e., until their local goals are achieved):
\begin{enumerate}
	\item The initial hierarchy state is $\langle\rmname_0,u^0_0, \top, []\rangle$ and the option stack $\opth$ is empty. We select options to fill $\opth$. The first option is chosen from $u^0_0$ in $\rmname_0$ using a policy induced by $\qfunc_0$. At this state, the available options are $\opt^{1,\neg\RabbitE}_{0,0,\top}$ and $\opt^{2,\top}_{0,0,\top}$. Let us assume that the former is chosen. Then an option from the initial state of $\rmname_1$ under context $\neg\Rabbit$ is chosen, which can only be $\opt^{\top,\SugarcaneE}_{1,0,\neg\RabbitE}$. Since this option is a formula option (the call is made to $\leaf$), we do not select any more options and the option stack is $\opth=\langle\opt^{1,\neg\RabbitE}_{0,0,\top}, \opt^{\top,\SugarcaneE}_{1,0,\neg\RabbitE}\rangle$.
	
	\item The agent selects options according to the formula option in $\opth$, $\opt^{\top,\SugarcaneE}_{1,0,\neg\RabbitE}$, whose policy is induced by $\qfunc_{\SugarcaneE\land\neg\RabbitE}$. Let us assume that the policy tells the agent to turn right. Since the label at this location is empty, the hierarchy state remains the same; therefore, no options terminate, and the option stack does not change.
	
	\item Let us assume that the agent moves forward twice, thus observing $\{\Sugarcane\}$. The hierarchy state then becomes $\langle\rmname_1, u^1_1, \top, [\langle u^0_0, u^1_0, \rmname_0, \rmname_1, \neg\Rabbit, \top \rangle] \rangle$ (see Appendix~\ref{app:hierarchy_traversal_full_example} for a step-by-step application of the hierarchical transition function). We check which options in $\opth$ have terminated starting from the last chosen one. The formula option $\opt^{\top,\SugarcaneE}_{1,0,\neg\RabbitE}$ terminates because the hierarchy state has changed. In contrast, the call option $\opt^{1,\neg\RabbitE}_{0,0,\top}$ does not terminate since there is an item in the call stack, $\langle u^0_0, u^1_0, \rmname_0, \rmname_1, \neg\Rabbit, \top \rangle$ that can be mapped into it (meaning that the option is running).
	
	\item An experience $(\obstuple, \opt^{\top,\SugarcaneE}_{1,0,\neg\RabbitE}, \obstuple')$ is formed for the terminated option, where $\obstuple$ and $\obstuple'$ are the observed tuples on initiation and termination respectively. This tuple is added to the replay buffer associated with the RM where the option appears, $\mathcal{D}_1$, since it achieved its goal (i.e., a label that satisfied $\Sugarcane\land\neg\Rabbit$ was observed).
	
	\item We align $\opth$ with the new stack. In this case, $\opth$ remains unchanged since its only option can be mapped into an item of the new stack.
	
	\item We start a new step. Since the option stack does not contain a formula option, we select new options from the current hierarchy state according to a policy induced by $\qfunc_1$. In this case, there is a single eligible option: $\opt^{\top,\WorkbenchE}_{1,1,\top}$.
\end{enumerate}

In the second scenario, we observe what occurs when the HRM traversal differs from the options chosen by the agent:
\begin{enumerate}
	\item The initial step is like the one in the previous scenario, but we assume $\opt^{2,\top}_{0,0,\top}$ is selected instead. Then, since this is a call option, an option from the initial state of $\rmname_2$ under context $\top$ is chosen, which can only be $\opt^{\top,\RabbitE}_{2,0,\top}$. The option stack thus becomes $\opth=\langle\opt^{2,\top}_{0,0,\top},\opt^{\top,\RabbitE}_{2,0,\top}\rangle$.
	\item Let us assume that by taking actions according to $\opt^{\top,\RabbitE}_{2,0,\top}$ we end up observing $\{\Sugarcane\}$. Like in the previous scenario, the hierarchy state becomes $\langle\rmname_1, u^1_1, \top, [\langle u^0_0, u^1_0, \rmname_0, \rmname_1, \neg\Rabbit, \top \rangle] \rangle$. We check which options in $\opth$ have terminated. The formula option  $\opt^{\top,\RabbitE}_{2,0,\top}$ terminates since the hierarchy state has changed, and the call option $\opt^{2,\top}_{0,0,\top}$ also terminates since it cannot be mapped into an item of the call stack. Note that these options should intuitively finish since the HRM is being traversed through a path different from that chosen by the agent.
	\item The replay buffers are not updated for these options since they have not achieved their local goals.
	\item We align $\opth$ with the new stack. The only item of the stack $\langle u^0_0, u^1_0, \rmname_0, \rmname_1, \neg\Rabbit, \top \rangle$ can be mapped into option $\opt^{1,\neg\RabbitE}_{0,0,\top}$. We assume that this option starts on the same tuple $\obstuple$ and that it has run for the same number of steps as the last terminated option $\opt^{2,\top}_{0,0,\top}$.
\end{enumerate}

\section{HRM Learning Implementation Details}
\label{app:hierarchy_learning}
In this appendix, we present some implementation details omitted in Section~\ref{sec:hierarchy_learning}. First, we explain the specifics of our curriculum learning mechanism (Appendix~\ref{app:curriculum_learning}). Second, we describe how an HRM is learned from traces using ILASP (Appendix~\ref{app:learning_hrms_ilasp}). Finally, we describe additional  details of the algorithm that interleaves RL and HRM learning  (Appendix~\ref{app:interleaving_algorithm}).

\subsection{Curriculum Learning}
\label{app:curriculum_learning}
We here describe the details of the curriculum learning method described in Section~\ref{sec:hierarchy_learning}. When an episode is completed for $\mdp_{ij}$, $\avgreturn_{ij}$ is updated using the episode's undiscounted return $r$ as $\avgreturn_{ij}\leftarrow \beta \avgreturn_{ij} + (1-\beta)r$, where $\beta\in[0,1]$ is a hyperparameter. A score $\taskinstancescore_{ij}=1 - \avgreturn_{ij}$ is computed from the return and used to determine the probability of selecting tasks and instances. Note that this scoring function, also used in the curriculum method by \citet{AndreasKL17}, assumes that the undiscounted return ranges between 0 and 1 (see Section~\ref{sec:background}). The probability of choosing task $i$ is $\max_j \taskinstancescore_{ij} / \sum_k \max_l \taskinstancescore_{kl}$; that is, the task for which an instance is performing very poorly has a higher probability. Having selected task $i$, the probability of choosing instance $j$ is $\taskinstancescore_{ij}/\sum_k \taskinstancescore_{ik}$, i.e.~instances where performance is worse have a higher probability of being chosen.  The average undiscounted returns $R_{ij}$ for each task-instance pair are periodically updated using the undiscounted return obtained by the greedy policies in a single evaluation episode.

\subsection{Learning an HRM from Traces with ILASP}
\label{app:learning_hrms_ilasp}
We formalize the task of learning an HRM using ILASP~\citep{ILASP_system}, an inductive logic programming system that learns answer set programs (ASP) from examples. We refer the reader to \citet{GelfondK14} for an introduction to ASP, and to \citet{Law18} for ILASP. Our formalization is close to that by \citet{FurelosBlancoLJBR21} for flat finite-state machines. Without loss of generality, as stated in Section~\ref{sec:hierarchy_learning}, we assume that each RM has exactly one accepting and one rejecting state.  

We first describe how HRMs are represented in ASP (Appendix~\ref{app:hrm_asp_repr}), and then explain the encoding of the HRM learning task in ILASP (Appendix~\ref{app:hrm_ilasp_repr}). Finally, we detail the version of ILASP and the flags we use to run it (Appendix~\ref{app:ilasp_hyperparams}).

\subsubsection{Representation of an HRM in Answer Set Programming}
\label{app:hrm_asp_repr}
In this section, we explain how HRMs are represented using Answer Set Programming (ASP). First, we describe how traces are represented. Then, we present how HRMs themselves are represented and also introduce the general rules that describe the behavior of these hierarchies. Finally, we prove the correctness of the representation. We use $\asprepr(X)$ to denote the ASP representation of $X$ (e.g., a trace).

\begin{definition}[ASP representation of a label trace]
	\label{def:asp_trace_representation}
	Given a label trace $\trace=\langle \proplabel_0,\ldots,\proplabel_n\rangle$, $M(\trace)$ denotes the set of ASP facts that describe it:
	\begin{equation*}
		\asprepr(\trace) = \begin{matrix*}[l]
			\left\lbrace \mathtt{label(}p, t). \mid 0 \leq t \leq n, p \in \proplabel_t\right\rbrace \cup \\
			\left\lbrace\mathtt{step}(t). \mid 0 \leq t \leq n \right\rbrace \cup \\
			\left\lbrace\mathtt{ last(}n\mathtt{).}\right\rbrace.
		\end{matrix*}
	\end{equation*}
\end{definition}
The $\mathtt{label}(p,t)$ fact indicates that proposition $p \in \propset$ is observed in step $t$, $\mathtt{step}(t)$ states that $t$ is a step of the trace, and $\mathtt{last}(n)$ indicates that the trace ends in step $n$.

\begin{example}
	The set of ASP facts for the label trace $\trace=\langle\lbrace \Iron\rbrace,\lbrace\rbrace,\lbrace \Table\rbrace\rangle$ is $\asprepr(\trace)=\lbrace \mathtt{label}(\Iron,0).,\allowbreak~\mathtt{label}(\Table,2).,\allowbreak~\mathtt{step}(0).,\allowbreak~\mathtt{step}(1).,\allowbreak~\mathtt{step}(2).,\allowbreak~\mathtt{last}(2).\rbrace$.	
\end{example}

\begin{definition}[ASP representation of an HRM]
	Given an HRM $\hrmtuple$, $\asprepr(\hrm)=\bigcup_{\rmname_i \in \machineset \setminus \{\leaf\}} \asprepr(\rmname_i)$, where:
	\begin{align*}
		\asprepr(\rmname_i) &= \asprepr_\rmstates(\rmname_i)\cup \asprepr_\rmtransition(\rmname_i),\\
		\asprepr_\rmstates(\rmname_i) &= \begin{matrix}
			\left\lbrace\mathtt{state}(u,\rmname_i). \mid u \in \rmstates_i \right\rbrace,
		\end{matrix}
	\end{align*}
	{
		\small
		\begin{equation*}
			\asprepr_\varphi\left(\rmname_i\right)=\left\lbrace\begin{array}{@{}l|c@{}}
				\mathtt{call}(u,u',x+e,\rmname_i,\rmname_j).                                       &  \\
				\bar{\varphi}(u,u',x+e,\rmname_i,\mathtt{T) \codeif not~label(}p_1\mathtt{,T), step(T).} &  \rmname_j \in \machineset, u, u' \in \rmstates_i, \\
				\multicolumn{1}{c|}{\vdots}                                                              & \varphi_i(u,u',\rmname_j) \neq \bot, \\
				\bar{\varphi}(u,u',x+e,\rmname_i,\mathtt{T) \codeif not~label(}p_n\mathtt{,T), step(T).} & x=\sum^{j-1}_{k=0}|\varphi_i(u,u',\rmname_k)|, \\
				\bar{\varphi}(u,u',x+e,\rmname_i,\mathtt{T) \codeif label(}p_{n+1}\mathtt{,T), step(T).} & e \in [1, |\varphi_i(u,u',\rmname_j)|], \\
				\multicolumn{1}{c|}{\vdots}                                                              & \phi_{e} \in \varphi_i(u,u',\rmname_j), \\
				\bar{\varphi}(u,u',x+e,\rmname_i,\mathtt{T) \codeif label(}p_m,\mathtt{T), step(T).}     & \begin{aligned}\phi_{e}  &= p_1 \land \cdots \land p_n\\ &\land \neg p_{n+1} \land \cdots \land \neg p_m\end{aligned}
			\end{array}\right\rbrace.
		\end{equation*}
	}
\end{definition}
Note that each non-leaf RM $\rmname_i$ in the hierarchy is associated with its own set of rules $\asprepr(\rmname_i)$, which are described as follows:
\begin{itemize}
	\item Facts $\mathtt{state}(u,\rmname_i)$ indicate that $u$ is a state of RM $\rmname_i$.
	\item Facts $\mathtt{call}(u,u',e,\rmname_i,\rmname_j)$ indicate that edge $e$ between states $u$ and $u'$ in RM $\rmname_i$ is labeled with a call to RM $\rmname_j$.
	\item Normal rules whose \emph{head} is of the form $\bar{\varphi}(u,u',e,\rmname_i, \mathtt{T})$ indicate that the transition from state $u$ to $u'$ with edge $e$ in RM $\rmname_i$ does not hold at step $\mathtt{T}$. The \emph{body} of these rules consists of a single $\mathtt{label}(p, \mathtt{T})$ literal and a $\mathtt{step(T)}$ atom indicating that $\mathtt{T}$ is a step. Commonly, variables are represented using upper case letters in ASP, which is the case of steps $\mathtt{T}$ here.
\end{itemize}
There are some important things to take into account regarding the encoding: 
\begin{itemize}
	\item There is no leaf RM $\leaf$. We later introduce the ASP rules to emulate it.
	\item The edge identifiers $e$ between a given pair of states $(u,u')$ range from 1 to the total number of conjunctions/disjuncts between them. Note that in $\asprepr_\rmtransition$ we assume that the leaf RM has an index, just like the other RMs in the HRM. The index could be $n$ since the rest are numbered from 0 to $n-1$.
\end{itemize}

\begin{example}
	The following rules represent the HRM in Figure~\ref{fig:book_hierarchy}:
	\begingroup
	\allowdisplaybreaks
	\begin{align*}
		\scriptsize
		\begin{matrix*}[l]
			\begin{Bmatrix*}[l]
				\mathtt{state}(u^0_0,\rmname_0).~ \mathtt{state}(u^1_0,\rmname_0).~ \mathtt{state}(u^2_0,\rmname_0). &
				\mathtt{state}(u^3_0,\rmname_0).~
				\mathtt{state}(u^A_0,\rmname_0). \\
				\mathtt{call}(u^0_0, u^1_0, 1, \rmname_0, \rmname_1).~\mathtt{call}(u^0_0, u^2_0, 1, \rmname_0, \rmname_2). & \mathtt{call}(u^1_0, u^3_0, 1, \rmname_0, \rmname_2).~\mathtt{call}(u^2_0, u^3_0, 1, \rmname_0, \rmname_1). \\
				\mathtt{call}(u^3_0, u^A_0, 1, \rmname_0, \rmname_\top). & \bar{\varphi}(u^0_0,u^1_0,1,\rmname_0,\mathtt{T}) \codeif \mathtt{label}(\Rabbit, \mathtt{T),step(T).}
				\\
				\bar{\varphi}(u^3_0,u^A_0,1,\rmname_0,\mathtt{T}) \codeif \mathtt{not~label}(\Table, \mathtt{T),step(T).}
			\end{Bmatrix*} \cup\\
			\begin{Bmatrix*}[l]
				\mathtt{state}(u^0_1,\rmname_1).~ \mathtt{state}(u^1_1,\rmname_1).~ \mathtt{state}(u^A_1,\rmname_1). & \mathtt{call}(u^0_1, u^1_1, 1, \rmname_1, \rmname_\top).~\mathtt{call}(u^1_1, u^A_1, 1, \rmname_1, \rmname_\top).\\
				\bar{\varphi}(u^0_1,u^1_1,1,\rmname_1,\mathtt{T}) \codeif \mathtt{not~label}(\Sugarcane, \mathtt{T),step(T).} & \bar{\varphi}(u^1_1,u^A_1,1,\rmname_1,\mathtt{T}) \codeif \mathtt{not~label}(\Workbench, \mathtt{T),step(T).}
			\end{Bmatrix*} \cup\\
			\begin{Bmatrix*}[l]
				\mathtt{state}(u^0_2,\rmname_2).~ \mathtt{state}(u^1_2,\rmname_2).~ \mathtt{state}(u^A_2,\rmname_2). & \mathtt{call}(u^0_2, u^1_2, 1, \rmname_2, \rmname_\top).~\mathtt{call}(u^1_2, u^A_2, 1, \rmname_2, \rmname_\top).\\
				\bar{\varphi}(u^0_2,u^1_2,1,\rmname_2,\mathtt{T}) \codeif \mathtt{not~label}(\Rabbit, \mathtt{T),step(T).} & \bar{\varphi}(u^1_2,u^A_2,1,\rmname_2,\mathtt{T}) \codeif \mathtt{not~label}(\Workbench, \mathtt{T),step(T).}
			\end{Bmatrix*}.
		\end{matrix*}
	\end{align*}
	\endgroup
\end{example}

\paragraph{General Rules.} The following sets of rules, whose union is denoted by $\generalrules=\cup_{i=0}^5\generalrules_i$, represent how an HRM functions (e.g., how transitions are taken or the acceptance/rejection criteria). For simplicity, all initial, accepting and rejecting states are denoted by $u^0$, $u^A$ and $u^R$ respectively.

The rule below is the inversion of the negation of the state transition function $\bar{\rmtransition}$. Note that the predicate for $\rmtransition$ includes the called RM $\mathtt{M2}$ as an argument.
\begin{align*}
	\generalrules_0 = \begin{Bmatrix*}[l]
		\mathtt{\rmtransition(X,Y,E,M,M2,T) \codeif not~\bar{\rmtransition}(X,Y,E,M,T), call(X,Y,E,M,M2), step(T).}
	\end{Bmatrix*}.
\end{align*}

The rule set $\generalrules_1$ introduces the $\mathtt{pre\_sat(X,M,T)}$ predicate, which encodes the exit condition presented in Section~\ref{sec:formalism} and indicates whether a call from state $\mathtt{X}$ of RM $\mathtt{M}$ can be started at time $\mathtt{T}$. The first rule corresponds to the base case and indicates that if the leaf $\leaf$ is called then the condition is satisfied if the associated formula is satisfied. The second rule applies to calls to non-leaf RMs, where we need to satisfy the context of the call (like in the base case), and also check whether a call from the initial state of the potentially called RM can be started.
\begin{align*}
	\generalrules_1=
	\begin{Bmatrix*}[l]
		\mathtt{pre\_sat(X,M,T) \codeif \varphi(X,\_,\_,M}, \leaf\mathtt{,T).}\\
		\mathtt{pre\_sat(X,M,T) \codeif \varphi(X,\_,\_,M, M2,T), pre\_sat(}u^0, \mathtt{M2, T), M2!\mathord{=}}\leaf.
	\end{Bmatrix*}.
\end{align*}

The rule set $\generalrules_2$ introduces the $\mathtt{reachable(X,M,TO,T2)}$ predicate, which  indicates that state $\mathtt{X}$ of RM $\mathtt{M}$ is reached between steps $\mathtt{T0}$ and $\mathtt{T2}$. The latter step can also be seen as the step we are currently at. The first fact indicates that the initial state of the root RM is reached from step $\mathtt{0}$ to step $\mathtt{0}$. The second rule indicates that the initial state of a non-root RM is reached from step $\mathtt{T}$ to step $\mathtt{T}$ (i.e., it is reached anytime). The third rule represents the loop transition in the initial state of the root $\rmname_r$: we stay there if no call can be started at $\mathtt{T}$ (i.e., we are not moving in the HRM). The fourth rule is analogous to the third but for the accepting state of the root instead of the initial state. Remember this is the only accepting state in the HRM that does not return control to the calling RM. The fifth rule is also similar to the previous ones: it applies to states reached after $\mathtt{TO}$ that are non-accepting, which excludes looping in initial states of non-root RMs at the time of starting them (i.e., loops are permitted in the initial state of a non-root RM if we can reach it afterwards by going back to it). The last rule indicates that $\mathtt{Y}$ is reached at step $\mathtt{T2}$ in RM $\mathtt{M}$ started at $\mathtt{T0}$ if there is an outgoing transition from the current state $\mathtt{X}$ to $\mathtt{Y}$ at time $\mathtt{T}$ that holds between $\mathtt{T}$ and $\mathtt{T2}$, and state $\mathtt{X}$ has been reached between $\mathtt{T0}$ and $\mathtt{T}$. We will later see how $\delta$ is defined.
\begin{align*}
	\generalrules_2=
	\begin{Bmatrix*}[l]
		\mathtt{reachable}(u^0, \rmname_r, \mathtt{0, 0}).\\
		\mathtt{reachable}(u^0, \mathtt{M, T, T) \codeif state}(u^0, \mathtt{M), M !{=}} \rmname_r,\mathtt{ step(T).}\\
		{
			\arraycolsep=1.4pt
			\begin{matrix*}[l]
				\mathtt{reachable(X, M, T0, T\mathord{+}1)}  \codeif& \mathtt{reachable(X, M, T0, T), not~pre\_sat(X,M,T), } \\
				&\mathtt{step(T),X\mathord{=}}u^0, \mathtt{M}\mathord{=}\rmname_r.
			\end{matrix*}
		}\\
		{
			\arraycolsep=1.4pt
			\begin{matrix*}[l]
				\mathtt{reachable(X, M, T0, T\mathord{+}1)} \codeif& \mathtt{reachable(X, M, T0, T), not~pre\_sat(X, M, T), }\\
				&\mathtt{step(T),X}\mathord{=}u^A\mathtt{, M}\mathord{=}\rmname_r.\\
			\end{matrix*}
		}\\
		{
			\arraycolsep=1.4pt
			\begin{matrix*}[l]		
				\mathtt{reachable(X, M, T0, T\mathord{+}1)} \codeif& \mathtt{reachable(X, M, T0, T), not~pre\_sat(X,M,T), }\\ 
				&\mathtt{step(T),TO\mathord{<}T, X!\mathord{=}}u^A.\\
			\end{matrix*}
		}\\
		\mathtt{reachable(Y, M, T0, T2) \codeif reachable(X, M, T0, T), \delta(X, Y, M, T, T2).}
	\end{Bmatrix*}.
\end{align*}

The rule set $\generalrules_3$ introduces two predicates. The predicate $\mathtt{satisfied(M, T0, TE)}$ indicates that RM $\mathtt{M}$ is satisfied if its accepting state $u^A$ is reached between steps $\mathtt{T0}$ and $\mathtt{TE}$. Likewise, the predicate $\mathtt{failed(M, T0, TE)}$ indicates that RM $\mathtt{M}$ fails if its rejecting state $u^R$ is reached between steps $\mathtt{T0}$ and $\mathtt{TE}$. These two descriptions correspond to the first and third rules. The second rule applies to the leaf RM $\leaf$, which always returns control immediately; thus, it is always satisfied between any two consecutive steps.
\begin{align*}
	\generalrules_3=\begin{Bmatrix*}[l]
		\mathtt{satisfied(M, T0, TE) \codeif reachable(}u^A\mathtt{, M, T0, TE).}\\
		\mathtt{satisfied}(\leaf\mathtt{, T, T\mathord{+}1) \codeif step(T).}\\
		\mathtt{failed(M, T0, TE) \codeif reachable(}u^R\mathtt{, M, T0, TE).}
	\end{Bmatrix*}
\end{align*}

The following set, $\generalrules_4$, encodes multi-step transitions within an RM. The predicate $\mathtt{\delta(X, Y, M, T, T2)}$ expresses that the transition from state $\mathtt{X}$ to state $\mathtt{Y}$ in RM $\mathtt{M}$ is satisfied between steps $\mathtt{T}$ and $\mathtt{T2}$. The first rule indicates that this occurs if the context labeling a call to an RM $\mathtt{M2}$ is satisfied and that RM is also satisfied (i.e., its accepting state is reached) between these two steps. In contrast, the second rule is used for the case in which the rejecting state of the called RM is reached between those steps. In the latter case, we transition to the local rejecting state $u^R$ of $\mathtt{M}$ (i.e., the state we would have transitioned to does not matter). This follows from the assumption that rejecting states are global rejectors (see Section~\ref{sec:formalism}). The idea of this rule is that rejection is propagated bottom-up in the HRM.
\begin{align*}
	\generalrules_4 = \begin{Bmatrix*}[l]
		\mathtt{\delta(X, Y, M, T, T2) \codeif \varphi(X, Y, \_, M, M2, T), satisfied(M2, T, T2).}\\
		\mathtt{\delta(X,} u^R, \mathtt{M, T, T2) \codeif \varphi(X, \_, \_, M, M2, T), failed(M2, T, T2).}
	\end{Bmatrix*}.
\end{align*}

The last set, $\generalrules_5$, encodes the accepting/rejecting criteria. Remember that the $\mathtt{last(T)}$ predicate indicates that $\mathtt{T}$ is the last step of a trace. Therefore, the trace is accepted if the root RM is satisfied from the initial step $\mathtt{0}$ to step $\mathtt{T+1}$ (the step after the last step of the trace, once the final label has been processed). In contrast, the trace is rejected if a rejecting state in the hierarchy is reached between these two same steps.
\begin{align*}
	\generalrules_5=\begin{Bmatrix*}[l]
		\mathtt{accept \codeif last(T), satisfied(}\rmname_r,\mathtt{0, T\mathord{+}1).}\\
		\mathtt{reject \codeif last(T), failed(}\rmname_r\mathtt{, 0, T\mathord{+}1).}
	\end{Bmatrix*}
\end{align*}

Unlike the formalism introduced in Section~\ref{sec:formalism}, this encoding does not use stacks, which would be costly to do. Here we know the trace to be processed and, therefore, the RMs can be evaluated bottom-up; that is, we start evaluating the lowest level RMs first on different subtraces, and the result of this evaluation is used in higher level RMs.

We now prove the correctness of the ASP encoding. To do so, we first introduce what means for an HRM to be valid with respect to a trace, as well as a definition and a theorem due to \citet{GelfondL88} that will help us derive the proof.
\begin{definition}
	Given a label trace $\trace^*$, where $\ast \in \lbrace G,D,I\rbrace$, an HRM $\hrm$ is valid with respect to $\trace^*$ if $\hrm$ accepts $\trace^*$ and $\ast=G$ (i.e., $\trace^*$ is a goal trace), or $\hrm$ rejects $\trace^*$ and $\ast=D$ (i.e., $\trace^*$ is a dead-end trace), or $\hrm$ does not accept nor reject $\trace^*$ and $\ast=I$ (i.e., $\trace^*$ is an incomplete trace).
	\label{def:valid_trace}
\end{definition}
\begin{definition}
	An ASP program $\aspprogram$ is stratified when there is a partition 
	\begin{equation*}
		\aspprogram = \aspprogram_0 \cup \aspprogram_1 \cup \cdots \cup \aspprogram_n
		\tag{$\aspprogram_i$ and $\aspprogram_j$ disjoint for all $i\neq j$}
	\end{equation*}
	such that, (1)~for every predicate $p$, the definition of $p$ (all clauses with $p$ in the head) is contained in one of the partitions $\aspprogram_i$	and, (2) ~for each $1 \leq i \leq n$,  if a predicate occurs positively in a clause of $\aspprogram_i$ then its definition is contained within $\bigcup_{j \leq i} \aspprogram_j$, and if a predicate occurs negatively in a clause of $\aspprogram_i$ then its definition is contained within $\bigcup_{j < i} \aspprogram_j$.
	\label{def:stratified_program}
\end{definition}
\begin{theorem}
	If an ASP program $\aspprogram$ is stratified, then it has a unique answer set.
	\label{theorem:stratification_unique_as}
\end{theorem}

\begin{restatable}[Correctness of the ASP encoding]{proposition}{aspcorrectness}
	Given a finite label trace $\trace^\ast$, where $\ast \in \lbrace G, D, I\rbrace$, and an HRM $\hrmtuple$ that is valid with respect to $\trace^\ast$, the program $\aspprogram=\asprepr(\hrm) \cup \generalrules \cup \asprepr(\trace^\ast)$ has a unique answer set $AS$ and (1) $\mathtt{accept}\in \answerset$ if and only if $\ast = G$, and (2) $\mathtt{reject} \in \answerset$ if and only if $\ast = D$.
	\label{prop:asp_correctness}
\end{restatable}
\begin{proof}
	First, we prove that the program $\aspprogram=\asprepr(\hrm) \cup \generalrules \cup \asprepr(\trace^\ast)$, where $\generalrules = \bigcup_{i=0}^5\generalrules_i$, has a unique answer set. By Theorem~\ref{theorem:stratification_unique_as}, if $\aspprogram$ is stratified then it has a unique answer set. We show there is a way of partitioning $\aspprogram$ following the constraints in Definition~\ref{def:stratified_program}. A possible partition is $\aspprogram=\aspprogram_0 \cup \aspprogram_1 \cup \aspprogram_2 \cup \aspprogram_3$, where $\aspprogram_0=\asprepr(\trace^\ast)$, $\aspprogram_1 = \asprepr(\hrm)$, $\aspprogram_2 = \generalrules_0 \cup \generalrules_1$, $\aspprogram_3 = \generalrules_2 \cup \generalrules_3 \cup \generalrules_4 \cup \generalrules_5$. The unique answer set $\answerset = \answerset_0 \cup \answerset_1 \cup \answerset_2 \cup \answerset_3$, where $\answerset_i$ corresponds to partition $\aspprogram_i$, is shown in Figure~\ref{fig:answer_set_proof}. For simplicity, $\trace^\ast[t]$ denotes the $t$-th label in trace $\trace^\ast$, $\trace^\ast[t:]$ denotes the subtrace starting from the $t$-th label onwards, and $\rmname_i(\trace^\ast)$ denotes the hierarchy traversal using RM $\rmname_i$ as the root.
	
	\begin{figure}
		\begingroup
		\allowdisplaybreaks 
		\small
		\begin{align*}
			\answerset_0 &= \left\lbrace \mathtt{label(}p, t). \mid 0 \leq t \leq n, p \in \proplabel_t\right\rbrace \cup
			\left\lbrace\mathtt{step}(t). \mid 0 \leq t \leq n \right\rbrace \cup
			\left\lbrace\mathtt{ last(}n\mathtt{).}\right\rbrace,\\\\
			\answerset_1 &= \begin{matrix*}[l]
				\left\lbrace\mathtt{state}(u,\rmname_i). \mid \rmname_i \in \machineset\setminus\{\leaf\}, u \in \rmstates_i \right\rbrace\cup\\
				\left\lbrace\begin{array}{@{}l|l@{}}
					\mathtt{call}(u,u',x+e,\rmname_i,\rmname_j). & \rmname_i \in \machineset \setminus \{\leaf\}, \rmname_j \in \machineset, u,u' \in \rmstates_i, \rmtransition_i(u,u',\rmname_j) \neq \bot,\\
					& x = \sum_{k=0}^{j-1}|\rmtransition_i(u,u',\rmname_k)|, e \in [1, |\rmtransition_i(u,u',\rmname_j)|]
				\end{array}\right\rbrace\cup\\
				\left\lbrace\begin{array}{@{}l|l@{}}
					\bar{\varphi}(u,u',x+e,\rmname_i,t). &  0\leq t \leq n, \rmname_i \in \machineset \setminus \{\leaf\}, \rmname_j \in \machineset, u,u' \in \rmstates_i, \\
					& \varphi_i(u,u',\rmname_j) \neq  \bot, x = \sum_{k=0}^{j-1}|\varphi_i(u,u',\rmname_k)|, \\
					& e \in [1, |\varphi_i(u,u',\rmname_j)|],\trace^\ast[t] \not\models \varphi_i(u,u',\rmname_j)[e]
				\end{array}\right\rbrace
			\end{matrix*},\\\\
			\answerset_2 &= \begin{matrix*}[l]
				\left\lbrace\begin{array}{@{}l|l@{}}
					\varphi(u,u',x+e,\rmname_i,t). &  0\leq t \leq n, \rmname_i \in \machineset \setminus \{\leaf\}, \rmname_j \in \machineset, u,u' \in \rmstates_i, \\
					& \varphi_i(u,u',\rmname_j) \neq  \bot, x = \sum_{k=0}^{j-1}|\varphi_i(u,u',\rmname_k)|,  \\
					&e \in [1, |\varphi_i(u,u',\rmname_j)|],\trace^\ast[t] \models \varphi_i(u,u',\rmname_j)[e]
				\end{array}\right\rbrace\cup\\
				\left\lbrace\begin{array}{@{}l|l@{}}
					\mathtt{pre\_sat}(u,\rmname_i,t). & 0 \leq t \leq n, \rmname_i \in \machineset \setminus \{\leaf\}, u \in \rmstates_i, \trace^\ast[t] \models \excond_{i,u,\top}
				\end{array}\right\rbrace
			\end{matrix*},\\\\
			\answerset_3 &= \begin{matrix*}[l]
				\left\lbrace\mathtt{reachable}(u^0,\rmname_r,0,0).  \right\rbrace \cup\\
				\left\lbrace\mathtt{reachable}(u^0,\rmname_i,t,t). \mid  0 \leq t \leq n, \rmname_i \in \machineset \setminus \{\leaf, \rmname_r\}, u^0 \in \rmstates_i \right\rbrace \cup\\	
				\left\lbrace\begin{array}{@{}l|l@{}}
					\mathtt{reachable}(u,\rmname_r,t_1,t_2). & 0 \leq t_1 < t_2 \leq n+1, u \in \rmstates_r,  \\
					&\hrm(\trace^\ast[t_1:])[t_2-t_1] = \langle \rmname_r, u,\cdot, \cdot \rangle
				\end{array}\right\rbrace\cup\\				
				\left\lbrace\begin{array}{@{}l|l@{}}
					\mathtt{reachable}(u,\rmname_i,t_1,t_2). & 0\leq t_1 < t_2 \leq n+1, \rmname_i \in \machineset \setminus \{\rmname_r, \leaf\}, u \in \rmstates_i,\\
					& \trace^\ast[t_1] \models \excond_{i,u^0,\top},\\
					& \rmname_i(\trace^\ast[t_1:])[t_2-t_1] = \langle \rmname_i, u, \cdot, \cdot \rangle, \\
					& \rmname_i(\trace^\ast[t_1:])[t_2-t_1-1] \neq \langle\rmname_i, u^A,  \cdot, \cdot \rangle
				\end{array}\right\rbrace\cup\\
				\left\lbrace \mathtt{satisfied}(\rmname_r,t_1,t_2) \mid 0\leq t_1 < t_2 \leq n+1, \hrm(\trace^\ast[t_1:])[t_2-t_1] = \langle \rmname_r,  u^A, \cdot, \cdot \rangle \right\rbrace\\
				\left\lbrace\begin{array}{@{}l|l@{}}
					\mathtt{satisfied}(\rmname_i,t_1,t_2). & 0\leq t_1 < t_2 \leq n+1, \rmname_i \in \machineset \setminus \{\rmname_r, \leaf\}, \\
					& \trace^\ast[t_1] \models \excond_{i,u^0,\top},\\
					& \rmname_i(\trace^\ast[t_1:])[t_2-t_1] = \langle \rmname_i, u^A, \cdot, \cdot \rangle, \\
					& \rmname_i(\trace^\ast[t_1:])[t_2-t_1-1] \neq \langle \rmname_i, u^A, \cdot, \cdot \rangle
				\end{array}\right\rbrace\cup\\
				\left\lbrace \mathtt{satisfied}(\leaf, t, t+1) \mid 0 \leq t \leq n \right\rbrace \cup\\
				\left\lbrace \mathtt{failed}(\rmname_r,t_1,t_2) \mid 0\leq t_1 < t_2 \leq n+1, \hrm(\trace^\ast[t_1:])[t_2-t_1] = \langle \cdot, u^R,  \cdot, \cdot \rangle \right\rbrace \cup\\
				\left\lbrace\begin{array}{@{}l|l@{}}
					\mathtt{failed}(\rmname_i,t_1,t_2). & 0\leq t_1 < t_2 \leq n+1, \rmname_i \in \machineset \setminus \{\rmname_r, \leaf\}, \\
					& \trace^\ast[t_1] \models \excond_{i,u^0,\top},\\
					& \rmname_i(\trace^\ast[t_1:])[t_2-t_1] = \langle \cdot,u^R,  \cdot, \cdot \rangle
				\end{array}\right\rbrace\cup\\
				\left\lbrace\begin{array}{@{}l|l@{}}
					\delta(u,u',\rmname_i,t,t+1). &  0\leq t \leq n, \rmname_i \in \machineset \setminus \{\leaf\}, u,u' \in \rmstates_i,\\
					& \trace^\ast[t_1] \models \varphi_i(u,u',\leaf)
				\end{array}\right\rbrace\cup\\
				\left\lbrace\begin{array}{@{}l|l@{}}
					\delta(u,u',\rmname_i,t_1,t_2). &  0\leq t_1 < t_2 \leq n+1, \rmname_i \in \machineset \setminus \{\leaf\}, u,u' \in \rmstates_i,\\
					& \exists \rmname_j \in \machineset\setminus\{\leaf\} \textnormal{ s.t. } \phi = \varphi_i(u,u',\rmname_j), \trace^\ast[t_1] \models \excond_{j,u^0,\phi},\\
					&\rmname_j(\trace^\ast[t_1:])[t_2-t_1] = \langle \rmname_j,u^A,\cdot,\cdot\rangle,\\
					& \rmname_j(\trace^\ast[t_1:])[t_2-t_1-1] \neq \langle \rmname_j, u^A, \cdot, \cdot \rangle
				\end{array}\right\rbrace\cup\\
				\left\lbrace\begin{array}{@{}l|l@{}}
					\delta(u,u^R,\rmname_i,t_1,t_2). & \rmname_i \in \machineset \setminus \{\leaf\}, u \in \rmstates_i, 0\leq t_1 < t_2 \leq n+1,\\
					& \exists \rmname_j \in \machineset\setminus\{\leaf\} \textnormal{ s.t. }\phi = \varphi_i(u,u',\rmname_j), \trace^\ast[t_1] \models \excond_{j,u^0,\phi},\\
					&\rmname_j(\trace^\ast[t_1:])[t_2-t_1] = \langle \rmname_k,u^R,\cdot,\cdot\rangle, \rmname_k \in \machineset
				\end{array}\right\rbrace\cup\\
				\left\lbrace \mathtt{accept} \mid \hrm(\trace^\ast)[n+1] = \langle \rmname_r, u^A, \cdot, \cdot \rangle \right\rbrace \cup\\
				\left\lbrace \mathtt{reject} \mid \hrm(\trace^\ast)[n+1] = \langle \rmname_k, u^R, \cdot, \cdot \rangle, \rmname_k \in \machineset\setminus\{\leaf\} \right\rbrace
			\end{matrix*}.
		\end{align*}
		\endgroup
		\caption{Answer sets for each of the partitions in the program $\aspprogram=\asprepr(\hrm) \cup \generalrules \cup \asprepr(\trace^\ast)$, where $H$ is an HRM, $\generalrules$ is the set of general rules and $\trace^\ast$ is a label trace.}
		\label{fig:answer_set_proof}
	\end{figure}
	
	We now prove that $\mathtt{accept} \in \answerset$ if and only if $\ast = G$ (i.e., the trace achieves the goal). If $\ast=G$ then, since the hierarchy is valid with respect to $\trace^\ast$ (see Definition~\ref{def:valid_trace}), the hierarchy traversal $\hrm(\trace^\ast)$ finishes in the accepting state $u^A$ of the root; that is, $\hrm(\trace^\ast)[n+1] = \langle \rmname_r, u^A_r, \cdot, \cdot \rangle$. This holds if and only if $\mathtt{accept} \in \answerset$.
	
	The proof showing that $\mathtt{reject} \in \answerset$ if and only if $\ast =D$ (i.e., the trace reaches a dead-end) is similar to the previous one. If $\ast=D$ then, since the hierarchy is valid with respect to $\trace^\ast$, the hierarchy traversal $\hrm(\trace^\ast)$ finishes in a rejecting state $u^R$; that is, $\hrm(\trace^\ast)[n+1] = \langle \rmname_k, u^R, \cdot, \cdot\rangle$, where $\rmname_k \in \machineset$. This holds if and only if $\mathtt{reject} \in \answerset$.
\end{proof}

\subsubsection{Representation of the HRM Learning Task in ILASP}
\label{app:hrm_ilasp_repr}
We here formalize the learning of an HRM and its mapping to a general ILASP learning task. We start by defining the HRM learning task introduced in Section~\ref{sec:hierarchy_learning}.

\begin{definition}
	An \emph{HRM learning task} is a tuple $\hrmlearningtask=\langle r, \rmstates, \propset, \machineset, \machinesetcall, u^0, u^A, u^R, \traceset, \maxdisjuncts \rangle$, where $r$ is the index of the root RM in the HRM; $\rmstates \supseteq \{u^0,u^A,u^R\}$ is a set of states of the root RM always containing an initial state $u^0$, an accepting state $u^A$, and a rejecting state $u^R$; $\propset$ is a set of propositions; $\machineset \supseteq \{\leaf\}$ is a set of RMs;  $\machinesetcall\subseteq \machineset$ is a set of callable RMs; $\traceset=\traceset^G\cup\traceset^D\cup\traceset^I$ is a set of label traces; and $\maxdisjuncts$ is the maximum number of conjunctions/disjuncts in each formula. An HRM $\hrm=\langle\machineset \cup \{\rmname_r\},\rmname_r,\propset\rangle$ is a solution of $\hrmlearningtask$ if and only if it is valid with respect to all the traces in $\traceset$.
\end{definition}

We make some assumptions about the sets of RMs $\machineset$: (i)~all RMs reachable from RMs in $\machinesetcall$ must be in $\machineset$, (ii)~all RMs in $\machineset$ are deterministic, and (iii)~all RMs in $\machineset$ are defined over the same set of propositions $\propset$ (or a subset of it).

For completeness, we provide the definition of an ILASP task introduced by \citet{LawRB16}. The first definition corresponds to the form of the examples taken by ILASP, while the second corresponds to the ILASP tasks themselves.
\begin{definition}
	A \emph{context-dependent partial interpretation} (CDPI) is a pair $\langle \langle e^{inc}, e^{exc} \rangle,\allowbreak e^{ctx} \rangle$, where $\langle e^{inc}, e^{exc}\rangle$ is a pair of sets of atoms, called a \emph{partial interpretation}, and $e^{ctx}$ is an ASP program called a \emph{context}. A program $\aspprogram$ \emph{accepts} a CDPI $\langle \langle e^{inc}, e^{exc} \rangle, e^{ctx} \rangle$ if and only if there is an answer set $\answerset$ of $\aspprogram \cup e^{ctx}$ such that $e^{inc} \subseteq \answerset$ and $e^{exc}~\cap \answerset = \emptyset$.
\end{definition}

\begin{definition}
	An \emph{ILASP task} is a tuple $\ilasplearningtask=\langle \ilaspbk,\ilasphypspace,\langle \ilaspexamples^+, \ilaspexamples^-\rangle\rangle$ where $\ilaspbk$ is the ASP background knowledge, which describes a set of known concepts before learning; $\ilasphypspace$ is the set of ASP rules allowed in the hypotheses; and $\ilaspexamples^+$ and $\ilaspexamples^-$ are sets of CDPIs called, respectively, the positive and negative examples. A hypothesis $\ilasphypothesis \subseteq \ilasphypspace$ is an \emph{inductive solution} of $\ilasplearningtask$ if and only if (i)~$\forall e \in \ilaspexamples^+$, $\ilaspbk \cup \ilasphypothesis$ accepts $e$, and (ii)~$\forall e \in \ilaspexamples^-$, $\ilaspbk \cup \ilasphypothesis$ does not accept $e$.
\end{definition}

Given an HRM learning task $\hrmlearningtask$, we map it into an ILASP learning task $\asprepr(\hrmlearningtask)=\langle \ilaspbk,\ilasphypspace,\langle \ilaspexamples^+, \emptyset \rangle\rangle$ and use the ILASP system~\citep{ILASP_system} to find an inductive solution $\asprepr_\varphi(\hrm) \subseteq \ilasphypspace$ that covers the examples. Note that we do not use \emph{negative examples} ($\ilaspexamples^-=\emptyset$). We define the components of $\asprepr(\hrmlearningtask)$ below.

\paragraph{Background Knowledge.}
The background knowledge $\ilaspbk=\ilaspbk_\rmstates \cup \ilaspbk_\machineset \cup \generalrules$ is a set of rules that describe the behavior of the HRM. The set $\ilaspbk_\rmstates$ consists of $\mathtt{state}(u, \rmname_r)$ facts for each state $u\in \rmstates$ of the root RM with index $r$ we aim to induce, whereas $\ilaspbk_\machineset = \bigcup_{\rmname_i \in \machineset\setminus \{\leaf\} }\asprepr(\machineset_i)$ contains the ASP representations of all RMs. Finally, $\generalrules$ is the set of general rules introduced in Appendix~\ref{app:hrm_asp_repr} that defines how HRMs process label traces. Importantly, the index of the root $r$ in these rules must correspond to the one used in $\hrmlearningtask$.

\paragraph{Hypothesis Space.}
The hypothesis space $\ilasphypspace$ contains all $\mathtt{ed}$ and $\bar{\rmtransition}$ rules that characterize a transition from a non-terminal state $u \in \rmstates \setminus \{u^A,u^R\}$ to a different state $u' \in \rmstates \setminus \{u\}$ using edge $i\in[1,\maxdisjuncts]$. Formally, it is defined as
\begin{equation*}
	\ilasphypspace=\left\lbrace\begin{array}{@{}l|c@{}}
		\mathtt{call}(u,u',i,\rmname).                                                  & u \in \rmstates\setminus \left\lbrace u^A,u^R \right\rbrace, \\
		\bar{\varphi}(u,u',i,\rmname,\mathtt{T) \codeif label}(p, \mathtt{T), step(T).} & u' \in \rmstates\setminus \left\lbrace u \right\rbrace,  i \in \left[1, \maxdisjuncts \right], \\
		\bar{\varphi}(u,u',i,\rmname,\mathtt{T) \codeif not~label}(p, \mathtt{T), step(T).}  & \rmname \in \machinesetcall, p \in \propset
	\end{array}\right\rbrace.
\end{equation*}

\paragraph{Example Sets.}
Given a set of traces $\traceset=\traceset^G \cup \traceset^D \cup \traceset^I$, the set of \emph{positive examples} is defined as 
\begin{equation*}
	\ilaspexamples^+=\lbrace\langle e^*, \asprepr(\trace) \rangle \mid \ast \in \lbrace G,D,I \rbrace,\trace \in \traceset^* \rbrace,
\end{equation*}
where $e^G=\langle \lbrace\mathtt{accept}\rbrace,\lbrace\mathtt{reject}\rbrace \rangle$, $e^D=\langle \lbrace\mathtt{reject}\rbrace, \lbrace\mathtt{accept}\rbrace \rangle$, and $e^I=\langle \lbrace\rbrace,\lbrace\mathtt{accept, reject}\rbrace \rangle$ are the partial interpretations for goal, dead-end and incomplete traces. The $\mathtt{accept}$ and  $\mathtt{reject}$ atoms express whether a trace is accepted or rejected by the HRM; hence, goal traces must only be accepted, dead-end traces must only be rejected, and incomplete traces cannot be accepted or rejected. Note that the context of each example is the set of ASP facts $\asprepr(\trace)$ that represents the corresponding trace (see Definition~\ref{def:asp_trace_representation}).

\paragraph{Correctness of the Learning Task.} The following theorem captures the correctness of the HRM learning task.
\begin{theorem}
	Given an HRM learning task $\hrmlearningtask=\langle r, \rmstates, \propset,\machineset,\machinesetcall, u^0,u^A,u^R,\traceset,\maxdisjuncts\rangle$, an HRM $\hrm=\langle\machineset \cup \{\rmname_r\}, \rmname_r,\propset \rangle$ is a solution of $\hrmlearningtask$ if and only if $\asprepr_\varphi(\rmname_r)$ is an inductive solution of $\asprepr(\hrmlearningtask)=\langle \ilaspbk, \ilasphypspace, \langle \ilaspexamples^+,\emptyset\rangle\rangle$. 
	\label{theorem:learning_task_correctness}
\end{theorem}
\begin{proof}
	Assume $\hrm$ is a solution of $\hrmlearningtask$.
	
	$\iff$ $\hrm$ is valid with respect to all traces in $\traceset$ (i.e., $\hrm$ accepts all traces in $\traceset^G$, rejects all traces in $\traceset^D$ and does not accept nor reject any trace in $\traceset^I$).
	
	$\iff$ By Proposition~\ref{prop:asp_correctness}, for each trace $\trace^\ast\in\traceset^*$ where $\ast \in \lbrace G,D,I\rbrace$, $\asprepr(\hrm) \cup \generalrules \cup \asprepr(\trace^\ast)$ has a unique answer set $\answerset$ and (1) $\mathtt{accept}\in AS$ if and only if $\ast = G$, and (2) $\mathtt{reject} \in \answerset$ if and only if $\ast = D$.
	
	$\iff$ For each example $e \in \ilaspexamples^+$, $\generalrules\cup \asprepr(\hrm)$ accepts $e$.
	
	$\iff$ For each example $e \in \ilaspexamples^+$, $\ilaspbk\cup \asprepr_\varphi(\rmname_r)$ accepts $e$ (the two programs are identical).
	
	$\iff \asprepr_\varphi(\rmname_r)$ is an inductive solution of $\asprepr(\hrmlearningtask)$.
\end{proof}

\paragraph{Constraints.} 
We introduce several constraints encoding structural properties of the HRMs we want to learn. Some of these constraints are expressed in terms of facts $\mathtt{pos}(u,u',e,m,p)$ and $\mathtt{neg}(u,u',e,m,p)$, which indicate that proposition $p\in\propset$ appears positively (resp.~negatively) in edge $e$ from state $u$ to state $u'$ in RM $\rmname_m$. These facts are derived from $\bar{\varphi}$ rules in $\asprepr(\hrm)$ and injected in the ILASP tasks using meta-program injection~\citep{LawRB18}.

The following set of constraints ensures that the learned root RM is \emph{deterministic} using the \emph{saturation} technique~\citep{EiterG95}. The idea is to check determinism top-down by selecting two edges from a given state in the root, each associated with a set of literals. Initially, the set of literals is formed by those in the formula labeling the edges. If a selected edge calls a non-leaf RM, we select an edge from the initial state of the called RM, augment the set of literals with the associated formula, and repeat the process until a call to the leaf RM is reached. We then check if the literal sets are mutually exclusive. If there is a pair of edges from the root that are not mutually exclusive, the solution is discarded. The set of rules is shown below. The first rule states that we keep two saturation IDs, one for each of the edges we select next and for which mutual exclusivity is checked. The second rule chooses a state $\mathtt{X}$ in the root, whereas the third rule selects two edges from this state and assigns a saturation ID to each of them. The fourth rule indicates that if one of the edges we have selected so far calls a non-leaf RM, we select one of the edges from the initial state of the called RM and create a new edge with the same saturation ID. The fifth and sixth rules take the propositions for each set of edges (one per saturation ID). The next three rules indicate that if the edges are mutually exclusive (i.e., a proposition appears positively in one set and negatively in the other) or they are the same, then the answer set is saturated. The saturation itself is encoded in the following three rules: an answer set is saturated by adding every possible $\mathtt{ed\_mtx}$ and $\mathtt{root\_point}$ atoms to the answer set. Due to the minimality of answer sets in disjunctive answer set programming, this ``maximal'' interpretation can only be an answer set if there is no smaller answer set. This will be the case if and only if every choice of edges satisfies the condition (i.e. every choice of $\mathtt{ed\_mtx}$ and $\mathtt{root\_point}$ atoms results in saturation). The constraint encoded in the final rule then discards answer sets in which saturation did not occur, meaning that the remaining solutions must satisfy the condition.
\begin{align*}
	\small
	\begin{Bmatrix*}[l]
		\mathtt{sat\_id(1;2).}\\
		\mathtt{root\_point(X, M) : call(X, \_, \_, M, \_), M\mathord{=}}\rmname_r.\\
		\mathtt{ed\_mtx((X, Y, E, M, M2), SatID) : call(X, Y, E, M, M2) \codeif root\_point(X, M), sat\_id(SatID).}\\
		{
			\arraycolsep=1.4pt
			\begin{matrix*}[l]
				\mathtt{ed\_mtx((u0, Y2, E2, M2, M3), SatID) : call(u0, Y2, E2, M2, M3) \codeif}&\mathtt{ ed\_mtx((\_, \_, \_, \_, M2), SatID),}\\
				&\mathtt{M2!\mathord{=}}\leaf.
			\end{matrix*}
		}\\
		\mathtt{pos\_prop(P, ID) \codeif ed\_mtx((X, Y, E, M, \_), ID), pos(X, Y, E, M, P).}\\
		\mathtt{neg\_prop(P, ID) \codeif ed\_mtx((X, Y, E, M, \_), ID), neg(X, Y, E, M, P).}\\
		\mathtt{saturate \codeif pos\_prop(P, 1), neg\_prop(P, 2).}\\
		\mathtt{saturate \codeif pos\_prop(P, 2), neg\_prop(P, 1).}\\
		\mathtt{saturate \codeif ed\_mtx((X, Y, \_, M, M2), 1), ed\_mtx((X, Y, \_, M, M2), 2), root\_point(X, M).}\\
		\mathtt{root\_point(X, M) \codeif call(X, \_, \_, M, \_), saturate, M\mathord{=}}\rmname_r.\\
		\mathtt{ed\_mtx((X, Y, E, M, M2), SatID) \codeif call(X, Y, E, M, M2), M\mathord{=}}\rmname_r\mathtt{, sat\_id(SatID), saturate.}\\
		\mathtt{ed\_mtx((u0, Y, E, M, M2), SatID) \codeif call(u0, Y, E, M, M2), sat\_id(SatID), saturate.}\\
		\mathtt{\codeif not~saturate.}
	\end{Bmatrix*}
\end{align*}

Other required constraints to learn sensible HRMs are shown below. The first rule prevents an edge from being labeled with calls to two different RMs. The second rule prevents edges from being labeled with the same literal both positively and negatively.
\begin{align*}
	\begin{Bmatrix*}[l]
		\mathtt{\codeif call(X, Y, E, M, M2), call(X, Y, E, M, M3), M2!\mathord{=}M3.}\\
		\mathtt{\codeif pos(X, Y, E, M, P), neg(X, Y, E, M, P).}
	\end{Bmatrix*}
\end{align*}

The following constraints are used to speed up the learning of an HRM. First, we extend the \emph{symmetry breaking} method by \citet{FurelosBlancoLJBR21}, originally proposed for flat RMs, to our hierarchical setting. The main advantage of this method is that it accelerates learning without restricting the family of learnable HRMs. Other constraints analogous to those in previous work~\citep{FurelosBlancoLJBR21} that speed up the learning process further are enumerated below. For simplicity, some of these constraints use the auxiliary rule below to define the $\mathtt{ed(X, Y, E, M)}$ predicate, which is equivalent to the $\mathtt{call(X, Y, E, M, M2)}$ predicate but omitting the called RM:
\begin{align*}
	\mathtt{ed(X, Y, E, M) \codeif call(X, Y, E, M, \_).}
\end{align*}
The constraints are the following:
\begin{itemize}
	\item Rule out inductive solutions where an edge calling the leaf $\leaf$ is labeled by a formula formed only by negated propositions. The rule below enforces a proposition to occur positively whenever a proposition appears negatively in an edge calling $\leaf$.
	\begin{align*}
		\mathtt{\codeif neg(X, Y, E, M, \_), not~pos(X, Y, E, M, \_), call(X, Y, E, M,} \leaf).
	\end{align*}
	
	\item Rule out any inductive solution where an edge from $\mathtt{X}$ to $\mathtt{Y}$ with index $\mathtt{E}$ is not labeled by a positive or a negative literal. This rule only applies to calls to the leaf $\leaf$, thus avoiding unconditional transitions.
	\begin{align*}
		\mathtt{\codeif not~pos(X, Y, E, M, \_), not~neg(X, Y, E, M, \_), call(X, Y, E, M}, \leaf).
	\end{align*}
	
	\item Rule out inductive solutions containing states different from the accepting and rejecting states without outgoing edges. In general, these states are not interesting.
	\begin{align*}
		\begin{Bmatrix*}[l]
			\mathtt{has\_outgoing\_edges(X, M) \codeif ed(X, \_, \_, M).}\\
			\mathtt{\codeif state(X, M), not~has\_outgoing\_edges(X, M), X!}\mathord{=}u^A\mathtt{, X!}\mathord{=}u^R.
		\end{Bmatrix*}
	\end{align*}
	
	\item Rule out inductive solutions containing cycles; that is, solutions where two states can be reached from each other. The $\mathtt{path(X,Y,M)}$ predicate indicates there is a directed path (i.e., a sequence of directed edges) from $\mathtt{X}$ to $\mathtt{Y}$ in RM $\mathtt{M}$. The first rule states that there is a path from $\mathtt{X}$ to $\mathtt{Y}$ if there is an edge from $\mathtt{X}$ to $\mathtt{Y}$. The second rule indicates that there is a path from $\mathtt{X}$ to $\mathtt{Y}$ if there is an edge from $\mathtt{X}$ to an intermediate state $\mathtt{Z}$ from which there is a path to $\mathtt{Y}$. Finally, the third rule discards the solutions where $\mathtt{X}$ and $\mathtt{Y}$ can be reached from each other through directed edges.
	\begin{align*}
		\begin{Bmatrix*}[l]
			\mathtt{path(X, Y, M) \codeif ed(X, Y, \_, M).}\\
			\mathtt{path(X, Y, M) \codeif ed(X, Z, \_, M), path(Z, Y, M).}\\
			\mathtt{\codeif path(X, Y, M), path(Y, X, M).}
		\end{Bmatrix*}		
	\end{align*}
\end{itemize}

\subsubsection{ILASP Flags}
\label{app:ilasp_hyperparams}
We use ILASP2~\citep{ILASP_system} to learn the HRMs. For efficiency, the default calls to the underlying ASP solver are modified to be made with the flag \texttt{---opt-mode=ignore}, meaning that non-minimal solutions might be obtained (i.e., solutions involving more rules than needed), so the learned root might contain some unnecessary edges. In practice, the solutions produced by ILASP rarely contain such edges and, if they do, these edges eventually disappear by observing an appropriate counterexample.  We hypothesize that using this flag helps since no optimization is made every time ILASP is called. We highlight that this notion of minimality is not related to that of a minimal RM (i.e., an RM with the fewest number of states) described in Section~\ref{sec:hierarchy_learning}.

\subsection{Interleaving Algorithm}
\label{app:interleaving_algorithm}

	Akin to some methods for learning RMs~\citep{IcarteWKVCM19,FurelosBlancoLJBR21}, we compress label traces by merging consecutive equal labels into a single one; e.g., $\langle\{\},\{\Iron\},\{\Iron\},\{\},\{\},\{\Table\},\{\Cow\}\rangle$ becomes $\langle\{\},\{\Iron\},\{\},\{\Table\},\{\Cow\}\rangle$. Our method does not require traces to be compressed, but performance is enhanced since traces usually become shorter.

\section{Experimental Details}
\label{app:experimental_details}
In this section, we describe the details of the experiments introduced in Section~\ref{sec:experimental_results}. First, we discuss the implementation of the domains, the network architecture we used for each of them, and provide HRM examples for the different tasks~(Appendix~\ref{app:experimental_details_domains}). Second, we provide the list of hyperparameters used in the methods evaluated in this paper (Appendix~\ref{app:hyperparameters}). Finally, we provide all specific results summarized in Section~\ref{sec:experimental_results} (Appendix~\ref{app:extended_results}). All timed experiments ran on 3.40GHz Intel\textsuperscript{\textregistered} Core\texttrademark~i7-6700 processors, while non-timed experiments have also run on 2.90GHz Intel\textsuperscript{\textregistered} Core\texttrademark~i7-10700, 4.20GHz Intel\textsuperscript{\textregistered}Core\texttrademark~i7-7700K, and 3.20GHz Intel\textsuperscript{\textregistered} Core\texttrademark~i7-8700 processors.

\subsection{Domains}
\label{app:experimental_details_domains}

We here describe how the domains we use in our experiments are implemented, the architecture of the DQNs, and provide some example HRMs for the tasks we have considered.

\subsubsection{CraftWorld}
\textbf{Implementation.~} This domain is based on MiniGrid~\citep{Chevalier-Boisvert18}, thus inheriting many of its features. At each step, the agent observes a $W\times H\times 3$ tensor, where $W$ and $H$ are the width and height of the grid. The three channels contain the object IDs, the color IDs, and object state IDs (including the orientation of the agent) respectively. Each of the objects we define (except for the lava $\Lava$, which already existed in MiniGrid) has its own object and color IDs. Before providing the agent with the state, the content of all matrices is scaled between -1 and 1. Note that even though the agent gets a full view of the grid, it is still unaware of the completion degree of a task. Other works have previously used the full view of the grid~\citep{IglCLTZDH19,JiangGR21}.

The grids are randomly generated. In all settings (OP, OPL, FR, FRL), the agent and the objects are randomly assigned an unoccupied position. In the case of FR and FRL, no object occupies a position between rooms or its adjoining positions. There is a single object per object type (i.e., proposition) in OP and OPL, whereas there can be one or two per type in FR and FRL. Finally, there is a single lava location in OPL, which is randomly assigned (like the rest of the propositions), whereas in FRL there are four fixed lava locations placed in the intersections between doors as shown in Figure~\ref{fig:craftworld_frl}.

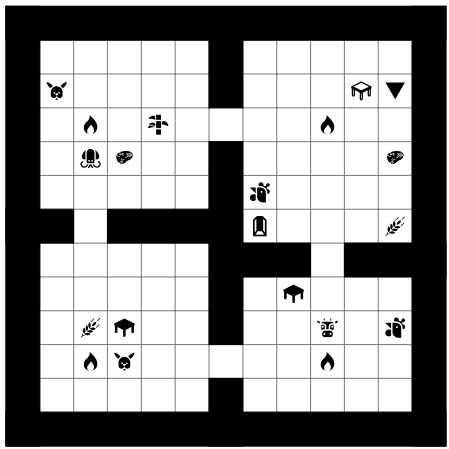
\begin{figure}
	\tikzset{digit/.style = { minimum height = 5mm, minimum width=5mm, anchor=center }}
	\newcommand{\setcell}[3]{\edef\x{#2 - 0.5}\edef\y{12.5 - #1}\node[digit,name={#1-#2}] at (\x, \y) {#3};}
	\centering
	\begin{tikzpicture}[scale=0.45]
		\draw[gray] (0, 0) grid (13, 13);
		
		\draw[black,fill=black] (0, 0) rectangle (1, 13);
		\draw[black,fill=black] (0, 0) rectangle (13, 1);
		\draw[black,fill=black] (12, 0) rectangle (13, 13);
		\draw[black,fill=black] (0, 12) rectangle (13, 13);
		
		\draw[black,fill=black] (6, 0) rectangle (7, 2);
		\draw[black,fill=black] (6, 3) rectangle (7, 9);
		\draw[black,fill=black] (6, 10) rectangle (7, 12);
		
		\draw[black,fill=black] (0, 6) rectangle (2, 7);
		\draw[black,fill=black] (3, 6) rectangle (6, 7);
		\draw[black,fill=black] (7, 5) rectangle (9, 6);
		\draw[black,fill=black] (10, 5) rectangle (13, 6);
		
		\setcell{2}{2}{\Rabbit}
		\setcell{2}{11}{\Workbench}
		\setcell{2}{12}{\Agent}
		
		\setcell{3}{3}{\Lava}
		\setcell{3}{5}{\Sugarcane}
		\setcell{3}{10}{\Lava}
		
		\setcell{4}{3}{\Squid}
		\setcell{4}{4}{\Redstone}
		\setcell{4}{12}{\Redstone}
		
		\setcell{5}{8}{\Chicken}
		
		\setcell{6}{8}{\Iron}
		\setcell{6}{12}{\Wheat}
		
		\setcell{8}{9}{\Table}
		
		\setcell{9}{3}{\Wheat}
		\setcell{9}{4}{\Table}
		\setcell{9}{10}{\Cow}
		\setcell{9}{12}{\Chicken}
		
		\setcell{10}{3}{\Lava}		
		\setcell{10}{4}{\Rabbit}
		\setcell{10}{10}{\Lava}
	\end{tikzpicture}
	\caption{An instance of the \cw grid in the FRL setting.}
	\label{fig:craftworld_frl}
\end{figure}

\textbf{Network Architecture.~}The DQNs for \cw consist of a 3-layer convolutional neural network (CNN) with 16, 32, and 32 filters respectively. All kernels are $2\times 2$ and use a stride of 1. In the FR and FRL settings, there is a max pooling layer with kernel size $2\times 2$ after the first convolutional layer. This part of the architecture is based on that by \citet{IglCLTZDH19} and \citet{JiangGR21}, who also work on MiniGrid using the full view of the grid. In DQNs associated with formulas, the CNN's output is fed to a 3-layer multilayer perceptron (MLP) where the hidden layer has 256 rectifier units and the output layer has a single output for each action. In the case of DQNs for RMs, the output of the CNN is extended with the encoding of the RM state and the context (as discussed in Appendix~\ref{app:option_hierarchy_selection}) before being fed to a 3-layer MLP where the hidden layer has 256 rectifier units and the output layer has a single output for each call in the RM.

\textbf{Examples of HRMs.~} Figure~\ref{fig:all_craftworld_automata} shows minimal root RMs for the \cw tasks listed in Table~\ref{tab:craftworld_tasks}. Note that (i)~since two or more propositions can never occur simultaneously, the mutual exclusivity between formulas could be enforced differently, and (ii)~these RMs correspond to the settings without dead-ends (thus, they do not include rejecting states).

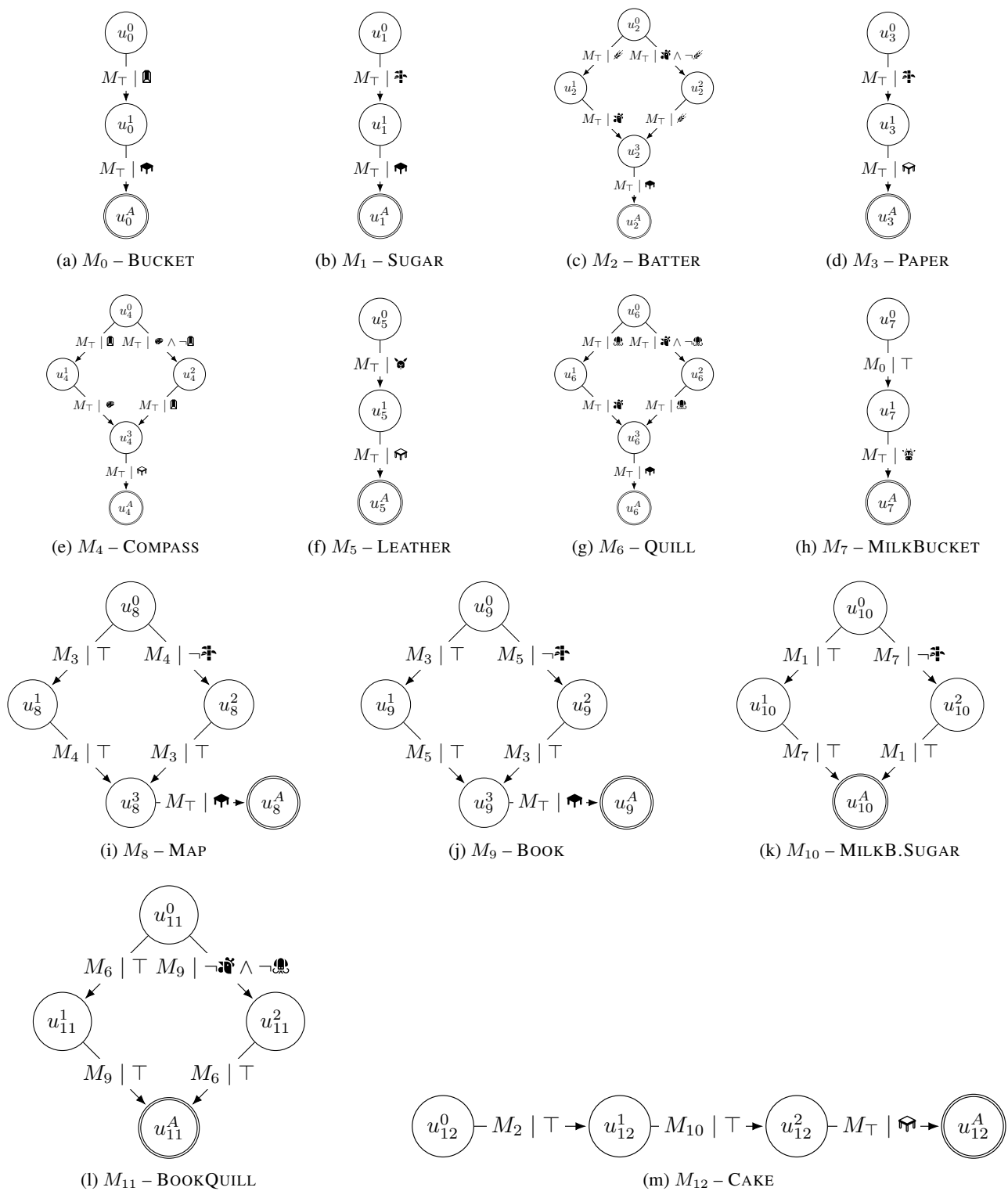
\begin{figure}[h]
	\centering
	\begin{subfigure}[t]{0.245\textwidth}
		\centering
		\resizebox{0.28\linewidth}{!}{
			\begin{tikzpicture}[shorten >=1pt,node distance=2cm,on grid,auto,every initial by arrow/.style ={-Latex} ]
				\node[state] (u_0)   {$u^0_0$};
				\node[state] (u_1) [below =5.5em of u_0]  {$u^1_0$};
				\node[state,accepting] (u_acc) [below =5.5em of u_1]  {$u^A_0$};
				
				\path[-Latex] (u_0) edge node[pos=0.5] [in place] {$\leaf \mid \Iron$} (u_1);
				\path[-Latex] (u_1) edge node[pos=0.5] [in place] {$\leaf \mid \Table$} (u_acc);
			\end{tikzpicture}
		}
		\caption{$\rmname_0$ -- \textsc{Bucket}}
	\end{subfigure} \hfill
	\begin{subfigure}[t]{0.245\textwidth}
		\centering
		\resizebox{0.28\linewidth}{!}{
			\begin{tikzpicture}[shorten >=1pt,node distance=2cm,on grid,auto]
				\node[state] (u_0)   {$u^0_1$};
				\node[state] (u_1) [below =5.5em of u_0]  {$u^1_1$};
				\node[state,accepting] (u_acc) [below =5.5em of u_1]  {$u^A_1$};
				
				\path[-Latex] (u_0) edge node[pos=0.5] [in place] {$\leaf \mid \Sugarcane$} (u_1);
				\path[-Latex] (u_1) edge node[pos=0.5] [in place] {$\leaf \mid \Table$} (u_acc);
			\end{tikzpicture}
		}
		\caption{$\rmname_1$ -- \textsc{Sugar}}
	\end{subfigure}	\hfill
	\begin{subfigure}[t]{0.245\textwidth}
		\centering
		\resizebox{0.67\linewidth}{!}{
			\begin{tikzpicture}[shorten >=1pt,node distance=2cm,on grid,auto]
				\node[state] (u_0)   {$u^0_2$};
				\node[state] (u_1) [below left =7.01em of u_0]  {$u^1_2$};
				\node[state] (u_2) [below right =7.01em of u_0]  {$u^2_2$};
				\node[state] (u_3) [below right =7.01em of u_1]  {$u^3_2$};
				\node[state,accepting] (u_acc) [below =5.5em of u_3]  {$u^A_2$};
				
				\path[-Latex] (u_0) edge node[pos=0.5] [in place] {$\leaf\mid\Wheat$} (u_1);
				\path[-Latex] (u_1) edge node[pos=0.5] [in place] {$\leaf\mid\Chicken$} (u_3);
				
				\path[-Latex] (u_0) edge node[pos=0.5] [in place] {$\leaf\mid\Chicken \land \neg \Wheat$} (u_2);
				\path[-Latex] (u_2) edge node[pos=0.5] [in place] {$\leaf\mid\Wheat$} (u_3);
				
				\path[-Latex] (u_3) edge node[pos=0.5] [in place] {$\leaf\mid\Table$} (u_acc);
			\end{tikzpicture}
		}
		\caption{$\rmname_2$ -- \textsc{Batter}}
	\end{subfigure} \hfill
	\begin{subfigure}[t]{0.245\textwidth}
		\centering
		\resizebox{0.28\linewidth}{!}{
			\begin{tikzpicture}[shorten >=1pt,node distance=2cm,on grid,auto]
				\node[state] (u_0)   {$u^0_3$};
				\node[state] (u_1) [below =5.5em of u_0]  {$u^1_3$};
				\node[state,accepting] (u_acc) [below =5.5em of u_1]  {$u^A_3$};
				
				\path[-Latex] (u_0) edge node[pos=0.5] [in place] {$\leaf\mid\Sugarcane$} (u_1);
				\path[-Latex] (u_1) edge node[pos=0.5] [in place] {$\leaf\mid\Workbench$} (u_acc);
			\end{tikzpicture}
		}
		\caption{$\rmname_3$ -- \textsc{Paper}}
	\end{subfigure}
	
	\vspace{10pt}
	
	\begin{subfigure}[t]{0.245\textwidth}
		\centering
		\resizebox{0.67\linewidth}{!}{
			\begin{tikzpicture}[shorten >=1pt,node distance=2cm,on grid,auto]
				\node[state] (u_0)   {$u^0_4$};
				\node[state] (u_1) [below left =7.01em of u_0]  {$u^1_4$};
				\node[state] (u_2) [below right =7.01em of u_0]  {$u^2_4$};
				\node[state] (u_3) [below right =7.01em of u_1]  {$u^3_4$};
				\node[state,accepting] (u_acc) [below =5.5em of u_3]  {$u^A_4$};
				
				\path[-Latex] (u_0) edge node[pos=0.5] [in place] {$\leaf\mid\Iron$} (u_1);
				\path[-Latex] (u_1) edge node[pos=0.5] [in place] {$\leaf\mid\Redstone$} (u_3);
				
				\path[-Latex] (u_0) edge node[pos=0.5] [in place] {$\leaf\mid\Redstone \land \neg \Iron$} (u_2);
				\path[-Latex] (u_2) edge node[pos=0.5] [in place] {$\leaf\mid\Iron$} (u_3);
				
				\path[-Latex] (u_3) edge node[pos=0.5] [in place] {$\leaf\mid\Workbench$} (u_acc);
			\end{tikzpicture}
		}
		\caption{$\rmname_4$ -- \textsc{Compass}}
	\end{subfigure} \hfill
	\begin{subfigure}[t]{0.245\textwidth}
		\centering
		\resizebox{0.28\linewidth}{!}{
			\begin{tikzpicture}[shorten >=1pt,node distance=2cm,on grid,auto]
				\node[state] (u_0)   {$u^0_5$};
				\node[state] (u_1) [below =5.5em of u_0]  {$u^1_5$};
				\node[state,accepting] (u_acc) [below =5.5em of u_1]  {$u^A_5$};
				
				\path[-Latex] (u_0) edge node[pos=0.5] [in place] {$\leaf\mid\Rabbit$} (u_1);
				\path[-Latex] (u_1) edge node[pos=0.5] [in place] {$\leaf\mid\Workbench$} (u_acc);
			\end{tikzpicture}
		}
		\caption{$\rmname_5$ -- \textsc{Leather}}
	\end{subfigure} \hfill
	\begin{subfigure}[t]{0.245\textwidth}
		\centering
		\resizebox{0.67\linewidth}{!}{
			\begin{tikzpicture}[shorten >=1pt,node distance=2cm,on grid,auto]
				\node[state] (u_0)   {$u^0_6$};
				\node[state] (u_1) [below left =7.01em of u_0]  {$u^1_6$};
				\node[state] (u_2) [below right =7.01em of u_0]  {$u^2_6$};
				\node[state] (u_3) [below right =7.01em of u_1]  {$u^3_6$};
				\node[state,accepting] (u_acc) [below =5.5em of u_3]  {$u^A_6$};
				
				\path[-Latex] (u_0) edge node[pos=0.5] [in place] {$\leaf\mid\Squid$} (u_1);
				\path[-Latex] (u_1) edge node[pos=0.5] [in place] {$\leaf\mid\Chicken$} (u_3);
				\path[-Latex] (u_0) edge node[pos=0.5] [in place] {$\leaf\mid\Chicken\land \neg \Squid$} (u_2);
				\path[-Latex] (u_2) edge node[pos=0.5] [in place] {$\leaf\mid\Squid$} (u_3);
				\path[-Latex] (u_3) edge node[pos=0.5] [in place] {$\leaf\mid\Table$} (u_acc);
			\end{tikzpicture}
		}
		\caption{$\rmname_6$ -- \textsc{Quill}}
	\end{subfigure} \hfill
	\begin{subfigure}[t]{0.245\textwidth}
		\centering
		\resizebox{0.28\linewidth}{!}{
			\begin{tikzpicture}[shorten >=1pt,node distance=2cm,on grid,auto]
				\node[state] (u_0)   {$u^0_7$};
				\node[state] (u_1) [below =5.5em of u_0]  {$u^1_7$};
				\node[state,accepting] (u_acc) [below =5.5em of u_1]  {$u^A_7$};
				
				\path[-Latex] (u_0) edge node[pos=0.5] [in place] {$\rmname_0\mid\top$} (u_1);
				\path[-Latex] (u_1) edge node[pos=0.5] [in place] {$\leaf\mid\Cow$} (u_acc);
			\end{tikzpicture}
		}
		\caption{$\rmname_7$ -- \textsc{MilkBucket}}
	\end{subfigure}
	
	\vspace{10pt}
	
	\begin{subfigure}[t]{0.3\textwidth}
		\centering
		\resizebox{\linewidth}{!}{
			\begin{tikzpicture}[shorten >=1pt,node distance=2cm,on grid,auto]
				\node[state] (u_0)   {$u^0_8$};
				\node[state] (u_1) [below left =7.01em of u_0]  {$u^1_8$};
				\node[state] (u_2) [below right =7.01em of u_0]  {$u^2_8$};
				\node[state] (u_3) [below right =7.01em of u_1]  {$u^3_8$};
				\node[state,accepting] (u_acc) [right =7.2em of u_3]  {$u^A_8$};
				
				\path[-Latex] (u_0) edge node[pos=0.5] [in place] {$\rmname_3 \mid \top$} (u_1);
				\path[-Latex] (u_0) edge node[pos=0.5] [in place] {$\rmname_4 \mid \neg \Sugarcane$} (u_2);
				\path[-Latex] (u_1) edge node[pos=0.5] [in place] {$\rmname_4 \mid \top$} (u_3);
				\path[-Latex] (u_2) edge node[pos=0.5] [in place] {$\rmname_3 \mid \top$} (u_3);
				\path[-Latex] (u_3) edge node[pos=0.45] [in place] {$\leaf\mid\Table$} (u_acc);
			\end{tikzpicture}
		}
		\caption{$\rmname_8$ -- \textsc{Map}}
	\end{subfigure} \hfill
	\begin{subfigure}[t]{0.3\textwidth}
		\centering
		\resizebox{\linewidth}{!}{
			\begin{tikzpicture}[shorten >=1pt,node distance=2cm,on grid,auto]
				\node[state] (u_0)   {$u^0_9$};
				\node[state] (u_1) [below left =7.01em of u_0]  {$u^1_9$};
				\node[state] (u_2) [below right =7.01em of u_0]  {$u^2_9$};
				\node[state] (u_3) [below right =7.01em of u_1]  {$u^3_9$};
				\node[state,accepting] (u_acc) [right =7.2em of u_3]  {$u^A_9$};
				
				\path[-Latex] (u_0) edge node[pos=0.5] [in place] {$\rmname_3\mid\top$} (u_1);
				\path[-Latex] (u_1) edge node[pos=0.5] [in place] {$\rmname_{5}\mid\top$} (u_3);
				\path[-Latex] (u_0) edge node[pos=0.5] [in place] {$\rmname_5\mid\neg \Sugarcane$} (u_2);
				\path[-Latex] (u_2) edge node[pos=0.5] [in place] {$\rmname_{3}\mid\top$} (u_3);
				\path[-Latex] (u_3) edge node[pos=0.45] [in place] {$\leaf\mid\Table$} (u_acc);
			\end{tikzpicture}
		}
		\caption{$\rmname_9$ -- \textsc{Book}}
	\end{subfigure} \hfill
	\begin{subfigure}[t]{0.3\textwidth}
		\centering
		\resizebox{0.85\linewidth}{!}{
			\begin{tikzpicture}[shorten >=1pt,node distance=2cm,on grid,auto]
				\node[state] (u_0)   {$u^0_{10}$};
				\node[state] (u_1) [below left =7.01em of u_0]  {$u^1_{10}$};
				\node[state] (u_2) [below right =7.01em of u_0]  {$u^2_{10}$};
				\node[state,accepting] (u_acc) [below right =7.01em of u_1]  {$u^A_{10}$};
				
				\path[-Latex] (u_0) edge node[pos=0.5] [in place] {$\rmname_1\mid\top$} (u_1);
				\path[-Latex] (u_1) edge node[pos=0.5] [in place] {$\rmname_{7}\mid\top$} (u_acc);
				\path[-Latex] (u_0) edge node[pos=0.5] [in place] {$\rmname_7\mid\neg\Sugarcane$} (u_2);
				\path[-Latex] (u_2) edge node[pos=0.5] [in place] {$\rmname_{1}\mid\top$} (u_acc);
			\end{tikzpicture}
		}
		\caption{$\rmname_{10}$ -- \textsc{MilkB.Sugar}}
	\end{subfigure}
	
	\vspace{10pt}
	
	\begin{subfigure}[t]{0.33\textwidth}
		\centering
		\resizebox{0.85\linewidth}{!}{
			\begin{tikzpicture}[shorten >=1pt,node distance=2.25cm,on grid,auto]
				\node[state] (u_0)   {$u^0_{11}$};
				\node[state] (u_1) [below left =7.01em of u_0]  {$u^1_{11}$};
				\node[state] (u_2) [below right =7.01em of u_0]  {$u^2_{11}$};
				\node[state,accepting] (u_acc) [below right =7.01em of u_1]  {$u^A_{11}$};
				
				\path[-Latex] (u_0) edge node[pos=0.5] [in place] {$\rmname_6\mid\top$} (u_1);
				\path[-Latex] (u_1) edge node[pos=0.5] [in place] {$\rmname_{9}\mid\top$} (u_acc);
				\path[-Latex] (u_0) edge node[pos=0.5] [in place] {$\rmname_9\mid\neg \Chicken \land \neg \Squid$} (u_2);
				\path[-Latex] (u_2) edge node[pos=0.5] [in place] {$\rmname_{6}\mid\top$} (u_acc);
			\end{tikzpicture}
		}
		\caption{$\rmname_{11}$ -- \textsc{BookQuill}}
	\end{subfigure} \hfill
	\begin{subfigure}[t]{0.6\linewidth}
		\centering
		\resizebox{\linewidth}{!}{
			\begin{tikzpicture}[shorten >=1pt,node distance=2.5cm,on grid,auto]
				\node[state] (u_0)   {$u^0_{12}$};
				\node[state] (u_1) [right =7.8em of u_0]  {$u^1_{12}$};
				\node[state] (u_2) [right =7.8em of u_1]  {$u^2_{12}$};
				\node[state,accepting] (u_acc) [right =7.8em of u_2]  {$u^A_{12}$};
				
				\path[-Latex] (u_0) edge node[pos=0.45] [in place] {$\rmname_2\mid\top$} (u_1);
				\path[-Latex] (u_1) edge node[pos=0.45] [in place] {$\rmname_{10}\mid\top$} (u_2);
				\path[-Latex] (u_2) edge node[pos=0.45] [in place] {$\leaf\mid\Workbench$} (u_acc);
			\end{tikzpicture}
		}
		\caption{$\rmname_{12}$ -- \textsc{Cake}}
	\end{subfigure}
	\caption{Root reward machines for each of the \cw tasks.}
	\label{fig:all_craftworld_automata}
\end{figure}

\subsubsection{WaterWorld}
\textbf{Implementation.~} This domain (cf.~Figure~\ref{fig:ww_picture}) has a continuous state space. The states are vectors containing the absolute position and velocity of the agent, and the relative positions and velocities of the other balls. The agent does not know the color of each ball. In all settings (WOD and WD), a \ww instance is created by assigning a random position and direction to each ball. Like in \cw, the agent does not know the degree of completion of a task.

\begin{figure}
	\centering
	\includegraphics[scale=0.9]{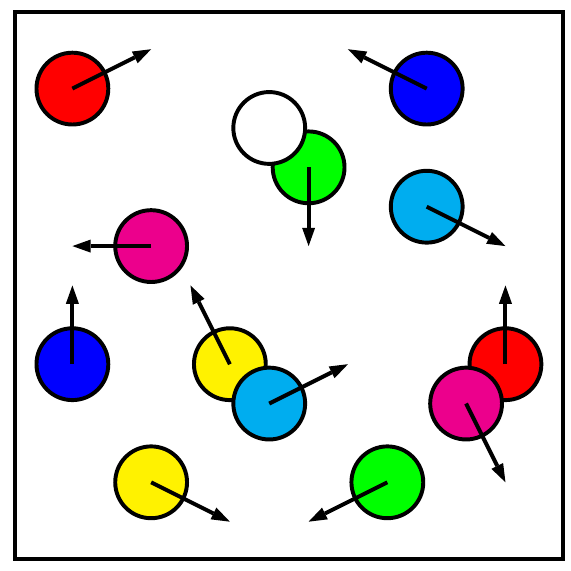}
	\caption{An instance of the \ww \citep{IcarteKVM18} in the WOD setting. Image taken from \cite{FurelosBlancoLJBR21}.}
	\label{fig:ww_picture}
\end{figure}

\textbf{Network Architecture.~}The architecture for \ww is a simple modification of the one introduced by \citet{IcarteKVM18}. The formula DQNs consist of a 5-layer MLP, where each of the 3 hidden layers has 512 rectifier units. The DQN for the RMs share the same architecture and, like in \cw, the state from the environment is extended with the state and context encodings.

\textbf{Examples of HRMs.~} Figure~\ref{fig:all_waterworld_automata} shows minimal root RMs for the tasks listed in Figure~\ref{fig:learning_hrm_and_policy} (right). Note that these RMs correspond to the settings without dead-ends; thus, they do not include rejecting states.

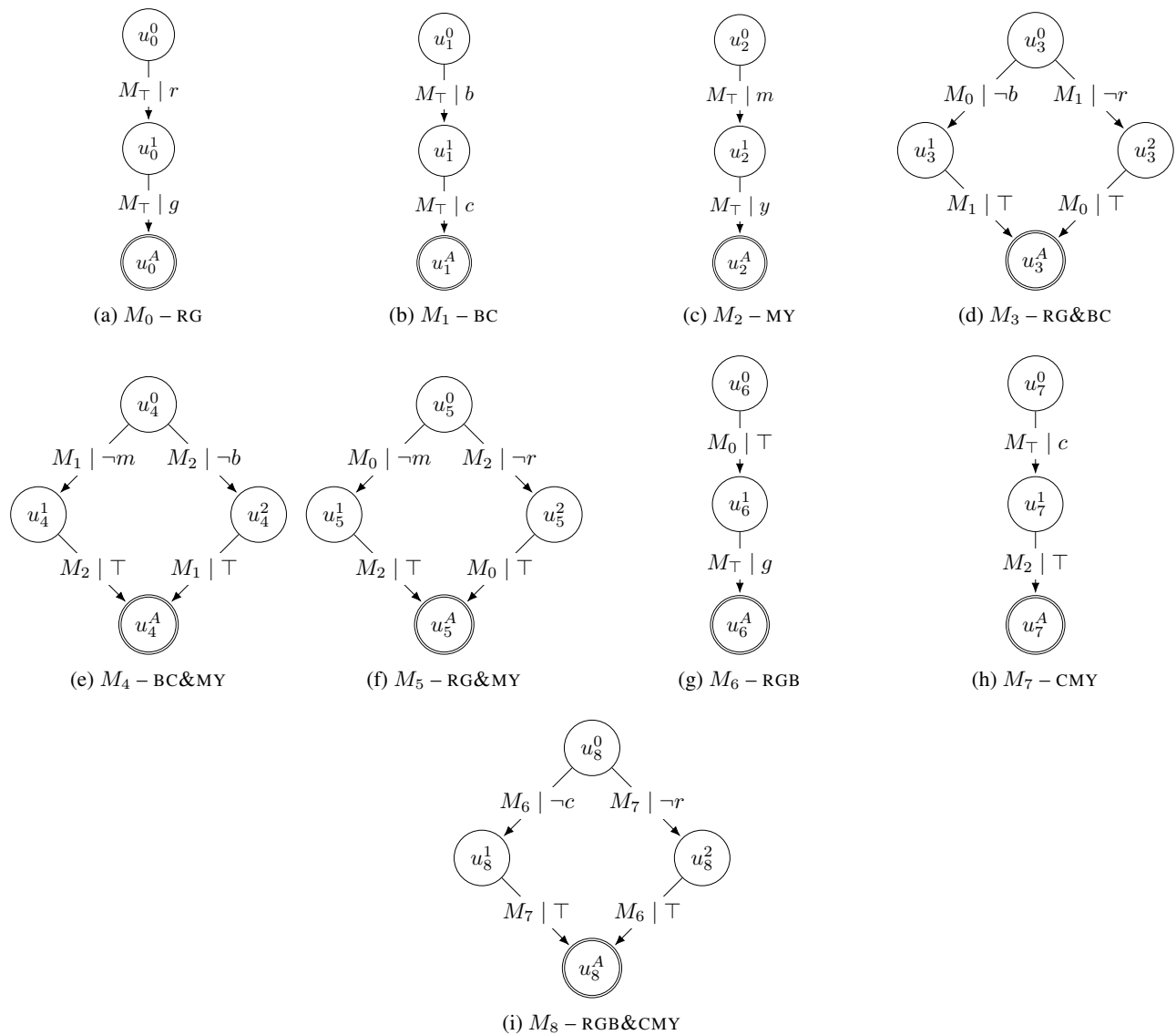
\begin{figure}[h]
	\centering
	\begin{subfigure}[t]{0.245\textwidth}
		\centering
		\resizebox{0.28\linewidth}{!}{
			\begin{tikzpicture}[shorten >=1pt,node distance=2cm,on grid,auto,every initial by arrow/.style ={-Latex} ]
				\node[state] (u_0)   {$u^0_0$};
				\node[state] (u_1) [below =5.5em of u_0]  {$u^1_0$};
				\node[state,accepting] (u_acc) [below =5.5em of u_1]  {$u^A_0$};
				
				\path[-Latex] (u_0) edge node[pos=0.5] [in place] {$\leaf \mid r$} (u_1);
				\path[-Latex] (u_1) edge node[pos=0.5] [in place] {$\leaf \mid g$} (u_acc);
			\end{tikzpicture}
		}
		\caption{$\rmname_0$ -- \textsc{rg}}
	\end{subfigure} \hfill
	\begin{subfigure}[t]{0.245\textwidth}
		\centering
		\resizebox{0.27\linewidth}{!}{
			\begin{tikzpicture}[shorten >=1pt,node distance=2cm,on grid,auto,every initial by arrow/.style ={-Latex} ]
				\node[state] (u_0)   {$u^0_1$};
				\node[state] (u_1) [below =5.5em of u_0]  {$u^1_1$};
				\node[state,accepting] (u_acc) [below =5.5em of u_1]  {$u^A_1$};
				
				\path[-Latex] (u_0) edge node[pos=0.5] [in place] {$\leaf \mid b$} (u_1);
				\path[-Latex] (u_1) edge node[pos=0.5] [in place] {$\leaf \mid c$} (u_acc);
			\end{tikzpicture}
		}
		\caption{$\rmname_1$ -- \textsc{bc}}
	\end{subfigure}	\hfill
	\begin{subfigure}[t]{0.245\textwidth}
		\centering
		\resizebox{0.3\linewidth}{!}{
			\begin{tikzpicture}[shorten >=1pt,node distance=2cm,on grid,auto,every initial by arrow/.style ={-Latex} ]
				\node[state] (u_0)   {$u^0_2$};
				\node[state] (u_1) [below =5.5em of u_0]  {$u^1_2$};
				\node[state,accepting] (u_acc) [below =5.5em of u_1]  {$u^A_2$};
				
				\path[-Latex] (u_0) edge node[pos=0.5] [in place] {$\leaf \mid m$} (u_1);
				\path[-Latex] (u_1) edge node[pos=0.5] [in place] {$\leaf \mid y$} (u_acc);
			\end{tikzpicture}
		}
		\caption{$\rmname_2$ -- \textsc{my}}
	\end{subfigure} \hfill
	\begin{subfigure}[t]{0.245\textwidth}
		\centering
		\resizebox{\linewidth}{!}{
			\begin{tikzpicture}[shorten >=1pt,node distance=2cm,on grid,auto,every initial by arrow/.style ={-Latex} ]
				\node[state] (u_0)   {$u^0_{3}$};
				\node[state] (u_1) [below left =7.01em of u_0]  {$u^1_{3}$};
				\node[state] (u_2) [below right =7.01em of u_0]  {$u^2_{3}$};
				\node[state,accepting] (u_acc) [below right =7.01em of u_1]  {$u^A_{3}$};
				
				\path[-Latex] (u_0) edge node[pos=0.5] [in place] {$\rmname_0\mid\neg b$} (u_1);
				\path[-Latex] (u_1) edge node[pos=0.5] [in place] {$\rmname_{1}\mid\top$} (u_acc);
				\path[-Latex] (u_0) edge node[pos=0.5] [in place] {$\rmname_1\mid\neg r$} (u_2);
				\path[-Latex] (u_2) edge node[pos=0.5] [in place] {$\rmname_{0}\mid\top$} (u_acc);
			\end{tikzpicture}
		}
		\caption{$\rmname_{3}$ -- \textsc{rg\&bc}}
	\end{subfigure}
	
	\vspace{10pt}
	
	\begin{subfigure}[t]{0.245\textwidth}
		\centering
		\resizebox{\linewidth}{!}{
			\begin{tikzpicture}[shorten >=1pt,node distance=2cm,on grid,auto,every initial by arrow/.style ={-Latex} ]
				\node[state] (u_0)   {$u^0_{4}$};
				\node[state] (u_1) [below left =7.01em of u_0]  {$u^1_{4}$};
				\node[state] (u_2) [below right =7.01em of u_0]  {$u^2_{4}$};
				\node[state,accepting] (u_acc) [below right =7.01em of u_1]  {$u^A_{4}$};
				
				\path[-Latex] (u_0) edge node[pos=0.5] [in place] {$\rmname_1\mid\neg m$} (u_1);
				\path[-Latex] (u_1) edge node[pos=0.5] [in place] {$\rmname_{2}\mid\top$} (u_acc);
				\path[-Latex] (u_0) edge node[pos=0.5] [in place] {$\rmname_2\mid\neg b$} (u_2);
				\path[-Latex] (u_2) edge node[pos=0.5] [in place] {$\rmname_{1}\mid\top$} (u_acc);
			\end{tikzpicture}
		}
		\caption{$\rmname_4$ -- \textsc{bc\&my}}
	\end{subfigure} \hfill
	\begin{subfigure}[t]{0.245\textwidth}
		\centering
		\resizebox{\linewidth}{!}{
			\begin{tikzpicture}[shorten >=1pt,node distance=2cm,on grid,auto,every initial by arrow/.style ={-Latex} ]
				\node[state] (u_0)   {$u^0_{5}$};
				\node[state] (u_1) [below left =7.01em of u_0]  {$u^1_{5}$};
				\node[state] (u_2) [below right =7.01em of u_0]  {$u^2_{5}$};
				\node[state,accepting] (u_acc) [below right =7.01em of u_1]  {$u^A_{5}$};
				
				\path[-Latex] (u_0) edge node[pos=0.5] [in place] {$\rmname_0\mid\neg m$} (u_1);
				\path[-Latex] (u_1) edge node[pos=0.5] [in place] {$\rmname_{2}\mid\top$} (u_acc);
				\path[-Latex] (u_0) edge node[pos=0.5] [in place] {$\rmname_2\mid\neg r$} (u_2);
				\path[-Latex] (u_2) edge node[pos=0.5] [in place] {$\rmname_{0}\mid\top$} (u_acc);
			\end{tikzpicture}
		}
		\caption{$\rmname_5$ -- \textsc{rg\&my}}
	\end{subfigure}	\hfill
	\begin{subfigure}[t]{0.245\textwidth}
		\centering
		\resizebox{0.3\linewidth}{!}{
			\begin{tikzpicture}[shorten >=1pt,node distance=2cm,on grid,auto,every initial by arrow/.style ={-Latex} ]
				\node[state] (u_0)   {$u^0_6$};
				\node[state] (u_1) [below =5.5em of u_0]  {$u^1_6$};
				\node[state,accepting] (u_acc) [below =5.5em of u_1]  {$u^A_6$};
				
				\path[-Latex] (u_0) edge node[pos=0.5] [in place] {$\rmname_0 \mid \top$} (u_1);
				\path[-Latex] (u_1) edge node[pos=0.5] [in place] {$\leaf \mid g$} (u_acc);
			\end{tikzpicture}
		}
		\caption{$\rmname_6$ -- \textsc{rgb}}
	\end{subfigure} \hfill
	\begin{subfigure}[t]{0.245\textwidth}
		\centering
		\resizebox{0.3\linewidth}{!}{
			\begin{tikzpicture}[shorten >=1pt,node distance=2cm,on grid,auto,every initial by arrow/.style ={-Latex} ]
				\node[state] (u_0)   {$u^0_7$};
				\node[state] (u_1) [below =5.5em of u_0]  {$u^1_7$};
				\node[state,accepting] (u_acc) [below =5.5em of u_1]  {$u^A_7$};
				
				\path[-Latex] (u_0) edge node[pos=0.5] [in place] {$\leaf \mid c$} (u_1);
				\path[-Latex] (u_1) edge node[pos=0.5] [in place] {$\rmname_2 \mid \top$} (u_acc);
			\end{tikzpicture}
		}
		\caption{$\rmname_7$ -- \textsc{cmy}}
	\end{subfigure}
	
	\vspace{10pt}
	
	\begin{subfigure}[t]{0.245\textwidth}
		\centering
		\resizebox{\linewidth}{!}{
			\begin{tikzpicture}[shorten >=1pt,node distance=2cm,on grid,auto,every initial by arrow/.style ={-Latex} ]
				\node[state] (u_0)   {$u^0_{8}$};
				\node[state] (u_1) [below left =7.01em of u_0]  {$u^1_{8}$};
				\node[state] (u_2) [below right =7.01em of u_0]  {$u^2_{8}$};
				\node[state,accepting] (u_acc) [below right =7.01em of u_1]  {$u^A_{8}$};
				
				\path[-Latex] (u_0) edge node[pos=0.5] [in place] {$\rmname_6\mid\neg c$} (u_1);
				\path[-Latex] (u_1) edge node[pos=0.5] [in place] {$\rmname_{7}\mid\top$} (u_acc);
				\path[-Latex] (u_0) edge node[pos=0.5] [in place] {$\rmname_7\mid\neg r$} (u_2);
				\path[-Latex] (u_2) edge node[pos=0.5] [in place] {$\rmname_{6}\mid\top$} (u_acc);
			\end{tikzpicture}
		}
		\caption{$\rmname_8$ -- \textsc{rgb\&cmy}}
	\end{subfigure}
	
	\caption{Root reward machines for each of the \ww tasks.}
	\label{fig:all_waterworld_automata}
\end{figure}

\subsection{Hyperparameters}
\label{app:hyperparameters}
\textbf{Policy and HRM Learning.}~Table~\ref{tab:hyperparameters} lists the hyperparameters used in the experiments with our approach. We also provide the hyperparameters used for CRM~\citep{IcarteKVM22} to learn policies in flat HRMs. The DQNs for CRM are like those associated with formulas in our approach.

\textbf{Flat HRM Learning Baselines.}~We briefly describe the methods used to learn flat HRMs in Section~\ref{sec:learning_flat_hrms}. Each run consists of \num{150000} episodes, and the set of instances is exactly the same across methods. The core difference with respect to learning non-flat HRMs is that there is a single task for which the HRM is learned. Our method, \learningmethod, is therefore not able to reuse previously learned HRMs for other tasks; however, it still uses the same hyperparameters (see Table~\ref{tab:hyperparameters}). In the case of DeepSynth~\citep{HasanbeigJAMK21}, JIRP~\citep{XuGAMNTW20} and LRM~\citep{IcarteWKVCM19}, we exclusively evaluate their RM learning components using traces collected through random walks.\footnote{The codebases for DeepSynth (\url{https://github.com/grockious/deepsynth}) and LRM (\url{https://bitbucket.org/RToroIcarte/lrm}) are linked in the papers, whereas the one for JIRP (\url{https://github.com/corazza/stochastic-reward-machines}) was referred to us by one of the authors through personal communication.} For a fair comparison against \learningmethod (both in the non-flat and flat learning cases), we (i)~compress the traces using the methodology described in Appendix~\ref{app:interleaving_algorithm}, and (ii)~use the OP and WOD settings of \cw and \ww respectively, where observing goal traces by randomly exploring the environment is relatively easy (especially for simple tasks such as \textsc{MilkBucket}). In these approaches, a different instance is selected at each episode following a cyclic order (i.e.,~$1$, $2$,\ldots, $I-1$, $I$, $1$, $2$, \ldots). The proposition set in these approaches includes a proposition covering the case where none of the original propositions are observed (if needed). In the case of LRM, one of the parameters is the maximum number of RM states, which we set to that of the minimal RM. Akin to other approaches, we modify DeepSynth to call the learner only when a counterexample trace is observed instead of calling it periodically, which repeatedly produced the same RM and resulted in avoidable timeouts.

\begin{table}[h]
	\caption{List of hyperparameters and their values.}
	\label{tab:hyperparameters}
	\centering
	\resizebox{0.9\linewidth}{!}{
		\begin{tabular}{lrr}
			\toprule
			Parameter                                & \cw                        & \ww \\
			\midrule
			\emph{General}\\
			
			~~Episodes                               &                            &  \\ 
			~~~~Without HRM learning                 & \num{100000} (OP, OPL); \num{200000} (FR, FRL)    & \num{100000} (WOD); \num{200000} (WD)   \\ 
			~~~~With HRM learning                    & \num{150000} (OP, OPL); \num{300000} (FR, FRL)    & \num{150000} (WOD); \num{300000} (WD)   \\
			~~Maximum episode length                 & \num{1000}                & \num{1000} \\
			~~Num. of instances $I$                  & \num{10}                  & \num{10} \\
			
			\midrule
			
			\emph{HRM policy learning (Section~\ref{sec:policy_learning}, Appendix~\ref{app:option_hierarchy_selection})} \\				
			
			~~Learning rate $\alpha$                 & \num{5e-4}                 & \num{1e-5} \\
			~~Learning rate (SMDP) $\alpha$          & \num{5e-4}                 & \num{1e-3} \\
			~~Optimizer                              & RMSprop~\citep{HintonSS12} & RMSprop~\citep{HintonSS12}\\
			~~Discount $\mdpdiscount$                & 0.9                        & 0.9 \\
			~~Discount (SMDP) $\mdpdiscount$         & 0.99                       & 0.99 \\
			~~Updated formula Q-functions per step   & 4                          & 4 \\
			
			~~Replay memory size                     & \num{500000}               & \num{500000} \\
			~~Replay start size                      & \num{100000}               & \num{100000}\\
			~~Target network update frequency        & \num{1500}                 & \num{1500} \\
			
			~~Replay memory size (SMDP)              & \num{10000}                & \num{10000} \\
			~~Replay start size (SMDP)               & \num{1000}                 & \num{1000}  \\
			~~Target network update frequency (SMDP) & 500                        & 500 \\
			
			~~Minibatch size                         & 32                         & 32 \\
			
			~~Initial exploration                    & 1.0                        & 1.0 \\
			~~Final exploration                      & 0.1                        & 0.1 \\
			~~Annealing steps                        & \num{2000000}              & \num{5000000}   \\
			~~Annealing steps (SMDP)                 & \num{10000}                & \num{10000}    \\
			
			\midrule
			
			\emph{HRM learning (Section~\ref{sec:hierarchy_learning}, Appendix~\ref{app:hierarchy_learning})} \\
			
			~~Curriculum weight $\beta$                             & 0.99                       & 0.99 \\
			~~Curriculum threshold                                  & 0.85                       & 0.75 \\
			~~Curriculum update frequency (\# episodes)             & 100                        & 100  \\
			~~ILASP time budget                                     & 2 hours                    & 2 hours \\
			~~Num. collected goal traces $\numgoaltracesexp$ (height 1)           & 25                         & 25 \\
			~~Num. collected goal traces $\numgoaltracesexp$ (height $\geq$2)  & 150                        & 150 \\
			~~Num. goal traces $\numshortestgoaltracesexp$ to learn first HRM          & 10                         & 10 \\
			
			\midrule
			
			\emph{CRM~\cite{IcarteKVM22}}\\
			
			~~Learning rate $\alpha$                 & \num{5e-4}                 & \num{1e-5} \\
			~~Optimizer                              & RMSprop~\citep{HintonSS12} & RMSprop~\citep{HintonSS12}\\
			~~Discount                               & 0.99                       & 0.99  \\
			~~Replay memory size                     &  \num{1000000}              & \num{1000000} \\
			~~Replay start size                      & \num{100000}               & \num{100000}\\
			~~Target network update frequency        & \num{1500}                 & \num{1500} \\
			~~Minibatch size                         & 32                         & 32 \\
			
			~~Initial exploration                    & 1.0                        & 1.0 \\
			~~Final exploration                      & 0.1                        & 0.1 \\
			~~Annealing steps           & \num{100000000}           & \num{2000000}  \\
			\bottomrule
	\end{tabular}}
\end{table}

\subsection{Extended Results}
\label{app:extended_results}
We here present the tables and figures on which the discussion in Section~\ref{sec:experimental_results} is based.

\subsubsection{Learning of Non-Flat HRMs}
\label{app:extended_results_non_flat_hrm_learning}
We present tables containing the results for the HRM learning component of \learningmethod. The content of the columns is the following left-to-right: (1)~task name; (2)~number of runs in which at least one goal trace was observed; (3)~number of runs in which at least one HRM was learned; (4)~time spent to learn the HRMs; (5)~number of calls made to ILASP to learn the HRMs; (6)~number of states of the final HRM; (7)~number of edges of the final HRM; (8)~number of episodes between the learning of the first HRM and the activation of the task's level; (9)~number of example traces of a given type (G = goal, D = dead-end, I = incomplete); and (10)~length of the example traces of a given type. In addition, the bottom of the tables contains the number of completed runs (i.e., the number of runs that have not timed out), the total time spent on learning the HRMs, and the total number of calls made to ILASP. In the case of \cw, Table~\ref{tab:non_flat_hrm_default_cw} shows the results for the default case (all lower level RMs are callable and options are used for exploration), Table~\ref{tab:non_flat_hrm_restricted_cw} shows the results when the set of callable RMs contains only those actually needed, and Table~\ref{tab:non_flat_hrm_action_explore_cw} shows the results using primitive actions for exploration instead of options. Analogous results are shown for \ww in Tables~\ref{tab:non_flat_hrm_default_ww}, \ref{tab:non_flat_hrm_restricted_ww} and \ref{tab:non_flat_hrm_action_explore_ww}.

The performance decay for \textsc{rgb\&cmy} observed in Figure~\ref{fig:learning_hrm_and_policy} is due to a new RM for \textsc{rg\&bc} being learned, which is indicated by a vertical line for the latter occurring exactly at the time of the decay. Following our curriculum method (see Section~\ref{sec:hierarchy_learning} and Appendix~\ref{app:curriculum_learning}), the average return for \textsc{rg\&bc} is reset to 0 and the current level is set to 2; hence, the agent stops performing \textsc{rgb\&cmy} (level~3), which causes the performance decay (the reward is 0 while a task is not active). When the average return for \textsc{rg\&bc} is again above the threshold, the agent continues learning \textsc{rgb\&cmy}.

\begin{table}
	\caption{Results of \learningmethod in \cw for the default case.}
	\label{tab:non_flat_hrm_default_cw}
	\begin{subtable}{\linewidth}
		\caption{OP}
		\centering
		\resizebox{0.9\linewidth}{!}{
			\begin{tabular}{lrrrrrrrrrrrr}
				\toprule[1.5pt]
				\multicolumn{1}{c}{Task} & \multicolumn{1}{c}{\# G} & \multicolumn{1}{c}{\# L} & \multicolumn{1}{c}{Time (s.)} & \multicolumn{1}{c}{Calls} & \multicolumn{1}{c}{States} & \multicolumn{1}{c}{Edges} & \multicolumn{1}{c}{Ep. First HRM} & \multicolumn{2}{c}{\# Examples} & \multicolumn{1}{c}{} & \multicolumn{2}{c}{Example Length} \\
				\cmidrule{9-10} \cmidrule{12-13}
				&&&&&&&\multicolumn{1}{c}{$(\times 10^2)$}
				&\multicolumn{1}{c}{G} & \multicolumn{1}{c}{I}&&\multicolumn{1}{c}{G} & \multicolumn{1}{c}{I}\\
				\midrule
				\textsc{Batter} & 5 & 5 & 11.1 (1.7) & 17.8 (1.9) & 5.0 (0.0) & 5.2 (0.2) & 1.8 (0.1) & 12.2 (0.7) & 11.6 (1.4) &  & 26.5 (2.1) & 24.2 (3.2)\\
				\textsc{Bucket} & 5 & 5 & 0.9 (0.0) & 3.6 (0.2) & 3.0 (0.0) & 2.0 (0.0) & 1.7 (0.1) & 10.0 (0.0) & 1.6 (0.2) &  & 19.4 (1.1) & 19.3 (5.7)\\
				\textsc{Compass} & 5 & 5 & 135.4 (73.3) & 18.6 (1.6) & 5.0 (0.0) & 5.2 (0.2) & 1.8 (0.2) & 11.8 (0.6) & 12.8 (1.4) &  & 28.7 (1.9) & 20.3 (2.8)\\
				\textsc{Leather} & 5 & 5 & 0.9 (0.0) & 3.8 (0.2) & 3.0 (0.0) & 2.0 (0.0) & 1.8 (0.1) & 10.0 (0.0) & 1.8 (0.2) &  & 16.7 (1.7) & 17.9 (4.4)\\
				\textsc{Paper} & 5 & 5 & 0.8 (0.1) & 3.4 (0.2) & 3.0 (0.0) & 2.0 (0.0) & 1.6 (0.1) & 10.0 (0.0) & 1.4 (0.2) &  & 19.8 (2.0) & 40.6 (27.0)\\
				\textsc{Quill} & 5 & 5 & 18.0 (3.5) & 19.8 (1.2) & 5.0 (0.0) & 5.2 (0.2) & 2.1 (0.1) & 13.2 (0.4) & 12.6 (1.1) &  & 29.6 (2.5) & 24.4 (3.2)\\
				\textsc{Sugar} & 5 & 5 & 0.8 (0.1) & 3.2 (0.2) & 3.0 (0.0) & 2.0 (0.0) & 1.7 (0.2) & 10.0 (0.0) & 1.2 (0.2) &  & 17.7 (1.6) & 17.5 (3.2)\\
				\textsc{Book} & 5 & 5 & 191.2 (36.4) & 22.8 (2.6) & 5.0 (0.0) & 5.8 (0.2) & 6.0 (0.2) & 11.4 (0.7) & 17.4 (2.2) &  & 20.5 (1.8) & 24.8 (1.5)\\
				\textsc{Map} & 5 & 5 & 549.4 (149.5) & 33.4 (3.2) & 5.0 (0.0) & 5.6 (0.2) & 6.0 (0.2) & 12.2 (0.6) & 27.2 (2.9) &  & 29.5 (3.2) & 28.7 (1.7)\\
				\textsc{MilkBucket} & 5 & 5 & 1.5 (0.2) & 4.6 (0.4) & 3.0 (0.0) & 2.0 (0.0) & 6.8 (0.5) & 10.0 (0.0) & 2.6 (0.4) &  & 11.6 (0.7) & 15.3 (4.3)\\
				\textsc{BookQuill} & 5 & 5 & 17.9 (1.4) & 19.6 (1.1) & 4.0 (0.0) & 4.0 (0.0) & 3.8 (0.1) & 10.0 (0.0) & 16.6 (1.1) &  & 27.2 (1.3) & 20.8 (1.4)\\
				\textsc{MilkB.Sugar} & 5 & 5 & 7.3 (1.2) & 12.4 (1.2) & 4.0 (0.0) & 4.0 (0.0) & 3.8 (0.1) & 10.2 (0.2) & 9.2 (1.2) &  & 16.9 (0.8) & 14.3 (1.7)\\
				\textsc{Cake} & 5 & 5 & 74.5 (25.7) & 26.4 (3.7) & 4.0 (0.0) & 3.2 (0.2) & 2.1 (0.1) & 10.2 (0.2) & 23.2 (3.6) &  & 38.4 (0.9) & 22.7 (1.6)\\
				\midrule[1.5pt]
				\multicolumn{13}{l}{\textbf{Completed Runs} = 5 \hfill \textbf{Total Time (s.)} = 1009.8 (122.3) \hfill \textbf{Total Calls} = 189.4 (4.1)} \\
				\bottomrule[1.5pt]
			\end{tabular}
		}
	\end{subtable}
	
	\vspace*{.1cm}
	
	\begin{subtable}{\linewidth}
		\caption{OPL}
		\centering
		\resizebox{0.9\linewidth}{!}{
			\begin{tabular}{lrrrrrrrrrrrrrr}
				\toprule[1.5pt]
				\multicolumn{1}{c}{Task} & \multicolumn{1}{c}{\# G} & \multicolumn{1}{c}{\# L} & \multicolumn{1}{c}{Time (s.)} & \multicolumn{1}{c}{Calls} & \multicolumn{1}{c}{States} & \multicolumn{1}{c}{Edges} & \multicolumn{1}{c}{Ep. First HRM} & \multicolumn{3}{c}{\# Examples} & \multicolumn{1}{c}{} & \multicolumn{3}{c}{Example Length} \\
				\cmidrule{9-11} \cmidrule{13-15}
				&&&&&&&\multicolumn{1}{c}{$(\times 10^2)$}
				&\multicolumn{1}{c}{G} & \multicolumn{1}{c}{D} & \multicolumn{1}{c}{I}&&\multicolumn{1}{c}{G} & \multicolumn{1}{c}{D} & \multicolumn{1}{c}{I}\\
				\midrule
				\textsc{Batter} & 5 & 5 & 13.7 (2.9) & 23.0 (3.0) & 6.0 (0.0) & 9.2 (0.2) & 12.0 (1.0) & 11.4 (0.4) & 7.0 (1.2) & 10.6 (1.6) &  & 20.4 (1.1) & 18.7 (1.6) & 12.1 (1.7)\\
				\textsc{Bucket} & 5 & 5 & 1.8 (0.2) & 7.2 (0.6) & 4.0 (0.0) & 4.0 (0.0) & 8.0 (0.5) & 10.2 (0.2) & 2.2 (0.2) & 2.8 (0.4) &  & 10.2 (0.5) & 13.4 (1.9) & 6.8 (1.7)\\
				\textsc{Compass} & 5 & 5 & 13.1 (1.7) & 22.0 (1.7) & 6.0 (0.0) & 9.2 (0.2) & 10.4 (1.4) & 11.0 (0.6) & 6.8 (1.0) & 10.2 (1.0) &  & 17.2 (1.6) & 20.9 (1.9) & 14.3 (0.8)\\
				\textsc{Leather} & 5 & 5 & 1.9 (0.2) & 7.0 (0.5) & 4.0 (0.0) & 4.0 (0.0) & 6.9 (0.5) & 10.0 (0.0) & 2.4 (0.2) & 2.6 (0.4) &  & 11.1 (0.9) & 16.9 (5.6) & 8.9 (3.3)\\
				\textsc{Paper} & 5 & 5 & 2.0 (0.2) & 7.6 (0.6) & 4.0 (0.0) & 4.0 (0.0) & 7.7 (1.1) & 10.0 (0.0) & 3.0 (0.3) & 2.6 (0.4) &  & 10.1 (0.9) & 18.9 (3.3) & 5.6 (0.8)\\
				\textsc{Quill} & 5 & 5 & 11.3 (1.2) & 22.0 (1.2) & 6.0 (0.0) & 9.2 (0.2) & 12.8 (1.5) & 10.6 (0.2) & 6.4 (0.7) & 11.0 (0.9) &  & 15.3 (1.3) & 13.5 (1.0) & 12.1 (1.4)\\
				\textsc{Sugar} & 5 & 5 & 1.7 (0.1) & 6.4 (0.4) & 4.0 (0.0) & 4.0 (0.0) & 6.5 (0.7) & 10.0 (0.0) & 2.4 (0.2) & 2.0 (0.3) &  & 9.6 (0.6) & 15.3 (3.6) & 16.6 (9.2)\\
				\textsc{Book} & 5 & 5 & 427.8 (201.6) & 32.6 (4.2) & 6.0 (0.0) & 6.6 (0.2) & 5.6 (0.2) & 12.0 (0.3) & 3.6 (0.7) & 23.0 (3.4) &  & 21.6 (1.5) & 25.9 (3.4) & 23.7 (1.3)\\
				\textsc{Map} & 5 & 5 & 647.9 (110.7) & 38.6 (3.6) & 6.0 (0.0) & 6.4 (0.2) & 5.6 (0.2) & 11.2 (0.4) & 3.8 (0.9) & 29.6 (3.5) &  & 23.1 (1.0) & 27.8 (4.6) & 26.1 (0.4)\\
				\textsc{MilkBucket} & 5 & 5 & 2.1 (0.2) & 5.4 (0.4) & 4.0 (0.0) & 3.0 (0.0) & 7.6 (0.5) & 10.0 (0.0) & 1.4 (0.4) & 2.0 (0.0) &  & 11.1 (0.5) & 26.3 (6.5) & 15.2 (5.8)\\
				\textsc{BookQuill} & 5 & 5 & 18.7 (2.3) & 16.6 (1.3) & 4.0 (0.0) & 4.0 (0.0) & 3.7 (0.2) & 10.0 (0.0) & 0.4 (0.2) & 13.2 (1.4) &  & 29.0 (1.1) & 6.2 (5.5) & 27.8 (1.4)\\
				\textsc{MilkB.Sugar} & 5 & 5 & 7.7 (0.7) & 12.2 (0.9) & 4.0 (0.0) & 4.0 (0.0) & 3.8 (0.2) & 10.0 (0.0) & 0.2 (0.2) & 9.0 (0.9) &  & 16.0 (0.9) & 1.6 (1.6) & 16.3 (1.3)\\
				\textsc{Cake} & 5 & 5 & 472.9 (216.6) & 36.0 (6.0) & 5.0 (0.0) & 4.6 (0.2) & 2.1 (0.0) & 10.0 (0.0) & 1.6 (0.4) & 31.4 (5.7) &  & 39.5 (1.2) & 41.5 (8.6) & 26.9 (0.8)\\
				\midrule[1.5pt]
				\multicolumn{15}{l}{\textbf{Completed Runs} = 5 \hfill \textbf{Total Time (s.)} = 1622.6 (328.7) \hfill \textbf{Total Calls} = 236.6 (9.3)} \\
				\bottomrule[1.5pt]
			\end{tabular}
		}
	\end{subtable}
	
	\vspace*{.1cm}
	
	\begin{subtable}{\linewidth}
		\caption{FR}
		\centering
		\resizebox{0.9\linewidth}{!}{
			\begin{tabular}{lrrrrrrrrrrrr}
				\toprule[1.5pt]
				\multicolumn{1}{c}{Task} & \multicolumn{1}{c}{\# G} & \multicolumn{1}{c}{\# L} & \multicolumn{1}{c}{Time (s.)} & \multicolumn{1}{c}{Calls} & \multicolumn{1}{c}{States} & \multicolumn{1}{c}{Edges} & \multicolumn{1}{c}{Ep. First HRM} & \multicolumn{2}{c}{\# Examples} & \multicolumn{1}{c}{} & \multicolumn{2}{c}{Example Length} \\
				\cmidrule{9-10} \cmidrule{12-13}
				&&&&&&&\multicolumn{1}{c}{$(\times 10^2)$}
				&\multicolumn{1}{c}{G} & \multicolumn{1}{c}{I}&&\multicolumn{1}{c}{G} & \multicolumn{1}{c}{I}\\
				\midrule
				\textsc{Batter} & 5 & 5 & 12.3 (1.7) & 17.6 (1.3) & 5.0 (0.0) & 5.4 (0.2) & 9.2 (1.2) & 11.6 (0.4) & 12.0 (1.2) &  & 30.3 (2.3) & 27.8 (2.0)\\
				\textsc{Bucket} & 5 & 5 & 1.2 (0.1) & 3.8 (0.2) & 3.0 (0.0) & 2.0 (0.0) & 6.7 (0.9) & 10.0 (0.0) & 1.8 (0.2) &  & 16.6 (2.5) & 28.5 (4.1)\\
				\textsc{Compass} & 5 & 5 & 14.1 (1.6) & 20.2 (1.7) & 5.0 (0.0) & 5.2 (0.2) & 9.8 (0.7) & 11.6 (0.6) & 14.6 (1.2) &  & 26.5 (0.8) & 26.5 (2.1)\\
				\textsc{Leather} & 5 & 5 & 1.1 (0.1) & 3.6 (0.2) & 3.0 (0.0) & 2.0 (0.0) & 4.5 (0.7) & 10.0 (0.0) & 1.6 (0.2) &  & 13.4 (1.3) & 16.7 (3.6)\\
				\textsc{Paper} & 5 & 5 & 1.2 (0.0) & 4.0 (0.0) & 3.0 (0.0) & 2.0 (0.0) & 4.9 (0.9) & 10.0 (0.0) & 2.0 (0.0) &  & 12.4 (1.1) & 10.9 (2.5)\\
				\textsc{Quill} & 5 & 5 & 8.9 (0.9) & 16.0 (0.8) & 5.0 (0.0) & 5.2 (0.2) & 9.4 (1.7) & 10.6 (0.2) & 11.4 (0.6) &  & 25.4 (0.3) & 25.5 (2.7)\\
				\textsc{Sugar} & 5 & 5 & 1.1 (0.1) & 3.8 (0.2) & 3.0 (0.0) & 2.0 (0.0) & 5.2 (0.3) & 10.0 (0.0) & 1.8 (0.2) &  & 15.3 (1.7) & 21.0 (10.1)\\
				\textsc{Book} & 5 & 5 & 220.2 (83.3) & 25.2 (3.4) & 5.0 (0.0) & 5.6 (0.2) & 6.1 (0.2) & 10.2 (0.2) & 21.0 (3.4) &  & 21.9 (1.0) & 18.4 (0.7)\\
				\textsc{Map} & 5 & 5 & 628.3 (85.4) & 37.8 (3.7) & 5.0 (0.0) & 5.6 (0.2) & 5.8 (0.1) & 10.0 (0.0) & 33.8 (3.7) &  & 26.4 (1.0) & 21.4 (0.7)\\
				\textsc{MilkBucket} & 5 & 5 & 1.9 (0.2) & 5.0 (0.3) & 3.0 (0.0) & 2.0 (0.0) & 9.8 (0.7) & 10.0 (0.0) & 3.0 (0.3) &  & 13.2 (0.7) & 12.8 (3.2)\\
				\textsc{BookQuill} & 5 & 5 & 12.9 (2.2) & 15.6 (1.7) & 4.0 (0.0) & 4.0 (0.0) & 3.9 (0.1) & 10.0 (0.0) & 12.6 (1.7) &  & 29.0 (1.5) & 13.3 (0.8)\\
				\textsc{MilkB.Sugar} & 5 & 5 & 7.2 (0.6) & 12.0 (0.7) & 4.0 (0.0) & 4.0 (0.0) & 3.9 (0.2) & 10.0 (0.0) & 9.0 (0.7) &  & 18.9 (0.9) & 10.1 (1.0)\\
				\textsc{Cake} & 5 & 5 & 121.1 (41.1) & 34.0 (4.8) & 4.0 (0.0) & 3.0 (0.0) & 2.2 (0.0) & 10.0 (0.0) & 31.0 (4.8) &  & 42.2 (1.7) & 16.2 (1.1)\\
				\midrule[1.5pt]
				\multicolumn{13}{l}{\textbf{Completed Runs} = 5 \hfill \textbf{Total Time (s.)} = 1031.6 (150.3) \hfill \textbf{Total Calls} = 198.6 (11.3)} \\
				\bottomrule[1.5pt]
			\end{tabular}
		}
	\end{subtable}
	
	\vspace*{.1cm}
	
	\begin{subtable}{\linewidth}
		\caption{FRL}
		\centering
		\resizebox{0.9\linewidth}{!}{
			\begin{tabular}{lrrrrrrrrrrrrrr}
				\toprule[1.5pt]
				\multicolumn{1}{c}{Task} & \multicolumn{1}{c}{\# G} & \multicolumn{1}{c}{\# L} & \multicolumn{1}{c}{Time (s.)} & \multicolumn{1}{c}{Calls} & \multicolumn{1}{c}{States} & \multicolumn{1}{c}{Edges} & \multicolumn{1}{c}{Ep. First HRM} & \multicolumn{3}{c}{\# Examples} & \multicolumn{1}{c}{} & \multicolumn{3}{c}{Example Length} \\
				\cmidrule{9-11} \cmidrule{13-15}
				&&&&&&&\multicolumn{1}{c}{$(\times 10^2)$}
				&\multicolumn{1}{c}{G} & \multicolumn{1}{c}{D} & \multicolumn{1}{c}{I}&&\multicolumn{1}{c}{G} & \multicolumn{1}{c}{D} & \multicolumn{1}{c}{I}\\
				\midrule
				\textsc{Batter} & 5 & 5 & 11.3 (1.4) & 23.4 (2.5) & 6.0 (0.0) & 9.2 (0.2) & 468.4 (121.9) & 10.4 (0.2) & 7.6 (0.9) & 11.4 (1.9) &  & 11.9 (0.6) & 10.1 (1.3) & 9.9 (0.4)\\
				\textsc{Bucket} & 5 & 5 & 2.3 (0.2) & 7.0 (0.3) & 4.0 (0.0) & 4.0 (0.0) & 129.5 (69.4) & 10.2 (0.2) & 2.8 (0.2) & 2.0 (0.3) &  & 7.8 (0.5) & 9.9 (1.7) & 6.4 (2.1)\\
				\textsc{Compass} & 5 & 5 & 13.0 (1.9) & 24.6 (2.2) & 6.0 (0.0) & 9.4 (0.2) & 550.8 (156.4) & 10.4 (0.2) & 7.8 (1.0) & 12.4 (1.2) &  & 12.5 (1.6) & 9.4 (1.0) & 8.4 (0.5)\\
				\textsc{Leather} & 5 & 5 & 2.5 (0.3) & 7.8 (0.7) & 4.0 (0.0) & 4.0 (0.0) & 89.0 (18.0) & 10.0 (0.0) & 3.2 (0.4) & 2.6 (0.4) &  & 7.3 (0.4) & 9.3 (1.7) & 3.7 (0.4)\\
				\textsc{Paper} & 5 & 5 & 2.2 (0.1) & 7.0 (0.3) & 4.0 (0.0) & 4.0 (0.0) & 82.7 (18.8) & 10.0 (0.0) & 3.0 (0.0) & 2.0 (0.3) &  & 6.9 (0.7) & 10.2 (1.8) & 4.7 (2.7)\\
				\textsc{Quill} & 5 & 5 & 11.6 (1.1) & 23.8 (1.5) & 6.0 (0.0) & 9.6 (0.2) & 458.9 (61.0) & 10.6 (0.2) & 8.0 (0.9) & 11.2 (1.2) &  & 11.9 (0.6) & 13.1 (2.7) & 9.2 (0.8)\\
				\textsc{Sugar} & 5 & 5 & 2.7 (0.2) & 8.4 (0.7) & 4.0 (0.0) & 4.0 (0.0) & 103.5 (39.5) & 10.0 (0.0) & 3.6 (0.4) & 2.8 (0.5) &  & 8.2 (0.7) & 10.1 (1.9) & 5.0 (1.1)\\
				\textsc{Book} & 5 & 5 & 301.7 (98.1) & 36.4 (1.9) & 6.0 (0.0) & 6.8 (0.2) & 5.3 (0.1) & 10.2 (0.2) & 5.0 (0.7) & 27.2 (1.9) &  & 21.7 (1.1) & 18.8 (2.2) & 16.1 (0.6)\\
				\textsc{Map} & 5 & 5 & 754.1 (158.2) & 44.6 (2.6) & 6.0 (0.0) & 7.0 (0.0) & 5.5 (0.2) & 10.2 (0.2) & 5.2 (0.4) & 35.2 (2.3) &  & 25.6 (0.5) & 20.4 (2.9) & 18.7 (0.6)\\
				\textsc{MilkBucket} & 5 & 5 & 2.8 (0.1) & 6.6 (0.2) & 4.0 (0.0) & 3.0 (0.0) & 6.9 (0.4) & 10.0 (0.0) & 2.0 (0.0) & 2.6 (0.2) &  & 12.5 (0.8) & 13.1 (3.7) & 7.4 (2.2)\\
				\textsc{BookQuill} & 5 & 5 & 19.8 (2.9) & 19.6 (1.6) & 4.0 (0.0) & 4.0 (0.0) & 4.3 (0.1) & 10.0 (0.0) & 0.8 (0.4) & 15.8 (1.2) &  & 28.4 (1.1) & 2.7 (1.3) & 13.5 (0.9)\\
				\textsc{MilkB.Sugar} & 5 & 5 & 8.8 (0.9) & 12.6 (1.0) & 4.0 (0.0) & 4.0 (0.0) & 4.0 (0.1) & 10.0 (0.0) & 1.2 (0.5) & 8.4 (0.7) &  & 19.3 (1.3) & 3.7 (2.0) & 10.7 (2.0)\\
				\textsc{Cake} & 5 & 5 & 344.0 (87.7) & 46.2 (4.9) & 5.0 (0.0) & 4.8 (0.2) & 2.8 (0.1) & 10.0 (0.0) & 2.8 (0.7) & 40.4 (4.5) &  & 44.5 (2.3) & 21.8 (2.2) & 17.3 (1.0)\\
				\midrule[1.5pt]
				\multicolumn{15}{l}{\textbf{Completed Runs} = 5 \hfill \textbf{Total Time (s.)} = 1476.8 (175.3) \hfill \textbf{Total Calls} = 268.0 (6.5)} \\
				\bottomrule[1.5pt]
			\end{tabular}
		}
	\end{subtable}
\end{table}

\begin{table}
	\caption{Results of \learningmethod in \cw with a restricted set of callable RMs.}
	\label{tab:non_flat_hrm_restricted_cw}
	\begin{subtable}{\linewidth}
		\caption{OP}
		\centering
		\resizebox{0.9\linewidth}{!}{
			\begin{tabular}{lrrrrrrrrrrrr}
				\toprule[1.5pt]
				\multicolumn{1}{c}{Task} & \multicolumn{1}{c}{\# G} & \multicolumn{1}{c}{\# L} & \multicolumn{1}{c}{Time (s.)} & \multicolumn{1}{c}{Calls} & \multicolumn{1}{c}{States} & \multicolumn{1}{c}{Edges} & \multicolumn{1}{c}{Ep. First HRM} & \multicolumn{2}{c}{\# Examples} & \multicolumn{1}{c}{} & \multicolumn{2}{c}{Example Length} \\
				\cmidrule{9-10} \cmidrule{12-13}
				&&&&&&&\multicolumn{1}{c}{$(\times 10^2)$}
				&\multicolumn{1}{c}{G} & \multicolumn{1}{c}{I}&&\multicolumn{1}{c}{G} & \multicolumn{1}{c}{I}\\
				\midrule
				\textsc{Batter} & 5 & 5 & 11.2 (1.6) & 17.8 (1.9) & 5.0 (0.0) & 5.2 (0.2) & 1.8 (0.1) & 12.2 (0.7) & 11.6 (1.4) &  & 26.5 (2.1) & 24.2 (3.2)\\
				\textsc{Bucket} & 5 & 5 & 0.9 (0.0) & 3.6 (0.2) & 3.0 (0.0) & 2.0 (0.0) & 1.7 (0.1) & 10.0 (0.0) & 1.6 (0.2) &  & 19.4 (1.1) & 19.3 (5.7)\\
				\textsc{Compass} & 5 & 5 & 15.5 (4.2) & 18.6 (1.6) & 5.0 (0.0) & 5.2 (0.2) & 1.8 (0.2) & 11.8 (0.6) & 12.8 (1.4) &  & 28.7 (1.9) & 20.3 (2.8)\\
				\textsc{Leather} & 5 & 5 & 0.9 (0.0) & 3.8 (0.2) & 3.0 (0.0) & 2.0 (0.0) & 1.8 (0.1) & 10.0 (0.0) & 1.8 (0.2) &  & 16.7 (1.7) & 17.9 (4.4)\\
				\textsc{Paper} & 5 & 5 & 0.9 (0.0) & 3.4 (0.2) & 3.0 (0.0) & 2.0 (0.0) & 1.6 (0.1) & 10.0 (0.0) & 1.4 (0.2) &  & 19.8 (2.0) & 40.6 (27.0)\\
				\textsc{Quill} & 5 & 5 & 18.2 (3.5) & 19.8 (1.2) & 5.0 (0.0) & 5.2 (0.2) & 2.1 (0.1) & 13.2 (0.4) & 12.6 (1.1) &  & 29.6 (2.5) & 24.4 (3.2)\\
				\textsc{Sugar} & 5 & 5 & 0.8 (0.0) & 3.2 (0.2) & 3.0 (0.0) & 2.0 (0.0) & 1.7 (0.2) & 10.0 (0.0) & 1.2 (0.2) &  & 17.7 (1.6) & 17.5 (3.2)\\
				\textsc{Book} & 5 & 5 & 45.8 (4.5) & 19.6 (0.9) & 5.0 (0.0) & 5.6 (0.2) & 6.0 (0.2) & 11.2 (1.0) & 14.4 (0.9) &  & 21.6 (1.8) & 21.0 (1.7)\\
				\textsc{Map} & 5 & 5 & 64.1 (10.6) & 22.0 (2.6) & 5.0 (0.0) & 5.2 (0.2) & 6.1 (0.2) & 10.8 (0.4) & 17.2 (2.7) &  & 22.5 (1.6) & 23.0 (1.2)\\
				\textsc{MilkBucket} & 5 & 5 & 1.2 (0.1) & 4.4 (0.4) & 3.0 (0.0) & 2.0 (0.0) & 6.8 (0.3) & 10.0 (0.0) & 2.4 (0.4) &  & 12.1 (0.7) & 15.3 (1.6)\\
				\textsc{BookQuill} & 5 & 5 & 4.5 (0.8) & 10.2 (1.4) & 4.0 (0.0) & 4.0 (0.0) & 3.9 (0.1) & 10.0 (0.0) & 7.2 (1.4) &  & 26.1 (0.8) & 22.4 (0.9)\\
				\textsc{MilkB.Sugar} & 5 & 5 & 3.5 (0.5) & 9.6 (1.3) & 4.0 (0.0) & 4.0 (0.0) & 3.9 (0.1) & 10.2 (0.2) & 6.4 (1.2) &  & 17.4 (0.5) & 12.5 (0.8)\\
				\textsc{Cake} & 5 & 5 & 9.1 (0.9) & 17.0 (0.9) & 4.0 (0.0) & 3.2 (0.2) & 2.1 (0.1) & 10.0 (0.0) & 14.0 (0.9) &  & 37.5 (1.9) & 18.0 (1.9)\\
				\midrule[1.5pt]
				\multicolumn{13}{l}{\textbf{Completed Runs} = 5 \hfill \textbf{Total Time (s.)} = 176.6 (13.1) \hfill \textbf{Total Calls} = 153.0 (3.6)} \\
				\bottomrule[1.5pt]
			\end{tabular}
		}
	\end{subtable}
	
	\vspace*{.1cm} 
	
	\begin{subtable}{\linewidth}
		\caption{OPL}
		\centering
		\resizebox{0.9\linewidth}{!}{
			\begin{tabular}{lrrrrrrrrrrrrrr}
				\toprule[1.5pt]
				\multicolumn{1}{c}{Task} & \multicolumn{1}{c}{\# G} & \multicolumn{1}{c}{\# L} & \multicolumn{1}{c}{Time (s.)} & \multicolumn{1}{c}{Calls} & \multicolumn{1}{c}{States} & \multicolumn{1}{c}{Edges} & \multicolumn{1}{c}{Ep. First HRM} & \multicolumn{3}{c}{\# Examples} & \multicolumn{1}{c}{} & \multicolumn{3}{c}{Example Length} \\
				\cmidrule{9-11} \cmidrule{13-15}
				&&&&&&&\multicolumn{1}{c}{$(\times 10^2)$}
				&\multicolumn{1}{c}{G} & \multicolumn{1}{c}{D} & \multicolumn{1}{c}{I}&&\multicolumn{1}{c}{G} & \multicolumn{1}{c}{D} & \multicolumn{1}{c}{I}\\
				\midrule
				\textsc{Batter} & 5 & 5 & 13.9 (3.0) & 23.0 (3.0) & 6.0 (0.0) & 9.2 (0.2) & 12.0 (1.0) & 11.4 (0.4) & 7.0 (1.2) & 10.6 (1.6) &  & 20.4 (1.1) & 18.7 (1.6) & 12.1 (1.7)\\
				\textsc{Bucket} & 5 & 5 & 1.8 (0.1) & 7.2 (0.6) & 4.0 (0.0) & 4.0 (0.0) & 8.0 (0.5) & 10.2 (0.2) & 2.2 (0.2) & 2.8 (0.4) &  & 10.2 (0.5) & 13.4 (1.9) & 6.8 (1.7)\\
				\textsc{Compass} & 5 & 5 & 13.2 (1.7) & 22.0 (1.7) & 6.0 (0.0) & 9.2 (0.2) & 10.4 (1.4) & 11.0 (0.6) & 6.8 (1.0) & 10.2 (1.0) &  & 17.2 (1.6) & 20.9 (1.9) & 14.3 (0.8)\\
				\textsc{Leather} & 5 & 5 & 1.9 (0.1) & 7.0 (0.5) & 4.0 (0.0) & 4.0 (0.0) & 6.9 (0.5) & 10.0 (0.0) & 2.4 (0.2) & 2.6 (0.4) &  & 11.1 (0.9) & 16.9 (5.6) & 8.9 (3.3)\\
				\textsc{Paper} & 5 & 5 & 2.0 (0.2) & 7.6 (0.6) & 4.0 (0.0) & 4.0 (0.0) & 7.7 (1.1) & 10.0 (0.0) & 3.0 (0.3) & 2.6 (0.4) &  & 10.1 (0.9) & 18.9 (3.3) & 5.6 (0.8)\\
				\textsc{Quill} & 5 & 5 & 11.5 (1.3) & 22.0 (1.2) & 6.0 (0.0) & 9.2 (0.2) & 12.8 (1.5) & 10.6 (0.2) & 6.4 (0.7) & 11.0 (0.9) &  & 15.3 (1.3) & 13.5 (1.0) & 12.1 (1.4)\\
				\textsc{Sugar} & 5 & 5 & 1.6 (0.1) & 6.4 (0.4) & 4.0 (0.0) & 4.0 (0.0) & 6.5 (0.7) & 10.0 (0.0) & 2.4 (0.2) & 2.0 (0.3) &  & 9.6 (0.6) & 15.3 (3.6) & 16.6 (9.2)\\
				\textsc{Book} & 5 & 5 & 69.0 (20.5) & 21.8 (2.2) & 6.0 (0.0) & 6.2 (0.2) & 5.5 (0.1) & 10.4 (0.2) & 5.2 (0.9) & 12.2 (1.9) &  & 20.4 (1.3) & 21.2 (2.0) & 20.8 (1.7)\\
				\textsc{Map} & 5 & 5 & 76.5 (6.0) & 24.2 (1.3) & 6.0 (0.0) & 6.4 (0.2) & 5.7 (0.3) & 11.6 (0.8) & 4.0 (0.3) & 14.6 (1.0) &  & 24.8 (3.0) & 21.4 (2.1) & 25.7 (0.8)\\
				\textsc{MilkBucket} & 5 & 5 & 1.7 (0.2) & 6.0 (0.6) & 4.0 (0.0) & 3.0 (0.0) & 7.5 (0.7) & 10.2 (0.2) & 1.4 (0.2) & 2.4 (0.2) &  & 11.7 (0.7) & 25.4 (6.4) & 14.2 (3.4)\\
				\textsc{BookQuill} & 5 & 5 & 5.3 (0.9) & 10.8 (1.4) & 4.0 (0.0) & 4.0 (0.0) & 3.7 (0.1) & 10.0 (0.0) & 1.0 (0.5) & 6.8 (0.9) &  & 27.7 (1.0) & 11.2 (5.4) & 21.1 (1.7)\\
				\textsc{MilkB.Sugar} & 5 & 5 & 4.0 (0.9) & 9.8 (1.7) & 4.0 (0.0) & 4.0 (0.0) & 3.8 (0.1) & 10.0 (0.0) & 1.6 (0.7) & 5.2 (1.2) &  & 18.4 (0.7) & 8.3 (2.9) & 15.6 (1.7)\\
				\textsc{Cake} & 5 & 5 & 16.2 (0.4) & 20.8 (0.2) & 5.0 (0.0) & 4.0 (0.0) & 2.1 (0.1) & 10.0 (0.0) & 3.2 (0.2) & 14.6 (0.2) &  & 38.1 (0.9) & 22.5 (3.3) & 25.8 (1.7)\\
				\midrule[1.5pt]
				\multicolumn{15}{l}{\textbf{Completed Runs} = 5 \hfill \textbf{Total Time (s.)} = 218.6 (21.1) \hfill \textbf{Total Calls} = 188.6 (5.4)} \\
				\bottomrule[1.5pt]
			\end{tabular}
		}
	\end{subtable}
	
	\vspace*{.1cm}
	
	\begin{subtable}{\linewidth}
		\caption{FR}
		\centering
		\resizebox{0.9\linewidth}{!}{
			\begin{tabular}{lrrrrrrrrrrrr}
				\toprule[1.5pt]
				\multicolumn{1}{c}{Task} & \multicolumn{1}{c}{\# G} & \multicolumn{1}{c}{\# L} & \multicolumn{1}{c}{Time (s.)} & \multicolumn{1}{c}{Calls} & \multicolumn{1}{c}{States} & \multicolumn{1}{c}{Edges} & \multicolumn{1}{c}{Ep. First HRM} & \multicolumn{2}{c}{\# Examples} & \multicolumn{1}{c}{} & \multicolumn{2}{c}{Example Length} \\
				\cmidrule{9-10} \cmidrule{12-13}
				&&&&&&&\multicolumn{1}{c}{$(\times 10^2)$}
				&\multicolumn{1}{c}{G} & \multicolumn{1}{c}{I}&&\multicolumn{1}{c}{G} & \multicolumn{1}{c}{I}\\
				\midrule
				\textsc{Batter} & 5 & 5 & 12.6 (1.8) & 17.6 (1.3) & 5.0 (0.0) & 5.4 (0.2) & 9.2 (1.2) & 11.6 (0.4) & 12.0 (1.2) &  & 30.3 (2.3) & 27.8 (2.0)\\
				\textsc{Bucket} & 5 & 5 & 1.2 (0.1) & 3.8 (0.2) & 3.0 (0.0) & 2.0 (0.0) & 6.7 (0.9) & 10.0 (0.0) & 1.8 (0.2) &  & 16.6 (2.5) & 28.5 (4.1)\\
				\textsc{Compass} & 5 & 5 & 14.1 (1.5) & 20.2 (1.7) & 5.0 (0.0) & 5.2 (0.2) & 9.8 (0.7) & 11.6 (0.6) & 14.6 (1.2) &  & 26.5 (0.8) & 26.5 (2.1)\\
				\textsc{Leather} & 5 & 5 & 1.1 (0.1) & 3.6 (0.2) & 3.0 (0.0) & 2.0 (0.0) & 4.5 (0.7) & 10.0 (0.0) & 1.6 (0.2) &  & 13.4 (1.3) & 16.7 (3.6)\\
				\textsc{Paper} & 5 & 5 & 1.2 (0.1) & 4.0 (0.0) & 3.0 (0.0) & 2.0 (0.0) & 4.9 (0.9) & 10.0 (0.0) & 2.0 (0.0) &  & 12.4 (1.1) & 10.9 (2.5)\\
				\textsc{Quill} & 5 & 5 & 9.3 (0.8) & 16.0 (0.8) & 5.0 (0.0) & 5.2 (0.2) & 9.4 (1.7) & 10.6 (0.2) & 11.4 (0.6) &  & 25.4 (0.3) & 25.5 (2.7)\\
				\textsc{Sugar} & 5 & 5 & 1.4 (0.2) & 3.8 (0.2) & 3.0 (0.0) & 2.0 (0.0) & 5.2 (0.3) & 10.0 (0.0) & 1.8 (0.2) &  & 15.3 (1.7) & 21.0 (10.1)\\
				\textsc{Book} & 5 & 5 & 43.8 (13.0) & 20.0 (1.9) & 5.0 (0.0) & 5.4 (0.2) & 6.0 (0.1) & 10.0 (0.0) & 16.0 (1.9) &  & 21.9 (1.0) & 14.7 (1.4)\\
				\textsc{Map} & 5 & 5 & 85.2 (13.4) & 22.2 (2.5) & 5.0 (0.0) & 5.2 (0.2) & 5.9 (0.1) & 10.2 (0.2) & 18.0 (2.6) &  & 26.5 (0.9) & 18.2 (1.2)\\
				\textsc{MilkBucket} & 5 & 5 & 1.4 (0.1) & 4.4 (0.2) & 3.0 (0.0) & 2.0 (0.0) & 10.2 (0.9) & 10.0 (0.0) & 2.4 (0.2) &  & 13.0 (0.8) & 12.2 (2.8)\\
				\textsc{BookQuill} & 5 & 5 & 6.3 (0.9) & 13.2 (1.7) & 4.0 (0.0) & 4.0 (0.0) & 3.8 (0.1) & 10.0 (0.0) & 10.2 (1.7) &  & 30.6 (2.0) & 11.9 (1.2)\\
				\textsc{MilkB.Sugar} & 5 & 5 & 4.8 (0.6) & 11.8 (1.3) & 4.0 (0.0) & 4.0 (0.0) & 3.8 (0.1) & 10.0 (0.0) & 8.8 (1.3) &  & 19.8 (0.7) & 8.6 (1.0)\\
				\textsc{Cake} & 5 & 5 & 12.5 (1.8) & 20.8 (2.5) & 4.0 (0.0) & 3.0 (0.0) & 2.3 (0.1) & 10.0 (0.0) & 17.8 (2.5) &  & 44.2 (2.6) & 13.1 (1.2)\\
				\midrule[1.5pt]
				\multicolumn{13}{l}{\textbf{Completed Runs} = 5 \hfill \textbf{Total Time (s.)} = 194.9 (17.6) \hfill \textbf{Total Calls} = 161.4 (7.0)} \\
				\bottomrule[1.5pt]
			\end{tabular}
		}
	\end{subtable}
	
	\vspace*{.1cm}
	
	\begin{subtable}{\linewidth}
		\caption{FRL}
		\centering
		\resizebox{0.9\linewidth}{!}{
			\begin{tabular}{lrrrrrrrrrrrrrr}
				\toprule[1.5pt]
				\multicolumn{1}{c}{Task} & \multicolumn{1}{c}{\# G} & \multicolumn{1}{c}{\# L} & \multicolumn{1}{c}{Time (s.)} & \multicolumn{1}{c}{Calls} & \multicolumn{1}{c}{States} & \multicolumn{1}{c}{Edges} & \multicolumn{1}{c}{Ep. First HRM} & \multicolumn{3}{c}{\# Examples} & \multicolumn{1}{c}{} & \multicolumn{3}{c}{Example Length} \\
				\cmidrule{9-11} \cmidrule{13-15}
				&&&&&&&\multicolumn{1}{c}{$(\times 10^2)$}
				&\multicolumn{1}{c}{G} & \multicolumn{1}{c}{D} & \multicolumn{1}{c}{I}&&\multicolumn{1}{c}{G} & \multicolumn{1}{c}{D} & \multicolumn{1}{c}{I}\\
				\midrule
				\textsc{Batter} & 5 & 5 & 11.2 (1.4) & 23.4 (2.5) & 6.0 (0.0) & 9.2 (0.2) & 468.4 (121.9) & 10.4 (0.2) & 7.6 (0.9) & 11.4 (1.9) &  & 11.9 (0.6) & 10.1 (1.3) & 9.9 (0.4)\\
				\textsc{Bucket} & 5 & 5 & 2.4 (0.1) & 7.0 (0.3) & 4.0 (0.0) & 4.0 (0.0) & 129.5 (69.4) & 10.2 (0.2) & 2.8 (0.2) & 2.0 (0.3) &  & 7.8 (0.5) & 9.9 (1.7) & 6.4 (2.1)\\
				\textsc{Compass} & 5 & 5 & 13.1 (1.9) & 24.6 (2.2) & 6.0 (0.0) & 9.4 (0.2) & 550.8 (156.4) & 10.4 (0.2) & 7.8 (1.0) & 12.4 (1.2) &  & 12.5 (1.6) & 9.4 (1.0) & 8.4 (0.5)\\
				\textsc{Leather} & 5 & 5 & 2.5 (0.4) & 7.8 (0.7) & 4.0 (0.0) & 4.0 (0.0) & 89.0 (18.0) & 10.0 (0.0) & 3.2 (0.4) & 2.6 (0.4) &  & 7.3 (0.4) & 9.3 (1.7) & 3.7 (0.4)\\
				\textsc{Paper} & 5 & 5 & 2.1 (0.1) & 7.0 (0.3) & 4.0 (0.0) & 4.0 (0.0) & 82.7 (18.8) & 10.0 (0.0) & 3.0 (0.0) & 2.0 (0.3) &  & 6.9 (0.7) & 10.2 (1.8) & 4.7 (2.7)\\
				\textsc{Quill} & 5 & 5 & 11.6 (1.2) & 23.8 (1.5) & 6.0 (0.0) & 9.6 (0.2) & 458.9 (61.0) & 10.6 (0.2) & 8.0 (0.9) & 11.2 (1.2) &  & 11.9 (0.6) & 13.1 (2.7) & 9.2 (0.8)\\
				\textsc{Sugar} & 5 & 5 & 2.6 (0.2) & 8.4 (0.7) & 4.0 (0.0) & 4.0 (0.0) & 103.5 (39.5) & 10.0 (0.0) & 3.6 (0.4) & 2.8 (0.5) &  & 8.2 (0.7) & 10.1 (1.9) & 5.0 (1.1)\\
				\textsc{Book} & 5 & 5 & 62.2 (13.2) & 27.4 (2.2) & 6.0 (0.0) & 6.6 (0.2) & 5.3 (0.1) & 10.2 (0.2) & 5.6 (0.6) & 17.6 (1.7) &  & 23.0 (1.0) & 16.5 (2.0) & 13.4 (1.0)\\
				\textsc{Map} & 5 & 5 & 131.3 (28.0) & 34.0 (3.0) & 6.0 (0.0) & 6.6 (0.2) & 5.5 (0.2) & 10.2 (0.2) & 6.8 (0.7) & 23.0 (2.4) &  & 26.2 (0.7) & 16.9 (1.6) & 14.5 (0.5)\\
				\textsc{MilkBucket} & 5 & 5 & 2.7 (0.7) & 6.6 (0.6) & 4.0 (0.0) & 3.0 (0.0) & 6.8 (0.3) & 10.0 (0.0) & 2.2 (0.2) & 2.4 (0.4) &  & 12.0 (0.8) & 9.4 (1.0) & 9.9 (2.3)\\
				\textsc{BookQuill} & 5 & 5 & 6.8 (0.6) & 12.6 (0.7) & 4.0 (0.0) & 4.0 (0.0) & 4.4 (0.2) & 10.0 (0.0) & 1.6 (0.5) & 8.0 (0.3) &  & 32.3 (2.3) & 4.9 (1.5) & 11.6 (1.1)\\
				\textsc{MilkB.Sugar} & 5 & 5 & 5.4 (0.6) & 12.2 (1.0) & 4.0 (0.0) & 4.0 (0.0) & 4.0 (0.1) & 10.2 (0.2) & 1.0 (0.4) & 8.0 (0.4) &  & 20.4 (1.7) & 2.9 (1.4) & 10.3 (1.2)\\
				\textsc{Cake} & 5 & 5 & 16.3 (1.2) & 21.2 (1.0) & 5.0 (0.0) & 4.0 (0.0) & 2.8 (0.0) & 10.0 (0.0) & 2.6 (0.2) & 15.6 (1.2) &  & 47.7 (1.9) & 15.0 (0.7) & 16.0 (0.9)\\
				\midrule[1.5pt]
				\multicolumn{15}{l}{\textbf{Completed Runs} = 5 \hfill \textbf{Total Time (s.)} = 270.1 (34.6) \hfill \textbf{Total Calls} = 216.0 (5.1)} \\
				\bottomrule[1.5pt]
			\end{tabular}
		}
	\end{subtable}
\end{table}

\begin{table}
	\caption{Results of \learningmethod in \cw without exploration using options.}
	\label{tab:non_flat_hrm_action_explore_cw}
	\begin{subtable}{\linewidth}
		\caption{OP}
		\resizebox{\linewidth}{!}{
			\begin{tabular}{lrrrrrrrrrrrr}
				\toprule[1.5pt]
				\multicolumn{1}{c}{Task} & \multicolumn{1}{c}{\# G} & \multicolumn{1}{c}{\# L} & \multicolumn{1}{c}{Time (s.)} & \multicolumn{1}{c}{Calls} & \multicolumn{1}{c}{States} & \multicolumn{1}{c}{Edges} & \multicolumn{1}{c}{Ep. First HRM} & \multicolumn{2}{c}{\# Examples} & \multicolumn{1}{c}{} & \multicolumn{2}{c}{Example Length} \\
				\cmidrule{9-10} \cmidrule{12-13}
				&&&&&&&\multicolumn{1}{c}{$(\times 10^2)$}
				&\multicolumn{1}{c}{G} & \multicolumn{1}{c}{I}&&\multicolumn{1}{c}{G} & \multicolumn{1}{c}{I}\\
				\midrule
				\textsc{Batter} & 5 & 5 & 11.2 (1.7) & 17.8 (1.9) & 5.0 (0.0) & 5.2 (0.2) & 1.8 (0.1) & 12.2 (0.7) & 11.6 (1.4) &  & 26.5 (2.1) & 24.2 (3.2)\\
				\textsc{Bucket} & 5 & 5 & 0.9 (0.0) & 3.6 (0.2) & 3.0 (0.0) & 2.0 (0.0) & 1.7 (0.1) & 10.0 (0.0) & 1.6 (0.2) &  & 19.4 (1.1) & 19.3 (5.7)\\
				\textsc{Compass} & 5 & 5 & 15.6 (4.1) & 18.6 (1.6) & 5.0 (0.0) & 5.2 (0.2) & 1.8 (0.2) & 11.8 (0.6) & 12.8 (1.4) &  & 28.7 (1.9) & 20.3 (2.8)\\
				\textsc{Leather} & 5 & 5 & 0.9 (0.1) & 3.8 (0.2) & 3.0 (0.0) & 2.0 (0.0) & 1.8 (0.1) & 10.0 (0.0) & 1.8 (0.2) &  & 16.7 (1.7) & 17.9 (4.4)\\
				\textsc{Paper} & 5 & 5 & 0.9 (0.1) & 3.4 (0.2) & 3.0 (0.0) & 2.0 (0.0) & 1.6 (0.1) & 10.0 (0.0) & 1.4 (0.2) &  & 19.8 (2.0) & 40.6 (27.0)\\
				\textsc{Quill} & 5 & 5 & 18.3 (3.6) & 19.8 (1.2) & 5.0 (0.0) & 5.2 (0.2) & 2.1 (0.1) & 13.2 (0.4) & 12.6 (1.1) &  & 29.6 (2.5) & 24.4 (3.2)\\
				\textsc{Sugar} & 5 & 5 & 0.9 (0.0) & 3.2 (0.2) & 3.0 (0.0) & 2.0 (0.0) & 1.7 (0.2) & 10.0 (0.0) & 1.2 (0.2) &  & 17.7 (1.6) & 17.5 (3.2)\\
				\textsc{Book} & 5 & 5 & 529.0 (164.2) & 21.2 (1.4) & 5.0 (0.0) & 5.8 (0.2) & 6.8 (0.2) & 10.2 (0.2) & 17.0 (1.5) &  & 33.0 (2.6) & 23.7 (1.3)\\
				\textsc{Map} & 5 & 5 & 1924.2 (443.5) & 28.0 (3.8) & 5.0 (0.0) & 5.4 (0.2) & 7.8 (0.4) & 10.4 (0.2) & 23.6 (3.7) &  & 40.1 (1.0) & 29.4 (1.3)\\
				\textsc{MilkBucket} & 5 & 5 & 1.6 (0.2) & 4.4 (0.4) & 3.0 (0.0) & 2.0 (0.0) & 6.1 (0.3) & 10.0 (0.0) & 2.4 (0.4) &  & 16.0 (1.0) & 14.2 (1.3)\\
				\textsc{BookQuill} & 5 & 5 & 42.7 (10.1) & 24.6 (3.9) & 4.0 (0.0) & 4.0 (0.0) & 6.8 (0.2) & 10.0 (0.0) & 21.6 (3.9) &  & 55.8 (2.7) & 21.2 (1.1)\\
				\textsc{MilkB.Sugar} & 5 & 5 & 8.1 (0.8) & 11.8 (1.0) & 4.0 (0.0) & 4.0 (0.0) & 4.9 (0.1) & 10.2 (0.2) & 8.6 (1.2) &  & 31.1 (0.7) & 13.1 (0.8)\\
				\textsc{Cake} & 5 & 5 & 198.3 (47.5) & 43.0 (5.3) & 4.0 (0.0) & 3.8 (0.2) & 5.5 (0.2) & 10.0 (0.0) & 40.0 (5.3) &  & 65.0 (0.9) & 22.0 (0.9)\\
				\midrule[1.5pt]
				\multicolumn{13}{l}{\textbf{Completed Runs} = 5 \hfill \textbf{Total Time (s.)} = 2752.8 (503.2) \hfill \textbf{Total Calls} = 203.2 (11.8)} \\
				\bottomrule[1.5pt]
			\end{tabular}
		}
	\end{subtable}
	
	\vspace*{.1cm}
	
	\begin{subtable}{\linewidth}
		\caption{OPL}
		\resizebox{\linewidth}{!}{
			\begin{tabular}{lrrrrrrrrrrrrrr}
				\toprule[1.5pt]
				\multicolumn{1}{c}{Task} & \multicolumn{1}{c}{\# G} & \multicolumn{1}{c}{\# L} & \multicolumn{1}{c}{Time (s.)} & \multicolumn{1}{c}{Calls} & \multicolumn{1}{c}{States} & \multicolumn{1}{c}{Edges} & \multicolumn{1}{c}{Ep. First HRM} & \multicolumn{3}{c}{\# Examples} & \multicolumn{1}{c}{} & \multicolumn{3}{c}{Example Length} \\
				\cmidrule{9-11} \cmidrule{13-15}
				&&&&&&&\multicolumn{1}{c}{$(\times 10^2)$}
				&\multicolumn{1}{c}{G} & \multicolumn{1}{c}{D} & \multicolumn{1}{c}{I}&&\multicolumn{1}{c}{G} & \multicolumn{1}{c}{D} & \multicolumn{1}{c}{I}\\
				\midrule
				\textsc{Batter} & 5 & 5 & 14.1 (3.2) & 23.0 (3.0) & 6.0 (0.0) & 9.2 (0.2) & 12.0 (1.0) & 11.4 (0.4) & 7.0 (1.2) & 10.6 (1.6) &  & 20.4 (1.1) & 18.7 (1.6) & 12.1 (1.7)\\
				\textsc{Bucket} & 5 & 5 & 1.8 (0.1) & 7.2 (0.6) & 4.0 (0.0) & 4.0 (0.0) & 8.0 (0.5) & 10.2 (0.2) & 2.2 (0.2) & 2.8 (0.4) &  & 10.2 (0.5) & 13.4 (1.9) & 6.8 (1.7)\\
				\textsc{Compass} & 5 & 5 & 13.5 (1.8) & 22.0 (1.7) & 6.0 (0.0) & 9.2 (0.2) & 10.4 (1.4) & 11.0 (0.6) & 6.8 (1.0) & 10.2 (1.0) &  & 17.2 (1.6) & 20.9 (1.9) & 14.3 (0.8)\\
				\textsc{Leather} & 5 & 5 & 1.8 (0.1) & 7.0 (0.5) & 4.0 (0.0) & 4.0 (0.0) & 6.9 (0.5) & 10.0 (0.0) & 2.4 (0.2) & 2.6 (0.4) &  & 11.1 (0.9) & 16.9 (5.6) & 8.9 (3.3)\\
				\textsc{Paper} & 5 & 5 & 2.0 (0.2) & 7.6 (0.6) & 4.0 (0.0) & 4.0 (0.0) & 7.7 (1.1) & 10.0 (0.0) & 3.0 (0.3) & 2.6 (0.4) &  & 10.1 (0.9) & 18.9 (3.3) & 5.6 (0.8)\\
				\textsc{Quill} & 5 & 5 & 11.8 (1.3) & 22.0 (1.2) & 6.0 (0.0) & 9.2 (0.2) & 12.8 (1.5) & 10.6 (0.2) & 6.4 (0.7) & 11.0 (0.9) &  & 15.3 (1.3) & 13.5 (1.0) & 12.1 (1.4)\\
				\textsc{Sugar} & 5 & 5 & 1.6 (0.1) & 6.4 (0.4) & 4.0 (0.0) & 4.0 (0.0) & 6.5 (0.7) & 10.0 (0.0) & 2.4 (0.2) & 2.0 (0.3) &  & 9.6 (0.6) & 15.3 (3.6) & 16.6 (9.2)\\
				\textsc{Book} & 5 & 5 & 224.8 (71.6) & 27.0 (1.9) & 6.0 (0.0) & 6.4 (0.2) & 139.7 (21.8) & 11.6 (0.4) & 3.2 (0.4) & 18.2 (1.4) &  & 22.0 (1.6) & 24.7 (6.5) & 23.5 (1.2)\\
				\textsc{Map} & 5 & 5 & 339.9 (33.6) & 33.0 (2.8) & 6.0 (0.0) & 6.4 (0.2) & 204.8 (27.1) & 10.6 (0.2) & 2.8 (0.5) & 25.6 (2.5) &  & 25.4 (0.8) & 21.8 (3.1) & 25.2 (1.1)\\
				\textsc{MilkBucket} & 5 & 5 & 3.5 (0.3) & 8.2 (0.6) & 4.0 (0.0) & 3.0 (0.0) & 47.6 (3.7) & 10.2 (0.2) & 2.6 (0.4) & 3.4 (0.4) &  & 10.3 (0.7) & 16.2 (1.7) & 14.2 (1.8)\\
				\textsc{BookQuill} & 5 & 5 & 19.0 (2.2) & 15.4 (1.5) & 4.0 (0.0) & 4.0 (0.0) & 383.4 (83.7) & 10.0 (0.0) & 1.0 (0.3) & 11.4 (1.3) &  & 38.2 (1.6) & 14.1 (4.9) & 23.8 (1.0)\\
				\textsc{MilkB.Sugar} & 5 & 5 & 11.4 (2.1) & 14.4 (1.7) & 4.0 (0.0) & 4.0 (0.0) & 87.4 (8.9) & 10.4 (0.2) & 1.0 (0.4) & 10.0 (1.3) &  & 19.7 (1.2) & 8.7 (4.9) & 17.6 (1.2)\\
				\textsc{Cake} & 4 & 1 & 277.4 (0.0) & 33.0 (0.0) & 5.0 (0.0) & 4.0 (0.0) & 264.1 (0.0) & 10.0 (0.0) & 2.0 (0.0) & 28.0 (0.0) &  & 46.7 (0.0) & 36.0 (0.0) & 22.9 (0.0)\\
				\midrule[1.5pt]
				\multicolumn{15}{l}{\textbf{Completed Runs} = 5 \hfill \textbf{Total Time (s.)} = 701.0 (111.2) \hfill \textbf{Total Calls} = 199.8 (6.9)} \\
				\bottomrule[1.5pt]
			\end{tabular}
		}
	\end{subtable}
	
	\vspace*{.1cm}
	
	\begin{subtable}{\linewidth}
		\caption{FRL}
		\resizebox{\linewidth}{!}{
			\begin{tabular}{lrrrrrrrrrrrrrr}
				\toprule[1.5pt]
				\multicolumn{1}{c}{Task} & \multicolumn{1}{c}{\# G} & \multicolumn{1}{c}{\# L} & \multicolumn{1}{c}{Time (s.)} & \multicolumn{1}{c}{Calls} & \multicolumn{1}{c}{States} & \multicolumn{1}{c}{Edges} & \multicolumn{1}{c}{Ep. First HRM} & \multicolumn{3}{c}{\# Examples} & \multicolumn{1}{c}{} & \multicolumn{3}{c}{Example Length} \\
				\cmidrule{9-11} \cmidrule{13-15}
				&&&&&&&\multicolumn{1}{c}{$(\times 10^2)$}
				&\multicolumn{1}{c}{G} & \multicolumn{1}{c}{D} & \multicolumn{1}{c}{I}&&\multicolumn{1}{c}{G} & \multicolumn{1}{c}{D} & \multicolumn{1}{c}{I}\\
				\midrule
				\textsc{Batter} & 5 & 5 & 11.1 (1.4) & 23.4 (2.5) & 6.0 (0.0) & 9.2 (0.2) & 468.4 (121.9) & 10.4 (0.2) & 7.6 (0.9) & 11.4 (1.9) &  & 11.9 (0.6) & 10.1 (1.3) & 9.9 (0.4)\\
				\textsc{Bucket} & 5 & 5 & 2.2 (0.1) & 7.0 (0.3) & 4.0 (0.0) & 4.0 (0.0) & 129.5 (69.4) & 10.2 (0.2) & 2.8 (0.2) & 2.0 (0.3) &  & 7.8 (0.5) & 9.9 (1.7) & 6.4 (2.1)\\
				\textsc{Compass} & 5 & 5 & 12.9 (1.9) & 24.6 (2.2) & 6.0 (0.0) & 9.4 (0.2) & 550.8 (156.4) & 10.4 (0.2) & 7.8 (1.0) & 12.4 (1.2) &  & 12.5 (1.6) & 9.4 (1.0) & 8.4 (0.5)\\
				\textsc{Leather} & 5 & 5 & 2.8 (0.4) & 7.8 (0.7) & 4.0 (0.0) & 4.0 (0.0) & 89.0 (18.0) & 10.0 (0.0) & 3.2 (0.4) & 2.6 (0.4) &  & 7.3 (0.4) & 9.3 (1.7) & 3.7 (0.4)\\
				\textsc{Paper} & 5 & 5 & 2.1 (0.1) & 7.0 (0.3) & 4.0 (0.0) & 4.0 (0.0) & 82.7 (18.8) & 10.0 (0.0) & 3.0 (0.0) & 2.0 (0.3) &  & 6.9 (0.7) & 10.2 (1.8) & 4.7 (2.7)\\
				\textsc{Quill} & 5 & 5 & 11.6 (1.1) & 23.8 (1.5) & 6.0 (0.0) & 9.6 (0.2) & 458.9 (61.0) & 10.6 (0.2) & 8.0 (0.9) & 11.2 (1.2) &  & 11.9 (0.6) & 13.1 (2.7) & 9.2 (0.8)\\
				\textsc{Sugar} & 5 & 5 & 2.6 (0.3) & 8.4 (0.7) & 4.0 (0.0) & 4.0 (0.0) & 103.5 (39.5) & 10.0 (0.0) & 3.6 (0.4) & 2.8 (0.5) &  & 8.2 (0.7) & 10.1 (1.9) & 5.0 (1.1)\\
				\textsc{Book} & 5 & 0 & 0.0 (0.0) & 0.0 (0.0) & 0.0 (0.0) & 0.0 (0.0) & 0.0 (0.0) & 0.0 (0.0) & 0.0 (0.0) & 0.0 (0.0) &  & 0.0 (0.0) & 0.0 (0.0) & 0.0 (0.0)\\
				\textsc{Map} & 3 & 0 & 0.0 (0.0) & 0.0 (0.0) & 0.0 (0.0) & 0.0 (0.0) & 0.0 (0.0) & 0.0 (0.0) & 0.0 (0.0) & 0.0 (0.0) &  & 0.0 (0.0) & 0.0 (0.0) & 0.0 (0.0)\\
				\textsc{MilkBucket} & 5 & 2 & 4.7 (0.5) & 11.0 (1.0) & 4.0 (0.0) & 3.0 (0.0) & 885.6 (142.3) & 10.0 (0.0) & 2.0 (0.0) & 7.0 (1.0) &  & 8.2 (0.4) & 10.2 (1.7) & 8.8 (0.2)\\
				\textsc{BookQuill} & 0 & 0 & 0.0 (0.0) & 0.0 (0.0) & 0.0 (0.0) & 0.0 (0.0) & 0.0 (0.0) & 0.0 (0.0) & 0.0 (0.0) & 0.0 (0.0) &  & 0.0 (0.0) & 0.0 (0.0) & 0.0 (0.0)\\
				\textsc{MilkB.Sugar} & 0 & 0 & 0.0 (0.0) & 0.0 (0.0) & 0.0 (0.0) & 0.0 (0.0) & 0.0 (0.0) & 0.0 (0.0) & 0.0 (0.0) & 0.0 (0.0) &  & 0.0 (0.0) & 0.0 (0.0) & 0.0 (0.0)\\
				\textsc{Cake} & 0 & 0 & 0.0 (0.0) & 0.0 (0.0) & 0.0 (0.0) & 0.0 (0.0) & 0.0 (0.0) & 0.0 (0.0) & 0.0 (0.0) & 0.0 (0.0) &  & 0.0 (0.0) & 0.0 (0.0) & 0.0 (0.0)\\
				\midrule[1.5pt]
				\multicolumn{15}{l}{\textbf{Completed Runs} = 5 \hfill \textbf{Total Time (s.)} = 47.1 (0.9) \hfill \textbf{Total Calls} = 106.4 (2.9)} \\
				\bottomrule[1.5pt]
			\end{tabular}
		}
	\end{subtable}
\end{table}

\begin{table}
	\caption{Results of \learningmethod in \ww for the default case.}
	\label{tab:non_flat_hrm_default_ww}
	\begin{subtable}{\linewidth}
		\caption{WOD}
		\resizebox{\linewidth}{!}{
			\begin{tabular}{lrrrrrrrrrrrr}
				\toprule[1.5pt]
				\multicolumn{1}{c}{Task} & \multicolumn{1}{c}{\# G} & \multicolumn{1}{c}{\# L} & \multicolumn{1}{c}{Time (s.)} & \multicolumn{1}{c}{Calls} & \multicolumn{1}{c}{States} & \multicolumn{1}{c}{Edges} & \multicolumn{1}{c}{Ep. First HRM} & \multicolumn{2}{c}{\# Examples} & \multicolumn{1}{c}{} & \multicolumn{2}{c}{Example Length} \\
				\cmidrule{9-10} \cmidrule{12-13}
				&&&&&&&\multicolumn{1}{c}{$(\times 10^2)$}
				&\multicolumn{1}{c}{G} & \multicolumn{1}{c}{I}&&\multicolumn{1}{c}{G} & \multicolumn{1}{c}{I}\\
				\midrule
				\textsc{rg} & 5 & 5 & 0.9 (0.0) & 4.0 (0.0) & 3.0 (0.0) & 2.0 (0.0) & 0.9 (0.1) & 10.0 (0.0) & 2.0 (0.0) &  & 11.2 (1.0) & 5.8 (1.1)\\
				\textsc{bc} & 5 & 5 & 0.9 (0.1) & 3.8 (0.2) & 3.0 (0.0) & 2.0 (0.0) & 0.8 (0.1) & 10.0 (0.0) & 1.8 (0.2) &  & 10.8 (0.8) & 11.9 (3.4)\\
				\textsc{my} & 5 & 5 & 0.9 (0.0) & 3.6 (0.2) & 3.0 (0.0) & 2.0 (0.0) & 0.7 (0.0) & 10.0 (0.0) & 1.6 (0.2) &  & 8.7 (0.8) & 6.6 (1.9)\\
				\textsc{rg\&bc} & 5 & 5 & 4.5 (0.3) & 13.4 (0.4) & 4.0 (0.0) & 4.0 (0.0) & 8.8 (0.3) & 11.8 (0.6) & 8.6 (0.7) &  & 12.2 (0.9) & 14.8 (1.2)\\
				\textsc{bc\&my} & 5 & 5 & 5.8 (1.0) & 15.6 (2.1) & 4.0 (0.0) & 4.0 (0.0) & 8.1 (0.2) & 12.8 (1.3) & 9.8 (1.5) &  & 13.2 (1.7) & 17.1 (1.6)\\
				\textsc{rg\&my} & 5 & 5 & 4.7 (0.5) & 13.2 (1.0) & 4.0 (0.0) & 4.0 (0.0) & 8.5 (0.2) & 10.8 (0.2) & 9.4 (0.9) &  & 12.2 (0.7) & 18.6 (1.2)\\
				\textsc{rgb} & 5 & 5 & 1.2 (0.1) & 4.8 (0.5) & 3.0 (0.0) & 2.0 (0.0) & 8.6 (0.2) & 10.0 (0.0) & 2.8 (0.5) &  & 7.8 (0.2) & 7.0 (1.4)\\
				\textsc{cmy} & 5 & 5 & 1.4 (0.2) & 5.4 (0.7) & 3.0 (0.0) & 2.0 (0.0) & 8.8 (0.5) & 10.0 (0.0) & 3.4 (0.7) &  & 8.0 (0.3) & 10.2 (1.3)\\
				\textsc{rgb\&cmy} & 5 & 5 & 15.1 (1.7) & 21.6 (1.7) & 4.0 (0.0) & 4.0 (0.0) & 2.3 (0.0) & 11.0 (0.4) & 17.6 (1.7) &  & 17.3 (0.4) & 22.6 (1.6)\\
				\midrule[1.5pt]
				\multicolumn{13}{l}{\textbf{Completed Runs} = 5 \hfill \textbf{Total Time (s.)} = 35.4 (2.0) \hfill \textbf{Total Calls} = 85.4 (3.1)} \\
				\bottomrule[1.5pt]
			\end{tabular}
		}
	\end{subtable}
	
	\vspace*{.1cm}
	
	\begin{subtable}{\linewidth}
		\caption{WD}
		\resizebox{\linewidth}{!}{
			\begin{tabular}{lrrrrrrrrrrrrrr}
				\toprule[1.5pt]
				\multicolumn{1}{c}{Task} & \multicolumn{1}{c}{\# G} & \multicolumn{1}{c}{\# L} & \multicolumn{1}{c}{Time (s.)} & \multicolumn{1}{c}{Calls} & \multicolumn{1}{c}{States} & \multicolumn{1}{c}{Edges} & \multicolumn{1}{c}{Ep. First HRM} & \multicolumn{3}{c}{\# Examples} & \multicolumn{1}{c}{} & \multicolumn{3}{c}{Example Length} \\
				\cmidrule{9-11} \cmidrule{13-15}
				&&&&&&&\multicolumn{1}{c}{$(\times 10^2)$}
				&\multicolumn{1}{c}{G} & \multicolumn{1}{c}{D} & \multicolumn{1}{c}{I}&&\multicolumn{1}{c}{G} & \multicolumn{1}{c}{D} & \multicolumn{1}{c}{I}\\
				\midrule
				\textsc{rg} & 5 & 5 & 1.9 (0.2) & 7.8 (1.1) & 4.0 (0.0) & 4.0 (0.0) & 3.0 (0.3) & 10.0 (0.0) & 3.0 (0.3) & 2.8 (0.9) &  & 7.0 (0.7) & 10.3 (1.6) & 4.2 (0.8)\\
				\textsc{bc} & 5 & 5 & 1.8 (0.3) & 7.2 (1.0) & 4.0 (0.0) & 4.0 (0.0) & 2.7 (0.3) & 10.2 (0.2) & 2.4 (0.2) & 2.6 (0.7) &  & 8.4 (0.5) & 6.9 (1.5) & 6.3 (1.7)\\
				\textsc{my} & 5 & 5 & 1.4 (0.1) & 5.8 (0.4) & 4.0 (0.0) & 4.0 (0.0) & 2.9 (0.3) & 10.0 (0.0) & 2.4 (0.2) & 1.4 (0.2) &  & 6.9 (0.4) & 5.8 (1.5) & 4.6 (1.6)\\
				\textsc{rg\&bc} & 5 & 5 & 11.7 (2.5) & 24.0 (3.8) & 4.8 (0.2) & 4.8 (0.2) & 12.0 (0.4) & 13.0 (0.7) & 6.2 (1.6) & 11.8 (1.9) &  & 11.7 (0.8) & 6.9 (1.5) & 11.8 (0.8)\\
				\textsc{bc\&my} & 5 & 5 & 9.5 (1.5) & 20.8 (2.4) & 4.8 (0.2) & 4.8 (0.2) & 11.5 (0.4) & 11.2 (0.6) & 5.2 (0.7) & 11.4 (1.6) &  & 10.6 (0.9) & 8.8 (1.4) & 13.2 (0.9)\\
				\textsc{rg\&my} & 5 & 5 & 5.4 (0.5) & 14.0 (1.3) & 4.2 (0.2) & 4.2 (0.2) & 11.8 (0.6) & 10.6 (0.2) & 3.2 (0.7) & 7.2 (1.1) &  & 9.8 (0.3) & 6.2 (1.9) & 11.6 (0.7)\\
				\textsc{rgb} & 5 & 5 & 2.5 (0.3) & 8.2 (0.7) & 4.0 (0.0) & 3.0 (0.0) & 11.9 (0.6) & 10.2 (0.2) & 3.0 (0.3) & 3.0 (0.5) &  & 7.9 (0.4) & 14.0 (1.8) & 10.6 (2.0)\\
				\textsc{cmy} & 5 & 5 & 3.6 (0.4) & 11.4 (1.2) & 4.0 (0.0) & 3.0 (0.0) & 10.8 (0.3) & 10.0 (0.0) & 4.2 (0.7) & 5.2 (0.6) &  & 7.9 (0.2) & 8.8 (1.6) & 10.5 (1.1)\\
				\textsc{rgb\&cmy} & 5 & 5 & 29.0 (4.1) & 31.4 (2.8) & 4.4 (0.2) & 4.4 (0.2) & 4.9 (0.3) & 11.2 (0.4) & 5.4 (1.6) & 21.8 (1.3) &  & 16.6 (0.8) & 7.4 (0.9) & 17.2 (0.6)\\
				\midrule[1.5pt]
				\multicolumn{15}{l}{\textbf{Completed Runs} = 5 \hfill \textbf{Total Time (s.)} = 67.0 (6.2) \hfill \textbf{Total Calls} = 130.6 (6.0)} \\
				\bottomrule[1.5pt]
			\end{tabular}
		}
	\end{subtable}
\end{table}

\begin{table}
	\caption{Results of \learningmethod in \ww with a restricted set of callable RMs.}
	\label{tab:non_flat_hrm_restricted_ww}
	\begin{subtable}{\linewidth}
		\caption{WOD}
		\resizebox{\linewidth}{!}{
			\begin{tabular}{lrrrrrrrrrrrr}
				\toprule[1.5pt]
				\multicolumn{1}{c}{Task} & \multicolumn{1}{c}{\# G} & \multicolumn{1}{c}{\# L} & \multicolumn{1}{c}{Time (s.)} & \multicolumn{1}{c}{Calls} & \multicolumn{1}{c}{States} & \multicolumn{1}{c}{Edges} & \multicolumn{1}{c}{Ep. First HRM} & \multicolumn{2}{c}{\# Examples} & \multicolumn{1}{c}{} & \multicolumn{2}{c}{Example Length} \\
				\cmidrule{9-10} \cmidrule{12-13}
				&&&&&&&\multicolumn{1}{c}{$(\times 10^2)$}
				&\multicolumn{1}{c}{G} & \multicolumn{1}{c}{I}&&\multicolumn{1}{c}{G} & \multicolumn{1}{c}{I}\\
				\midrule
				\textsc{rg} & 5 & 5 & 0.9 (0.0) & 4.0 (0.0) & 3.0 (0.0) & 2.0 (0.0) & 0.9 (0.1) & 10.0 (0.0) & 2.0 (0.0) &  & 11.2 (1.0) & 5.8 (1.1)\\
				\textsc{bc} & 5 & 5 & 0.9 (0.1) & 3.8 (0.2) & 3.0 (0.0) & 2.0 (0.0) & 0.8 (0.1) & 10.0 (0.0) & 1.8 (0.2) &  & 10.8 (0.8) & 11.9 (3.4)\\
				\textsc{my} & 5 & 5 & 0.9 (0.0) & 3.6 (0.2) & 3.0 (0.0) & 2.0 (0.0) & 0.7 (0.0) & 10.0 (0.0) & 1.6 (0.2) &  & 8.7 (0.8) & 6.6 (1.9)\\
				\textsc{rg\&bc} & 5 & 5 & 5.3 (0.4) & 15.2 (0.9) & 4.0 (0.0) & 4.0 (0.0) & 8.6 (0.3) & 12.4 (0.2) & 9.8 (0.8) &  & 14.7 (1.3) & 16.0 (0.8)\\
				\textsc{bc\&my} & 5 & 5 & 3.9 (0.1) & 12.4 (0.2) & 4.0 (0.0) & 4.0 (0.0) & 8.3 (0.4) & 11.8 (0.7) & 7.6 (0.7) &  & 11.2 (0.8) & 13.2 (1.0)\\
				\textsc{rg\&my} & 5 & 5 & 4.6 (0.3) & 13.8 (0.9) & 4.0 (0.0) & 4.0 (0.0) & 8.5 (0.2) & 10.2 (0.2) & 10.6 (0.9) &  & 10.7 (0.5) & 15.8 (1.6)\\
				\textsc{rgb} & 5 & 5 & 1.2 (0.1) & 4.8 (0.7) & 3.0 (0.0) & 2.0 (0.0) & 8.7 (0.2) & 10.2 (0.2) & 2.6 (0.7) &  & 8.3 (0.5) & 16.2 (3.8)\\
				\textsc{cmy} & 5 & 5 & 1.6 (0.2) & 6.2 (0.7) & 3.0 (0.0) & 2.0 (0.0) & 8.6 (0.5) & 10.0 (0.0) & 4.2 (0.7) &  & 8.0 (0.3) & 10.8 (1.2)\\
				\textsc{rgb\&cmy} & 5 & 5 & 5.7 (0.8) & 15.0 (1.6) & 4.0 (0.0) & 4.0 (0.0) & 2.6 (0.1) & 10.4 (0.4) & 11.6 (1.6) &  & 17.0 (1.1) & 15.9 (1.3)\\
				\midrule[1.5pt]
				\multicolumn{13}{l}{\textbf{Completed Runs} = 5 \hfill \textbf{Total Time (s.)} = 24.9 (0.9) \hfill \textbf{Total Calls} = 78.8 (2.7)} \\
				\bottomrule[1.5pt]
			\end{tabular}
		}
	\end{subtable}
	
	\vspace*{.1cm}
	
	\begin{subtable}{\linewidth}
		\caption{WD}
		\resizebox{\linewidth}{!}{
			\begin{tabular}{lrrrrrrrrrrrrrr}
				\toprule[1.5pt]
				\multicolumn{1}{c}{Task} & \multicolumn{1}{c}{\# G} & \multicolumn{1}{c}{\# L} & \multicolumn{1}{c}{Time (s.)} & \multicolumn{1}{c}{Calls} & \multicolumn{1}{c}{States} & \multicolumn{1}{c}{Edges} & \multicolumn{1}{c}{Ep. First HRM} & \multicolumn{3}{c}{\# Examples} & \multicolumn{1}{c}{} & \multicolumn{3}{c}{Example Length} \\
				\cmidrule{9-11} \cmidrule{13-15}
				&&&&&&&\multicolumn{1}{c}{$(\times 10^2)$}
				&\multicolumn{1}{c}{G} & \multicolumn{1}{c}{D} & \multicolumn{1}{c}{I}&&\multicolumn{1}{c}{G} & \multicolumn{1}{c}{D} & \multicolumn{1}{c}{I}\\
				\midrule
				\textsc{rg} & 5 & 5 & 1.9 (0.3) & 7.8 (1.1) & 4.0 (0.0) & 4.0 (0.0) & 3.0 (0.3) & 10.0 (0.0) & 3.0 (0.3) & 2.8 (0.9) &  & 7.0 (0.7) & 10.3 (1.6) & 4.2 (0.8)\\
				\textsc{bc} & 5 & 5 & 1.8 (0.3) & 7.2 (1.0) & 4.0 (0.0) & 4.0 (0.0) & 2.7 (0.3) & 10.2 (0.2) & 2.4 (0.2) & 2.6 (0.7) &  & 8.4 (0.5) & 6.9 (1.5) & 6.3 (1.7)\\
				\textsc{my} & 5 & 5 & 1.4 (0.1) & 5.8 (0.4) & 4.0 (0.0) & 4.0 (0.0) & 2.9 (0.3) & 10.0 (0.0) & 2.4 (0.2) & 1.4 (0.2) &  & 6.9 (0.4) & 5.8 (1.5) & 4.6 (1.6)\\
				\textsc{rg\&bc} & 5 & 5 & 6.9 (0.7) & 17.6 (1.5) & 4.6 (0.2) & 4.6 (0.2) & 12.0 (0.4) & 10.8 (0.2) & 4.6 (0.5) & 9.2 (1.0) &  & 10.4 (0.5) & 9.4 (1.8) & 12.6 (0.7)\\
				\textsc{bc\&my} & 5 & 5 & 9.3 (1.8) & 21.4 (2.9) & 4.8 (0.2) & 4.8 (0.2) & 11.7 (0.5) & 12.2 (1.0) & 5.8 (1.1) & 10.4 (1.8) &  & 11.4 (0.7) & 6.9 (1.3) & 12.1 (0.7)\\
				\textsc{rg\&my} & 5 & 5 & 7.8 (1.1) & 18.8 (1.9) & 4.8 (0.2) & 4.8 (0.2) & 11.8 (0.5) & 11.0 (0.3) & 4.8 (0.4) & 10.0 (2.0) &  & 9.8 (0.2) & 8.6 (0.8) & 13.0 (0.8)\\
				\textsc{rgb} & 5 & 5 & 2.1 (0.1) & 7.6 (0.2) & 4.0 (0.0) & 3.0 (0.0) & 11.9 (0.5) & 10.0 (0.0) & 2.6 (0.2) & 3.0 (0.0) &  & 7.6 (0.5) & 11.8 (1.7) & 10.7 (1.7)\\
				\textsc{cmy} & 5 & 5 & 2.3 (0.2) & 8.2 (0.8) & 4.0 (0.0) & 3.0 (0.0) & 10.7 (0.2) & 10.0 (0.0) & 2.2 (0.4) & 4.0 (0.5) &  & 7.8 (0.2) & 9.9 (1.6) & 8.9 (0.5)\\
				\textsc{rgb\&cmy} & 5 & 5 & 9.6 (1.5) & 20.6 (2.6) & 5.0 (0.0) & 5.0 (0.0) & 5.0 (0.6) & 10.2 (0.2) & 6.4 (0.9) & 11.0 (1.8) &  & 15.0 (0.5) & 12.3 (0.8) & 14.2 (1.2)\\
				\midrule[1.5pt]
				\multicolumn{15}{l}{\textbf{Completed Runs} = 5 \hfill \textbf{Total Time (s.)} = 42.9 (3.7) \hfill \textbf{Total Calls} = 115.0 (7.5)} \\
				\bottomrule[1.5pt]
			\end{tabular}
		}
	\end{subtable}
\end{table}

\begin{table}
	\caption{Results of \learningmethod in \ww without exploration using options.}
	\label{tab:non_flat_hrm_action_explore_ww}
	\begin{subtable}{\linewidth}
		\caption{WOD}
		\resizebox{\linewidth}{!}{
			\begin{tabular}{lrrrrrrrrrrrr}
				\toprule[1.5pt]
				\multicolumn{1}{c}{Task} & \multicolumn{1}{c}{\# G} & \multicolumn{1}{c}{\# L} & \multicolumn{1}{c}{Time (s.)} & \multicolumn{1}{c}{Calls} & \multicolumn{1}{c}{States} & \multicolumn{1}{c}{Edges} & \multicolumn{1}{c}{Ep. First HRM} & \multicolumn{2}{c}{\# Examples} & \multicolumn{1}{c}{} & \multicolumn{2}{c}{Example Length} \\
				\cmidrule{9-10} \cmidrule{12-13}
				&&&&&&&\multicolumn{1}{c}{$(\times 10^2)$}
				&\multicolumn{1}{c}{G} & \multicolumn{1}{c}{I}&&\multicolumn{1}{c}{G} & \multicolumn{1}{c}{I}\\
				\midrule
				\textsc{rg} & 5 & 5 & 0.9 (0.0) & 4.0 (0.0) & 3.0 (0.0) & 2.0 (0.0) & 0.9 (0.1) & 10.0 (0.0) & 2.0 (0.0) &  & 11.2 (1.0) & 5.8 (1.1)\\
				\textsc{bc} & 5 & 5 & 0.9 (0.1) & 3.8 (0.2) & 3.0 (0.0) & 2.0 (0.0) & 0.8 (0.1) & 10.0 (0.0) & 1.8 (0.2) &  & 10.8 (0.8) & 11.9 (3.4)\\
				\textsc{my} & 5 & 5 & 0.9 (0.0) & 3.6 (0.2) & 3.0 (0.0) & 2.0 (0.0) & 0.7 (0.0) & 10.0 (0.0) & 1.6 (0.2) &  & 8.7 (0.8) & 6.6 (1.9)\\
				\textsc{rg\&bc} & 5 & 5 & 4.2 (0.4) & 12.2 (0.9) & 4.0 (0.0) & 4.0 (0.0) & 9.5 (0.3) & 10.6 (0.2) & 8.6 (1.1) &  & 13.8 (0.2) & 15.3 (1.6)\\
				\textsc{bc\&my} & 5 & 5 & 4.3 (0.3) & 11.8 (0.7) & 4.0 (0.0) & 4.0 (0.0) & 9.8 (0.1) & 11.6 (0.2) & 7.2 (0.6) &  & 15.3 (0.9) & 16.7 (1.7)\\
				\textsc{rg\&my} & 5 & 5 & 4.6 (0.3) & 12.6 (0.7) & 4.0 (0.0) & 4.0 (0.0) & 9.5 (0.1) & 11.2 (0.4) & 8.4 (0.7) &  & 14.2 (0.9) & 14.8 (0.8)\\
				\textsc{rgb} & 5 & 5 & 1.2 (0.2) & 4.6 (0.7) & 3.0 (0.0) & 2.0 (0.0) & 9.0 (0.1) & 10.0 (0.0) & 2.6 (0.7) &  & 9.4 (0.4) & 8.7 (1.7)\\
				\textsc{cmy} & 5 & 5 & 1.4 (0.1) & 5.0 (0.5) & 3.0 (0.0) & 2.0 (0.0) & 8.8 (0.2) & 10.0 (0.0) & 3.0 (0.5) &  & 8.8 (0.2) & 10.6 (1.5)\\
				\textsc{rgb\&cmy} & 5 & 5 & 16.1 (1.1) & 19.8 (1.1) & 4.0 (0.0) & 4.0 (0.0) & 4.1 (0.1) & 11.2 (0.6) & 15.6 (1.3) &  & 26.0 (1.2) & 21.6 (0.9)\\
				\midrule[1.5pt]
				\multicolumn{13}{l}{\textbf{Completed Runs} = 5 \hfill \textbf{Total Time (s.)} = 34.4 (1.4) \hfill \textbf{Total Calls} = 77.4 (2.0)} \\
				\bottomrule[1.5pt]
			\end{tabular}
		}
	\end{subtable}
	
	\vspace*{.1cm}
	
	\begin{subtable}{\linewidth}
		\caption{WD}
		\resizebox{\linewidth}{!}{
			\begin{tabular}{lrrrrrrrrrrrrrr}
				\toprule[1.5pt]
				\multicolumn{1}{c}{Task} & \multicolumn{1}{c}{\# G} & \multicolumn{1}{c}{\# L} & \multicolumn{1}{c}{Time (s.)} & \multicolumn{1}{c}{Calls} & \multicolumn{1}{c}{States} & \multicolumn{1}{c}{Edges} & \multicolumn{1}{c}{Ep. First HRM} & \multicolumn{3}{c}{\# Examples} & \multicolumn{1}{c}{} & \multicolumn{3}{c}{Example Length} \\
				\cmidrule{9-11} \cmidrule{13-15}
				&&&&&&&\multicolumn{1}{c}{$(\times 10^2)$}
				&\multicolumn{1}{c}{G} & \multicolumn{1}{c}{D} & \multicolumn{1}{c}{I}&&\multicolumn{1}{c}{G} & \multicolumn{1}{c}{D} & \multicolumn{1}{c}{I}\\
				\midrule
				\textsc{rg} & 5 & 5 & 1.9 (0.3) & 7.8 (1.1) & 4.0 (0.0) & 4.0 (0.0) & 3.0 (0.3) & 10.0 (0.0) & 3.0 (0.3) & 2.8 (0.9) &  & 7.0 (0.7) & 10.3 (1.6) & 4.2 (0.8)\\
				\textsc{bc} & 5 & 5 & 1.8 (0.2) & 7.2 (1.0) & 4.0 (0.0) & 4.0 (0.0) & 2.7 (0.3) & 10.2 (0.2) & 2.4 (0.2) & 2.6 (0.7) &  & 8.4 (0.5) & 6.9 (1.5) & 6.3 (1.7)\\
				\textsc{my} & 5 & 5 & 1.4 (0.1) & 5.8 (0.4) & 4.0 (0.0) & 4.0 (0.0) & 2.9 (0.3) & 10.0 (0.0) & 2.4 (0.2) & 1.4 (0.2) &  & 6.9 (0.4) & 5.8 (1.5) & 4.6 (1.6)\\
				\textsc{rg\&bc} & 5 & 5 & 8.1 (1.4) & 18.2 (2.2) & 4.6 (0.2) & 4.6 (0.2) & 97.4 (4.2) & 10.8 (0.4) & 5.2 (0.6) & 9.2 (1.4) &  & 10.5 (0.5) & 10.3 (1.7) & 13.7 (1.5)\\
				\textsc{bc\&my} & 5 & 5 & 6.2 (0.5) & 15.6 (0.7) & 4.6 (0.2) & 4.6 (0.2) & 91.5 (5.8) & 10.6 (0.2) & 4.6 (0.6) & 7.4 (0.7) &  & 9.7 (0.2) & 7.0 (1.2) & 11.3 (1.0)\\
				\textsc{rg\&my} & 5 & 5 & 8.6 (1.8) & 19.2 (2.7) & 4.4 (0.2) & 4.4 (0.2) & 90.3 (5.3) & 11.2 (0.8) & 5.6 (0.9) & 9.4 (1.3) &  & 10.4 (0.7) & 7.7 (0.7) & 13.3 (0.9)\\
				\textsc{rgb} & 5 & 5 & 2.3 (0.1) & 7.6 (0.2) & 4.0 (0.0) & 3.0 (0.0) & 65.3 (1.6) & 10.2 (0.2) & 2.6 (0.4) & 2.8 (0.4) &  & 7.6 (0.3) & 11.1 (3.0) & 8.8 (1.7)\\
				\textsc{cmy} & 5 & 5 & 4.4 (0.6) & 13.2 (1.5) & 3.8 (0.2) & 3.0 (0.0) & 59.2 (2.9) & 10.2 (0.2) & 3.8 (0.6) & 7.2 (1.0) &  & 6.9 (0.5) & 8.2 (1.0) & 7.6 (0.8)\\
				\textsc{rgb\&cmy} & 5 & 5 & 32.1 (5.0) & 31.4 (3.3) & 4.4 (0.2) & 4.4 (0.2) & 125.7 (9.9) & 11.4 (0.5) & 5.8 (1.2) & 21.2 (2.2) &  & 17.1 (0.6) & 8.4 (1.6) & 17.8 (1.0)\\
				\midrule[1.5pt]
				\multicolumn{15}{l}{\textbf{Completed Runs} = 5 \hfill \textbf{Total Time (s.)} = 66.7 (6.6) \hfill \textbf{Total Calls} = 126.0 (6.3)} \\
				\bottomrule[1.5pt]
			\end{tabular}
		}
	\end{subtable}
\end{table}

\subsubsection{Learning of Flat HRMs}
\label{app:extended_results_flat_hrm_learning}
Table~\ref{tab:flat_hrms_table} shows the results of learning a non-flat HRM using \learningmethod, and the results of learning a flat HRM using several approaches (\learningmethod, DeepSynth, JIRP and LRM). An extended discussion of these results can be found in Section~\ref{sec:learning_flat_hrms}.

\begin{sidewaystable}
	\caption{Results of learning non-flat and flat HRMs using different methods. The columns are the following: the number of completed runs without timing out, the amount of time needed to learn the HRMs or RMs, and the number of states and edges of the RM.}
	\label{tab:flat_hrms_table}
	\resizebox{\linewidth}{!}{
		\begin{tabular}{lrrrrrrrrrrrrrrrrrrrrrrrr}
			\toprule[1.5pt]
			Task & \multicolumn{4}{c}{\learningmethod (Non-Flat)} & & \multicolumn{4}{c}{\learningmethod (Flat)} & & \multicolumn{4}{c}{DeepSynth} & & \multicolumn{4}{c}{JIRP} & & \multicolumn{4}{c}{LRM}\\
			\cmidrule{2-5} \cmidrule{7-10} \cmidrule{12-15} \cmidrule{17-20} \cmidrule{22-25}
			& \multicolumn{1}{c}{C} & \multicolumn{1}{c}{Time (s.)} & \multicolumn{1}{c}{States} & \multicolumn{1}{c}{Edges} & & \multicolumn{1}{c}{C} & \multicolumn{1}{c}{Time (s.)} & \multicolumn{1}{c}{States} & \multicolumn{1}{c}{Edges} & & \multicolumn{1}{c}{C} & \multicolumn{1}{c}{Time (s.)} & \multicolumn{1}{c}{States} & \multicolumn{1}{c}{Edges} & & \multicolumn{1}{c}{C} & \multicolumn{1}{c}{Time (s.)} & \multicolumn{1}{c}{States} & \multicolumn{1}{c}{Edges} & & \multicolumn{1}{c}{C} & \multicolumn{1}{c}{Time (s.)} & \multicolumn{1}{c}{States} & \multicolumn{1}{c}{Edges}\\
			\midrule
			\textsc{MilkBucket} & \textbf{5} & \textbf{1.5 (0.2)} & \textbf{3.0 (0.0)} & \textbf{2.0 (0.0)} & & \textbf{5} & 3.2 (0.6) & 4.0 (0.0) & 3.6 (0.2) & & \textbf{5} & 325.6 (29.7) & 13.4 (0.4) & 93.2 (1.7) & & \textbf{5} & 17.1 (5.5) & 4.0 (0.0) & 3.0 (0.0) & & \textbf{5} & 347.5 (64.5) & 4.0 (0.0) & 14.0 (1.0)\\
			\textsc{Book} & \textbf{5} & \textbf{191.2 (36.4)} & \textbf{5.0 (0.0)} & \textbf{5.8 (0.2)} & & 0 & - & - & - & & 5 & 288.9 (31.7) & 16.6 (3.1) & 119.0 (19.4) & & 0 & - & - & - & &  5 & 2261.0 (552.2)
			& 8.0 (0.0) & 31.2 (2.0) \\
			\textsc{BookQuill} & \textbf{5} & \textbf{17.9 (1.4)} & \textbf{4.0 (0.0)} & \textbf{4.0 (0.0)} & & 0 & - & - & - & & 5 & 308.6 (52.6) & 12.8 (0.5) & 92.8 (2.3) & & 0 & - & - & - & & 0 & - & - & - \\
			\textsc{Cake} & \textbf{5} & \textbf{74.5 (25.7)} & \textbf{4.0 (0.0)} & \textbf{3.2 (0.2)} & & 0 & - & - & - & & 4 & 290.6 (36.4) & 17.2 (2.5) & 110.2 (11.6) & & 0 & - & - & - & & 0 & - & - & - \\
			\midrule
			\textsc{rg} & \textbf{5} & \textbf{0.9 (0.0)} & \textbf{3.0 (0.0)} & \textbf{2.0 (0.0)} & & \textbf{5} &\textbf{ 0.9 (0.0)} & \textbf{3.0 (0.0)} & \textbf{2.0 (0.0)} & & 0 & - & - & - & & \textbf{5} & 32.3 (7.9) & 3.8 (0.2) & 82.4 (9.1) & & 0 & - & - & - \\
			\textsc{rg\&bc} & \textbf{5} & \textbf{4.5 (0.3)} & \textbf{4.0 (0.0)} & \textbf{4.0 (0.0)} & & 0 & - & - & -  & & 0 & - & - & - & & 0 & - & - & - & & 0 & - & - & -\\
			\textsc{rgb\&cmy} & \textbf{5} & \textbf{15.1 (1.7)} & \textbf{4.0 (0.0)} & \textbf{4.0 (0.0)} & &  0 & - & - & -  & & 0 & - & - & - & & 0 & - & - & - & & 0 & - & - & -\\
			\bottomrule[1.5pt]
		\end{tabular}
	}
\end{sidewaystable}

\subsubsection{Policy Learning in Handcrafted HRMs}
\label{app:extended_results_policy_learning}
Figure~\ref{fig:flat_vs_hrm_policy_rest} shows the plots for the settings omitted in the main paper (the remaining \cw setting, FRL, is shown in Figure~\ref{fig:flat_vs_hrm_policy}). As mentioned in Section~\ref{sec:experimental_results_policy_learning_handcrafted} and shown in the figure, the efficacy of non-flat HRMs is less evident in two scenarios. First, when the task's goal is reachable regardless of the chosen options (e.g.,~if there are no rejecting states, like in OP and FR) because the policies over options become irrelevant. Second, when the reward is not sparse, like in OPL (the grid is small) or \ww (the balls easily get near the agent, so even a poor agent can achieve the goal if the number of steps per episode is large). In addition, Figure~\ref{fig:flat_vs_hrm_policy_rest_cw_opl} shows a case where the convergence in the non-flat case is delayed with respect to the flat one. Remember that in DQNs~\cite{MnihKSRVBGRFOPB15}, learning does not start until the buffers contain a certain number of experiences. In our approach, as described in Section~\ref{sec:policy_learning}, there is a DQN and a replay buffer for each RM; thus, in the flat case, there is a single DQN and buffer, while in the non-flat case there are several. Filling the buffers in the non-flat case is slower since there are higher-level options (i.e., call options) that do not occur as often as others (i.e., formula options), hence explaining the convergence delay. Nevertheless, we emphasize that in more complex scenarios this does not occur, as shown in Figure~\ref{fig:flat_vs_hrm_policy}.
	
We observe CRM performs closely to the HRL methods when the reward is not very sparse; for instance, when the \cw grid is small (OP, OPL). However, as the reward becomes sparser, CRM struggles more to converge since it does not decompose tasks into independently solvable subtasks and instead relies on a single non-zero reward signal (i.e., the one coming from the transitions to the accepting state).

\begin{figure}
	\centering
	\begin{subfigure}{\linewidth}
		\centering
		\includegraphics[width=0.25\linewidth]{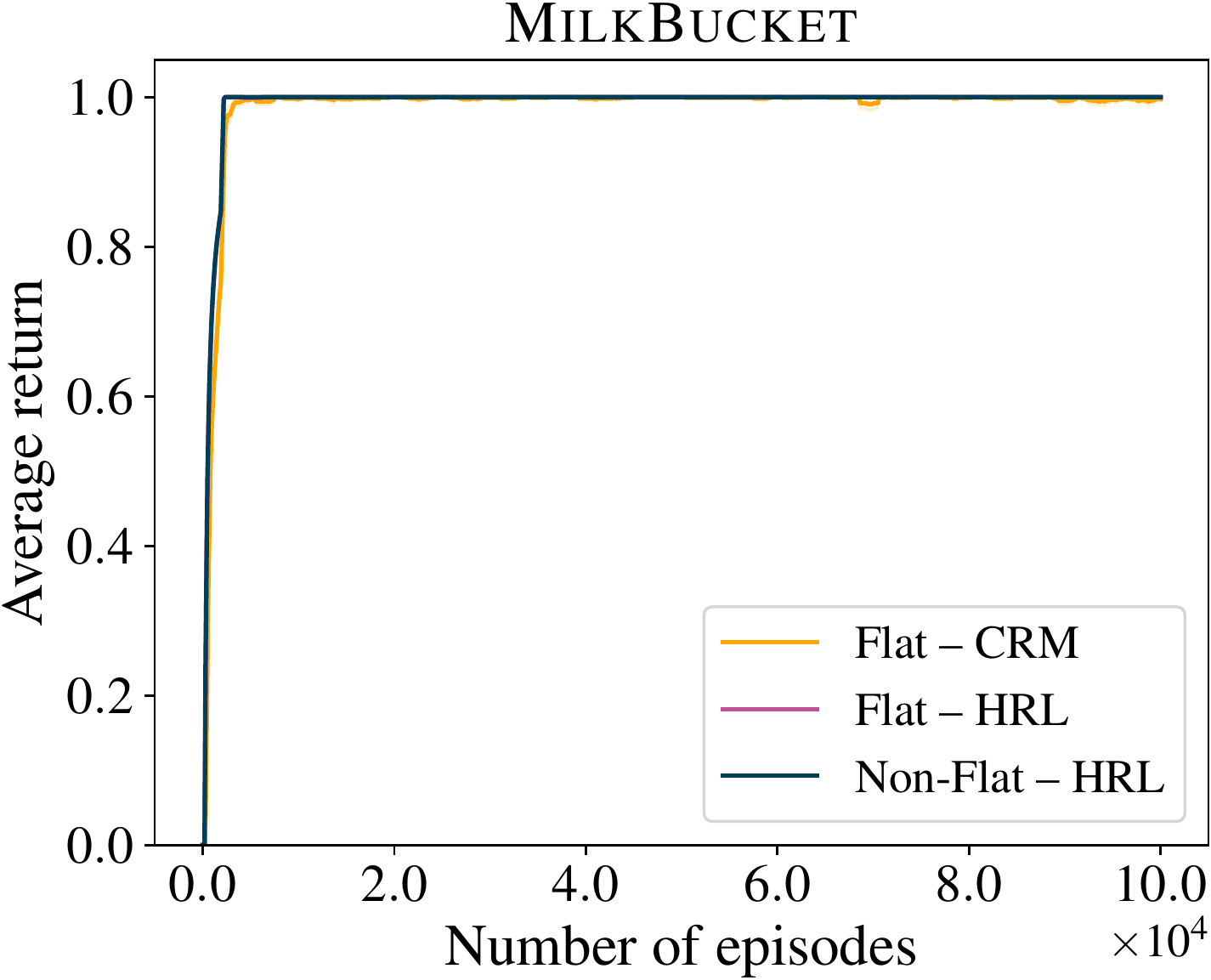}
		\includegraphics[width=0.25\linewidth]{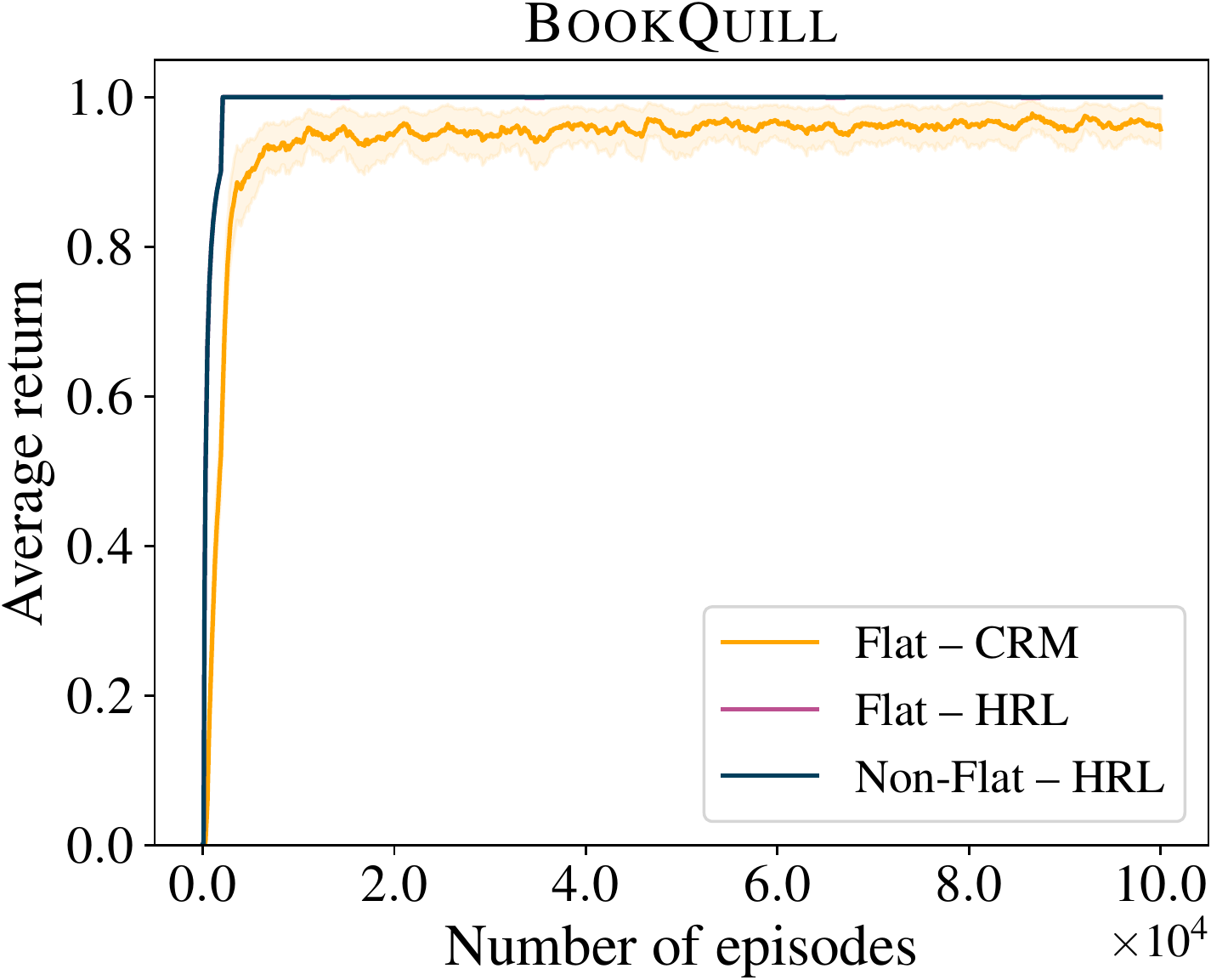}
		\includegraphics[width=0.25\linewidth]{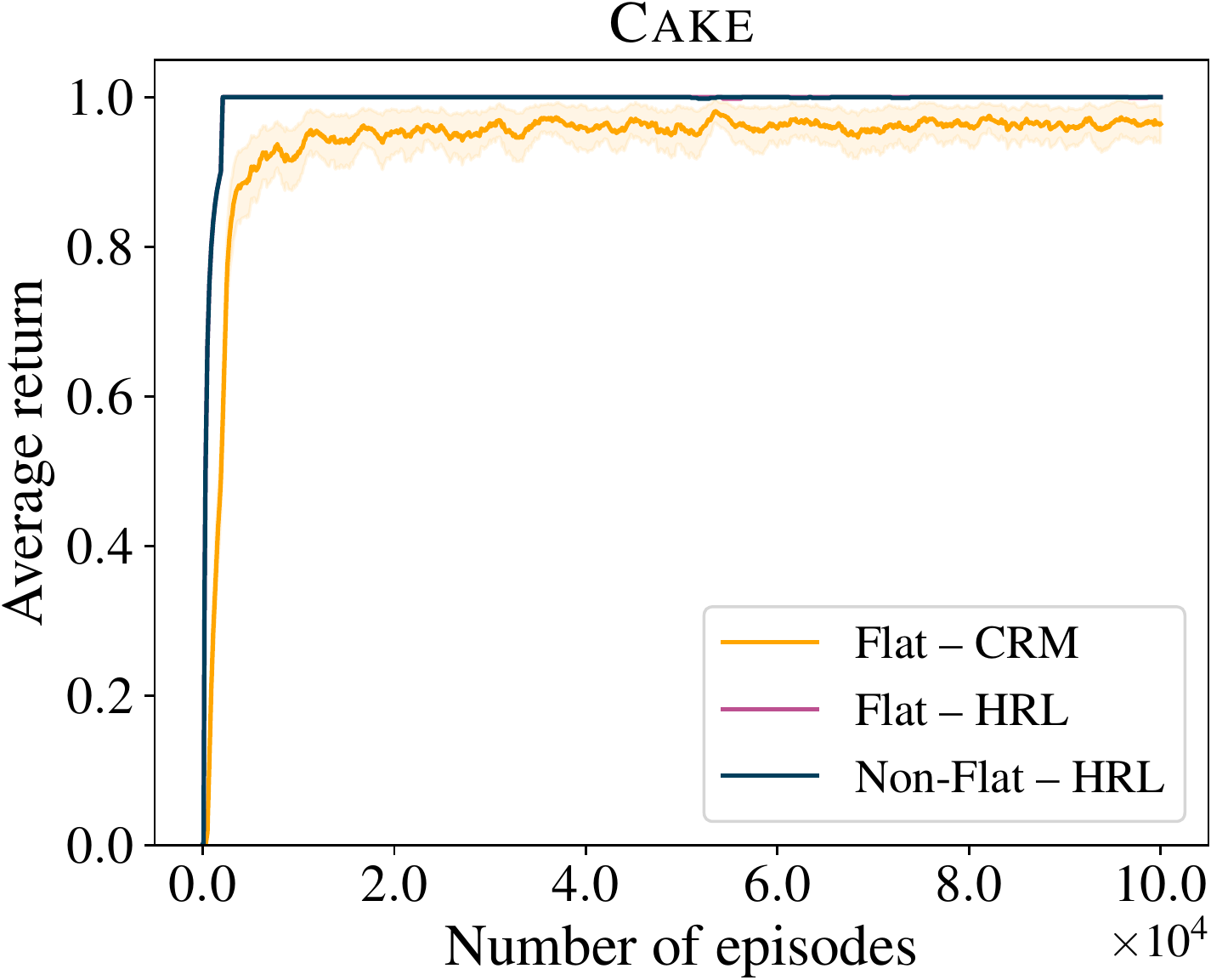}
		\caption{{\cw} -- OP}
	\end{subfigure}
	\begin{subfigure}{\linewidth}
		\centering
		\includegraphics[width=0.25\linewidth]{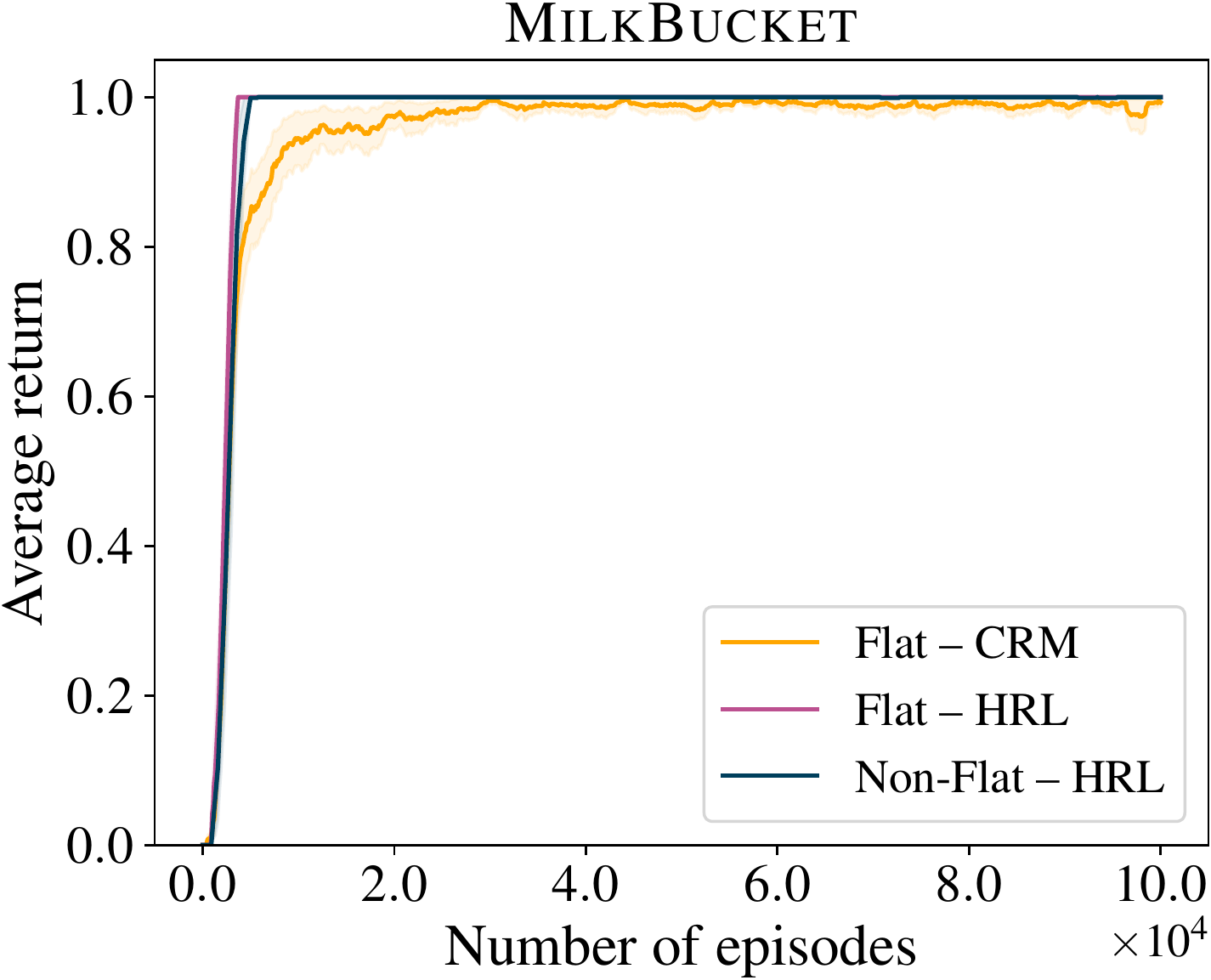}
		\includegraphics[width=0.25\linewidth]{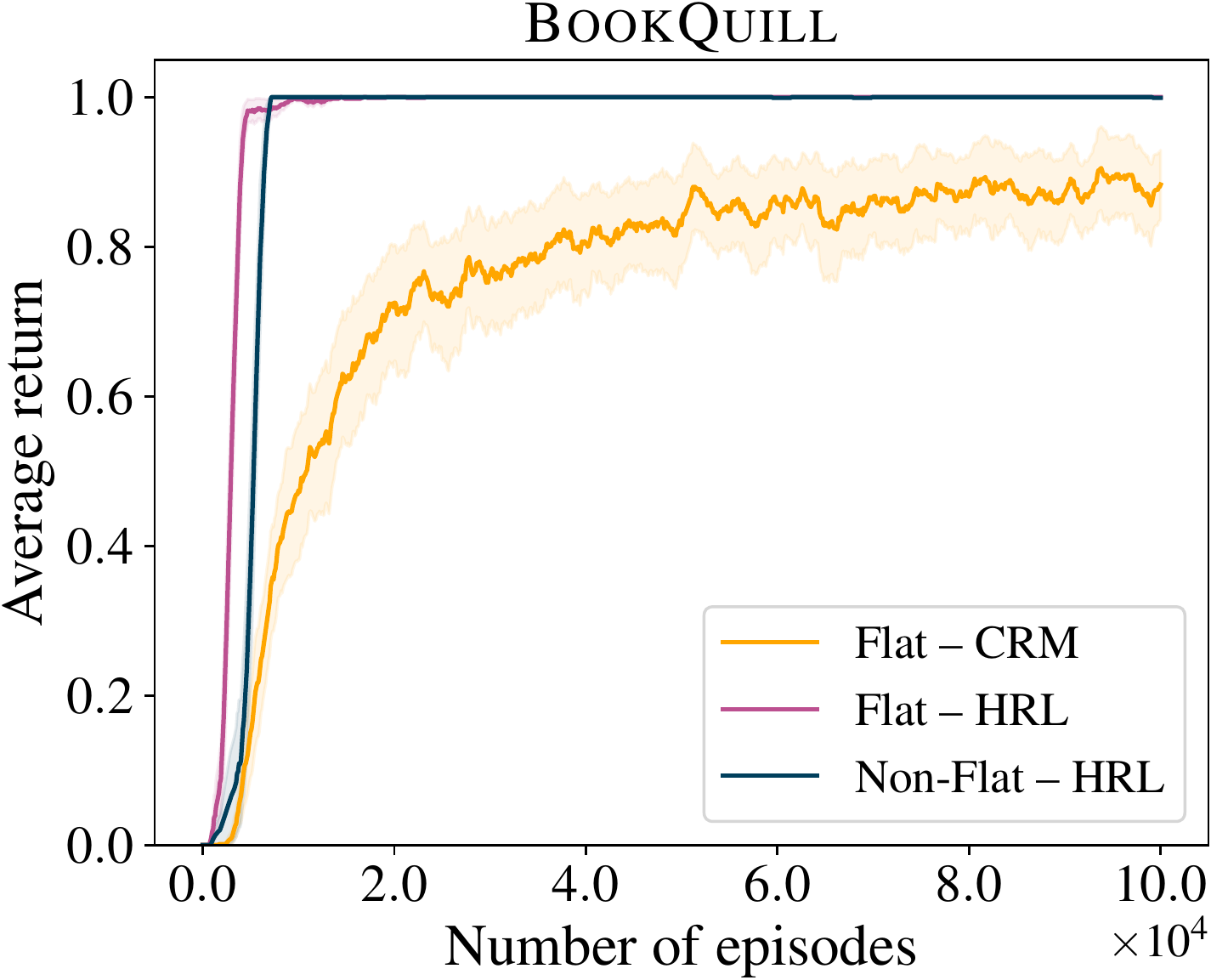}
		\includegraphics[width=0.25\linewidth]{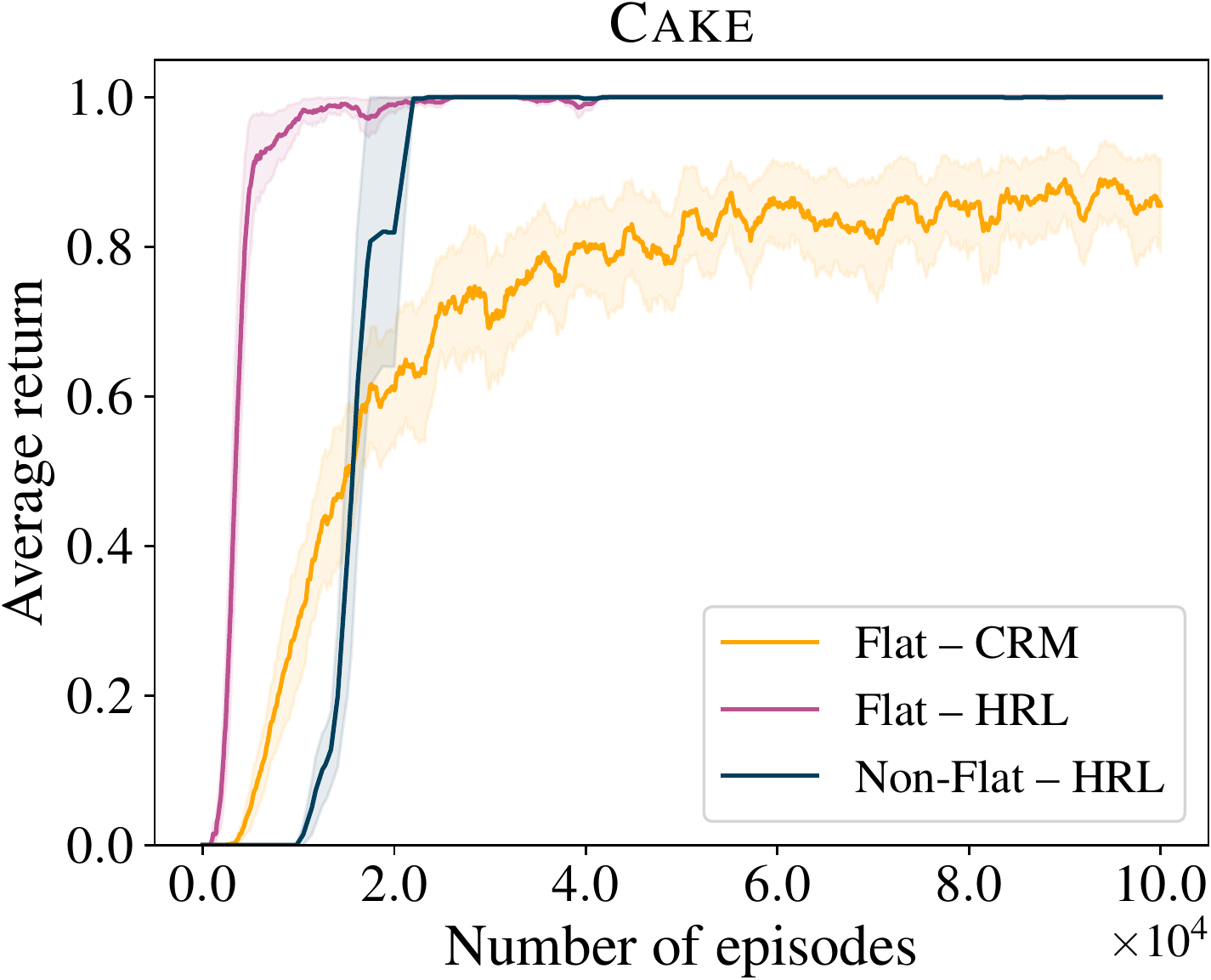}
		\caption{{\cw} -- OPL}
		\label{fig:flat_vs_hrm_policy_rest_cw_opl}
	\end{subfigure}	
	\begin{subfigure}{\linewidth}
		\centering
		\includegraphics[width=0.25\linewidth]{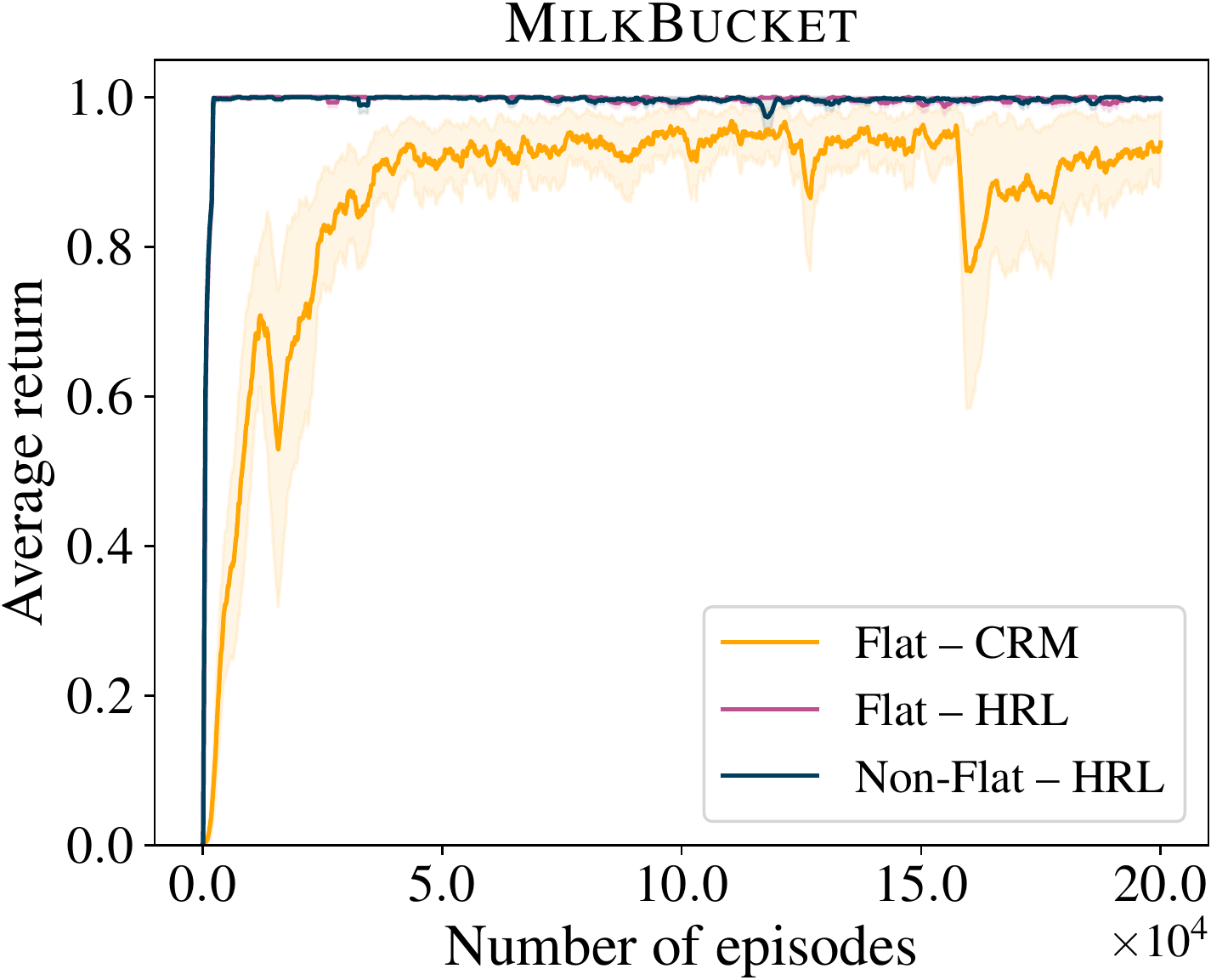}
		\includegraphics[width=0.25\linewidth]{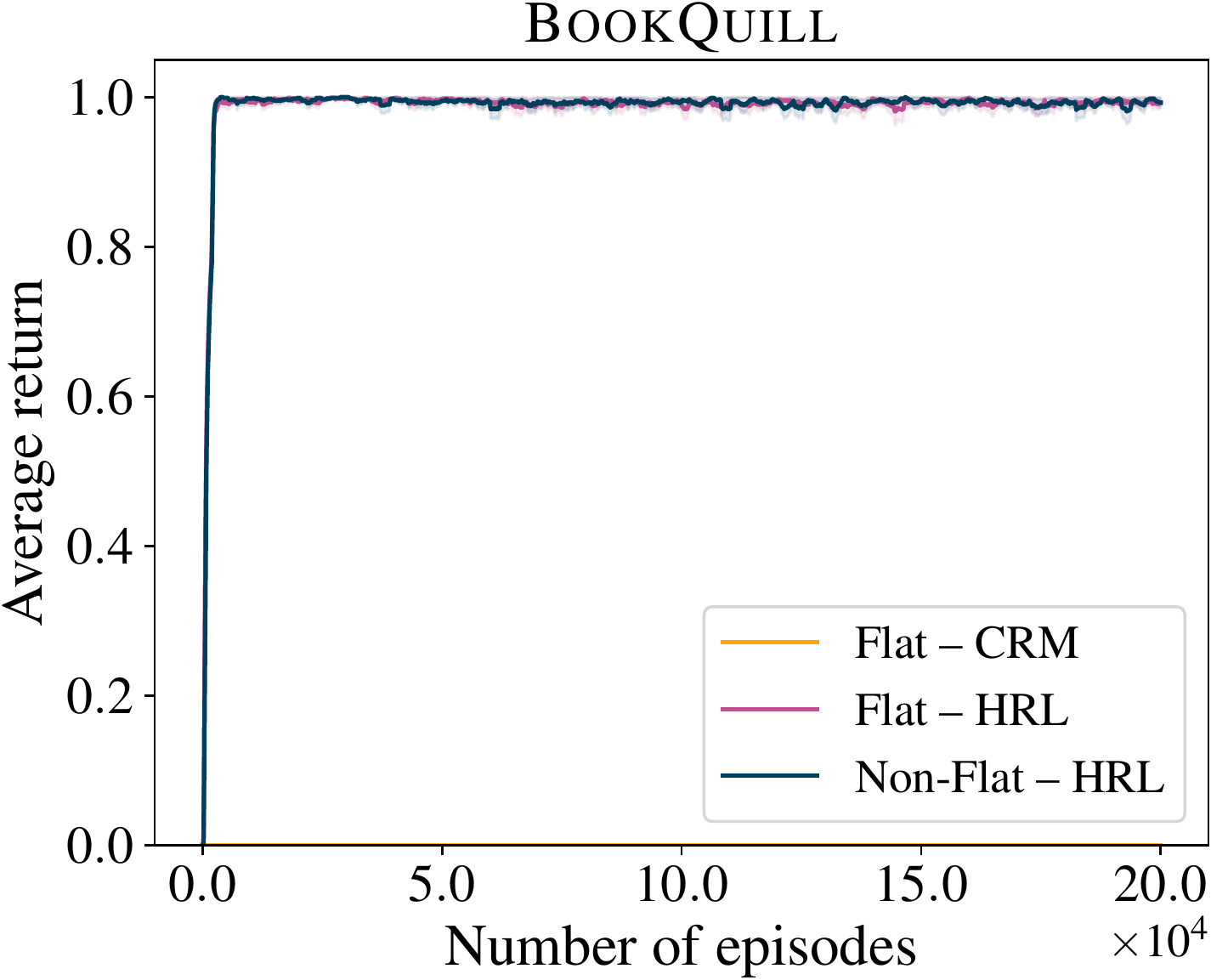}
		\includegraphics[width=0.25\linewidth]{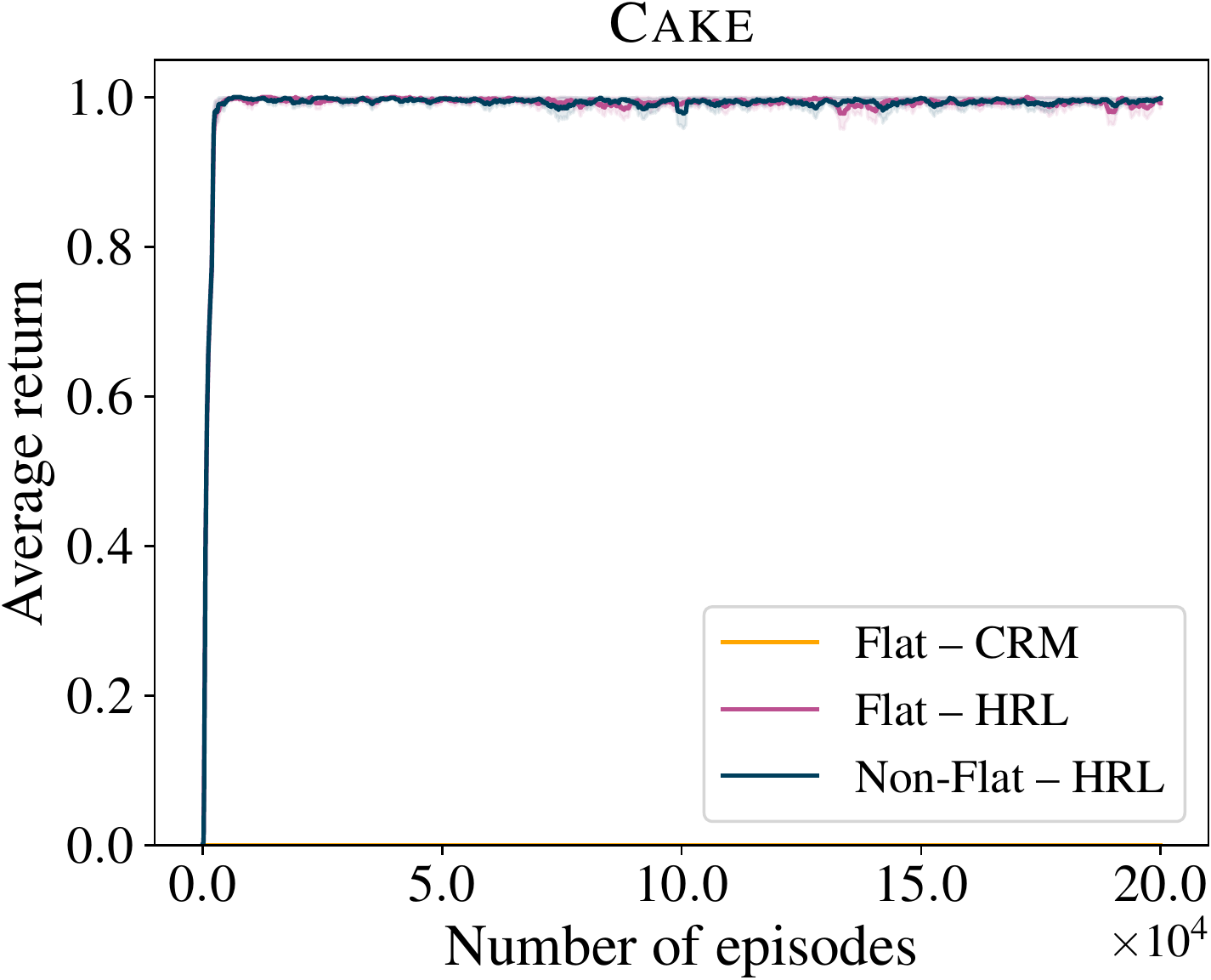}
		\caption{{\cw} -- FR}
	\end{subfigure}
	\begin{subfigure}{\linewidth}
		\centering
		\includegraphics[width=0.25\linewidth]{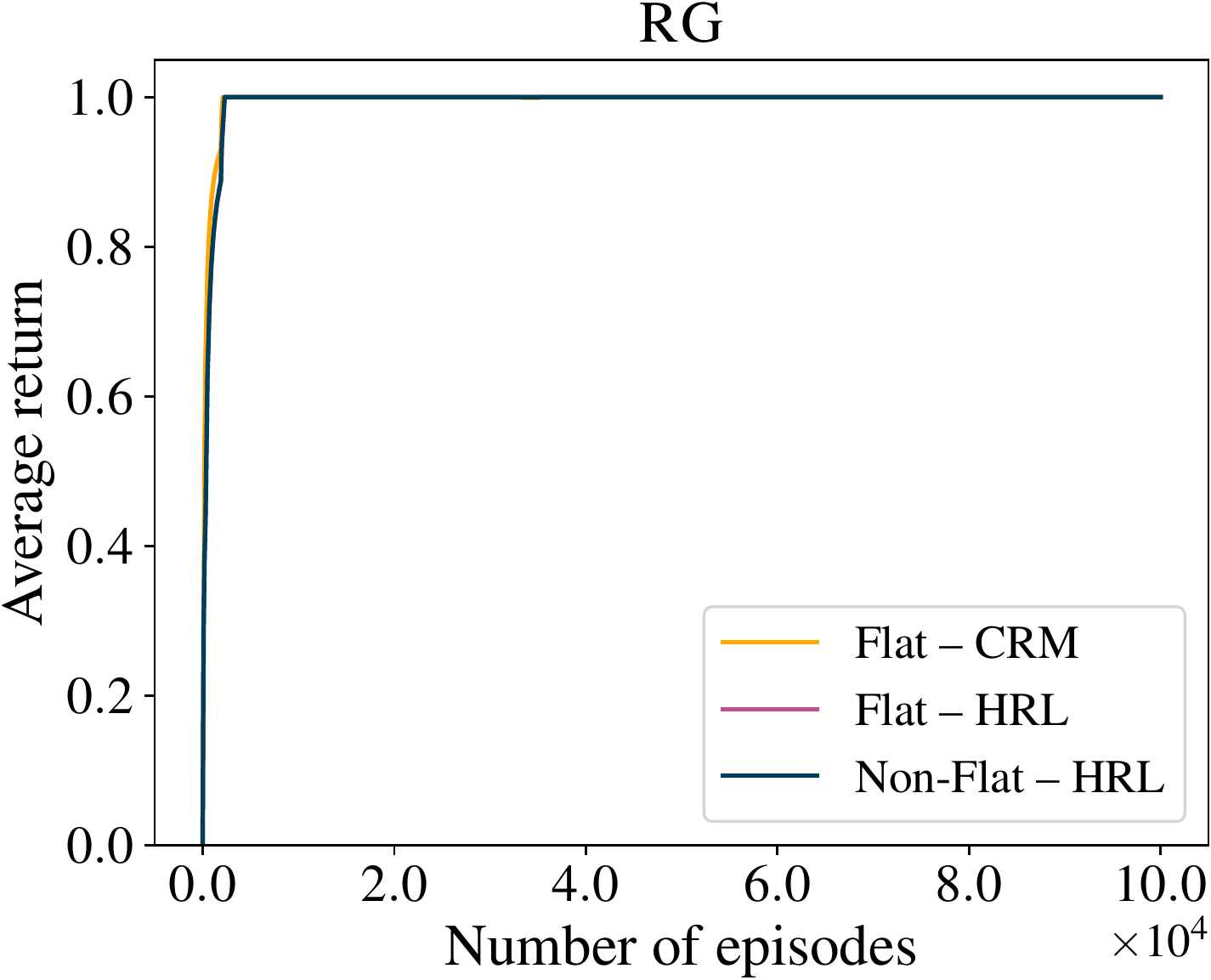}
		\includegraphics[width=0.25\linewidth]{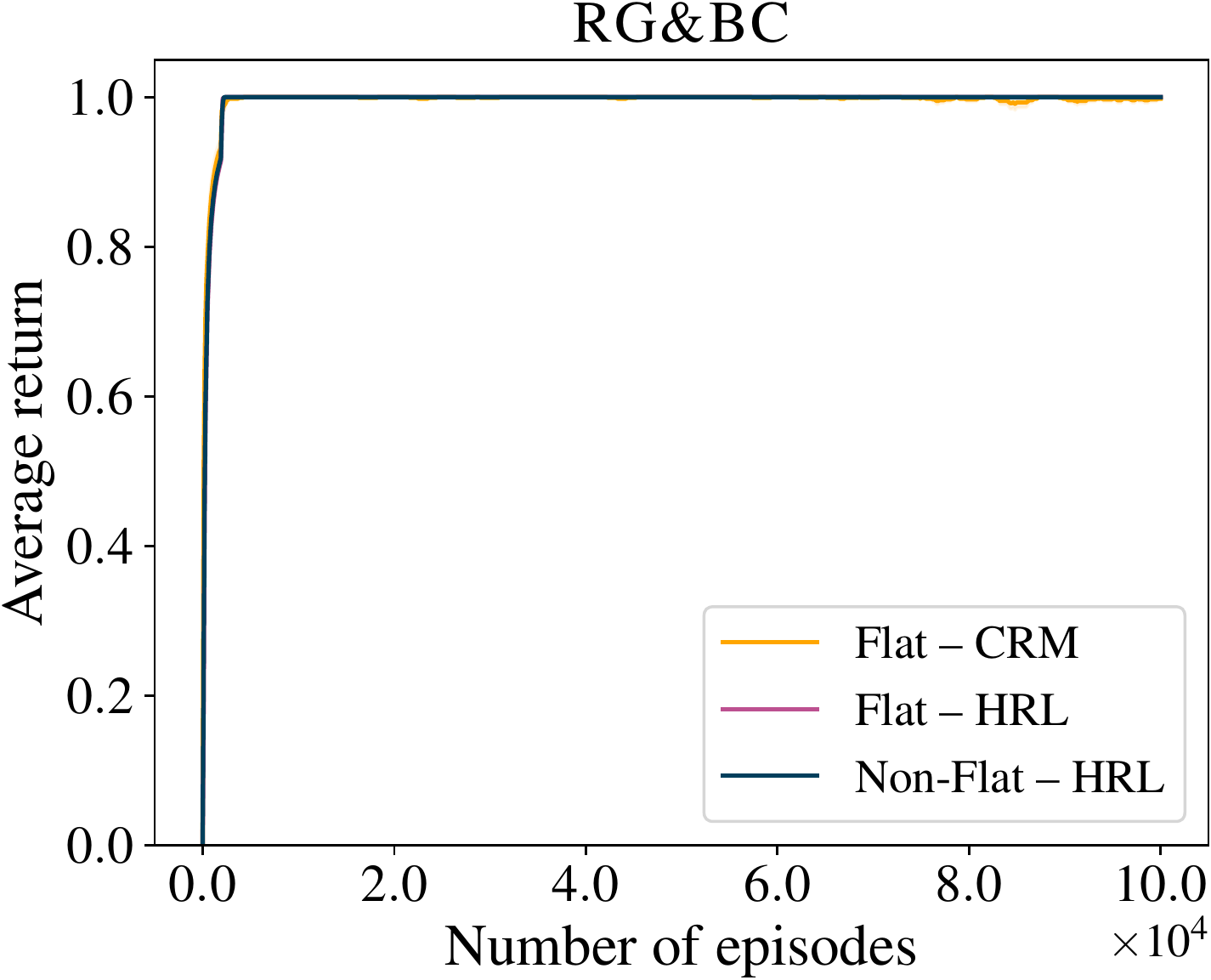}
		\includegraphics[width=0.25\linewidth]{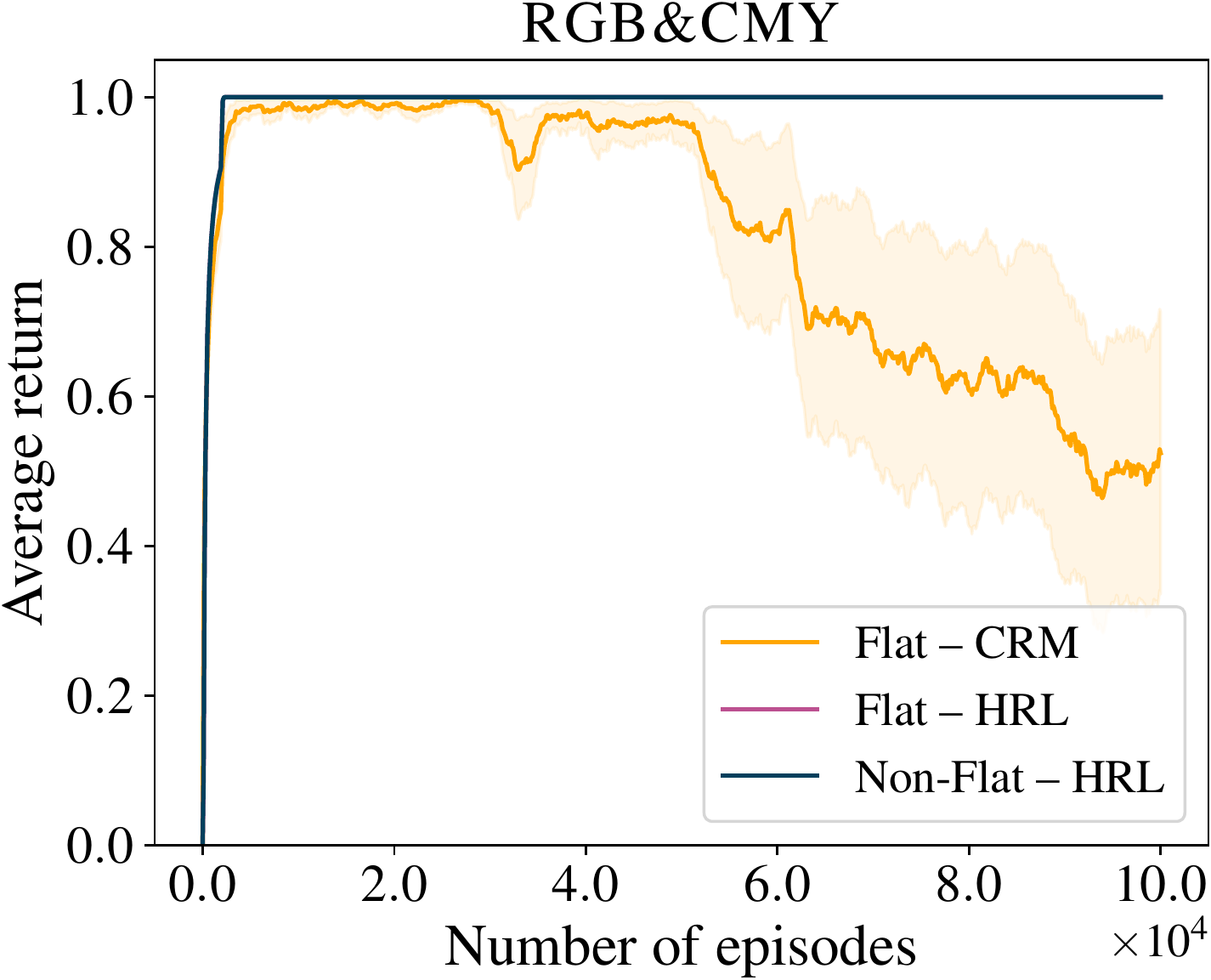}
		\caption{{\ww} -- WOD}
	\end{subfigure}
	\begin{subfigure}{\linewidth}
		\centering
		\includegraphics[width=0.25\linewidth]{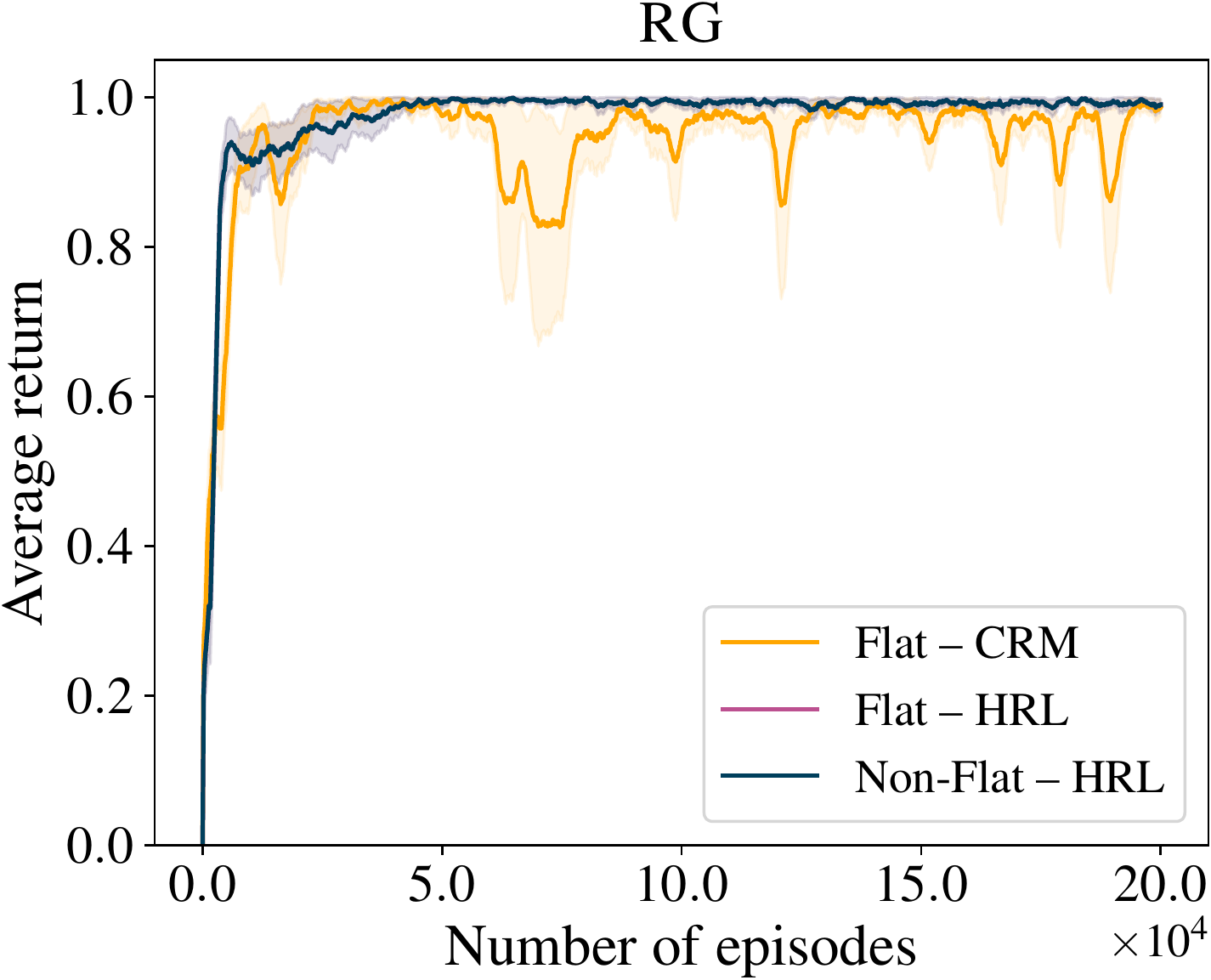}
		\includegraphics[width=0.25\linewidth]{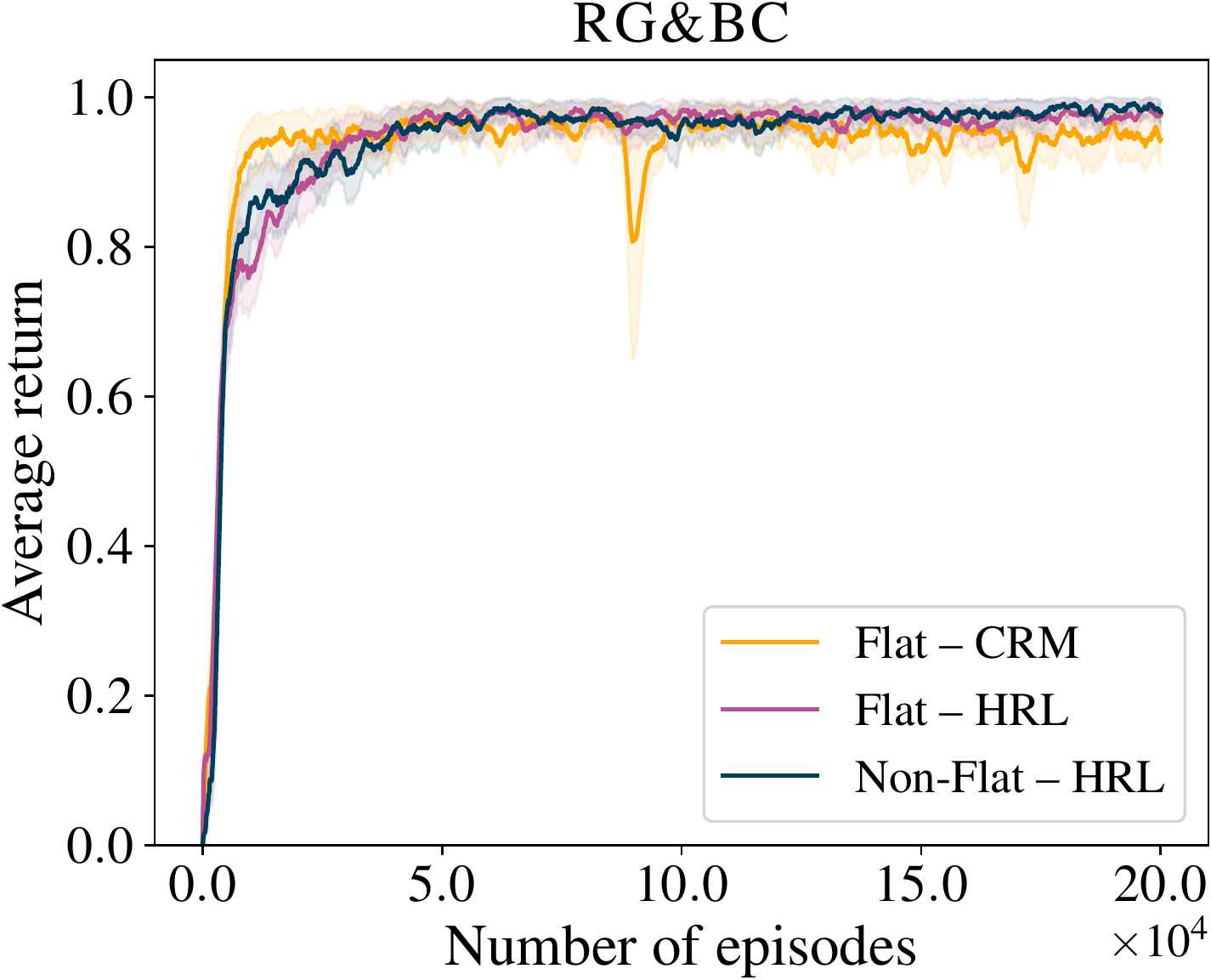}
		\includegraphics[width=0.25\linewidth]{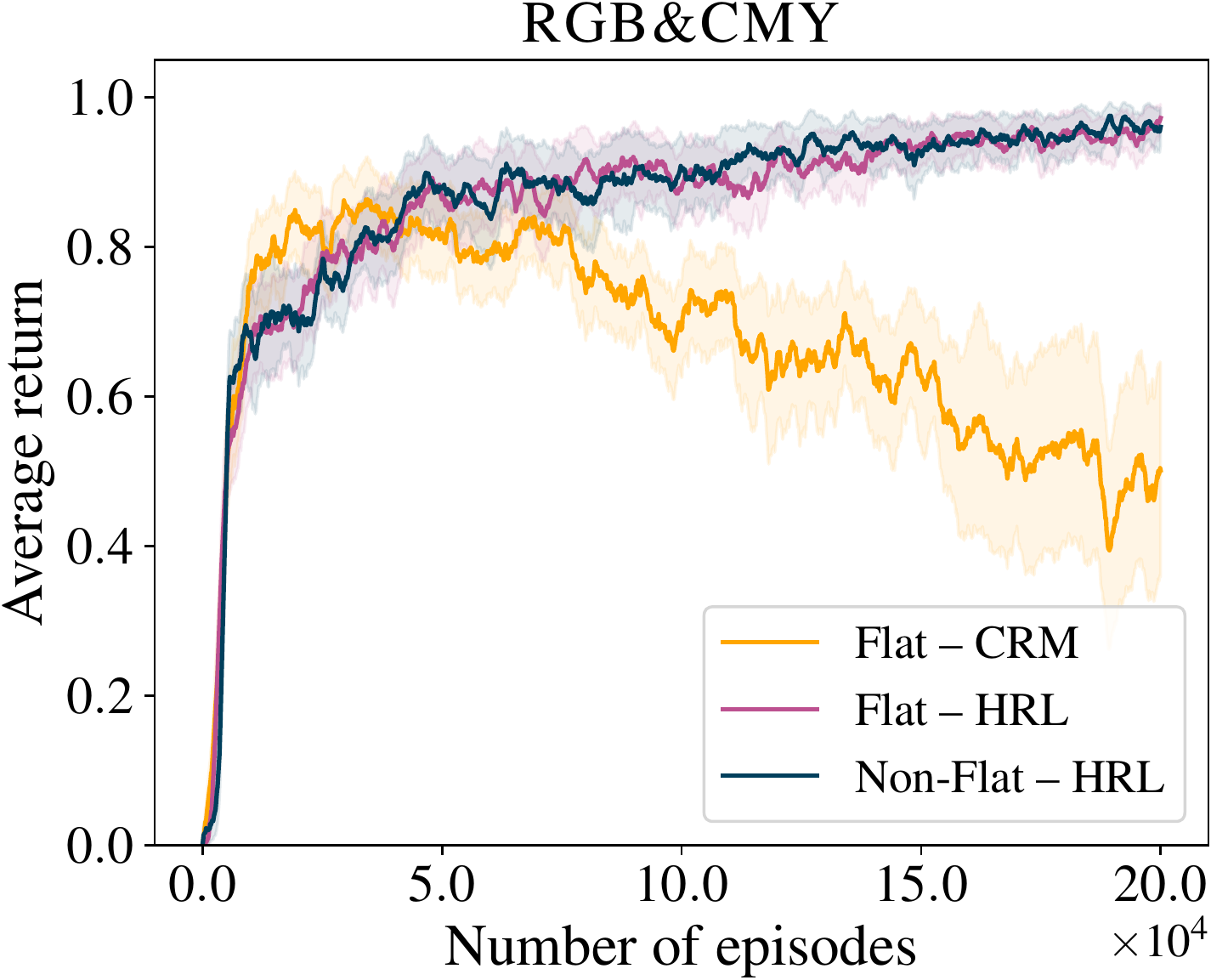}
		\caption{{\ww} -- WD}
	\end{subfigure}
	\caption{Learning curves comparing the performance of policy learning algorithms exploiting handcrafted non-flat and flat HRMs.}
	\label{fig:flat_vs_hrm_policy_rest}
\end{figure}

\end{document}